\documentclass[11pt]{article}

\usepackage{fullpage}
\usepackage[utf8]{inputenc} 
\usepackage[T1]{fontenc}    
\usepackage{hyperref}       
\usepackage{url}            
\usepackage{booktabs}       
\usepackage{amsfonts}       
\usepackage{nicefrac}       
\usepackage{microtype}      
\usepackage{tcolorbox}
\usepackage{diagbox}
\usepackage{mathtools}
\usepackage{authblk}
\usepackage{enumerate}
\usepackage[shortlabels]{enumitem}
\usepackage{algorithmic}
\usepackage{algorithm}
\usepackage{graphicx} 
\usepackage{caption}
\usepackage{subcaption}
\usepackage{amsmath}
\usepackage{amsthm}
\usepackage{amssymb}
\usepackage{tikz}
\usepackage{tablefootnote}
\usepackage{multirow}
\usepackage{enumerate}
\usepackage{color}
\usepackage{xcolor}
\usepackage{natbib}

\usetikzlibrary{arrows}

\allowdisplaybreaks[4]

\usepackage{mathrsfs}

\usepackage{hyperref}
\usepackage{bm,todonotes}

\allowdisplaybreaks

\newtheorem{thm}{Theorem}[section]
\newtheorem{lem}{Lemma}[section]
\newtheorem{cor}{Corollary}[section]
\newtheorem{prop}{Proposition}[section]

\newtheorem{rem}{Remark}[section]

\newenvironment{prf}{\begin{proof}}{\end{proof}}
\newtheorem{ass}{Assumption}

\hypersetup{
	colorlinks=true,
	filecolor=blue,
	citecolor = blue,
	urlcolor=cyan,
}

\usepackage{amsmath,amsfonts,bm}

\def\1{\bm{1}}

\DeclareMathAlphabet{\mathsfit}{\encodingdefault}{\sfdefault}{m}{sl}
\SetMathAlphabet{\mathsfit}{bold}{\encodingdefault}{\sfdefault}{bx}{n}

\def\D{{\bf D}}

\def\e{{\bf e}}

\def\g{{\bf g}}
\def\h{{\bf h}}

\def\S{{\bf S}}

\def\U{{\bf U}}
\def\u{{\bf u}}

\def\W{{\bf W}}

\def\X{{\bf X}}

\def\Xi{\boldsymbol{\xi}}

\def\Phi{\boldsymbol{\phi}}

\def\x{{\boldsymbol{x}}}

\def\U{\boldsymbol{U}}

\def\Y{{\bf Y}}
\def\y{{\bf y}}

\def\0{{\bf 0}}
\def\1{{\bf 1}}

\title{
Faithful Federated Tabular Data Synthesis
}
\author[1]{Zilong Cao\thanks{Email: \texttt{nwu\_czl@stumail.nwu.edu.cn}}}
\author[1]{Hai Zhang\thanks{Email: \texttt{zhanghai@nwu.edu.cn}}}
\author[2]{Xuan Bi\thanks{Corresponding author. Email: \texttt{xbi@umn.edu}}}

\affil[1]{School of Mathematics, Northwest University}
\affil[2]{Carlson School of Management, University of Minnesota}

\date{}
\begin{document}

\maketitle
\begin{abstract}
Synthetic data is an important data source in the AI era. 
This paper studies the generation of mixed-type tabular synthetic data in federated settings, where two key challenges arise: ensuring distributional fidelity and effectively handling mixed-type data. 
We propose Federated Integrative Two-Step Synthesis (FITS), a framework for mixed-type tabular data synthesis. 
At its core is a two-step federated algorithm for estimating conditional cumulative distribution functions (CDFs) using local linear regression. 
Building on this algorithm, we develop a federated synthesis procedure based on the probability chain rule to ensure distributional fidelity. 
To handle mixed-type data, FITS introduces a continualization method that lifts discrete variables into a continuous space and recovers their discrete distributions without loss at both the density-function and CDF levels. 
This transformation provides a unified continuous representation of mixed-type variables, facilitating the application of local linear regression and synthetic data generation.
Our theoretical analysis characterizes the conditional density and cumulative distribution functions of the continualized mixed-type variables and establishes the uniform convergence of the resulting federated conditional CDF estimators. 
Experiments demonstrate that FITS-generated synthetic data improves downstream task performance for all clients, with the benefits depending on how the synthetic data is used.
\end{abstract}

\section{Introduction}
Synthetic data generation has received increasing attention as the AI era progresses, with synthetic data becoming an important data source in many fields, such as large language models \citep{wang2023self}, imaging \citep{kim2022transferable}, and medical research \citep{kokosi2022synthetic}. 
However, synthetic data generation (or, broadly speaking, data synthesis) is itself a nontrivial methodological problem, rather than a simple consequence of assuming a generative distribution \citep{alaa2022faithful}. 
In addition to the intrinsic challenges of generating synthetic data, a fundamental issue is the limited availability of real data, which hinders the production of high-quality synthetic data. Such scarcity frequently arises from data silos, where data holders prioritize privacy and legal considerations. This severely limits or even prevents access to their data.
Federated learning (FL, \cite{McMahan2016CommunicationEfficientLO,kairouz2021advances}) provides an efficient and privacy-preserving framework for synthetic data generation, addressing this dilemma through a distributed paradigm that eliminates centralized storage and avoids centralizing raw data while transmitting only processed data from clients to ensure security. 
Numerous federated synthetic data generation efforts focus on images and other data formats \citep{little2023federated}; 
nonetheless, the development of synthetic tabular data generation in federated learning remains limited, despite its substantial importance across various domains.

Existing studies have explored generative adversarial networks \citep{zhao2021fed}, diffusion models \citep{sattarov2025overcoming}, and marginal-based synthesis \citep{li2026heterofedsyn}. 
However, several major challenges remain unresolved. 
(i) Communication costs remain substantial. 
Federated synthetic data generation methods, especially deep learning-based models, require extensive communication to learn the generative model. 
Several reports \citep{zhu2019deep, kairouz2021advances,nist2025genomicsppfl} emphasize that reducing communication is essential for improving both the efficiency and safety of federated learning. 
(ii) Lack of faithfulness remains a concern \citep{alaa2022faithful}. 
Generative adversarial networks and diffusion models lack theoretical interpretability, which may lead to synthetic data that is unfaithful.
For example, many deep-learning-based models cannot guarantee that the synthetic data is faithful to the original distribution.
Marginal-based synthesis, on the other hand, relies only on low-order marginal information and may have limitations in approximating high-dimensional distributions.
(iii) Existing methods remain inefficient in handling mixed-type data. 
For example, black-box approaches, such as generative adversarial networks and diffusion models, typically use one-hot encodings or continuous embeddings for discrete variables; 
however, one-hot encodings are not genuinely continuous and substantially increase dimensionality, whereas non-invertible continuous embeddings destroy the discrete structure and compromise interpretability.
Marginal-based synthesis, by contrast, discretizes continuous variables to facilitate the computation of marginal statistics, thereby losing information carried by continuous variables.
Nonparametric methods provide established approaches to handling mixed-type data.
\citet{li2008nonparametric} estimate conditional CDFs and quantiles using generalized product kernels that combine continuous and discrete kernels.
Alternatively, \citet{nagler2018generic} shows that suitably chosen independent noisy jitter allows nonparametric density estimators designed for continuous data to be applied to mixed data.
These approaches provide important foundations for mixed-data estimation but do not develop a federated synthesis procedure.

To address these challenges, we propose a federated framework, termed federated integrative two-step synthesis for mixed-type tabular data (FITS), which supports the generation of multivariate mixed-type tabular data based on the chain rule. 
FITS is characterized by a two-step communication procedure.
In the first step, clients transmit weighted Hessian-type matrices to the central server, which computes the inverse of the aggregated matrix and sends only its first row back to the clients. 
In the second step, clients send their local response-smoothed CDF contribution vectors to the central server. 
The server then constructs the global conditional cumulative distribution functions (CDFs), which are subsequently used for inverse sampling and data synthesis. 
For mixed-type data, we employ an efficient technique that lifts discrete variables into a continuous space while allowing inversion back to the original discrete space without incurring additional dimensional cost. 
FITS includes a core federated algorithm for estimating conditional cumulative distribution functions via local linear regression. 
In theory, local linear regression (LLR, \cite{fan1996estimation}) can effectively reduce estimator bias and has been studied for conditional CDF estimation \citep{hall1999methods,xie2023uniform} and widely used in various applications \citep{cai2011analysis, cheng2013local, Chakrabortty2017EfficientAA,cheng2020estimating}. 
Furthermore, we propose client-designed and penalty-adaptive kernels to enable effective LLR fitting in federated settings with mixed-type data.
At the client level, the client-designed kernel offers both fixed-bandwidth and K-nearest-neighbor (KNN) variants, with the choice determined by each client. Under this design, each client can specify kernel components, including the kernel function and bandwidth parameters.
The penalty-adaptive kernel, by contrast, operates at the kernel-function level and penalizes continualized samples whose categories do not match that of the query.

FITS offers several distinctive advantages for federated synthesis of mixed-type tabular data.
First, unlike iterative federated optimization, FITS requires only the upload of processed data and a \emph{two-step} procedure to achieve distribution estimation.
Second, by estimating conditional CDFs and combining them through the chain rule, FITS directly targets the underlying joint distribution, providing a principled basis for distributionally faithful data synthesis.
Third, we utilize an invertible continualization technique that maps continuous and discrete variables into a unified continuous domain without increasing the data dimension. This transformation preserves the recoverability of the original discrete variables while providing a foundation for applying established methods developed for continuous data and facilitates synthetic data generation through a unified procedure.
Finally, FITS supports client-designed kernels that accommodate client-specific data characteristics and a penalty-adaptive mechanism that reduces kernel-fitting error caused by category mismatches after continualization.

Our contributions can be summarized as follows:

From a methodological perspective, we propose a novel FITS framework for federated mixed-type tabular data synthesis. 
It supports federated conditional CDF estimation and synthetic data generation.
In addition, FITS introduces an invertible discrete-to-continuous lifting technique, the client-designed kernel for the FL scheme, and the penalty-adaptive kernel for the mixed-type data.

In theory, we characterize the proposed continualization technique at both the density and distribution-function levels, showing, in particular, its blockwise smoothness.
Moreover, under this blockwise structure, we establish uniform convergence rates for the federated local linear conditional CDF estimator using Gaussian kernels with unbounded support. 
In particular, direct kernel smoothing near blockwise boundaries can yield inconsistent estimates; we therefore develop a modified estimation procedure that restores consistency at these points.
Finally, we establish the uniform consistency of the per-step conditional CDF estimates used in the chain-rule generation process.

We evaluate FITS on simulated and real datasets under imbalanced client sample sizes and compare it with two federated data synthesis baselines, HeteroFedSyn \citep{li2026heterofedsyn} and FedTabDiff \citep{sattarov2025overcoming}, using three downstream models: linear regression, logistic regression, and FT-Transformer (a deep learning model, \citep{gorishniy2021revisiting}), covering both regression and classification tasks.  
The results show that FITS-generated synthetic data can substantially improve downstream performance and provide the best distributional fidelity compared to synthetic data generated by the other methods. 

The remainder of the paper is organized as follows. 
Section~\ref{sec:background} introduces the federated setup, notation, and local linear formulation of conditional CDF estimation. 
Section~\ref{sec:algorithm} presents FITS for continuous data, including two-step federated aggregation, client-designed kernel schemes, and synthetic data generation through the chain rule and inverse sampling. 
Section~\ref{sec:general_mix} extends the framework to mixed-type data through continualization of discrete variables and penalty-adaptive kernels. 
Section~\ref{sec:theory} establishes uniform convergence results for both continuous and continualized mixed-type data. 
Section~\ref{sec:experiments} evaluates the utility of synthetic data through simulation and real-data studies. 
The final section summarizes the findings and outlines future research directions. 
The appendices provide auxiliary theoretical tools, detailed proofs, and supplementary experiments.

\section{Background}
\label{sec:background}

Subsection~\ref{subsec:fed_setup} establishes the federated learning setup and the general notation used throughout the paper. Subsection~\ref{subsec:llr_formulation} introduces local linear regression, explains how it estimates conditional CDFs, and describes its adaptation to the federated learning scheme.

\subsection{Federated setup and notation}
\label{subsec:fed_setup}

We consider a federated learning system consisting of a central
server (CS) and $K$ clients. Client $k\in\{1,\ldots,K\}$ holds a
client-specific hold-out dataset
\[
    \mathcal H_k
    =
    \bigl\{\h_i^{(k)}\bigr\}_{i=1}^{m_k},
    \qquad
    \h_i^{(k)}
    =
    \bigl(\h_{i,:d}^{(k)},\,h_{i,d+1}^{(k)}\bigr)
    \in\mathbb R^{d+1},
\]
where $\h_{i,:d}^{(k)} = (h_{i,1}^{(k)},...,h_{i,d}^{(k)})\in\mathcal X\subset\mathbb R^d$ is the
covariate vector and $h_{i,d+1}^{(k)}\in\mathcal Z\subset\mathbb R$
is the response. We write $M=\sum_{k=1}^K m_k$ for the total
hold-out sample size. Throughout the paper, we assume that the
samples $\{\h_i^{(k)}: i\in[m_k],\ k\in[K]\}$ are i.i.d.\ draws
from a common joint distribution.

Let $\X\in\mathcal X$ and $Z\in\mathcal Z$ denote the covariate
vector and response, respectively. The estimand of interest is
the conditional distribution function
\begin{equation}\label{eq:target}
    F(z\mid \x)
    =
    \mathbb P(Z\le z\mid \X=\x),
    \qquad
    z\in\mathcal Z,\ \x\in\mathcal X.
\end{equation}
It is the canonical object from which arbitrary conditional functionals, including means, quantiles, and probability masses, can be derived, and samples can be generated by inverse-CDF sampling.

The federated system is subject to the following communication and data-access constraints.
The clients communicate with the CS in synchronous rounds; in each round, a client may upload one or more deterministic functions of its local data $\mathcal H_k$ (and any quantities previously received from the CS), and the CS may broadcast a deterministic function of the received uploads to all clients.
The following raw-data objects are \emph{never} transmitted: the individual samples $\{\h_i^{(k)}\}_i$, the local response vectors $\{h_{i,d+1}^{(k)}\}_i$, and the local design matrices used in weighted regression.
The federated estimator developed in Section~\ref{sec:algorithm} respects these constraints by transmitting only \emph{aggregated sufficient statistics}---a $(d{+}1)\times(d{+}1)$ design summary in the first round and an $L$-dimensional CDF contribution vector in the second round---each of which is a many-to-one function of the underlying raw data.

\paragraph{Coordinate grids.}
For each coordinate $j\in[d]$ we fix an evaluation grid $\g_j=(g_{1j},\ldots,g_{Gj})^\top\in\mathbb R^G$, and write $\mathcal G=\{\g_j\}_{j=1}^{d+1}$.
The grids serve two purposes: in Section~\ref{sec:algorithm}, they parameterize the evaluation points of $\widehat F(\cdot\mid\x)$. And the grids are used to generate synthetic data via linear interpolation.

\subsection{Local-linear formulation of the conditional CDF}
\label{subsec:llr_formulation}

We estimate $F(z\mid \x)$ by applying local linear regression to a smoothed CDF score of the response with respect to the continuous covariates. 
Let $\mathcal K:\mathbb R\to[0,1]$ denote a continuous CDF kernel. 
The local linear estimator $\widehat F(z\mid \x)=\hat\beta_0$ of $F(z\mid \x)$ at fixed $\x$ and $z$ is defined as the intercept of the weighted least-squares fit
\begin{equation}\label{eq:llr_global}
(\hat\beta_0,\hat{\bm\beta}_{:d})
=
\arg\min_{(\beta_0,\bm\beta_{:d})}
\sum_{i=1}^{M}
\Bigl(
\mathcal K_i(z)
-\beta_0
-\bm\beta_{:d}^\top
(\x- \h_{i,:d})
\Bigr)^{2}
\varpi_i,
\end{equation}
where $\bm\beta_{:d} = (\beta_1,\ldots,\beta_d) \in \mathbb{R}^d$ denotes the $d$-dimensional slope vector, $\varpi_i$ is a nonnegative kernel weight on the covariate coordinates, and $\h_{i,:d}$ is the covariate vector of the $i$th observation. 
The choice of $\varpi_i$ is left abstract here and is generally specified as the kernel evaluation of the query point $\x$ at the sample point $\h_{i,:d}$. 
For each query point $\x\in \mathcal X$ and each evaluation threshold $z\in\mathcal Z$, define the smoothed response indicator with bandwidth $w^{(i)}_z>0$, generally setting $w^{(i)}_z = w_z >0$, as
\begin{equation}\label{eq:smoothed_indicator}
\mathcal K_i(z)
:=
\mathcal K\left(
\frac{z- h_{i,d+1}}{w^{(i)}_z}
\right),
\qquad i\in[M],
\end{equation}
where $\{\h_i\}_{i=1}^M$ denotes the pooled hold-out sample under a temporary common index, with $\h_i=(\h_{i,:d},h_{i,d+1})$. 
Under standard kernel assumptions, $\mathbb E[\mathcal K_i(z)\mid \X_i=\x]\to F(z\mid\x)$ as $w^{(i)}_z\to0$.

The local-linear estimator \citep{fan1992design} has two well-known statistical advantages over the Nadaraya--Watson type local-constant estimator:
the leading bias term contributed by the marginal density of the covariates vanishes, and the estimator is automatically bias-corrected at the boundary of the covariate support. 
These properties carry over to the federated setting without modification, because, as shown below, the federated estimator coincides exactly with the centralized local linear estimator on the pooled sample.

For the federated decomposition, we record the closed-form
solution of \eqref{eq:llr_global}. Write
\[
    \U_i
    :=
    \bm\iota(\h_{i,:d}-\x)
    =\bigl(1,(\h_{i,:d}-\x)^\top\bigr)^\top
    \in\mathbb R^{d+1},
    \qquad
    \W:=\mathrm{diag}(\varpi_1,\ldots,\varpi_M),
\]
where $\bm\iota(\y) = (1,\y)^\top$, and let $\U\in\mathbb R^{M\times(d+1)}$ stack the $\U_i^\top$ as rows.
Set $\bm{\mathcal K}(z):=(\mathcal K_1(z),\ldots,\mathcal K_M(z))^\top$
and $\e_1=(1,0,\ldots,0)^\top\in\mathbb R^{d+1}$. The intercept of
\eqref{eq:llr_global} admits the closed form
\begin{equation}\label{eq:llr_intercept_closed}
    \hat\beta_0
    =
    \e_1^\top
    \bigl(\U^\top\W\U\bigr)^{-1}
    \U^\top\W\bm{\mathcal K}(z)
    =
    \sum_{i=1}^{M}
    \underbrace{
        \varpi_i\,\e_1^\top
        \bigl(\U^\top\W\U\bigr)^{-1}\U_i
    }_{\displaystyle =:\,\alpha_i}
    \,\mathcal K_i(z).
\end{equation}
The local-linear estimator is thus a linear combination of the smoothed response indicators with weights $\{\alpha_i\}_{i=1}^M$ that depend on $\x$ and the design $\U$ but not on $z$.
This separability is the structural property that enables the two-step federated learning procedure developed in Section~\ref{sec:algorithm}.

\section{The Proposed Method}
\label{sec:algorithm}

This section develops our framework FITS for continuous data,  focusing on federated conditional CDF estimation and synthetic data generation.
Subsection~\ref{subsec:two_step_fllr} formulates the two-step FL procedure for estimating LLR-based conditional CDFs, develops the core algorithm, and analyzes its communication cost and equivalence to the pooled setting.
Subsection~\ref{subsec:bandwidth_schemes} presents two interchangeable bandwidth schemes (fixed-bandwidth and KNN-kernel).
Subsection~\ref{subsec:generation} proposes a method for federated synthetic data generation that combines the chain rule with inverse sampling from conditional CDFs.

\subsection{Two-step federated aggregation}
\label{subsec:two_step_fllr}
\begin{figure}[t]
    \centering
    \includegraphics[width=\textwidth]{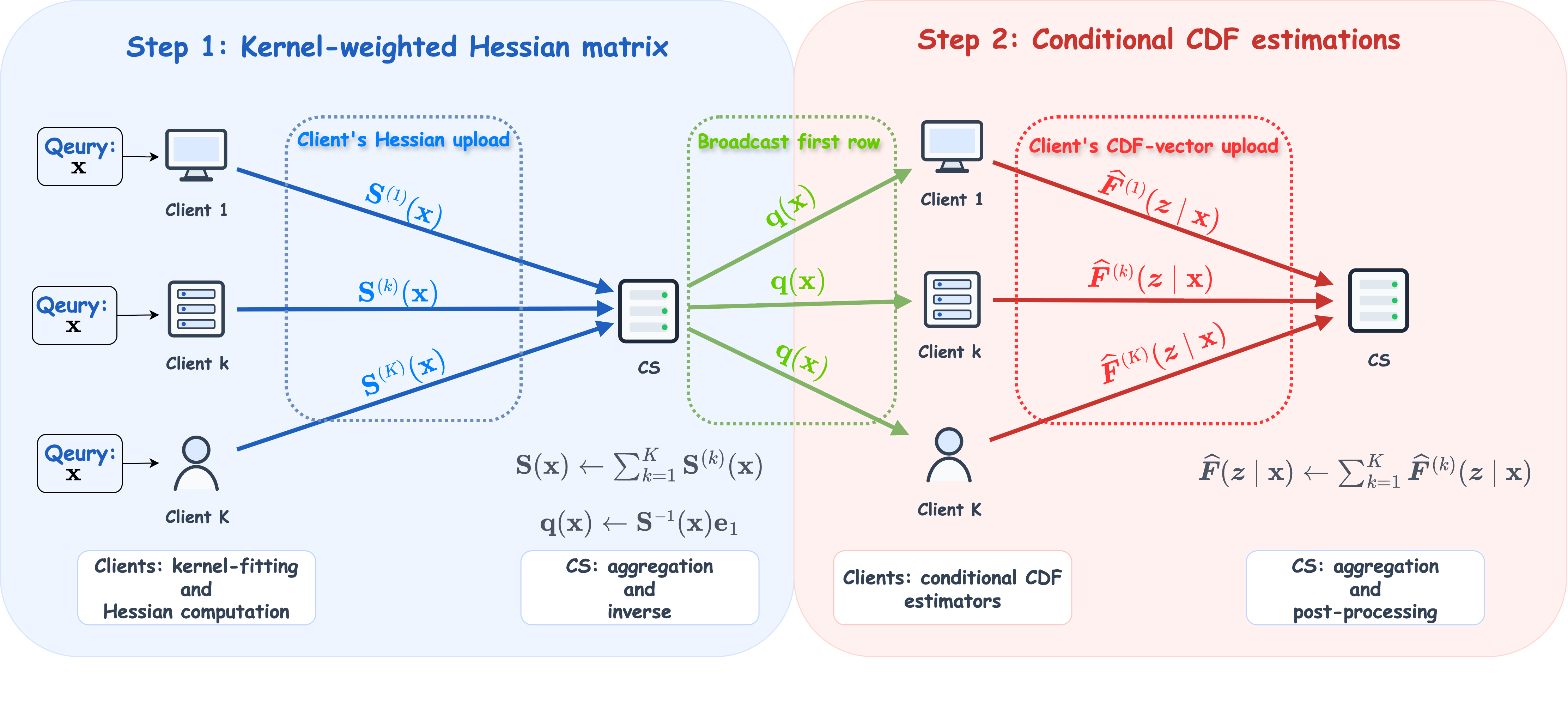}
    \caption{FITS: two-step aggregation. }
    \label{fig:twostep_agg}
\end{figure}
We describe our two-step federated aggregation procedure first, whose purpose is to perform federated estimation of the closed-form expression in \eqref{eq:llr_intercept_closed}.
The first step calculates the clients' kernel-weighted Hessian matrices, and the CS aggregates them and broadcasts the first-row vector of the inverse of the aggregated Hessian matrix.
In the second step, clients upload their vectors of conditional CDF estimates, after which the CS aggregates and post-processes them to obtain the final conditional CDF estimation vector.
Figure~\ref{fig:twostep_agg} illustrates the two-step procedure.

\paragraph{Step 1: aggregation of the kernel-weighted Hessian.}
For each client $k$, define the local block matrix
\begin{equation}\label{eq:local_block}
    \S^{(k)}(\x)
    :=
    \sum_{i}
    \varpi_i^{(k)}\,
    \U_i^{(k)}
    \U_i^{(k)\top}
    \in\mathbb R^{(d+1)\times(d+1)},
\end{equation}
where $\U_i^{(k)}=\bm\iota(\h_{i,:d}^{(k)}-\x)$ and
$\varpi_i^{(k)}$ is the (still abstract) covariate-kernel weight on the
$i$th local sample. We write $\S^{(k)}$ in block form as
\[
    \S^{(k)}(\x)
    =
    \begin{pmatrix}
        S_0^{(k)} & \S_1^{(k)\top}\\
        \S_1^{(k)} & \S_2^{(k)}
    \end{pmatrix},
    \qquad
    S_0^{(k)}=\sum_i \varpi_i^{(k)},\
    \S_1^{(k)}=\sum_i \varpi_i^{(k)}\u_i^{(k)},\
    \S_2^{(k)}=\sum_i \varpi_i^{(k)}\u_i^{(k)}\u_i^{(k)\top},
\]
where $\u_i^{(k)}:=\h_{i,:d}^{(k)}-\x$. The
global weighted Hessian satisfies
\begin{equation}\label{eq:hessian_decomp}
    \U^\top\W\U
    =
    \sum_{k=1}^{K}\S^{(k)}(\x).
\end{equation}
In the first communication round, each client uploads $\S^{(k)}(\x)$---an object of size $(d{+}1)(d{+}2)/2$ when symmetry is exploited---to the CS. 
The CS aggregates the uploads to form $\S(\x)=\sum_k\S^{(k)}(\x)$ and computes the vector
\begin{equation}\label{eq:q_vector}
    \bm q(\x)
    :=
    \bigl(\S(\x)+\lambda_{\rm ridge}\mathbf I_{d+1}\bigr)^{-1}
    \e_1
    \ \in\mathbb R^{d+1},
\end{equation}
where $\lambda_{\rm ridge}\ge 0$ is a small ridge stabilizer, with
$\lambda_{\rm ridge}=0$ corresponding to the unregularized
solution $ \bm q(\x)^\top =\e_1^\top(\U^\top\W\U)^{-1}$. The CS then broadcasts
$\bm q(\x)$ to all clients.

\paragraph{Step 2: aggregation of conditional CDF estimates.}
Given $\bm q(\x)$, the federated weight in
\eqref{eq:llr_intercept_closed} factorizes as
\begin{equation}\label{eq:alpha_federated}
    \alpha_i^{(k)}
    =
    \varpi_i^{(k)}\,
    \bm q(\x)^\top
    \U_i^{(k)},
\end{equation}
which is computable locally from $\varpi_i^{(k)}$, $\U_i^{(k)}$, and
the broadcast $\bm q(\x)$. Each client then computes the
local response-smoothed CDF vector at the target grid
$\bm z=(z_1,\ldots,z_L)\in\mathbb R^L$,
\begin{equation}\label{eq:local_F}
    \widehat{\bm F}^{(k)}(\bm z\mid\x)
    :=
    \sum_{i}
    \alpha_i^{(k)}\,
    \mathcal K\!\left(
        \frac{\bm z -  h_{i,d+1}^{(k)}\mathbf 1_L}{w^{(k)}_z}
    \right)
    \in\mathbb R^{L},
\end{equation}
where $\mathbf 1_L \in \mathbb{R}^L$ is the all-ones vector, $w^{(k)}_z$ is the $k$th client's response bandwidth, and $\mathcal K$ is applied componentwise. The client then uploads $\widehat{\bm F}^{(k)}(\bm z\mid\x)$ to the CS. The CS
aggregates,
\begin{equation}\label{eq:global_F_raw}
    \widehat{\bm F}^{\rm raw}(\bm z\mid\x)
    =
    \sum_{k=1}^{K}
    \widehat{\bm F}^{(k)}(\bm z\mid\x)\in\mathbb R^{L},
\end{equation}
and applies a post-processing operator $\mathcal M_{\rm cdf}$ that
clips to $[0,1]$, isotonizes in $z$, and normalizes the
right endpoint to $1$ to obtain the final estimate
$\widehat{\bm F}(\bm z\mid\x)$. 

\begin{algorithm}[H]
\caption{FITS: Two-Step Federated Estimation of the Conditional CDF}
\label{alg:two_step_fllr}
\begin{algorithmic}[1]
\REQUIRE Query covariate ${\x}\in{\mathcal X}\subset\mathbb R^d$; target grid $\bm z=(z_1,\dots,z_L)\in\mathbb R^L$; response bandwidth $w_z$; clients $k=1,\dots,K$ with continuous hold-out data ${\mathcal H}_k=\{{\h}_i^{(k)}\}_{i=1}^{m_k}$, where ${\h}_i^{(k)}=({\h}_{i,:d}^{(k)},{h}_{i,d+1}^{(k)})$; kernel-bandwidth scheme parameter $r_0 \in \mathbb N$.
\STATE \textbf{----------------- Step 1: aggregation of the kernel-weighted Hessian -----------------}
\FOR{$k=1,\dots,K$ \textbf{in parallel}}
    \STATE According to Eq.~\eqref{eq:client_fitting_func}, run $\mathcal W_k(\bm x;\mathcal H_k,r_0)$ to get the index set of selected samples $\mathcal N_{c_k}^{(k)}\subseteq[m_k]$ and the samples' kernel weights $\varpi_i^{(k)},\ i\in\mathcal N_{c_k}^{(k)}$. 
    \STATE According to the summation Eq.~\eqref{eq:hessian_decomp}, client $k$ obtains $\S^{(k)}( \x)$ and uploads it to CS.
\ENDFOR
\STATE \textbf{// CS aggregates and broadcasts}
\STATE CS computes and broadcasts $q( \x)$ according to Eq.~\eqref{eq:q_vector}.

\STATE \textbf{----------------- Step 2: aggregation of conditional CDF estimates -----------------}
\FOR{$k=1,\dots,K$ \textbf{in parallel}}
    \STATE For each $i\in\mathcal N_{c_k}^{(k)}$, client $k$ computes the federated local linear weight $\alpha^{(k)}_i$ according to Eq.~\eqref{eq:alpha_federated}.
    \STATE Compute and upload the local response-smoothed CDF vector $\widehat{\bm F}^{(k)}(\bm z\mid \x)$ according to Eq.~\eqref{eq:local_F}.
\ENDFOR

\STATE \textbf{// CS aggregates and post-processes}
\STATE According to Eq.~\eqref{eq:global_F_raw}, CS gets $\widehat{\bm F}^{\rm raw}(\bm z\mid \x)$ and applies CDF post-processing $\mathcal M_{\rm cdf}$ to it. 
\RETURN $\widehat{\bm F}(\bm z\mid \x) = \mathcal M_{\rm cdf}[\widehat{\bm F}^{\rm raw}(\bm z\mid \x)]$.
\end{algorithmic}
\end{algorithm}
\begin{rem}[Communication cost]\label{rem:cost} The two-step procedure requires, per query $\x$: (i) Round 1. Each client uploads a symmetric
    $(d{+}1)\times(d{+}1)$ matrix, costing $(d{+}1)(d{+}2)/2$
    scalars per client; the CS broadcasts a $(d{+}1)$-dimensional
    vector $\bm q(\x)$. (ii) Round 2. Each client uploads an $L$-dimensional CDF
    vector with elements $\widehat{\bm F}^{(k)}(z_i\mid\x), \ i\in[L]$. The per-query communication is therefore $O(d^2+L)$ per client, independent of the local sample size $m_k$. 
\end{rem}

The full procedure is given in Algorithm~\ref{alg:two_step_fllr} and the communication cost is $O(d^2+L)$ as stated in Remark~\ref{rem:cost}. 
Here, in Algorithm~\ref{alg:two_step_fllr}, $\mathcal W_k(\bm x;\mathcal H_k,r_0)$ denotes the client-designed procedure described in Subsection~\ref{subsec:bandwidth_schemes}, where clients can choose their own kernel-fitting schemes.
It returns the selected local samples together with their weights $\varpi$, and allows each client to adaptively choose a kernel-fitting scheme, either the fixed-bandwidth scheme or the KNN-kernel scheme.
The fixed-bandwidth scheme uses all local hold-out samples, so that $\mathcal N_{c_k}^{(k)}=[m_k]$. By contrast, the KNN-kernel scheme searches for the neighbors of $\x$ and uses a subset $\mathcal N_{c_k}^{(k)}\subset [m_k]$.

Combining \eqref{eq:hessian_decomp}--\eqref{eq:alpha_federated} with \eqref{eq:llr_intercept_closed} yields the following algebraic equivalence between the federated and pooled estimators.

\begin{prop}[Federated--pooled equivalence]
\label{prop:fed_pooled_equiv}
Fix any query $\x\in\mathcal X$ and threshold
$z\in\mathcal Z$. Suppose the kernel weights $\{\varpi_i^{(k)}: i\in
[m_k],\ k\in[K]\}$ used by the federated algorithm coincide with
those that would be assigned to the same observations under the
centralized rule on the pooled sample, and that
$\lambda_{\rm ridge}=0$ and $\U^\top\W\U$ is invertible. Then
\[
    \widehat F^{\rm raw}(z\mid\x)
    \ =\
    \sum_{k=1}^{K}\widehat F^{(k)}(z\mid\x)
    \ =\
    \hat\beta_0,
\]
where $\hat\beta_0$ is the centralized local linear estimator
\eqref{eq:llr_intercept_closed} computed on the pooled sample.
\end{prop}
\begin{prf}
Substituting \eqref{eq:hessian_decomp} into \eqref{eq:q_vector}
gives $\bm q(\x)^\top=\e_1^\top(\U^\top\W\U)^{-1}$ when $\lambda_{\rm ridge}=0$.
Hence, \eqref{eq:alpha_federated} equals
the centralized weight $\alpha_i$ in
\eqref{eq:llr_intercept_closed}. Summing \eqref{eq:local_F} over
$k$ reproduces the right-hand side of
\eqref{eq:llr_intercept_closed}.
\end{prf}

Proposition~\ref{prop:fed_pooled_equiv} establishes that Algorithm~\ref{alg:two_step_fllr} is not an approximation of the centralized estimator but an exact reformulation of it under federated data access. 
The practical implication is that all statistical properties derived for the centralized estimator---in particular the uniform convergence rates of Section~\ref{sec:theory}---transfer to the federated estimator without further argument.

\subsection{Client-designed kernel}
\label{subsec:bandwidth_schemes}

In this subsection, we propose two client-designed kernel schemes: the fixed-bandwidth scheme and the KNN-kernel scheme, supporting the $\mathcal W_k(\bm x;\mathcal H_k,r_0)$ procedure in line 3 of Algorithm~\ref{alg:two_step_fllr}.
Clients may choose either scheme to select local samples and compute their kernel weights.
Moreover, both schemes leverage the client's local data to standardize the selected samples and the query point before kernel fitting, thereby improving the stability and effectiveness of kernel fitting.
The standardization procedure is performed entirely on the client side: the marginal means, standard deviations, and raw covariate scales are not transmitted, in accordance with the data-access constraints in Subsection~\ref{subsec:fed_setup}.

\emph{Fixed-bandwidth scheme.}
Under the fixed-bandwidth scheme, each client uses all of its
local hold-out samples, so that
$\mathcal N_{c_k}^{(k)}=[m_k]$, together with a client-specific
covariate bandwidth $w_x^{(k)}>0$. After the covariate coordinates have been locally standardized using the client's marginal means and standard deviations, the weight assigned to observation $i$ is
\begin{equation}\label{eq:fixed_bw_weight}
    \varpi_i^{(k)}
    =
    \bm\kappa\!\left(
        \frac{\bar\x^{(k)}-\bar\h_i^{(k)}}
             {w_x^{(k)}}
    \right)\rho_i^{(k)}, \qquad i \in [m_k]
\end{equation}
where $\bm\kappa(\cdot)=\prod_{j=1}^{d}\kappa_j(\cdot)$ is the product kernel and barred quantities denote locally Z-score standardized coordinates.
$\rho_i^{(k)}$ is the category penalty introduced in Subsection~\ref{subsec:mixed_adaptation} and is used to reduce the kernel-fitting error caused by mismatched categories; if there is no discrete variable, $\rho_i^{(k)} = 1$.
This scheme is conceptually simple and directly inherits the usual fixed-bandwidth local linear interpretation, but it may lead to highly uneven effective sample sizes due to client imbalance.
Similarly, the effective sample size problem also affects the pooled method.
A possible approach is based on KNN-type methods \citep{fan1992variable,benezet2025learning}.

\emph{KNN-kernel scheme.}
Under the KNN-kernel scheme, the effective bandwidth is chosen adaptively at each client based on the distance to the $c_k$-th nearest neighbor of the query point.
The local KNN budgets are allocated in proportion to the hold-out sample sizes:
\begin{equation}\label{eq:knn_budget}
    c_k
    =
    \max\!\left\{
        c_0,\
        \min\!\left(
            \big\lfloor \frac{m_k}{M} r_0\big\rceil,\
            m_k
        \right)
    \right\},
\end{equation}
where $r_0\in \mathbb N$ is a global KNN budget and $c_0  >1$ is the minimum number of nearest neighbors. 
Hence, the total number of observations contributing to the federated local fit is approximately $r_0$, matching the budget of a centralized KNN rule with total sample size $M$.
After locally standardizing the covariate coordinates, client $k$ selects the $c_k$ nearest neighbors of the query point and sets
\[
    r_{c_k}^{(k)}(\bar\x)
    =
    \max_{i\in\mathcal N_{c_k}^{(k)}}
    \|\bar\h_i^{(k)}-\bar\x^{(k)}\|.
\]
The corresponding kernel weight is
\begin{equation}\label{eq:knn_weight}
    \varpi_i^{(k)}
    =
    \bm\kappa\!\left(
        \frac{\bar\h_i^{(k)}-\bar\x^{(k)}}
             {\eta\,r_{c_k}^{(k)}(\bar\x)}
    \right)\rho_i^{(k)},
    \qquad
    i\in\mathcal N_{c_k}^{(k)},
\end{equation}
where $\eta>0$ is a fixed kernel-scale parameter. Typically, a bounded-support kernel $\bm \kappa(\cdot)$ is used for the KNN-kernel scheme.

Based on these schemes, we define a kernel-fitting function for each client that outputs the index set of selected samples and their kernel weights:
\begin{equation}\label{eq:client_fitting_func}
    \mathcal W_k(\bm x;\mathcal H_k,r_0) \to N_{c_k}^{(k)}, \{\varpi_i^{(k)}\}_i^{N_{c_k}^{(k)}}, \ k\in [K],
\end{equation} where $r_0 \in \mathbb N$, $r_0>0$ selects the KNN-kernel scheme, whereas $r_0=0$ selects the fixed-bandwidth scheme.
\begin{rem}
    A separate design choice concerns the scale on which the local linear expansion is formed.
    The kernel weights $\varpi_i^{(k)}$ are evaluated in the locally standardized covariate space, but the local linear design vector $\U_i^{(k)}=\bm\iota(\h_{i,:d}^{(k)}-\x)$ is constructed in the original coordinate space.
    This mismatch is intentional.
    Standardization stabilizes the kernel-fitting step under heterogeneous client scales;
    keeping the design in the native scale ensures that the local linear expansion targets $F(z\mid\x)$ in the original covariate system, so that the federated estimator remains an estimator of the population conditional distribution on ${\mathcal X}$ rather than on a client-dependent rescaling thereof.
\end{rem}

\subsection{Federated Synthetic Data Generation}
\label{subsec:generation}

\begin{figure}[ht]
    \centering
    \includegraphics[width=\textwidth]{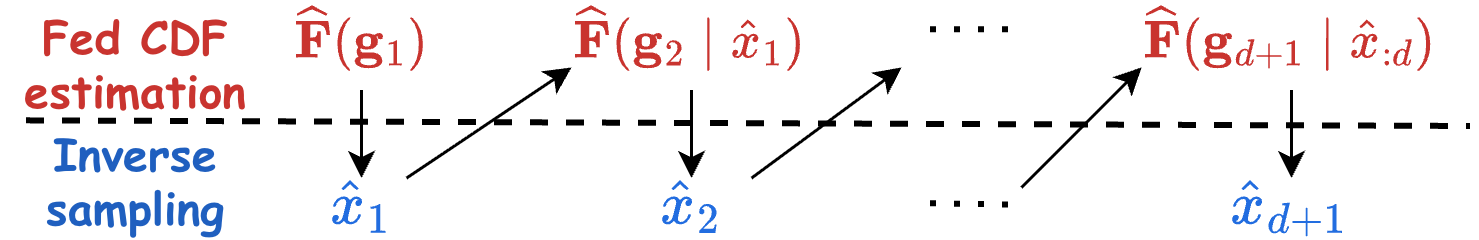}
    \caption{FITS: generation of a single synthetic sample by the chain rule. }
    \label{fig:syn}
\end{figure}
We combine the chain rule with inverse sampling from the federated conditional CDF estimator $\widehat{\bm F}(\bm z\mid\x)$ to generate synthetic data.
Without loss of generality, we assume that the chain order follows the variable sequence from $\X_1$ to $\X_{d+1}$ and denote $\X_{:j}=(\X_1,\ldots,\X_j)$.
The $j$-th sampling step is based on the one-dimensional conditional distribution function given the previously sampled $j-1$ dimensions:
\begin{equation}\label{eq:inv_sampling}
    \hat{x}_j = \widehat{\bm F}^{-1}
    (U \mid\hat{\x}_{:j-1}=(\hat{x}_1,\ldots,\hat{x}_{j-1})),
\end{equation}
where $U\sim \mathrm{Unif}[0,1]$, $\hat{x}_i$ denotes the sampled value of the $i$-th coordinate, and $\widehat{\bm F}^{-1}(\cdot \mid \hat{\x}_{:j-1})$ denotes the estimated inverse conditional CDF given the event $\hat{\x}_{:j-1}=(\hat{x}_1,\ldots,\hat{x}_{j-1})$.
For notational convenience, we denote $\widehat{\bm F}(z \mid \hat{\x}_{:0})$ as the unconditional CDF, which is equivalent to the CDF estimate for the first variable $\widehat{\bm F}(z)$.
For conditional CDF estimation, we estimate the conditional CDF using Algorithm~\ref{alg:two_step_fllr} and invert the estimate using linear interpolation.
For the first dimension, we use the following smooth KDE-based CDF estimator:
\begin{equation}\label{eq:first_dim_kde}
    \widehat{\bm F}(z)
    =
    \frac{1}{M}
    \sum_{k=1}^{K}
    \widehat{\bm F}^{(k)}(z)
    =
    \frac{1}{M}
    \sum_{k=1}^{K}
    \left[
    \sum_{i=1}^{m_k}
    \mathcal{K}
    \left(
    \frac{z-h_{i,1}^{(k)}}{w_z^{(k)}}
    \right)
    \right],
\end{equation}
where $\mathcal{K}(\cdot)$ is the kernel CDF.
The federated CDF estimation in this procedure is straightforward: each client uploads the estimated values $\widehat{\bm F}^{(k)}(z_\ell)$ on the grid points $\bm z=(z_1,\ldots,z_L)$.
Values between adjacent grid points are obtained by linear interpolation.
Specifically, if
$
U\in
(\widehat{\bm F}(z_{\ell-1}\mid\hat{\x}_{:j-1}),
\widehat{\bm F}(z_{\ell}\mid\hat{\x}_{:j-1})),
$
then
\begin{equation}\label{eq:interpolation}
    \hat{x}_j
    =
    z_{\ell-1}
    +
    \frac{
    U-\widehat{\bm F}(z_{\ell-1}\mid\hat{\x}_{:j-1})
    }{
    \widehat{\bm F}(z_\ell\mid\hat{\x}_{:j-1})
    -
    \widehat{\bm F}(z_{\ell-1}\mid\hat{\x}_{:j-1})
    }
    (z_\ell-z_{\ell-1}),
    \qquad \ell\in[L].
\end{equation}

\begin{algorithm}[ht]
\caption{FITS: Federated Synthetic Data Generation}
\label{alg:chain_generation}
\begin{algorithmic}[1]
\REQUIRE Number of synthetic samples $B$; coordinate grids $\mathcal G=\{\bm g_j\}_{j=1}^{d+1}$ in the chain order; clients $k=1,\ldots,K$ with hold-out data $\mathcal H_k=\{\h_i^{(k)}\}_{i=1}^{m_k}$, and the inputs required by Algorithm~\ref{alg:two_step_fllr}.
\FOR{$j=1,2,\ldots,d+1$}
    \FOR{$b=1,\ldots,B$}
        \STATE Estimate the CDFs in a federated manner
        \[\begin{cases}
            \widehat{\bm F}(\bm{g}_1) \text{ by Eq.~\eqref{eq:first_dim_kde}} , & j=1,\\[1mm]
                \widehat{\bm F}(\bm{g}_j \mid\hat{\x}_{b,:j-1} = (\hat{x}_{b,1},...,\hat{x}_{b,j-1})) \text{ by Algorithm~\ref{alg:two_step_fllr}}, & j>1.
        \end{cases} 
        \]

        \STATE CS draws $u_{bj}\sim\mathrm{Unif}(0,1)$ and samples
        \[
            \hat x_{b,j}
            \leftarrow
                \widehat{\bm F}^{-1}(u_{bj} \mid\hat{\x}_{:j-1} = (\hat{x}_1,...,\hat{x}_{j-1})) \text{ by Eq.~\eqref{eq:inv_sampling} and Eq.~\eqref{eq:interpolation}}.
        \]
    \ENDFOR
\ENDFOR

\RETURN  Synthetic data $\mathcal{B}=\{\hat{\x}_b = (\hat{x}_{b,1},...,\hat{x}_{b,d+1})\}_{b=1}^B$.
\end{algorithmic}
\end{algorithm}
\begin{rem}[Cost of synthetic data generation]
For $D$ variables, $B$ synthetic records, $K$ clients, and $L$ response grid points, each conditional CDF query with $r$ conditioning variables requires each client to upload $(r+1)(r+2)/2$ Hessian entries and $L$ CDF values. 
Thus, direct execution of the conditional generation steps in Algorithm~\ref{alg:chain_generation} requires
$
O\!\left(KB(D^3+DL)\right)
$
uploaded scalars in total.

\end{rem}

\section{Generalization to Mixed-Type Data}\label{sec:general_mix}
In this section, we propose a continualization method for discrete variables.
In Subsection~\ref{subsec:continualize_univariate}, the proposed continualization procedure lifts discrete variables into a continuous space and allows exact inversion back to the original discrete space without information loss.
In particular, we characterize continualization at the density-function level, thereby unifying mixed-type data within a continuous framework in Subsection~\ref{subsec:continualize_multi}.
For illustration, we provide a bivariate example in Subsection~\ref{subsec:toy_example}.
For practical implementation, we further develop a penalty-adaptive kernel method for mixed-type covariates in Subsection~\ref{subsec:mixed_adaptation}.
This method effectively mitigates the impact of mismatches between continualized discrete samples used for kernel fitting and the query point.

\subsection{Continualization of discrete variables}
\label{subsec:continualize_univariate}

The lifting is constructed independently for each discrete coordinate; hence, this subsection focuses on the univariate case.
Discrete variables include ordered variables and unordered (nominal/categorical) variables.
For ordered variables, we lift the discrete support into continuous adjacent intervals.
For unordered variables, we fix an arbitrary one-to-one numerical coding to define the continualization intervals.

\paragraph{Ordered discrete variables.}
Let $\Upsilon$ be a discrete random variable with finite ordered
support
\[
    \varUpsilon=\{\upsilon_1,\ldots,\upsilon_L\},
    \qquad
    \upsilon_1<\cdots<\upsilon_L,
    \qquad
    p_\ell=\mathbb P(\Upsilon=\upsilon_\ell),\quad
    \sum_{\ell=1}^L p_\ell=1.
\]
Set $\upsilon_0:=\upsilon_1-1$ and define the disjoint intervals $I_\ell$ and an interval-valued random variable $I_\Upsilon$ corresponding to $\Upsilon$:
\[
    I_\ell=(\upsilon_{\ell-1},\upsilon_\ell],
    \qquad
    I_\Upsilon=(\lfloor \Upsilon\rfloor,\Upsilon],
    \qquad
    \ell\in[L],
\]
where $\lfloor \Upsilon\rfloor$ represents the predecessor
$\upsilon_{\ell-1}$ of $\Upsilon=\upsilon_\ell$ in the ordered support, and
$|I_\ell|=\upsilon_\ell-\upsilon_{\ell-1}$.
Let $U\sim\mathrm{Unif}(0,|I_\Upsilon|)$ be independent of $\Upsilon$, and define the continualized variable
\begin{equation}\label{eq:continualize_univ}
    Y=\Upsilon-U.
\end{equation}
Then $Y$ takes values in $\bigcup_{\ell=1}^L I_\ell$, and on each
interval $I_\ell$, the mass $p_\ell$ is spread uniformly. The
density and distribution function of $Y$ are
\begin{equation*}
    f_Y(y)
    =
    \sum_{\ell=1}^L (p_\ell/|I_\ell|)\,\mathbb I\{y\in I_\ell\},
    \qquad
    F_Y(y)
    =
    \sum_{r<\ell}p_r+p_\ell\frac{(y-\upsilon_{\ell-1})}{|I_\ell|},
    \quad y\in I_\ell,
\end{equation*}
so $f_Y$ is piecewise constant and $F_Y$ is piecewise affine.

\paragraph{Inverse map via the ceiling operator.}
The original discrete variable is recovered from its continualization by
\begin{equation}\label{eq:ceiling}
    \Upsilon=\mathrm{ceil}(Y)
    :=\upsilon_\ell
    \text{ on the event }Y\in I_\ell.
\end{equation}
Equivalently, $\mathrm{ceil}$ is the canonical inverse map between
the partition $\{I_\ell\}_{\ell=1}^L$ of
$\bigcup_\ell I_\ell$ and the support $\varUpsilon$.

\paragraph{Unordered (nominal) discrete variables.}
For an unordered categorical variable with support
$\{c_1,\ldots,c_L\}$, the same construction applies after fixing
an arbitrary one-to-one coding $c_\ell\leftrightarrow\ell$.
We set
\[
    I_\ell=(\ell-1,\ell],
    \qquad
    Y=\ell-U \text{ on the event } \Upsilon=c_\ell,
    \qquad
    U\sim\mathrm{Unif}(0,1),
\]
and recover the original category through
\begin{equation}\label{eq:cat_operator}
    \Upsilon=\mathrm{cat}(Y):=c_{\lceil Y\rceil}.
\end{equation}
Because the $\mathrm{ceil}$ operator is also included in the $\mathrm{cat}$ operation, we use $\mathrm{cat}$ as the unified inverse operator.
The coding is used to define the partition and inverse map and introduces a Euclidean metric.
We summarize the properties of the univariate construction here.
The more general multidimensional construction is presented in Subsection~\ref{subsec:continualize_multi}, together with the corresponding theoretical results, Theorem~\ref{thm:continuity_continualization_pdf} and Corollary~\ref{cor:Cr_continualization}.

\begin{prop}[Univariate continualization]
\label{prop:univariate_continualize}
Let $\Upsilon$ be a discrete variable on
$\varUpsilon=\{\upsilon_1,\ldots,\upsilon_L\}$ (ordered or
nominal) and let $Y$ be its continualization as in
\eqref{eq:continualize_univ}. Then:
\begin{enumerate}[(i)]
    \item \textbf{Mass preservation.}
    $\mathbb P(Y\in I_\ell)=p_\ell=\mathbb P(\Upsilon=\upsilon_\ell)$
    for every $\ell\in[L]$.
    \item \textbf{Exact inverse.}
    For ordered variables, $\mathrm{ceil}(Y)=\Upsilon$ almost surely; for nominal variables, $\mathrm{cat}(Y)=\Upsilon$ almost surely.
    \item \textbf{Piecewise regularity.}
    The density $f_Y$ is piecewise constant on the partition
    $\{I_\ell\}_\ell$, and the distribution function $F_Y$ is
    continuous on $\bigcup_\ell I_\ell$ and piecewise affine
    within each $I_\ell$.
\end{enumerate}
\end{prop}

Properties~(i)--(iii) follow directly from the definition and therefore require no further proof.

\subsection{Multidimensional continualization and block structure}
\label{subsec:continualize_multi}

We now extend the univariate construction to the mixed-type
covariate vector $\X$. Write
\[
    \X=\bigl(\X^{(c)},\X^{(d)}\bigr),
\]
where $\X^{(c)}\in\mathcal X_c\subset\mathbb R^{d-\tilde d}$
collects the $d-\tilde d$ continuous coordinates and
$\X^{(d)}$ collects the $\tilde d$ discrete coordinates. For
each discrete coordinate $j\in[\tilde d]$, let
$\varUpsilon_j=\{\upsilon_{j,1},\ldots,\upsilon_{j,\tau_j}\}$
denote its support, and let
$\mathcal Y_j=\bigcup_{\ell=1}^{\tau_j}I_{j,\ell}$ denote the union
of its continualization intervals as defined in
Section~\ref{subsec:continualize_univariate}. The product support
of the discrete covariates is
\[
    \varUpsilon
    =
    \prod_{j=1}^{\tilde d}\varUpsilon_j .
\]
For each
$\bm\upsilon=(\upsilon_{1,\ell_1},\ldots,\upsilon_{\tilde d,\ell_{\tilde d}})
\in\varUpsilon$, define the associated block
\begin{equation}\label{eq:block}
    I(\bm\upsilon)
    =
    I_{1,\ell_1}\times\cdots\times I_{\tilde d,\ell_{\tilde d}}
    \ \subset\ \mathcal Y:=\prod_{j=1}^{\tilde d}\mathcal Y_j .
\end{equation}
Under the left-open right-closed convention, the blocks
$\{I(\bm\upsilon)\}_{\bm\upsilon\in\varUpsilon}$ form a partition
of $\mathcal Y$, and the indicator
$\mathbb I\{\Y\in I(\bm\upsilon)\}$ is equivalent to
$\mathbb I\{\X^{(d)}=\bm\upsilon\}$.
The block $I(\bm\upsilon)$ also defines a corresponding volume
$|I(\bm\upsilon)|$, which is used to normalize the lifted density function.

Let $\Y$ denote the joint continualization of $\X^{(d)}$, obtained
by applying \eqref{eq:continualize_univ} coordinatewise with
mutually independent uniform random variables. The resulting
continualized vector is
\begin{equation}\label{eq:continualize_multi}
    \widetilde\X
    =
    \bigl(\X^{(c)},\Y\bigr)
    \in\mathcal X_c\times\mathcal Y
    =:\widetilde{\mathcal X}.
\end{equation}
Throughout the paper, $\x$ denotes a generic point in the original
mixed-type covariate space $\mathcal X$, while
$\tilde\x=(\x^{(c)},\y)$ denotes the corresponding point in
$\widetilde{\mathcal X}$. The inverse discrete map
\[
    \tilde\x
    \mapsto
    \bigl(\x^{(c)},\mathrm{cat}_1(y_1),\ldots,
    \mathrm{cat}_{\tilde d}(y_{\tilde d})\bigr)
\]
recovers the original mixed-type variable exactly.

The block structure naturally separates $\widetilde{\mathcal X}$
into an interior region and a boundary region. Define
\begin{equation}\label{eq:interior}
    \widetilde{\mathcal X}^\circ
    :=
    \mathcal X_c\times
    \mathcal Y^\circ
    =
    \mathcal X_c\times
    \bigcup_{\bm\upsilon\in\varUpsilon}
    I(\bm\upsilon)^\circ,
\end{equation}
where $I(\bm\upsilon)^\circ$ denotes the topological interior of
$I(\bm\upsilon)$ and $\mathcal Y^\circ$ denotes the union of the interiors of all blocks.
Its complement is
\[
    \widetilde{\mathcal X}^\bullet
    :=
    \mathcal X_c\times\mathcal Y^\bullet,
    \qquad
    \mathcal Y^\bullet
    :=
    \mathcal Y\setminus\mathcal Y^\circ .
\]
This set is the union of all block boundaries in the $\y$-coordinates
and will be referred to as the covariate jumping set (block-boundary set), since
the lifted joint density of $\widetilde\X$ may exhibit jump discontinuities
across $\widetilde{\mathcal X}^\bullet$.

This construction yields a piecewise smooth lifted law. On each
block $\mathcal X_c\times I(\bm\upsilon)^\circ$, the joint density
of $\widetilde\X$ inherits the smoothness of the conditional density of
$\X^{(c)}$ given $\X^{(d)}=\bm\upsilon$, because the continualization
is uniform within the corresponding discrete cell. Across
$\widetilde{\mathcal X}^\bullet$, the density may jump, since the
conditional densities of $\X^{(c)}$ on adjacent blocks need not
match. Thus the lifted law is generally piecewise smooth rather
than globally smooth. A formal piecewise-$C^r$ regularity condition is given in
Section~\ref{sec:theory}, where it supports the blockwise uniform convergence analysis of conditional CDF estimation.

The same continualization device is applied to the response when
$Z$ is discrete. Specifically, if
$Z\in\{z_1,\ldots,z_L\}$ with $z_1<\cdots<z_L$, set
$z_0:=z_1-1$, $I_i=(z_{i-1},z_i]$, and define
$\tilde Z=Z-V$, where $V\sim\mathrm{Unif}(0,|I_Z|)$ is independent of
$(Z,\widetilde\X)$. Then
$\tilde Z\in\widetilde{\mathcal Z}:=\bigcup_{i=1}^L I_i$, and the
target in \eqref{eq:target} is estimated as
$F(\tilde z\mid\tilde\x)$. If $Z$ is already continuous, no
response continualization is performed and we set $\tilde z=z$. Section~\ref{sec:theory} covers both cases.

\subsection{An illustrative bivariate example}
\label{subsec:toy_example}
\begin{figure}[H]
    \centering

    \begin{subfigure}[t]{0.72\linewidth}
        \centering
        \includegraphics[width=\textwidth]{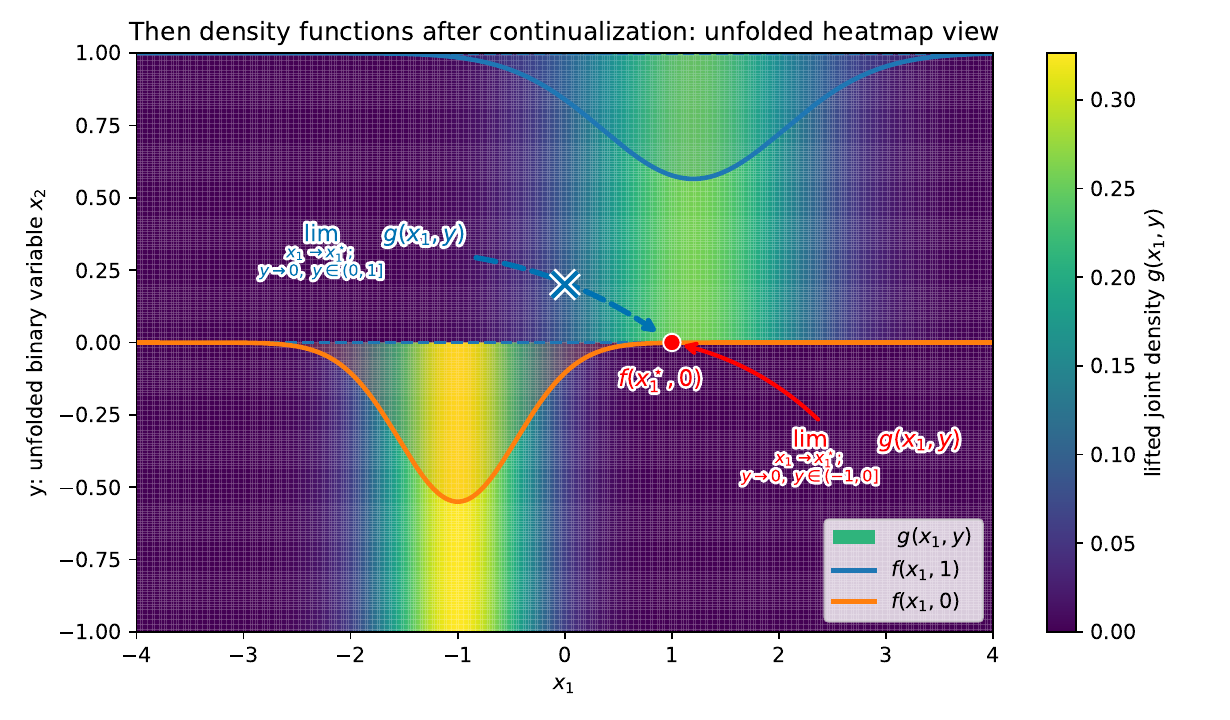}
        \caption{Lifted density after binary continualization.}
        \label{fig:toy_density_heatmap}
    \end{subfigure}

    \vspace{0.6em}

    \begin{subfigure}[t]{0.47\linewidth}
        \centering
        \includegraphics[width=\textwidth]{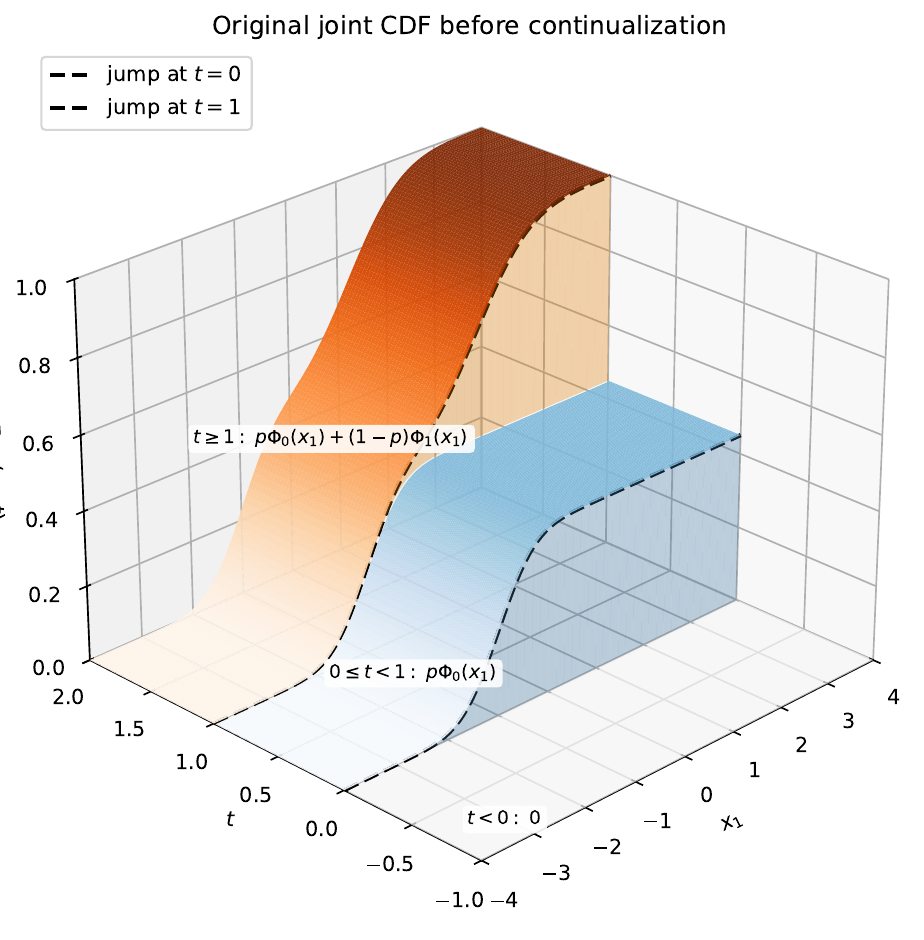}
        \caption{Original joint CDF.}
        \label{fig:toy_discrete_cdf}
    \end{subfigure}
    \hfill
    \begin{subfigure}[t]{0.47\linewidth}
        \centering
        \includegraphics[width=\textwidth]{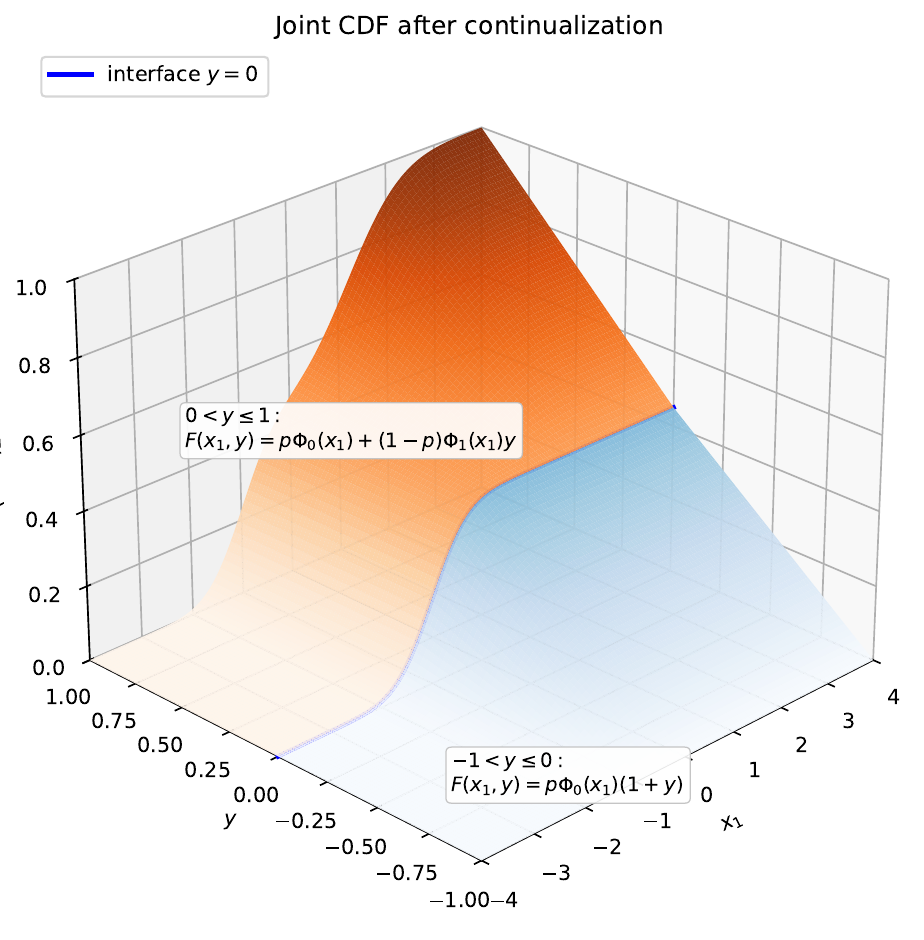}
        \caption{Continualized joint CDF.}
        \label{fig:toy_continualized_cdf}
    \end{subfigure}

    \caption{Illustration of binary continualization in a bivariate model with one continuous $X_1$ and one binary variable $X_2\in\{0,1\}$. Panel~(a) shows the lifted density
    $g(x_1,y)=f(x_1,0)\mathbb I\{{y\in(-1,0]}\} +f(x_1,1)\mathbb I \{y\in(0,1]\}$, where $f(x_1,0)=p f(x_1 \mid 0)$ and
    $f(x_1,1)=(1-p)f(x_1 \mid 1)$. The line $y=0$
    marks the internal block interface across which the lifted
    density may jump. Panel~(b) shows the original joint CDF
    $F_{\mathrm{orig}}(x_1,t)=\mathbb P(X_1\le x_1,X_2\le t)$,
    which is stepwise in the discrete coordinate. Panel~(c) shows
    the lifted CDF after setting $Y=X_2-U$, where
    $U\sim\mathrm{Unif}(0,1)$ is independent of $(X_1,X_2)$; on
    $(-1,0]$ and $(0,1]$, it is affine in $y$, with forms
    $p\Phi_0(x_1)(1+y)$ and
    $p\Phi_0(x_1)+(1-p)\Phi_1(x_1)y$, respectively.}
    \label{fig:toy_continualization}
\end{figure}

To fix ideas, consider a bivariate covariate $\x=(x_1,x_2)$, where $x_1$ is continuous on a compact set $\mathcal X_1\subset\mathbb R$ and $x_2\in\{0,1\}$ is binary with $\mathbb P(x_2=0)=p$. Suppose that the joint mass-density $f(x_1,x_2)$ is continuous in $x_1$ for each value of $x_2$ and the marginal density $f(x_1)>0$ is continuous on $\mathcal X_1$.

Then
\[
    f(x_1,x_2)
    =
    f(x_1,0)\mathbb I\{x_2=0\}
    +
    f(x_1,1)\mathbb I\{x_2=1\}.
\]

Under binary continualization, let $Y=X_2-U$, where $U\sim\mathrm{Unif}(0,1)$ is independent of $X_2$. The continualized variable takes values in $(-1,0]\cup(0,1]$, with blocks $I(0)=(-1,0]$ and $I(1)=(0,1]$, and the lifted covariate is $\tilde\x=(x_1,y)\in\mathcal X_1\times(-1,1]$. The lifted joint density has the blockwise form
\[
    g(x_1,y)
    =
    f(x_1,0)\mathbb I\{y\in I(0)\}
    +
    f(x_1,1)\mathbb I\{y\in I(1)\}.
\]

Thus $g$ is continuous, and as smooth as $f(\cdot,j)$, within each of the two open blocks $\mathcal X_1\times(-1,0)$ and $\mathcal X_1\times(0,1)$, but may jump across the internal interface $\mathcal X_1\times\{0\}$ whenever $f(x_1,0)\neq f(x_1,1)$. 
Figure~\ref{fig:toy_continualization} illustrates this lifted density and the corresponding change from a stepwise discrete-coordinate CDF to a piecewise affine continualized CDF.
The same blockwise interpretation applies to conditional quantities. 
For example, on the block corresponding to $x_2=0$,
\[
    g(x_1\mid y)
    =
    f(x_1\mid 0)
    =
    \frac{f(x_1,0)}{p},
    \qquad y\in(-1,0],
\]
and
\[
    g(y\mid x_1)
    =
    \mathbb P(X_2=0\mid x_1)
    =
    \frac{f(x_1,0)}{f(x_1)},
    \qquad y\in(-1,0].
\]

These lifted conditional functions are continuous in $x_1$ within the block and constant in the continualized coordinate $y$. The block $\mathcal X_1\times(0,1]$ is treated analogously, with $p$ replaced by $1-p$ and $f(x_1,0)$ replaced by $f(x_1,1)$. Finally, applying the inverse categorization map, here the ceiling map in \eqref{eq:ceiling}, collapses $\mathcal X_1\times(-1,0]$ and $\mathcal X_1\times(0,1]$ back to $\mathcal X_1\times\{0\}$ and $\mathcal X_1\times\{1\}$, respectively. Hence continualization creates a blockwise continuous representation without losing the original mixed-type information.

\subsection{Penalty-adaptive kernel}
\label{subsec:mixed_adaptation}

The continualization in Section~\ref{subsec:continualize_univariate} maps each discrete value $\upsilon_\ell$ to the interval $I_\ell=(\upsilon_{\ell-1},\upsilon_\ell]$, and the continualized query value $y_q$ lies somewhere inside the interval associated with its category.
If $y_q$, a continualized coordinate of the query point, is close to an interval boundary, a Euclidean KNN search in the continualized space may select observations that are geometrically close but belong to different discrete categories.
For the fixed-bandwidth scheme, kernel evaluation at $y_q$ may also assign excessive weight to mismatched neighboring observations, thereby affecting the effective fitting weights of observations with matched categories.
Such neighbors are undesirable because locality should respect both the continualized geometry and the recovered discrete labels.

\begin{figure}[ht]
    \centering
    \includegraphics[width=\textwidth]{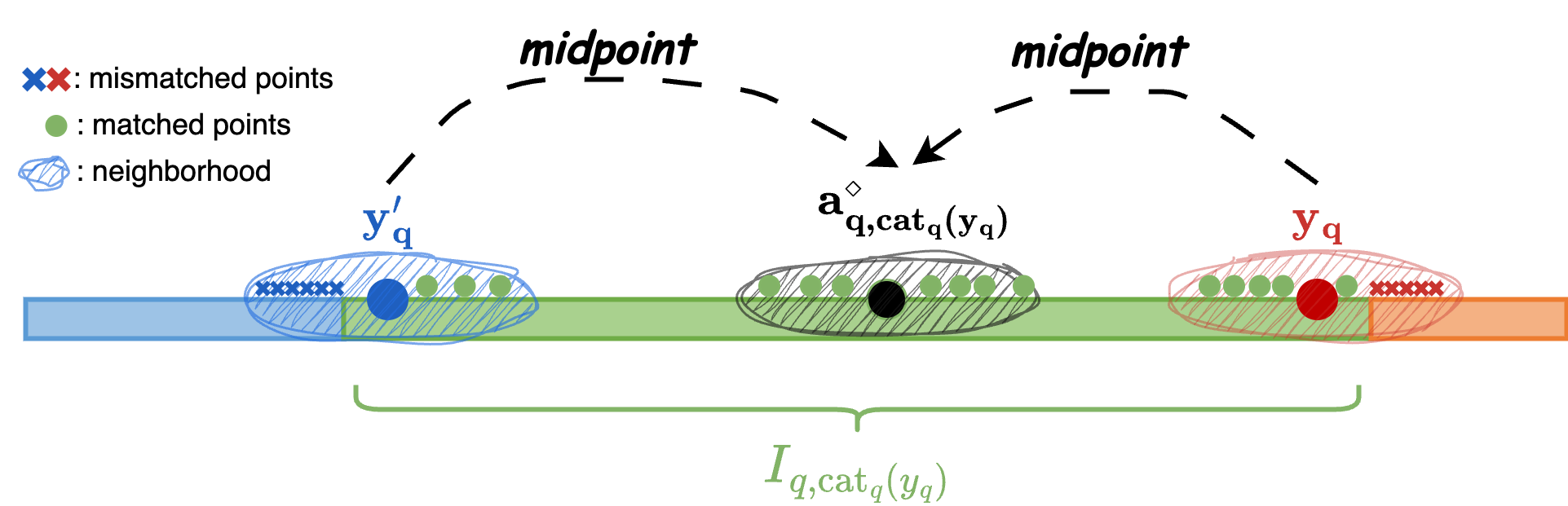}
    \caption{Boundary mismatches.}
    \label{fig:mid_point}
\end{figure}

We address this boundary effect by replacing each discrete covariate coordinate of the query point with the center of its continualization interval before standardization, neighbor selection, and kernel evaluation:
\begin{equation}\label{eq:interval_center}
    y_q^{\diamond}
    \leftarrow
    a_{q,\mathrm{cat}_q(y_q)}^{\diamond}
    :=
    \text{midpoint of }
    I_{q,\mathrm{cat}_q(y_q)},
    \qquad q\in\mathcal D,
\end{equation}
where $\mathcal D$ is the index set of discrete variables.
Because the lifted law is uniform within each discrete block (Proposition~\ref{prop:univariate_continualize} and Theorem~\ref{thm:conditional_distribution}), this query-side substitution does not alter the target conditional law within the corresponding block.
The corresponding theoretical guarantee is provided in Corollary~\ref{cor:jump_uniform}.
Its role is purely geometric: it moves the query point away from block interfaces so that the selected neighborhood is consistent with the query's discrete label.

We further multiply the continuous kernel weight by a
category-aware penalty. Define
\begin{equation}\label{eq:cat_penalty}
    \rho_i^{(k)}
    =
    \prod_{q\in \mathcal D}
    \rho_{i,q}^{(k)},
\end{equation}
with per-coordinate factor
\begin{equation}\label{eq:per_coord_penalty}
    \rho_{i,q}^{(k)}
    =
    \begin{cases}
        \exp\!\left(
            -\lambda_q
            \bigl|\mathrm{cat}_q(h_{i,q}^{(k)})
                -\mathrm{cat}_q(y^{\diamond}_q)\bigr|
        \right),
        & \mathrm{kind}_q=\mathrm{ordinal},\\[1mm]
        \mathbb I\bigl\{\mathrm{cat}_q(h_{i,q}^{(k)})
            =\mathrm{cat}_q(y^{\diamond}_q)\bigr\}
        +\gamma_q
        \bigl(1-\mathbb I\{\cdots\}\bigr),
        & \mathrm{kind}_q=\mathrm{nominal}.
    \end{cases}
\end{equation}
For ordinal variables, the penalty decays with the category
distance
$|\mathrm{cat}_q(h_{i,q}^{(k)})-\mathrm{cat}_q(y^{\diamond}_q)|$
at rate $\lambda_q>0$, preserving the ordering information of the original
support.
For nominal variables, a category mismatch receives a
multiplicative discount $\gamma_q\in(0,1]$, independent of the
numerical coding of the labels.

Thus, the mixed-type adaptation consists of two local deterministic modifications of the weighting procedure: a query-side interval-center substitution and a multiplicative category penalty.
Together, they make the final kernel weight reflect both geometric proximity in the continualized space and categorical proximity in the recovered discrete label space.
No changes are made to the two-step aggregation logic or to the standardization-versus-design scale separation described above.
In particular, the algebraic identity in Proposition~\ref{prop:fed_pooled_equiv} continues to hold for the mixed-type estimator, provided that the centralized local linear estimator is defined using the same weighting modifications.

\section{Theoretical Analysis}
\label{sec:theory}

In this section, we develop uniform convergence results for the
federated local linear conditional CDF estimator introduced in Section~\ref{sec:algorithm}. The convergence is uniform not only over the covariates but also over the response variable.
Specifically, the uniformity over the response applies to the function estimator $\hat{F}(z \mid \x)$ for all $z \in \mathbb{R}$ and provides a theoretical guarantee for the quality of data generation based on $\hat{F}(z \mid \x)$.
By Proposition~\ref{prop:fed_pooled_equiv}, the federated estimator coincides exactly with the centralized local linear estimator computed on the pooled sample whenever the kernel weights are matched.
The convergence analysis can therefore be carried out for the centralized estimator with total sample size $M=\sum_{k=1}^{K}m_k$, and the resulting rates transfer directly to the federated estimator.
Throughout this section, we consider the asymptotic regime $M\to\infty$ and let the bandwidths $w^{(k)}_x,w^{(k)}_z\to 0$ at rates to be specified.

Subsection~\ref{subsec:assumptions} states the regularity assumptions on the covariate marginal density, the conditional distribution function, and the kernels.
Subsection~\ref{subsec:simple_case} establishes uniform convergence in the basic setting where all covariates and the response are continuous, and serves as the analytical backbone for the mixed-variable case.
Subsection~\ref{subsec:piecewise_regularity} characterizes the piecewise smoothness of the lifted joint and conditional distributions induced by continualization, providing the regularity needed for the mixed-variable analysis.
Subsection~\ref{subsec:mixed_uniform} then extends the uniform convergence results to the continualized mixed-type setting, both on interior blocks and on the covariate jumping set, and establishes a strictly sharper rate for the conditional probability mass estimator when the response is discrete.

\subsection{Assumptions}
\label{subsec:assumptions}

We work with three groups of assumptions: a geometric and smoothness condition on the covariate marginal density, a
smoothness condition on the conditional distribution function, and a regularity condition on the kernels. 
The first two assumptions follow the standard local linear regression setup \citep{xie2023uniform} and are stated on a generic covariate domain $\mathcal X\subset\mathbb R^d$; 
the same conditions will be invoked blockwise in Subsection~\ref{subsec:mixed_uniform}. 
The kernel assumption is generalized to the Gaussian kernel with unbounded support.

\begin{ass}[Marginal density]
\label{ass:marginal}
The covariate domain $\mathcal X\subset\mathbb R^d$ is compact and
convex, and satisfies an internal cone condition: there exist
constants $\lambda_0,\lambda_1\in(0,1]$ such that, for every
$\x\in\mathcal X$ and every $w\in(0,\lambda_0]$, there exists
$\x'\in\mathcal X$ with $B(\x',\lambda_1 w)\subset B(\x,w)\cap
\mathcal X$, where $B(\x,\delta)$ denotes the Euclidean ball
centered at $\x$ with radius $\delta$. The covariate marginal
density $f$ satisfies:
\begin{enumerate}[(i)]
    \item $\inf_{\x\in\mathcal X}f(\x)>0$;
    \item $f\in C^{2}(\mathcal X^{\circ})$, with all second-order
    partial derivatives uniformly continuous and bounded on the interior domain
    $\mathcal X^{\circ}$.
\end{enumerate}
\end{ass}

\begin{ass}[Conditional distribution]
\label{ass:conditional}
The conditional distribution function
$F(z\mid\x)=\mathbb P(Z\le z\mid\X=\x)$ is twice continuously
differentiable on $\mathcal S=\mathcal{Z}\times\mathcal X$, where $\mathcal{Z} \subseteq \mathbb R$, and all
partial derivatives of order up to two are bounded and uniformly
continuous on $\mathcal S$.
\end{ass}

\begin{ass}[Kernel functions]
\label{ass:kernel}
The kernel functions used to define the weights $\varpi_i = \bm \kappa (\cdot)$ with the fixed bandwidth $w_x$ used in Section~\ref{sec:algorithm} are taken to be
Gaussian:
$\kappa(u)=\phi(u):=(2\pi)^{-1/2}e^{-u^{2}/2}$,
$\mathcal K(u)=\varPhi(u):=\int_{-\infty}^{u}\phi(v)\,dv$, and
$\bm\kappa(\mathbf t)=\prod_{j=1}^{d}\phi(t_j)$.
\end{ass}

Assumption~\ref{ass:marginal} is the standard interior cone condition under which uniform kernel approximation arguments apply. 
Assumption~\ref{ass:conditional} is the second-order smoothness condition required for local linear bias control.
Assumption~\ref{ass:kernel} generalizes the kernel family from kernels with bounded support to kernels with unbounded support, specifically the Gaussian kernel.

We adopt a common-bandwidth regime across clients in federated learning.
All clients use the same covariate bandwidth $w^{(k)}_x = w_x > 0$ and the same response bandwidth $w_z^{(k)}=w_z>0$. They satisfy $w_x,w_z\to 0$ and $Mw_x^{d}/|\log w_x|\to\infty$ as $M\to\infty$.
This is the standard fixed-bandwidth regime for $d$-dimensional local linear smoothing.

\subsection{Uniform convergence under continuous covariates}
\label{subsec:simple_case}

To isolate the analytic core, we first consider the case where
all covariates are continuous. The federated estimator
$\widehat F(z\mid\x)$ at the query point
$(\x,z)\in\mathcal X\times\mathbb R$ equals the centralized
local linear intercept \eqref{eq:llr_intercept_closed} on the
pooled sample, and so its analysis reduces to the standard
local linear conditional CDF analysis. We state the result here in a form that
will be reused blockwise in
Subsection~\ref{subsec:mixed_uniform}.

\begin{thm}[Uniform convergence under continuous covariates]
\label{thm:simple_uniform}
Let Assumptions~\ref{ass:marginal}--\ref{ass:kernel} hold. Then
\begin{equation}\label{eq:simple_uniform}
    \sup_{z\in\mathbb R,\ \x\in\mathcal X}
    \bigl|\widehat F(z\mid\x)-F(z\mid\x)\bigr|
    =
    O_{p}\!\left(
        w_x^{2}+w_z^{2}
        +\sqrt{\frac{|\log w_x|}{Mw_x^{d}}}
    \right).
\end{equation}
The deterministic term $w_x^{2}+w_z^{2}$ is the oracle bias and the
stochastic term $\sqrt{|\log w_x|/(Mw_x^{d})}$ controls the
empirical fluctuation.
\end{thm}

A sketch of the proof is as follows; the full argument is given
in Appendix~\ref{app:proof_simple_uniform}. Decompose the estimation error as
\begin{equation}\label{eq:est_decomposition}
    \widehat F(z\mid\x)-F(z\mid\x)
    =
    \underbrace{
        \widehat F(z\mid\x)
        -\beta_0^{\diamond}(z,\x,w_{z},w_{x})
    }_{\text{stochastic}}
	+
	\underbrace{
        \beta_0^{\diamond}(z,\x,w_{z},w_{x})
        -F(z\mid\x)
    }_{\text{bias}},
\end{equation}where $\beta_0^{\diamond}(z,\x,w_{z},w_{x})$ is the intercept of the minimizer of the expected objective corresponding to the empirical objective in \eqref{eq:llr_global}. 
Under Assumption~\ref{ass:conditional}, a second-order Taylor expansion of $F(z\mid\x)$ in both $\x$ and $z$ shows that the conditional bias is uniformly $O(w_x^{2}+w_z^{2})$, with the first-order term in $\x$ eliminated by the local linear construction. 
The stochastic term is handled by an empirical process argument over the Euclidean class
$\{(\x',z')\mapsto\bm\kappa((\x'-\x)/w_x)\mathcal K((z'-z)/w_z):
\x\in\mathcal X,\,z\in\mathbb R\}$, whose entropy is controlled
through Lemmas~\ref{lem:Euclidean_class}--\ref{lem:empirical_bound}
in Appendix~\ref{app:tools}.

The following corollary records the implication for pointwise
estimation, which is the form most directly useful in downstream
applications.

\begin{cor}[Conditional CDF estimation bound]
\label{cor:simple_pointwise}
Under the conditions of Theorem~\ref{thm:simple_uniform}, for any fixed $(z,\x)$ in the interior of $\mathbb R\times\mathcal X$,
\[
    |\widehat F(z\mid\x)-F(z\mid\x)|
    =
    O_{p}\!\left(
        w_x^{2}+w_z^{2}
        +\sqrt{\frac{|\log w_x|}{Mw_x^{d}}}
    \right),
\]
with the same constants as in \eqref{eq:simple_uniform}.
\end{cor}

When at least one covariate coordinate is discrete and is handled
through the continualization of
Section~\ref{subsec:continualize_univariate}, neither
Assumption~\ref{ass:marginal} nor Assumption~\ref{ass:conditional}
holds globally on $\widetilde{\mathcal X}$: the lifted density
generally jumps across $\widetilde{\mathcal X}^{\bullet}$, and the
lifted conditional distribution function inherits this jump
structure. The remainder of this section adapts
Theorem~\ref{thm:simple_uniform} to the lifted setting by
exploiting the blockwise regularity of the lifted laws.

\subsection{Regularity of the lifted joint and conditional laws}
\label{subsec:piecewise_regularity}

The lifted covariate vector $\widetilde\X=(\X^{(c)},\Y)$ defined in
Section~\ref{subsec:continualize_multi} has a joint density on
$\widetilde{\mathcal X}=\mathcal X_{c}\times\mathcal Y$ that is
piecewise smooth rather than globally smooth: on each open block
$\mathcal X_{c}\times I(\bm\upsilon)^{\circ}$ it inherits the
regularity of the conditional law of $\X^{(c)}$ given
$\X^{(d)}=\bm\upsilon$, while across the covariate jumping set
$\widetilde{\mathcal X}^{\bullet}$ it may exhibit finite jumps.
This subsection makes that clear and then derives the
corresponding piecewise structure of the lifted conditional
distribution function, which is the regularity input required
for the uniform convergence analysis in
Subsection~\ref{subsec:mixed_uniform}.

\subsubsection{Lifted joint density: blockwise continuity and
one-sided limits}

We first describe the lifted joint density $g$ at the level of
$C^{0}$-regularity, before turning to higher-order smoothness.

\begin{thm}[Continuity of the lifted joint density]
\label{thm:continuity_continualization_pdf}
Suppose that for every $\bm\upsilon\in\varUpsilon$ the mixed
density-mass section $\x^{(c)}\mapsto f(\x^{(c)},\bm\upsilon)$ is continuous on
$\mathcal X_{c}$. Let $g$ denote the joint density of the continualized variable vector
$\widetilde\X$. Then the following hold.
\begin{enumerate}[(i)]
    \item \textbf{Blockwise representation.} For every
    $\bm\upsilon\in\varUpsilon$ and every
    $(\x^{(c)},\y)\in\mathcal X_{c}\times I(\bm\upsilon)^{\circ}$,
    \[
        g(\x^{(c)},\y)
        =
        \frac{f(\x^{(c)},\bm\upsilon)}{|I(\bm\upsilon)|},
    \]
    so $g$ is continuous on each open block
    $\mathcal X_{c}\times I(\bm\upsilon)^{\circ}$ and independent
    of $\y$ on that block.
    \item \textbf{Jumps across $\widetilde{\mathcal X}^{\bullet}$.}
    Let
    $(\x_{0}^{(c)},\y_{0})\in\widetilde{\mathcal X}^{\bullet}$,
    and let $\bm\upsilon,\bm\upsilon'\in\varUpsilon$ be two indices of two blocks 
    sharing the boundary on which $\y_{0}$ lies. Then the
    one-sided limits of $g$ at $(\x_{0}^{(c)},\y_{0})$ from
    inside $\mathcal X_{c}\times I(\bm\upsilon)$ and from inside
    $\mathcal X_{c}\times I(\bm\upsilon')$ exist and equal
    $f(\x_{0}^{(c)},\bm\upsilon)/|I(\bm\upsilon)|$ and
    $f(\x_{0}^{(c)},\bm\upsilon')/|I(\bm\upsilon')|$,
    respectively; these limits coincide if and only if
    $f(\x_{0}^{(c)},\bm\upsilon)/|I(\bm\upsilon)|
     =f(\x_{0}^{(c)},\bm\upsilon')/|I(\bm\upsilon')|$.
	
 	\item \textbf{Piecewise continuity.} 
	The function $g$ is piecewise continuous on $\widetilde{\mathcal X}^\circ$, and at every interface point $(\x_1^{(c)},\y_1)\in\widetilde{\mathcal X}^{\bullet}$, $g|_{\mathcal X_{c}\times I(\bm\upsilon)}$ admits a continuous extension to the closed block
    $\mathcal X_{c}\times\overline{I(\bm\upsilon)}$, given by the same formula in~(i).
\end{enumerate}
\end{thm}

The proof is given in Appendix~\ref{app:proof_continuity_continualization_pdf}. 
Theorem~\ref{thm:continuity_continualization_pdf} captures the basic geometric picture: continualization unfolds the discrete
mass on $\varUpsilon$ into a continuous density on $\widetilde{\mathcal X}$ that is constant in $\y$ within each
block, has well-defined one-sided limits at every block boundary, and exhibits a finite jump exactly when the adjacent blockwise density values $f(\x^{(c)},\bm\upsilon)/|I(\bm\upsilon)|$ differ at the corresponding $\x^{(c)}$. 
Raising the smoothness assumption on the sections $f(\cdot,\bm\upsilon)$ to $C^{r}$ extends this $C^{0}$ picture to $C^{r}$ in a directly analogous manner. 
The proof, given in Appendix~\ref{app:proof_Cr_continualization}, applies Theorem~\ref{thm:continuity_continualization_pdf} to each partial derivative of order at most $r$.

\begin{cor}[Piecewise $C^{r}$ regularity of the lifted density]
\label{cor:Cr_continualization}
Let $r\ge 0$ be an integer and assume that for every
$\bm\upsilon\in\varUpsilon$ the section
$\x^{(c)}\mapsto f(\x^{(c)},\bm\upsilon)$ belongs to
$C^{r}(\mathcal X_{c})$, with all partial derivatives of order
up to $r$ continuous on $\mathcal X_{c}$
uniformly in $\bm\upsilon$. Then:
\begin{enumerate}[(i)]
    \item For every $\bm\upsilon\in\varUpsilon$,
    $g\in C^{r}\bigl(\mathcal X_{c}\times I(\bm\upsilon)^{\circ}\bigr)$,
    with all $\y$-derivatives identically zero and all
    $\x^{(c)}$-derivatives inherited from
    $f(\cdot,\bm\upsilon)/|I(\bm\upsilon)|$.
    \item Across $\widetilde{\mathcal X}^{\bullet}$, the one-sided
    limits of $g$ and of its $\x^{(c)}$-derivatives from adjacent
    blocks need not coincide; the magnitude of any such
    discontinuity follows the formula of
    Theorem~\ref{thm:continuity_continualization_pdf}(ii)
    applied to the corresponding derivative.
    \item The collection
    $\{g|_{\mathcal X_{c}\times I(\bm\upsilon)}\}_{\bm\upsilon\in\varUpsilon}$
    constitutes a piecewise $C^{r}$ structure on
    $\widetilde{\mathcal X}$. The restriction
    $g|_{\mathcal X_{c}\times I(\bm\upsilon)}$ admits a $C^{r}$ extension to $\mathcal X_{c}\times\overline{I(\bm\upsilon)}$.
\end{enumerate}
\end{cor}

\subsubsection{Lifted conditional distribution: blockwise
$C^{2}$ structure}

With the regularity of the joint density in hand, we now turn to
the lifted conditional distribution function
$F(z\mid\tilde\x)=\mathbb P(Z\le z\mid\widetilde\X=\tilde\x)$. Two
features of this function are critical for the analysis of
Subsection~\ref{subsec:mixed_uniform}: it is independent of $\y$
within each block (so $\y$-derivatives contribute no bias to
local linear smoothing), and, in the case of an ordered discrete
response, it is exactly affine in the continualized response
coordinate $\tilde z$ on each response block (so the
$\tilde z$-direction Taylor remainder vanishes whenever the
response smoothing window stays within a single block).

\begin{thm}[Lifted conditional distribution]
\label{thm:conditional_distribution}
Let $\widetilde\X=(\X^{(c)},\Y)$ be the continualized covariate random vector.
\begin{enumerate}[(i)]
    \item \textbf{Continuous response.} Suppose $Z$ is continuous
    and, for every $\bm\upsilon\in\varUpsilon$, the section
    $(z,\x^{(c)})\mapsto F(z\mid\x^{(c)},\bm\upsilon)$ belongs to
    $C^{2}(\mathcal Z\times\mathcal X_{c})$ with second-order
    partial derivatives uniformly continuous and bounded on
    $\mathcal Z\times\mathcal X_{c}$, uniformly in $\bm\upsilon$.
    Then
    \begin{equation}\label{eq:cdf_cont_blockwise}
        F(z\mid\x^{(c)},\y)
        =
        \sum_{\bm\upsilon\in\varUpsilon}
            F(z\mid\x^{(c)},\bm\upsilon)\,
            \mathbb I\{\y\in I(\bm\upsilon)\}
    \end{equation}
    belongs to
    $C^{2}\bigl(\mathcal Z\times\mathcal X_{c}\times I(\bm\upsilon)^{\circ}\bigr)$
    for every $\bm\upsilon$, is independent of $\y$ on each block
    (so all $\y$-derivatives vanish identically), and may jump
    across $\widetilde{\mathcal X}^{\bullet}$.

    \item \textbf{Discrete response after continualization.}
    Suppose $Z$ is ordered discrete on $\{z_{1},\ldots,z_{L}\}$
    and, for every $i\in[L]$ and $\bm\upsilon\in\varUpsilon$, the
    conditional mass
    $\x^{(c)}\mapsto P(z_{i}\mid\x^{(c)},\bm\upsilon)$ belongs to
    $C^{2}(\mathcal X_{c})$, uniformly in $i$ and $\bm\upsilon$.
    Then the lifted conditional density is
    \begin{equation}\label{eq:cdf_disc_density}
        f(\tilde z\mid\x^{(c)},\y)
        =
        \sum_{\bm\upsilon\in\varUpsilon}
        \sum_{i=1}^{L}
        \frac{P(z_{i}\mid\x^{(c)},\bm\upsilon)}{|I_{i}|}\,
        \mathbb I\{\tilde z\in I_{i}\}\,
        \mathbb I\{\y\in I(\bm\upsilon)\},
    \end{equation}
    a valid conditional density, and the lifted conditional CDF
    satisfies, for $\tilde z\in I_{i}$ and $\y\in I(\bm\upsilon)$,
    \begin{equation}\label{eq:cdf_disc_blockwise}
        F(\tilde z\mid\x^{(c)},\y)
        =
        \sum_{k=1}^{i-1}P(z_{k}\mid\x^{(c)},\bm\upsilon)
        +\frac{\tilde z-z_{i-1}}{|I_{i}|}\,
        P(z_{i}\mid\x^{(c)},\bm\upsilon).
    \end{equation}
    Consequently $F(\tilde z\mid\x^{(c)},\y)$ is $C^{2}$ on each
    $I_{i}^{\circ}\times\mathcal X_{c}\times I(\bm\upsilon)^{\circ}$,
    is independent of $\y$ on each block, and is \emph{exactly
    affine} in $\tilde z$ on each $I_{i}^{\circ}$.

    \item \textbf{Covariate jumping set.} In both cases (i) and
    (ii), at every covariate jumping point
    $(\x^{(c)},\y_{0})\in\widetilde{\mathcal X}^{\bullet}$, $F$ admits one-sided $C^{2}$ limits from every block whose closure contains the point in the sense of Theorem~\ref{thm:continuity_continualization_pdf}(ii); the values of $F$ and its $\x^{(c)}$-derivatives on adjacent blocks need not agree.
\end{enumerate}
\end{thm}

The two structural features mentioned at the start of this
discussion---block-internal $\y$-invariance in both cases and exact affineness in $\tilde z$ in the discrete response case---are not byproducts of the construction; they are
direct consequences of the uniform spread of probability mass within each continualized cell, and they are precisely what allows the uniform convergence theory of Subsection~\ref{subsec:mixed_uniform} to proceed blockwise from the analytic backbone of Theorem~\ref{thm:conditional_distribution}. 

\subsection{Uniform convergence under continualized mixed-type variates}
\label{subsec:mixed_uniform}

We now establish the uniform convergence rate of the federated
local linear conditional CDF estimator under continualized
mixed-type variates. 
The strategy is to apply Theorem~\ref{thm:simple_uniform} blockwise on each piecewise $C^{2}$ region identified in
Subsection~\ref{subsec:piecewise_regularity}, and then to combine
the blockwise bounds using the finiteness of $\varUpsilon$.
However, points near block boundaries require special care.

\begin{thm}[Uniform convergence on shrinking interior blocks]
\label{thm:int_uniform}
Let $\tilde\x=(\x^{(c)},\y)\in\widetilde{\mathcal X}^{\circ}$ and
suppose the following assumptions hold: (1) the continuous covariates satisfy Assumption~\ref{ass:marginal} on $\mathcal X_{c}$; (2) Assumption~\ref{ass:kernel} holds for both the covariate and response kernels; (3) for every $\bm\upsilon\in\varUpsilon$, $\x^{(c)}\mapsto f(\x^{(c)},\bm\upsilon)\in C^{2}(\mathcal X_{c})$ with $\inf_{\x^{(c)}}f(\x^{(c)},\bm\upsilon)>0$, uniformly in $\bm\upsilon$; (4) for every $\bm\upsilon\in\varUpsilon$, the conditional CDF $F(z\mid\x^{(c)},\bm\upsilon)$ satisfies Assumption~\ref{ass:conditional} on $\mathcal Z\times\mathcal X_{c}$, uniformly in $\bm\upsilon$. Then:
\begin{enumerate}[(i)]
    \item \textbf{Continuous response.} If \(Z\) is continuous, define
    \begin{equation*}
    \mathcal Y_{\epsilon}^{\circ}
    :=
    \left\{
    \y\in\mathcal Y^\circ:
    \min_{\substack{j\in[\tilde d]\\1\le \ell\le\tau_j-1}}
    |y_j-\upsilon_{j,\ell}|
    \ge\epsilon
    \right\},
    \qquad
    \widetilde{\mathcal X}_{\epsilon}^{\circ}
    :=
    \mathcal X_c\times\mathcal Y_{\epsilon}^{\circ},
    \end{equation*}
    where $\epsilon\ge C_\epsilon w_x\sqrt{|\log w_x|}$ and $C_\epsilon>2$.
    Then
    \begin{equation}\label{eq:int_uniform_cont}
        \sup_{z\in\mathbb R,\,\tilde\x\in \widetilde{\mathcal X}_{\epsilon}^{\circ}}
        \bigl|\widehat F(z\mid\tilde\x)-F(z\mid\tilde\x)\bigr|
        =
        O_p\!\left(
        w_x^2+w_z^2+
        \sqrt{\frac{|\log w_x|}{M w_x^d}}
        \right).
    \end{equation}

    \item \textbf{Discrete response after continualization.} If \(Z\) is ordered discrete on
    \(\{z_1,\ldots,z_L\}\), define
    \begin{equation*}
            \widetilde{\mathcal Z}_{\epsilon'}^{\circ}
            :=
            \left\{
            \tilde z\in\widetilde{\mathcal Z}^{\circ}:
            \min_{0\le i\le L}|\tilde z-z_i|
            \ge\epsilon'
            \right\},
    \end{equation*} where $\epsilon'\ge C_{\epsilon'}w_z\sqrt{|\log w_z|}$ and $C_{\epsilon'}>2$.
    Then
    \begin{equation}\label{eq:int_uniform_disc}
        \sup_{\substack{
        \tilde z\in\widetilde{\mathcal Z}_{\epsilon'}^{\circ}\\
        \tilde\x\in \widetilde{\mathcal X}_{\epsilon}^{\circ}
        }}
        \bigl|\widehat F(\tilde z\mid\tilde\x)-F(\tilde z\mid\tilde\x)\bigr|
        =
        O_p\!\left(
        w_x^2+w_z^2+
        \sqrt{\frac{|\log w_x|}{M w_x^d}}
        \right).
    \end{equation}
\end{enumerate}
\end{thm}

\begin{rem}
    The shrinking radii $\epsilon$ and $\epsilon'$ are of the same order as $O(w_x|\log w_x|)$ and $O(w_z|\log w_z|)$. When $w_x =w_z= n^{-\alpha}, \quad\alpha>0$, the shrinking radii are of order $O(\frac{\alpha\log n}{n^\alpha}) \to 0$ as $n\to \infty$. In addition, the theorem shows the interval-center substitution moves the query away from the excluded interface neighborhoods $(\mathcal Y_{\epsilon}^{\circ})^c$ effectively in practice.
\end{rem}
The proof is given in Appendix~\ref{app:proof_int_uniform}. Theorem~\ref{thm:int_uniform} directly applies only to the shrinking interior \(\mathcal X_{\epsilon}^{\circ}\). To extend the result to the excluded interface neighborhoods and the boundary points, for each \(\bm\upsilon\in\varUpsilon\), fix a representative
\[
\y_{\epsilon}^{\circ}
\in
I(\bm\upsilon)\cap\mathcal Y_{\epsilon}^{\circ},
\]
which exists for all sufficiently small \(\epsilon\). Under the left-open right-closed convention, every
\(\y\in\mathcal Y\), including a boundary point \(\y^\bullet\), belongs to a unique block \(I(\bm\upsilon)\). By Theorem~\ref{thm:conditional_distribution},
\[
F(z\mid\x^{(c)},\y)
=
F(z\mid\x^{(c)},\y_{\epsilon}^{\circ}),
\qquad
\y\in I(\bm\upsilon).
\]
Thus, although Theorem~\ref{thm:int_uniform} does not directly control the estimator at points within the excluded interface neighborhoods, their conditional distributions, including those at the boundary points, can be recovered by evaluating the estimator at the corresponding interior representatives. In what follows, with the continuous covariates $\x^{(c)}$ fixed, we therefore use
\[
\tilde\x_{\epsilon}^{\circ}
:=
(\x^{(c)},\y_{\epsilon}^{\circ})
\in \widetilde{\mathcal X}_{\epsilon}^{\circ}
\]
as the representative of any
\(\tilde\x=(\x^{(c)},\y)\) belonging to the same block.

\begin{cor}[Uniform convergence near covariate jumping points]
\label{cor:jump_uniform}
Under the conditions of Theorem~\ref{thm:int_uniform},
define
\[
\widetilde{\mathcal X}_{\epsilon}^{\bullet}
:=
\widetilde{\mathcal X}\setminus \widetilde{\mathcal X}_{\epsilon}^{\circ},
\]
which includes the excluded interface neighborhoods and
the covariate jumping points. For each
$\tilde\x_\epsilon^{\bullet}=(\x^{(c)},\y)\in
\widetilde{\mathcal X}_{\epsilon}^{\bullet}$,
let $\tilde\x_{\epsilon}^{\circ} \in \widetilde{\mathcal X}_{\epsilon}^{\circ}$
be its interior representative defined above.
Then the following hold.
\begin{enumerate}[(i)]
    \item If $Z$ is continuous,
    \begin{equation}\label{eq:jump_uniform_cont}
        \sup_{\substack{z\in\mathbb R\\
        \tilde\x_\epsilon^{\bullet} \in\widetilde{\mathcal X}_{\epsilon}^{\bullet}}}
        \bigl|
        \widehat F(z\mid\tilde\x_{\epsilon}^{\circ})
        -F(z\mid \tilde\x_\epsilon^{\bullet})
        \bigr|
        =
        O_p\!\left(
            w_x^2+w_z^2+
            \sqrt{\frac{|\log w_x|}{Mw_x^d}}
        \right).
    \end{equation}

    \item If $Z$ is ordered discrete and is continualized,
    \begin{equation}\label{eq:jump_uniform_disc}
        \sup_{\substack{
        \tilde z\in\widetilde{\mathcal Z}_{\epsilon'}^{\circ}\\
        \tilde\x_\epsilon^{\bullet}\in\widetilde{\mathcal X}_{\epsilon}^{\bullet}}}
        \bigl|
        \widehat F(\tilde z\mid
        \tilde\x_{\epsilon}^{\circ})
        -F(\tilde z\mid\tilde\x_\epsilon^{\bullet})
        \bigr|
        =
        O_p\!\left(
            w_x^2+w_z^2+
            \sqrt{\frac{|\log w_x|}{Mw_x^d}}
        \right).
    \end{equation}
\end{enumerate}
In particular, evaluating the federated estimator at the corresponding interior representative yields a uniformly consistent estimate of the conditional CDF throughout the excluded interface neighborhoods, including the jumping points, at the shrinking-interior rate.
\end{cor}

The proof is short: by Theorem~\ref{thm:conditional_distribution},
\[
F(z\mid\tilde\x_\epsilon^{\bullet})
=
F(z\mid\tilde\x_{\epsilon}^{\circ})
\]
for every
$\tilde\x\in\widetilde{\mathcal X}_{\epsilon}^{\bullet}$ and the corresponding interior representative $\tilde\x_{\epsilon}^{\circ} \in \widetilde{\mathcal X}_{\epsilon}^{\circ}$.
Consequently,
\[
\widehat F(z\mid\tilde\x_{\epsilon}^{\circ})
-F(z\mid\tilde\x_\epsilon^{\bullet})
=
\widehat F(z\mid\tilde\x_{\epsilon}^{\circ})
-F(z\mid\tilde\x_{\epsilon}^{\circ}),
\]
and the right-hand side is bounded uniformly by
Theorem~\ref{thm:int_uniform}.

Corollary~\ref{cor:jump_uniform} extends the shrinking-interior result to covariate interface neighborhoods, including the jumping points, through a deterministic blockwise reparametrization. This relies on the $\y$-invariance of the conditional CDF within each continualization block. Similarly, the response points near the boundaries $\tilde{z} \in \widetilde{\mathcal Z}_{\epsilon'}^{\bullet} := \widetilde{\mathcal Z}\setminus \widetilde{\mathcal Z}_{\epsilon'}^{\circ}$ also need to be analyzed carefully. We cannot obtain a convergence result analogous to \eqref{eq:jump_uniform_disc}, but the original conditional probability masses $p_{i}(\tilde\x):=\mathbb P(Z=z_{i}\mid\tilde\x)$ can be estimated by a finite difference of $\widehat F$ at the two interior representative points, with the resulting estimators attaining a strictly sharper rate.

\begin{cor}[Uniform convergence for conditional probability masses]
\label{cor:mass_uniform}
Under the conditions of Theorem~\ref{thm:int_uniform} for the
discrete response case, fix for each $i\in[L]$ two interior
evaluation points
$\tilde z_{i,-}^{\circ},\tilde z_{i,+}^{\circ}
\in I_i\cap\widetilde{\mathcal Z}_{\epsilon'}^{\circ}$
with $\tilde z_{i,-}^{\circ}<\tilde z_{i,+}^{\circ}$,
satisfying the separation condition
\begin{equation}\label{eq:separation}
    \inf_{1\le i\le L}
    \frac{\tilde z_{i,+}^{\circ}-\tilde z_{i,-}^{\circ}}
    {|I_i|}
    \ge c
    \quad\text{for some }c\in(0,1].
\end{equation}
Define the conditional mass estimator
\[
    \widehat p_i(\tilde\x)
    :=
    \frac{|I_i|}
    {\tilde z_{i,+}^{\circ}-\tilde z_{i,-}^{\circ}}
    \Bigl(
        \widehat F(\tilde z_{i,+}^{\circ}\mid\tilde\x)
        -
        \widehat F(\tilde z_{i,-}^{\circ}\mid\tilde\x)
    \Bigr).
\]
Then
\begin{equation}\label{eq:mass_uniform}
    \sup_{\substack{
        1\le i\le L\\
        \tilde\x\in\mathcal X_\epsilon^{\circ}
    }}
    \bigl|\widehat p_i(\tilde\x)-p_i(\tilde\x)\bigr|
    =
    O_p\!\left(
        w_x^2+
        \sqrt{\frac{|\log w_x|}{Mw_x^d}}
    \right)+o(w_z^2).
\end{equation}
\end{cor}

The three results of this section---Theorem~\ref{thm:int_uniform}, Corollary~\ref{cor:jump_uniform}, and Corollary~\ref{cor:mass_uniform}---together establish that the federated estimator developed in Section~\ref{sec:algorithm}
attains the uniform rate $O_{p}(w_x^{2}+w_z^{2}+\sqrt{|\log w_x|/(Mw_x^{d})})$ on the full continualized covariate domain in both the continuous- and discrete-response cases.

\section{Experiments}
\label{sec:experiments}
We compare FITS with two baseline methods, the marginal-based HeteroFedSyn \citep{li2026heterofedsyn} and the diffusion model FedTabDiff \citep{sattarov2024fedtabdiff}. We evaluate the synthetic data generators with three models, including linear regression and logistic regression for parameter estimation, and a deep learning model (FT-Transformer, \citep{gorishniy2021revisiting}) for prediction.  
The simulation study uses a controlled mixed-type data-generating mechanism with known downstream parameters, allowing direct assessment of parameter recovery under imbalanced federated sample sizes. The real-data study uses two mixed-type tabular data sets, one for regression and one for binary classification, to examine these generators' utility in practice.

Throughout this section, we ask not only whether synthetic data should replace local data, but also whether synthetic samples generated by FITS provide a statistically useful complement to client-level local data, especially for clients with limited sample sizes.
This distinction is important: in federated settings, the value of synthetic data depends not only on the quality of the generator but also on how synthetic and local data is combined in the downstream tasks. 
Table~\ref{tab:experiment_designs} summarizes the datasets and experimental designs.

\begin{table}[h]
\centering
\caption{Experimental designs for the simulation and real-data studies.}
\label{tab:experiment_designs}
\resizebox{\textwidth}{!}{
\begin{tabular}{ccccccccccc}
\specialrule{0.8pt}{0pt}{1.5pt}
\specialrule{0.8pt}{0pt}{3pt}
 & Simulation & Bike Sharing & Adult Income \\
\specialrule{0.8pt}{3pt}{1.5pt}
\specialrule{0.8pt}{0pt}{3pt}
Task & Regression and classification & Regression & Classification \\
Data source & Synthetic DGP & UCI Bike Sharing & UCI Adult Income \\
Full sample size & $10{,}000$ & $17{,}379$ & $45{,}222$ \\
Training sample size & $6{,}000$ & $13{,}903$ & $36{,}177$ \\
Test sample size & $4{,}000$ & $3{,}476$ & $9{,}045$ \\
Generated sample size & $3{,}000$ & $5{,}000$ & $20{,}000$ \\
Number of repetitions &  \multicolumn{3}{c}{$50$} \\
Clients &  \multicolumn{3}{c}{$10$} \\
Client proportions 
& \multicolumn{3}{c}{$(30\%,20\%,10\%,9\%,8\%,7\%,6\%,5\%,4\%,1\%)$} \\
Client size range & $60$--$1{,}800$ & $139$--$4{,}171$ & $362$--$10{,}853$ \\
Covariates/predictors  & $6$ & $12$ & $14$ \\
Variable types 
& Continuous, ordinal, nominal 
& Continuous, ordinal 
& Continuous, ordinal, nominal \\
FITS generation
& Joint [Shared $X$, task-specific $Y$]
& Joint $[X,Y]$
& Joint $[X,Y]$ \\
Parameter reference
& True coefficients
& Full-data estimator
& Full-data estimator \\
Distributional reference
& Full data
& Full data
& Full data \\
\specialrule{0.8pt}{3pt}{1.5pt}
\specialrule{0.8pt}{0pt}{0pt}
\end{tabular}
}
\end{table}

\subsection{Experimental setup}
\label{subsec:experiment_setup}

\subsubsection{Federated split with sample-size imbalance}

All experiments use \(K=10\) clients. Unless otherwise stated, the training data is randomly permuted and split according to the fixed proportion vector
$
(p_1,\ldots,p_{10})
=
(30\%,20\%,10\%,9\%,8\%,7\%$ $,6\%,5\%,4\%,1\%).
$ Thus client \(k\) receives approximately \(p_k n_{\rm train}\) observations. This creates a deliberately imbalanced federated system: the largest client contains about \(30\%\) of the training data, whereas the smallest client contains about \(1\%\).

This design is intended to stress-test the downstream data usage modes under client-size imbalance. If all clients had comparable sample sizes, the relative advantage of synthetic data would be harder to separate from ordinary sampling variability. In contrast, under the present split, estimators from the local datasets only are expected to be high-variance on the smallest clients and comparatively stable on the largest clients. This makes it possible to assess whether synthetic data primarily helps small clients and whether it degrades performance on large clients. The main experiments focus on sample-size imbalance while keeping the federated split otherwise random.

\subsubsection{Federated generator configurations}
For FITS, we use the Gaussian-kernel scheme and the KNN-kernel scheme, abbreviated as FITS-G and FITS-K respectively. 
The general settings include that the target bandwidth parameter $w_z$ is set to $0.05$ and the conditional CDF grid size is $G=300$. For covariates' kernel-fitting, FITS-G takes the Gaussian kernel with Scott's rule kernel bandwidth $w^{(k)}_x = n_{train}^{-1/(d+4)}$, where $n_{train}$ is the size of the full training dataset and the dimension $d$ increases as the covariates are added along the chain process.
FITS-K uses the Epanechnikov kernel with $\eta = 1.05$ and the global KNN budget is dimension-adaptive: when $d_{\rm cond}=1$, $r_0 = 350$; when $d_{\rm cond}=2$, $r_0 = 250$ and otherwise, $r_0 = 200$.
For HeteroFedSyn and FedTabDiff, Appendix~\ref{app:baseline_methods} gives the detailed configurations.

\subsubsection{Downstream data usage modes}

For each downstream task, we compare four data usage modes introduced in Section~\ref{subsec:integration_modes}, which are used to explore how synthetic and local data is combined in the downstream tasks.
Mode~1 (M1) uses only synthetic data. 
Mode~2 (M2) uses only the client's local data (and therefore does not use federated learning). 
Mode~3 (M3) pools local observations with a subset of the synthetic observations directly in a mixed-data objective.
Mode~4 (M4) combines local and synthetic observations through a loss-reweighted objective, where the synthetic component receives a tunable weight.

\subsection{Simulation experiments}
\label{subsec:simulation_experiments}

\subsubsection{Simulation design}

The simulation study uses a mixed-type tabular data-generating mechanism with six covariates,
$(X_B,X_{G_1},X_{G_2},X_O,X_L,X_{G_3}),$
and two task-specific responses, \(Y_{\rm lin}\) for linear regression and \(Y_{\rm log}\) for logistic regression. 
The variable \(X_B\) is binary, \(X_O\) is ordinal, and the remaining covariates are continuous. The dependence structure among the covariates and the model settings are shown in Section~\ref{subsec:cov_ship}. 

Operationally, the generator first produces a shared synthetic covariate table according to the chain order $(X_B,X_{G_1},X_{G_2},X_O,X_L,X_{G_3})$. 
The two task-specific responses, $Y_{\rm lin}$ and $Y_{\rm log}$, are then generated on top of this same covariate table. 
This design ensures that the linear and logistic downstream comparisons differ only in the response-generating mechanism, while sharing the same generated covariate distribution. 
This shared-covariate design removes sampling variation in the covariate distribution as a confounder when comparing the regression and classification tasks.
More details are described in Table~\ref{tab:task_chain_parameters}.

\subsubsection{Baseline comparisons}\label{sec:sim_baselines}
We conduct three sets of experiments in this section.
First, we compare FITS-G with two baselines, HeteroFedSyn and FedTabDiff, across all modes and clients.
We report the parameter estimation errors of linear and logistic regression, as well as the classification accuracy of FT-Transformer.
Second, we compare the Gaussian-kernel and KNN-kernel schemes, denoted by FITS-G and FITS-K, respectively.
Third, we evaluate FITS-G under the four modes and their respective parameter settings to identify effective ways to use FITS-generated synthetic data.
All results are reported as the mean (standard deviation) over 50 repeats, with lower errors and higher accuracy indicating better performance.

\begin{table}[ht]
\centering
\caption{Simulation: linear-regression parameter estimation error.}
\label{tab:simulation_linear_parameter_l2_error_fixed_modes}
\begingroup
\setlength{\tabcolsep}{4pt}
\renewcommand{\arraystretch}{1.35}
\resizebox{\textwidth}{!}{%
\begin{tabular}{@{}cc|*{10}{c}@{}}
\specialrule{0.8pt}{0pt}{1.5pt}
\specialrule{0.8pt}{0pt}{3pt}
\multirow{2}{*}{Mode} & \multirow{2}{*}{Baselines} & \multicolumn{10}{c}{Clients} \\
\cline{3-12}
 & & Client 0 (30\%) & Client 1 (20\%) & Client 2 (10\%) & Client 3 (9\%) & Client 4 (8\%) & Client 5 (7\%) & Client 6 (6\%) & Client 7 (5\%) & Client 8 (4\%) & Client 9 (1\%) \\
\specialrule{0.8pt}{3pt}{1.5pt}
\specialrule{0.8pt}{0pt}{3pt}
\multirow{3}{*}{M1} & FITS-G & \textbf{0.1584 (0.0698)} & \textbf{0.1584 (0.0698)} & \textbf{0.1584 (0.0698)} & \textbf{0.1584 (0.0698)} & \textbf{0.1584 (0.0698)} & \textbf{0.1584 (0.0698)} & \textbf{0.1584 (0.0698)} & \textbf{0.1584 (0.0698)} & \textbf{0.1584 (0.0698)} & \textbf{0.1584 (0.0698)} \\
 & HeteroFedSyn & 4.6669 (0.3037) & 4.6669 (0.3037) & 4.6669 (0.3037) & 4.6669 (0.3037) & 4.6669 (0.3037) & 4.6669 (0.3037) & 4.6669 (0.3037) & 4.6669 (0.3037) & 4.6669 (0.3037) & 4.6669 (0.3037) \\
 & FedTabDiff & 1.3846 (0.5311) & 1.3846 (0.5311) & 1.3846 (0.5311) & 1.3846 (0.5311) & 1.3846 (0.5311) & 1.3846 (0.5311) & 1.3846 (0.5311) & 1.3846 (0.5311) & 1.3846 (0.5311) & 1.3846 (0.5311) \\
\hline
\multirow{1}{*}{M2} & None & \textbf{0.0751 (0.0329)} & \textbf{0.0930 (0.0385)} & \textbf{0.1302 (0.0568)} & \textbf{0.1446 (0.0698)} & \textbf{0.1310 (0.0616)} & \textbf{0.1541 (0.0704)} & \textbf{0.1751 (0.0725)} & \textbf{0.2063 (0.0856)} & \textbf{0.2198 (0.1036)} & \textbf{0.4730 (0.1856)} \\
\hline
\multirow{3}{*}{\shortstack{M3 }} & FITS-G & \textbf{0.1060 (0.0435)} & \textbf{0.1225 (0.0549)} & \textbf{0.1350 (0.0622)} & \textbf{0.1406 (0.0620)} & \textbf{0.1405 (0.0638)} & \textbf{0.1413 (0.0629)} & \textbf{0.1438 (0.0654)} & \textbf{0.1467 (0.0663)} & \textbf{0.1501 (0.0669)} & \textbf{0.1545 (0.0682)} \\
 & HeteroFedSyn & 3.2202 (0.2503) & 3.5844 (0.2632) & 4.0457 (0.2720) & 4.1081 (0.2955) & 4.1559 (0.2868) & 4.2088 (0.2809) & 4.2857 (0.2856) & 4.3372 (0.2879) & 4.3956 (0.2900) & 4.5981 (0.2999) \\
 & FedTabDiff & 0.9799 (0.4452) & 1.0844 (0.4678) & 1.2112 (0.5020) & 1.2301 (0.5039) & 1.2438 (0.5066) & 1.2589 (0.5084) & 1.2780 (0.5128) & 1.2891 (0.5153) & 1.3106 (0.5159) & 1.3650 (0.5280) \\
\hline
\multirow{3}{*}{\shortstack{M4 }} & FITS-G & \textbf{0.0863 (0.0339)} & \textbf{0.1062 (0.0477)} & \textbf{0.1206 (0.0580)} & \textbf{0.1300 (0.0570)} & \textbf{0.1285 (0.0603)} & \textbf{0.1303 (0.0579)} & \textbf{0.1337 (0.0628)} & \textbf{0.1387 (0.0648)} & \textbf{0.1447 (0.0645)} & \textbf{0.1512 (0.0669)} \\
 & HeteroFedSyn & 2.4775 (0.2225) & 2.9268 (0.2443) & 3.5818 (0.2570) & 3.6796 (0.2920) & 3.7555 (0.2784) & 3.8404 (0.2678) & 3.9677 (0.2771) & 4.0556 (0.2784) & 4.1585 (0.2823) & 4.5319 (0.2971) \\
 & FedTabDiff & 0.7655 (0.3807) & 0.8985 (0.4177) & 1.0802 (0.4758) & 1.1101 (0.4787) & 1.1318 (0.4841) & 1.1565 (0.4877) & 1.1888 (0.4955) & 1.2078 (0.5007) & 1.2454 (0.5018) & 1.3462 (0.5249) \\
\specialrule{0.8pt}{3pt}{1.5pt}
\specialrule{0.8pt}{0pt}{0pt}
\end{tabular}
}
\endgroup
\end{table}

\begin{table}[ht]
\centering
\caption{Simulation: logistic-regression parameter estimation error.}
\label{tab:simulation_logistic_parameter_l2_error_fixed_modes}
\begingroup
\setlength{\tabcolsep}{4pt}
\renewcommand{\arraystretch}{1.35}
\resizebox{\textwidth}{!}{%
\begin{tabular}{@{}cc|*{10}{c}@{}}
\specialrule{0.8pt}{0pt}{1.5pt}
\specialrule{0.8pt}{0pt}{3pt}
\multirow{2}{*}{Mode} & \multirow{2}{*}{Baselines} & \multicolumn{10}{c}{Clients} \\
\cline{3-12}
 & & Client 0 (30\%) & Client 1 (20\%) & Client 2 (10\%) & Client 3 (9\%) & Client 4 (8\%) & Client 5 (7\%) & Client 6 (6\%) & Client 7 (5\%) & Client 8 (4\%) & Client 9 (1\%) \\
\specialrule{0.8pt}{3pt}{1.5pt}
\specialrule{0.8pt}{0pt}{3pt}
\multirow{3}{*}{M1} & FITS-G & \textbf{0.7111 (0.1401)} & \textbf{0.7111 (0.1401)} & \textbf{0.7111 (0.1401)} & \textbf{0.7111 (0.1401)} & \textbf{0.7111 (0.1401)} & \textbf{0.7111 (0.1401)} & \textbf{0.7111 (0.1401)} & \textbf{0.7111 (0.1401)} & \textbf{0.7111 (0.1401)} & \textbf{0.7111 (0.1401)} \\
 & HeteroFedSyn & 3.7460 (0.5029) & 3.7460 (0.5029) & 3.7460 (0.5029) & 3.7460 (0.5029) & 3.7460 (0.5029) & 3.7460 (0.5029) & 3.7460 (0.5029) & 3.7460 (0.5029) & 3.7460 (0.5029) & 3.7460 (0.5029) \\
 & FedTabDiff & 2.1627 (0.7893) & 2.1627 (0.7893) & 2.1627 (0.7893) & 2.1627 (0.7893) & 2.1627 (0.7893) & 2.1627 (0.7893) & 2.1627 (0.7893) & 2.1627 (0.7893) & 2.1627 (0.7893) & 2.1627 (0.7893) \\
\hline
\multirow{1}{*}{M2} & None & \textbf{0.3923 (0.1852)} & \textbf{0.5028 (0.2397)} & \textbf{0.6821 (0.3112)} & \textbf{0.8635 (0.3364)} & \textbf{0.7619 (0.3377)} & \textbf{0.8787 (0.4005)} & \textbf{0.8757 (0.4659)} & \textbf{1.0606 (0.5028)} & \textbf{1.1961 (0.6224)} & \textbf{10.6726 (16.6535)} \\
\hline
\multirow{3}{*}{\shortstack{M3 }} & FITS-G & \textbf{0.5117 (0.1249)} & \textbf{0.5694 (0.1246)} & \textbf{0.6281 (0.1209)} & \textbf{0.6370 (0.1343)} & \textbf{0.6446 (0.1351)} & \textbf{0.6602 (0.1370)} & \textbf{0.6615 (0.1381)} & \textbf{0.6714 (0.1420)} & \textbf{0.6681 (0.1364)} & \textbf{0.6983 (0.1363)} \\
 & HeteroFedSyn & 2.7278 (0.3807) & 2.9886 (0.4202) & 3.2917 (0.4382) & 3.3256 (0.4551) & 3.3909 (0.4637) & 3.4413 (0.4649) & 3.4779 (0.4677) & 3.5052 (0.4744) & 3.5438 (0.4838) & 3.6935 (0.4914) \\
 & FedTabDiff & 1.4479 (0.5721) & 1.6263 (0.6169) & 1.8196 (0.6498) & 1.8533 (0.6480) & 1.9047 (0.6881) & 1.9549 (0.7112) & 1.9653 (0.7155) & 1.9843 (0.7242) & 2.0228 (0.7459) & 2.1242 (0.7750) \\
\hline
\multirow{3}{*}{\shortstack{M4 }} & FITS-G & \textbf{0.4215 (0.1316)} & \textbf{0.4946 (0.1278)} & \textbf{0.5717 (0.1161)} & \textbf{0.5882 (0.1389)} & \textbf{0.5979 (0.1375)} & \textbf{0.6242 (0.1382)} & \textbf{0.6237 (0.1439)} & \textbf{0.6425 (0.1466)} & \textbf{0.6334 (0.1397)} & \textbf{0.6865 (0.1351)} \\
 & HeteroFedSyn & 2.2095 (0.3488) & 2.5441 (0.3994) & 2.9683 (0.4105) & 3.0208 (0.4436) & 3.1229 (0.4448) & 3.2045 (0.4459) & 3.2641 (0.4482) & 3.3091 (0.4592) & 3.3746 (0.4736) & 3.6439 (0.4824) \\
 & FedTabDiff & 1.1230 (0.4924) & 1.3417 (0.5581) & 1.5946 (0.5915) & 1.6461 (0.5847) & 1.7227 (0.6341) & 1.7997 (0.6657) & 1.8168 (0.6732) & 1.8451 (0.6853) & 1.9086 (0.7182) & 2.0893 (0.7651) \\
\specialrule{0.8pt}{3pt}{1.5pt}
\specialrule{0.8pt}{0pt}{0pt}
\end{tabular}
}
\endgroup
\end{table}

\begin{table}[ht]
\centering
\caption{Simulation: FT-Transformer classification accuracy on the test data.}
\label{tab:simulation_ft_classification_test_accuracy_fixed_modes}
\begingroup
\setlength{\tabcolsep}{4pt}
\renewcommand{\arraystretch}{1.35}
\resizebox{\textwidth}{!}{%
\begin{tabular}{@{}cc|*{10}{c}@{}}
\specialrule{0.8pt}{0pt}{1.5pt}
\specialrule{0.8pt}{0pt}{3pt}
\multirow{2}{*}{Mode} & \multirow{2}{*}{Baselines} & \multicolumn{10}{c}{Clients} \\
\cline{3-12}
 & & Client 0 (30\%) & Client 1 (20\%) & Client 2 (10\%) & Client 3 (9\%) & Client 4 (8\%) & Client 5 (7\%) & Client 6 (6\%) & Client 7 (5\%) & Client 8 (4\%) & Client 9 (1\%) \\
\specialrule{0.8pt}{3pt}{1.5pt}
\specialrule{0.8pt}{0pt}{3pt}
\multirow{3}{*}{M1} & FITS-G & \textbf{0.8608 (0.0079)} & \textbf{0.8608 (0.0079)} & \textbf{0.8608 (0.0079)} & \textbf{0.8608 (0.0079)} & \textbf{0.8608 (0.0079)} & \textbf{0.8608 (0.0079)} & \textbf{0.8608 (0.0079)} & \textbf{0.8608 (0.0079)} & \textbf{0.8608 (0.0079)} & \textbf{0.8608 (0.0079)} \\
 & HeteroFedSyn & 0.7650 (0.0221) & 0.7650 (0.0221) & 0.7650 (0.0221) & 0.7650 (0.0221) & 0.7650 (0.0221) & 0.7650 (0.0221) & 0.7650 (0.0221) & 0.7650 (0.0221) & 0.7650 (0.0221) & 0.7650 (0.0221) \\
 & FedTabDiff & 0.8582 (0.0075) & 0.8582 (0.0075) & 0.8582 (0.0075) & 0.8582 (0.0075) & 0.8582 (0.0075) & 0.8582 (0.0075) & 0.8582 (0.0075) & 0.8582 (0.0075) & 0.8582 (0.0075) & 0.8582 (0.0075) \\
\hline
\multirow{1}{*}{M2} & None & \textbf{0.8547 (0.0101)} & \textbf{0.8406 (0.0093)} & \textbf{0.8280 (0.0105)} & \textbf{0.8278 (0.0115)} & \textbf{0.8237 (0.0127)} & \textbf{0.8242 (0.0123)} & \textbf{0.8259 (0.0122)} & \textbf{0.8246 (0.0135)} & \textbf{0.8234 (0.0134)} & \textbf{0.8193 (0.0235)} \\
\hline
\multirow{3}{*}{\shortstack{M3 }} & FITS-G & \textbf{0.8692 (0.0080)} & \textbf{0.8672 (0.0070)} & \textbf{0.8644 (0.0074)} & \textbf{0.8645 (0.0075)} & \textbf{0.8629 (0.0082)} & \textbf{0.8638 (0.0072)} & \textbf{0.8635 (0.0084)} & \textbf{0.8634 (0.0081)} & \textbf{0.8617 (0.0079)} & \textbf{0.8603 (0.0081)} \\
 & HeteroFedSyn & 0.8263 (0.0144) & 0.8100 (0.0140) & 0.7974 (0.0163) & 0.7953 (0.0181) & 0.7914 (0.0170) & 0.7873 (0.0185) & 0.7845 (0.0195) & 0.7856 (0.0191) & 0.7835 (0.0197) & 0.7725 (0.0217) \\
 & FedTabDiff & 0.8685 (0.0068) & 0.8666 (0.0065) & 0.8643 (0.0076) & 0.8623 (0.0084) & 0.8623 (0.0075) & 0.8632 (0.0081) & 0.8624 (0.0081) & 0.8616 (0.0091) & 0.8611 (0.0074) & 0.8602 (0.0077) \\
\hline
\multirow{3}{*}{\shortstack{M4 }} & FITS-G & \textbf{0.8689 (0.0084)} & \textbf{0.8667 (0.0079)} & 0.8647 (0.0085) & 0.8633 (0.0078) & \textbf{0.8634 (0.0071)} & 0.8620 (0.0080) & \textbf{0.8627 (0.0083)} & 0.8612 (0.0096) & \textbf{0.8606 (0.0080)} & \textbf{0.8605 (0.0075)} \\
 & HeteroFedSyn & 0.8412 (0.0147) & 0.8247 (0.0133) & 0.8051 (0.0170) & 0.8027 (0.0187) & 0.8000 (0.0149) & 0.7964 (0.0147) & 0.7933 (0.0159) & 0.7899 (0.0194) & 0.7853 (0.0178) & 0.7757 (0.0203) \\
 & FedTabDiff & 0.8684 (0.0078) & 0.8663 (0.0071) & \textbf{0.8649 (0.0077)} & \textbf{0.8637 (0.0077)} & 0.8626 (0.0068) & \textbf{0.8621 (0.0070)} & 0.8623 (0.0068) & \textbf{0.8616 (0.0082)} & 0.8602 (0.0090) & 0.8602 (0.0084) \\
\specialrule{0.8pt}{3pt}{1.5pt}
\specialrule{0.8pt}{0pt}{0pt}
\end{tabular}
}
\endgroup
\end{table}

Tables~\ref{tab:simulation_linear_parameter_l2_error_fixed_modes},
\ref{tab:simulation_logistic_parameter_l2_error_fixed_modes}, and
\ref{tab:simulation_ft_classification_test_accuracy_fixed_modes}
present the first set of experiments comparing the synthetic data generation methods.
Unless stated otherwise, we always use all available synthetic data and each client's local data for model training in M3.
M4 also uses all local and synthetic data but employs a weighted loss function, in which $\lambda = 0.5$ specifies the weight assigned to each synthetic observation.
For each mode and client, the best result among the compared methods is shown in bold.
Overall, FITS-generated synthetic data yields better performance than the baselines.
Clients with smaller local datasets benefit more from synthetic data, which can even serve as an effective substitute for their local data in downstream model training.
For example, in Table~\ref{tab:simulation_linear_parameter_l2_error_fixed_modes}, training solely on FITS-G synthetic data (M1) begins to outperform training solely on local data (M2) at Client~6 (6\%), with parameter estimation errors of $0.1584$ and $0.1751$, respectively.
Similarly, in Table~\ref{tab:simulation_logistic_parameter_l2_error_fixed_modes}, this advantage begins at Client~3 (9\%).
Table~\ref{tab:simulation_ft_classification_test_accuracy_fixed_modes} also shows that FITS-G improves FT-Transformer classification accuracy across all clients.
Regarding data usage, combining synthetic and local data through M3 or M4 is more effective than using synthetic data alone through M1.
Across these tables, M3 and M4 generally outperform M1; for FT-Transformer classification in Table~\ref{tab:simulation_ft_classification_test_accuracy_fixed_modes}, they also outperform M2.

Next, we compare FITS-G and FITS-K in
Tables~\ref{tab:simulation_linear_parameter_l2_error_fixed_modes_fits_g_vs_fits_k},
\ref{tab:simulation_logistic_parameter_l2_error_fixed_modes_fits_g_vs_fits_k}, and
\ref{tab:simulation_ft_classification_test_accuracy_fixed_modes_fits_g_vs_fits_k}.
The performance differences between the Gaussian-kernel and KNN-kernel schemes are negligible across the downstream tasks.
These findings support the effectiveness of the proposed KNN-kernel scheme for FITS-based synthetic data generation.

The preceding experiments show that data usage modes affect downstream performance differently.
In the third set of experiments, we therefore vary the parameters of M3 and M4 for FITS-G to investigate how synthetic data can be used more effectively.
In Tables~\ref{tab:simulation_linear_parameter_l2_error_fixed_modes_fits_g_all_modes_parameters},
\ref{tab:simulation_logistic_parameter_l2_error_fixed_modes_fits_g_all_modes_parameters}, and
\ref{tab:simulation_ft_classification_test_accuracy_fixed_modes_fits_g_all_modes_parameters},
the best result across all modes and parameter settings is shown in bold for each client.
These results suggest that M4 is an effective way to exploit synthetic data, provided that the synthetic-observation loss weight $\lambda$ is chosen appropriately.

\subsubsection{Distributional fidelity}

Synthetic data needs to exhibit distributional fidelity.
In this section, we use a distributional metric called Energy Distance (ED, \cite{Szkely2013EnergySA}) to assess the fidelity of this synthetic data. Compared with the commonly used KL divergence, ED can be estimated directly from samples without requiring density estimation and avoids the infinite divergence that may occur due to mismatched supports between original and synthetic distributions.
\begin{equation*}
D_E(\hat P, \hat Q)=
\sqrt{
\frac{2}{nm}\sum_{i=1}^{n}\sum_{j=1}^{m}\| \bm x_i- \bm y_j\|_2
-\frac{1}{n^2}\sum_{i=1}^{n}\sum_{k=1}^{n}\|\bm x_i- \bm x_k\|_2
-\frac{1}{m^2}\sum_{j=1}^{m}\sum_{\ell=1}^{m}\|\bm y_j- \bm y_\ell\|_2
},
\end{equation*} where $\hat P$ is the empirical distribution of the sample $\{\bm x_i\}_{i=1}^n$ drawn from $P$; similarly, $\hat Q$ is the empirical distribution of the sample $\{\bm y_i\}_{i=1}^m$ drawn from $Q$. 
We evaluate the distributional distance between the synthetic data and the full data including the training samples and test samples.
We use the full data as the reference because it includes the training data used for synthesis and, owing to its larger sample size, is expected to provide a more reliable approximation to the underlying data distribution. 
This choice allows us to assess synthetic-data fidelity against a reference that incorporates the training observations while more broadly representing the underlying distribution. In the experiment, continuous and ordered variables are standardized using the means
and standard deviations computed from the real reference data.
Unordered categorical variables are one-hot encoded and each one-hot block is divided by
$\sqrt{2}$, giving unit Euclidean distance between distinct categories.

In Table~\ref{tab:simulation_energy_distance}, we compute three Energy Distances to assess the fidelity of the covariate distribution, the joint distribution of the covariates and the regression response, and the joint distribution of the covariates and the classification response.
FITS-G gives the lowest ED values among all synthetic data generators, providing the best distributional fidelity for the full-dimensional tabular data.
\begin{table}[ht]
\centering
\caption{Simulation: Energy Distance between synthetic and full data.}
\label{tab:simulation_energy_distance}
\begingroup
\setlength{\tabcolsep}{4pt}
\renewcommand{\arraystretch}{1.35}
\resizebox{\textwidth}{!}{%
\begin{tabular*}{\textwidth}{@{\extracolsep{\fill}}l*{3}{c}@{}}
\specialrule{0.8pt}{0pt}{1.5pt}
\specialrule{0.8pt}{0pt}{3pt}
Variables & HeteroFedSyn & FedTabDiff & FITS-G \\
\specialrule{0.8pt}{3pt}{1.5pt}
\specialrule{0.8pt}{0pt}{3pt}
$X$ & 0.0702 (0.0084) & 0.1215 (0.0493) & \textbf{0.0392 (0.0079)} \\
\hline
$(X,Y_{\mathrm{lin}})$ & 0.1074 (0.0073) & 0.1260 (0.0492) & \textbf{0.0429 (0.0076)} \\
\hline
$(X,Y_{\mathrm{log}})$ & 0.0752 (0.0074) & 0.1231 (0.0490) & \textbf{0.0409 (0.0079)} \\
\specialrule{0.8pt}{3pt}{1.5pt}
\specialrule{0.8pt}{0pt}{0pt}
\end{tabular*}
}
\endgroup
\end{table}

\subsubsection{Additional experiments}

Appendix~\ref{app:supp_results} provides supplementary results for
Section~\ref{sec:sim_baselines}.
For the first set of experiments, which compares FITS-G with HeteroFedSyn
and FedTabDiff across all modes and clients,
Tables~\ref{tab:simulation_linear_test_mse_fixed_modes},
\ref{tab:simulation_ft_regression_test_mse_fixed_modes}, and
\ref{tab:simulation_logistic_test_accuracy_fixed_modes}
report the test MSE of linear regression, the test MSE of FT-Transformer
regression, and the test accuracy of logistic regression, respectively.
For the second set of experiments, which compares the Gaussian-kernel
and KNN-kernel schemes, FITS-G and FITS-K, the corresponding results
are presented in
Tables~\ref{tab:simulation_linear_test_mse_fixed_modes_fits_g_vs_fits_k},
\ref{tab:simulation_ft_regression_test_mse_fixed_modes_fits_g_vs_fits_k}, and
\ref{tab:simulation_logistic_test_accuracy_fixed_modes_fits_g_vs_fits_k}.
For the third set of experiments, which examines FITS-G across the four
modes and their respective parameter settings to identify effective
ways to use synthetic data, the corresponding results are provided in
Tables~\ref{tab:simulation_linear_test_mse_fixed_modes_fits_g_all_modes_parameters},
\ref{tab:simulation_ft_regression_test_mse_fixed_modes_fits_g_all_modes_parameters}, and
\ref{tab:simulation_logistic_test_accuracy_fixed_modes_fits_g_all_modes_parameters}.

\subsection{Real-data experiments}
\label{subsec:real_data_experiments}

\subsubsection{Datasets and preprocessing}

We next evaluate the method on two real mixed-type tabular data sets. 
The first is the UCI Bike Sharing data set, used for a linear regression task with response variable \texttt{cnt}. 
After preprocessing, the dataset contains \(17{,}379\) observations, of which \(13{,}903\) are used for training and \(3{,}476\) for testing. 
As in the simulation study, the training set is randomly split across ten clients according to the fixed proportion vector $(p_1, ..., p_{10})$, yielding client sizes of $\{4171,2781,1391,1251,1112,973,834,695,556,139\}$.
The preprocessed predictors consist of four continuous variables and eight discrete variables. The response is standardized before fitting the generator and downstream linear models.
The second real-data experiment uses the Adult Income dataset for binary classification. After preprocessing, the dataset contains \(45{,}222\) observations, with \(36{,}177\) used for training and \(9{,}045\) for testing. 
The training dataset is randomly split according to the fixed proportion vector $(p_1, . . . , p_{10})$, yielding client sizes of $\{10853,7235,3618,3256,2894,2532,2171,1809,1447,362\}$.
The predictors include five continuous variables and nine discrete variables. 
The binary income response is encoded as \(0/1\).
For both data sets, missing or invalid entries are removed during preprocessing. 
Continuous predictors are standardized using complete-data means and standard deviations. 
Appendix~\ref{app:realdata_setting} provides more detailed information about the datasets and the chain order in FITS.
Unlike the simulation study, the true downstream parameters are unavailable in the real-data experiments. 
Therefore, parameter error is computed relative to the full-data reference estimator fitted on the complete preprocessed training and test data. Prediction metrics are computed on the test set.

\subsubsection{Bike dataset}
We evaluate regression performance on the Bike dataset.
As in the simulation study, this evaluation comprises three sets of
experiments and an assessment of distributional fidelity.

Tables~\ref{tab:bike_linear_parameter_l2_error_fixed_modes},
\ref{tab:bike_linear_test_mse_fixed_modes}, and
\ref{tab:bike_ft_regression_test_mse_fixed_modes}
present the first set of experiments comparing the synthetic data
generation methods, with $\lambda$ fixed at $0.1$ for M4.
The results show that FITS achieves the best performance across the
regression tasks and models considered on this dataset.
The second set of experiments compares FITS-G and FITS-K, as reported in
Tables~\ref{tab:bike_linear_parameter_l2_error_fixed_modes_fits_g_vs_fits_k},
\ref{tab:bike_linear_test_mse_fixed_modes_fits_g_vs_fits_k}, and
\ref{tab:bike_ft_regression_test_mse_fixed_modes_fits_g_vs_fits_k}.
The two schemes yield broadly comparable performance on the regression
tasks.
The third set of experiments, presented in
Tables~\ref{tab:bike_linear_parameter_l2_error_fixed_modes_fits_g_all_modes_parameters},
\ref{tab:bike_linear_test_mse_fixed_modes_fits_g_all_modes_parameters}, and
\ref{tab:bike_ft_regression_test_mse_fixed_modes_fits_g_all_modes_parameters},
further suggests that M4, with an appropriately chosen $\lambda$,
provides an effective way to combine synthetic data with real local data.
In particular, M4 with $\lambda=1$ is equivalent to M3 with $r=\infty$.
Regarding distributional fidelity, the Energy Distance results in
Table~\ref{tab:bike_energy_distance} indicate that FITS more faithfully
preserves the joint distribution of the Bike data.

\begin{table}[H]
\centering
\caption{Bike: linear-regression parameter estimation error.}
\label{tab:bike_linear_parameter_l2_error_fixed_modes}
\begingroup
\setlength{\tabcolsep}{4pt}
\renewcommand{\arraystretch}{1.35}
\resizebox{\textwidth}{!}{%
\begin{tabular}{@{}cc|*{10}{c}@{}}
\specialrule{0.8pt}{0pt}{1.5pt}
\specialrule{0.8pt}{0pt}{3pt}
\multirow{2}{*}{Mode} & \multirow{2}{*}{Baselines} & \multicolumn{10}{c}{Clients} \\
\cline{3-12}
 & & Client 0 (30\%) & Client 1 (20\%) & Client 2 (10\%) & Client 3 (9\%) & Client 4 (8\%) & Client 5 (7\%) & Client 6 (6\%) & Client 7 (5\%) & Client 8 (4\%) & Client 9 (1\%) \\
\specialrule{0.8pt}{3pt}{1.5pt}
\specialrule{0.8pt}{0pt}{3pt}
\multirow{3}{*}{M1} & FITS-G & \textbf{0.1240 (0.0443)} & \textbf{0.1240 (0.0443)} & \textbf{0.1240 (0.0443)} & \textbf{0.1240 (0.0443)} & \textbf{0.1240 (0.0443)} & \textbf{0.1240 (0.0443)} & \textbf{0.1240 (0.0443)} & \textbf{0.1240 (0.0443)} & \textbf{0.1240 (0.0443)} & \textbf{0.1240 (0.0443)} \\
 & HeteroFedSyn & 0.3068 (0.0503) & 0.3068 (0.0503) & 0.3068 (0.0503) & 0.3068 (0.0503) & 0.3068 (0.0503) & 0.3068 (0.0503) & 0.3068 (0.0503) & 0.3068 (0.0503) & 0.3068 (0.0503) & 0.3068 (0.0503) \\
 & FedTabDiff & 0.3077 (0.0708) & 0.3077 (0.0708) & 0.3077 (0.0708) & 0.3077 (0.0708) & 0.3077 (0.0708) & 0.3077 (0.0708) & 0.3077 (0.0708) & 0.3077 (0.0708) & 0.3077 (0.0708) & 0.3077 (0.0708) \\
\hline
\multirow{1}{*}{M2} & None & \textbf{0.1059 (0.0711)} & \textbf{0.1397 (0.0924)} & \textbf{0.2044 (0.1343)} & \textbf{0.2043 (0.1359)} & \textbf{0.2331 (0.1749)} & \textbf{0.2088 (0.1111)} & \textbf{0.2900 (0.1840)} & \textbf{0.3137 (0.1962)} & \textbf{0.3291 (0.2270)} & \textbf{0.6422 (0.4136)} \\
\hline
\multirow{3}{*}{\shortstack{M3 }} & FITS-G & \textbf{0.0933 (0.0370)} & \textbf{0.0976 (0.0402)} & \textbf{0.1138 (0.0442)} & \textbf{0.1117 (0.0460)} & \textbf{0.1126 (0.0413)} & \textbf{0.1114 (0.0425)} & \textbf{0.1186 (0.0406)} & \textbf{0.1169 (0.0424)} & \textbf{0.1173 (0.0399)} & \textbf{0.1224 (0.0445)} \\
 & HeteroFedSyn & 0.1970 (0.0506) & 0.2187 (0.0482) & 0.2517 (0.0482) & 0.2551 (0.0483) & 0.2604 (0.0488) & 0.2650 (0.0488) & 0.2699 (0.0485) & 0.2743 (0.0477) & 0.2818 (0.0484) & 0.2996 (0.0500) \\
 & FedTabDiff & 0.2006 (0.0499) & 0.2235 (0.0552) & 0.2523 (0.0594) & 0.2549 (0.0572) & 0.2658 (0.0613) & 0.2692 (0.0631) & 0.2754 (0.0638) & 0.2789 (0.0640) & 0.2819 (0.0643) & 0.2992 (0.0708) \\
\hline
\multirow{3}{*}{\shortstack{M4 }} & FITS-G & \textbf{0.0850 (0.0591)} & \textbf{0.1012 (0.0568)} & \textbf{0.1225 (0.0752)} & \textbf{0.1187 (0.0683)} & \textbf{0.1222 (0.0799)} & \textbf{0.1081 (0.0532)} & \textbf{0.1422 (0.0752)} & \textbf{0.1286 (0.0777)} & \textbf{0.1237 (0.0498)} & \textbf{0.1258 (0.0556)} \\
 & HeteroFedSyn & 0.1157 (0.0591) & 0.1228 (0.0571) & 0.1502 (0.0574) & 0.1522 (0.0532) & 0.1616 (0.0599) & 0.1624 (0.0551) & 0.1735 (0.0536) & 0.1789 (0.0517) & 0.1959 (0.0504) & 0.2571 (0.0497) \\
 & FedTabDiff & 0.1109 (0.0578) & 0.1366 (0.0683) & 0.1725 (0.0709) & 0.1637 (0.0690) & 0.2019 (0.0854) & 0.1849 (0.0651) & 0.2200 (0.0830) & 0.2253 (0.0800) & 0.2210 (0.0648) & 0.2612 (0.0812) \\
\specialrule{0.8pt}{3pt}{1.5pt}
\specialrule{0.8pt}{0pt}{0pt}
\end{tabular}
}
\endgroup
\end{table}

\begin{table}[H]
\centering
\caption{Bike: linear-regression MSE on the test data.}
\label{tab:bike_linear_test_mse_fixed_modes}
\begingroup
\setlength{\tabcolsep}{4pt}
\renewcommand{\arraystretch}{1.35}
\resizebox{\textwidth}{!}{%
\begin{tabular}{@{}cc|*{10}{c}@{}}
\specialrule{0.8pt}{0pt}{1.5pt}
\specialrule{0.8pt}{0pt}{3pt}
\multirow{2}{*}{Mode} & \multirow{2}{*}{Baselines} & \multicolumn{10}{c}{Clients} \\
\cline{3-12}
 & & Client 0 (30\%) & Client 1 (20\%) & Client 2 (10\%) & Client 3 (9\%) & Client 4 (8\%) & Client 5 (7\%) & Client 6 (6\%) & Client 7 (5\%) & Client 8 (4\%) & Client 9 (1\%) \\
\specialrule{0.8pt}{3pt}{1.5pt}
\specialrule{0.8pt}{0pt}{3pt}
\multirow{3}{*}{M1} & FITS-G & \textbf{0.6117 (0.0195)} & \textbf{0.6117 (0.0195)} & \textbf{0.6117 (0.0195)} & \textbf{0.6117 (0.0195)} & \textbf{0.6117 (0.0195)} & \textbf{0.6117 (0.0195)} & \textbf{0.6117 (0.0195)} & \textbf{0.6117 (0.0195)} & \textbf{0.6117 (0.0195)} & \textbf{0.6117 (0.0195)} \\
 & HeteroFedSyn & 0.7025 (0.0331) & 0.7025 (0.0331) & 0.7025 (0.0331) & 0.7025 (0.0331) & 0.7025 (0.0331) & 0.7025 (0.0331) & 0.7025 (0.0331) & 0.7025 (0.0331) & 0.7025 (0.0331) & 0.7025 (0.0331) \\
 & FedTabDiff & 0.6709 (0.0267) & 0.6709 (0.0267) & 0.6709 (0.0267) & 0.6709 (0.0267) & 0.6709 (0.0267) & 0.6709 (0.0267) & 0.6709 (0.0267) & 0.6709 (0.0267) & 0.6709 (0.0267) & 0.6709 (0.0267) \\
\hline
\multirow{1}{*}{M2} & None & \textbf{0.6099 (0.0189)} & \textbf{0.6106 (0.0189)} & \textbf{0.6135 (0.0193)} & \textbf{0.6140 (0.0190)} & \textbf{0.6157 (0.0189)} & \textbf{0.6152 (0.0196)} & \textbf{0.6172 (0.0198)} & \textbf{0.6192 (0.0200)} & \textbf{0.6216 (0.0209)} & \textbf{0.6582 (0.0274)} \\
\hline
\multirow{3}{*}{\shortstack{M3 }} & FITS-G & \textbf{0.6098 (0.0192)} & \textbf{0.6101 (0.0192)} & \textbf{0.6109 (0.0194)} & \textbf{0.6107 (0.0195)} & \textbf{0.6110 (0.0194)} & \textbf{0.6110 (0.0194)} & \textbf{0.6111 (0.0194)} & \textbf{0.6111 (0.0195)} & \textbf{0.6113 (0.0194)} & \textbf{0.6116 (0.0194)} \\
 & HeteroFedSyn & 0.6335 (0.0214) & 0.6438 (0.0218) & 0.6626 (0.0252) & 0.6651 (0.0255) & 0.6679 (0.0263) & 0.6714 (0.0265) & 0.6747 (0.0279) & 0.6777 (0.0276) & 0.6828 (0.0283) & 0.6964 (0.0318) \\
 & FedTabDiff & 0.6298 (0.0200) & 0.6366 (0.0207) & 0.6483 (0.0227) & 0.6497 (0.0228) & 0.6518 (0.0230) & 0.6535 (0.0232) & 0.6556 (0.0237) & 0.6578 (0.0241) & 0.6599 (0.0246) & 0.6675 (0.0261) \\
\hline
\multirow{3}{*}{\shortstack{M4 }} & FITS-G & \textbf{0.6096 (0.0190)} & \textbf{0.6100 (0.0189)} & \textbf{0.6114 (0.0193)} & \textbf{0.6112 (0.0194)} & \textbf{0.6120 (0.0191)} & \textbf{0.6117 (0.0193)} & \textbf{0.6121 (0.0192)} & \textbf{0.6121 (0.0195)} & \textbf{0.6124 (0.0196)} & \textbf{0.6122 (0.0192)} \\
 & HeteroFedSyn & 0.6105 (0.0191) & 0.6118 (0.0189) & 0.6168 (0.0204) & 0.6177 (0.0194) & 0.6193 (0.0200) & 0.6210 (0.0195) & 0.6235 (0.0205) & 0.6256 (0.0202) & 0.6318 (0.0203) & 0.6640 (0.0259) \\
 & FedTabDiff & 0.6108 (0.0189) & 0.6120 (0.0189) & 0.6168 (0.0196) & 0.6175 (0.0193) & 0.6197 (0.0193) & 0.6205 (0.0197) & 0.6227 (0.0196) & 0.6256 (0.0199) & 0.6283 (0.0212) & 0.6495 (0.0228) \\
\specialrule{0.8pt}{3pt}{1.5pt}
\specialrule{0.8pt}{0pt}{0pt}
\end{tabular}
}
\endgroup
\end{table}

\begin{table}[H]
\centering
\caption{Bike: FT-Transformer regression MSE on the test data.}
\label{tab:bike_ft_regression_test_mse_fixed_modes}
\begingroup
\setlength{\tabcolsep}{4pt}
\renewcommand{\arraystretch}{1.35}
\resizebox{\textwidth}{!}{%
\begin{tabular}{@{}cc|*{10}{c}@{}}
\specialrule{0.8pt}{0pt}{1.5pt}
\specialrule{0.8pt}{0pt}{3pt}
\multirow{2}{*}{Mode} & \multirow{2}{*}{Baselines} & \multicolumn{10}{c}{Clients} \\
\cline{3-12}
 & & Client 0 (30\%) & Client 1 (20\%) & Client 2 (10\%) & Client 3 (9\%) & Client 4 (8\%) & Client 5 (7\%) & Client 6 (6\%) & Client 7 (5\%) & Client 8 (4\%) & Client 9 (1\%) \\
\specialrule{0.8pt}{3pt}{1.5pt}
\specialrule{0.8pt}{0pt}{3pt}
\multirow{3}{*}{M1} & FITS-G & \textbf{0.2348 (0.0209)} & \textbf{0.2348 (0.0209)} & \textbf{0.2348 (0.0209)} & \textbf{0.2348 (0.0209)} & \textbf{0.2348 (0.0209)} & \textbf{0.2348 (0.0209)} & \textbf{0.2348 (0.0209)} & \textbf{0.2348 (0.0209)} & \textbf{0.2348 (0.0209)} & \textbf{0.2348 (0.0209)} \\
 & HeteroFedSyn & 0.8725 (0.0907) & 0.8725 (0.0907) & 0.8725 (0.0907) & 0.8725 (0.0907) & 0.8725 (0.0907) & 0.8725 (0.0907) & 0.8725 (0.0907) & 0.8725 (0.0907) & 0.8725 (0.0907) & 0.8725 (0.0907) \\
 & FedTabDiff & 0.6221 (0.0552) & 0.6221 (0.0552) & 0.6221 (0.0552) & 0.6221 (0.0552) & 0.6221 (0.0552) & 0.6221 (0.0552) & 0.6221 (0.0552) & 0.6221 (0.0552) & 0.6221 (0.0552) & 0.6221 (0.0552) \\
\hline
\multirow{1}{*}{M2} & None & \textbf{0.0687 (0.0060)} & \textbf{0.0727 (0.0041)} & \textbf{0.0896 (0.0065)} & \textbf{0.0944 (0.0056)} & \textbf{0.0990 (0.0070)} & \textbf{0.1079 (0.0079)} & \textbf{0.1154 (0.0107)} & \textbf{0.1262 (0.0142)} & \textbf{0.1522 (0.0192)} & \textbf{0.4313 (0.0741)} \\
\hline
\multirow{3}{*}{\shortstack{M3 }} & FITS-G & \textbf{0.1150 (0.0122)} & \textbf{0.1324 (0.0143)} & \textbf{0.1636 (0.0177)} & \textbf{0.1697 (0.0191)} & \textbf{0.1732 (0.0152)} & \textbf{0.1790 (0.0210)} & \textbf{0.1840 (0.0212)} & \textbf{0.1908 (0.0158)} & \textbf{0.1984 (0.0210)} & \textbf{0.2213 (0.0199)} \\
 & HeteroFedSyn & 0.2256 (0.0274) & 0.2960 (0.0336) & 0.4324 (0.0365) & 0.4506 (0.0408) & 0.4756 (0.0518) & 0.5095 (0.0569) & 0.5348 (0.0469) & 0.5696 (0.0545) & 0.6045 (0.0524) & 0.7849 (0.0732) \\
 & FedTabDiff & 0.2259 (0.0264) & 0.2922 (0.0343) & 0.4121 (0.0499) & 0.4217 (0.0528) & 0.4349 (0.0454) & 0.4573 (0.0473) & 0.4778 (0.0464) & 0.5021 (0.0512) & 0.5176 (0.0496) & 0.5849 (0.0500) \\
\hline
\multirow{3}{*}{\shortstack{M4 }} & FITS-G & \textbf{0.0717 (0.0057)} & \textbf{0.0774 (0.0065)} & \textbf{0.0956 (0.0110)} & \textbf{0.1034 (0.0085)} & \textbf{0.1053 (0.0090)} & \textbf{0.1169 (0.0120)} & \textbf{0.1240 (0.0126)} & \textbf{0.1363 (0.0164)} & \textbf{0.1499 (0.0160)} & \textbf{0.2106 (0.0237)} \\
 & HeteroFedSyn & 0.0857 (0.0069) & 0.1047 (0.0092) & 0.1797 (0.0229) & 0.1968 (0.0214) & 0.2196 (0.0268) & 0.2556 (0.0326) & 0.2871 (0.0396) & 0.3353 (0.0366) & 0.3895 (0.0442) & 0.6997 (0.0695) \\
 & FedTabDiff & 0.0905 (0.0094) & 0.1071 (0.0123) & 0.1765 (0.0300) & 0.1818 (0.0224) & 0.2097 (0.0418) & 0.2294 (0.0299) & 0.2552 (0.0389) & 0.2949 (0.0445) & 0.3555 (0.0460) & 0.5521 (0.0564) \\
\specialrule{0.8pt}{3pt}{1.5pt}
\specialrule{0.8pt}{0pt}{0pt}
\end{tabular}
}
\endgroup
\end{table}

\begin{table}[H]
\centering
\caption{Bike: Energy Distance between synthetic and full real data.}
\label{tab:bike_energy_distance}
\begingroup
\setlength{\tabcolsep}{4pt}
\renewcommand{\arraystretch}{1.35}
\resizebox{\textwidth}{!}{%
\begin{tabular*}{\textwidth}{@{\extracolsep{\fill}}l*{3}{c}@{}}
\specialrule{0.8pt}{0pt}{1.5pt}
\specialrule{0.8pt}{0pt}{3pt}
Variables & HeteroFedSyn & FedTabDiff & FITS-G \\
\specialrule{0.8pt}{3pt}{1.5pt}
\specialrule{0.8pt}{0pt}{3pt}
$X$ & 0.0903 (0.0093) & 0.1247 (0.0260) & \textbf{0.0595 (0.0042)} \\
\hline
$X,Y$ & 0.1008 (0.0072) & 0.1244 (0.0253) & \textbf{0.0599 (0.0044)} \\
\specialrule{0.8pt}{3pt}{1.5pt}
\specialrule{0.8pt}{0pt}{0pt}
\end{tabular*}
}
\endgroup
\end{table}

\subsubsection{Adult dataset}
We evaluate classification performance on the Adult dataset using
the same experimental structure.
Tables~\ref{tab:adult_logistic_parameter_l2_error_fixed_modes},
\ref{tab:adult_logistic_test_accuracy_fixed_modes}, and
\ref{tab:adult_ft_classification_test_accuracy_fixed_modes}
compare the synthetic data generation methods, with $\lambda$ fixed
at $0.05$ for M4.
The comparisons between FITS-G and FITS-K are presented in
Tables~\ref{tab:adult_logistic_parameter_l2_error_fixed_modes_fits_g_vs_fits_k},
\ref{tab:adult_logistic_test_accuracy_fixed_modes_fits_g_vs_fits_k}, and
\ref{tab:adult_ft_classification_test_accuracy_fixed_modes_fits_g_vs_fits_k}.
Finally,
Tables~\ref{tab:adult_logistic_parameter_l2_error_fixed_modes_fits_g_all_modes_parameters},
\ref{tab:adult_logistic_test_accuracy_fixed_modes_fits_g_all_modes_parameters}, and
\ref{tab:adult_ft_classification_test_accuracy_fixed_modes_fits_g_all_modes_parameters}
report the performance of FITS-G across the four modes and their
respective parameter settings.

In contrast to the findings on the other datasets, the results here show that HeteroFedSyn achieves lower parameter estimation errors than FITS for logistic regression (Table~\ref{tab:adult_logistic_parameter_l2_error_fixed_modes}). 
However, the two methods yield comparable test classification accuracy (Table~\ref{tab:adult_logistic_test_accuracy_fixed_modes}), while FITS achieves the best overall performance
with the more expressive FT-Transformer model (Table~\ref{tab:adult_ft_classification_test_accuracy_fixed_modes}).

One possible explanation is that two-way marginals capture much of
the distributional structure relevant to logistic regression on this
dataset. 
Preserving these pairwise relationships may therefore allow HeteroFedSyn to achieve better logistic-regression parameter estimates
than FITS. 
In contrast, FT-Transformer can exploit richer, higher-order dependencies, making the preservation of more distributional information beyond two-way marginals more important.
HeteroFedSyn may preserve these dependencies less effectively, which could explain its weaker performance relative to FITS with FT-Transformer.
The Energy Distance results in Table~\ref{tab:adult_energy_distance} are consistent with this interpretation: although HeteroFedSyn also achieves relatively low distances, FITS exhibits the highest overall distributional fidelity.

\begin{table}[H]
\centering
\caption{Adult: logistic-regression parameter estimation error.}
\label{tab:adult_logistic_parameter_l2_error_fixed_modes}
\begingroup
\setlength{\tabcolsep}{4pt}
\renewcommand{\arraystretch}{1.35}
\resizebox{\textwidth}{!}{%
\begin{tabular}{@{}cc|*{10}{c}@{}}
\specialrule{0.8pt}{0pt}{1.5pt}
\specialrule{0.8pt}{0pt}{3pt}
\multirow{2}{*}{Mode} & \multirow{2}{*}{Baselines} & \multicolumn{10}{c}{Clients} \\
\cline{3-12}
 & & Client 0 (30\%) & Client 1 (20\%) & Client 2 (10\%) & Client 3 (9\%) & Client 4 (8\%) & Client 5 (7\%) & Client 6 (6\%) & Client 7 (5\%) & Client 8 (4\%) & Client 9 (1\%) \\
\specialrule{0.8pt}{3pt}{1.5pt}
\specialrule{0.8pt}{0pt}{3pt}
\multirow{3}{*}{M1} & FITS-G & 1.0930 (0.1225) & 1.0930 (0.1225) & 1.0930 (0.1225) & 1.0930 (0.1225) & 1.0930 (0.1225) & 1.0930 (0.1225) & 1.0930 (0.1225) & 1.0930 (0.1225) & 1.0930 (0.1225) & 1.0930 (0.1225) \\
 & HeteroFedSyn & \textbf{0.7452 (0.2609)} & \textbf{0.7452 (0.2609)} & \textbf{0.7452 (0.2609)} & \textbf{0.7452 (0.2609)} & \textbf{0.7452 (0.2609)} & \textbf{0.7452 (0.2609)} & \textbf{0.7452 (0.2609)} & \textbf{0.7452 (0.2609)} & \textbf{0.7452 (0.2609)} & \textbf{0.7452 (0.2609)} \\
 & FedTabDiff & 2.3245 (0.1224) & 2.3245 (0.1224) & 2.3245 (0.1224) & 2.3245 (0.1224) & 2.3245 (0.1224) & 2.3245 (0.1224) & 2.3245 (0.1224) & 2.3245 (0.1224) & 2.3245 (0.1224) & 2.3245 (0.1224) \\
\hline
\multirow{1}{*}{M2} & None & \textbf{0.2712 (0.1082)} & \textbf{0.2810 (0.1239)} & \textbf{0.4893 (0.2531)} & \textbf{0.5538 (0.2660)} & \textbf{0.5820 (0.2656)} & \textbf{0.5930 (0.2823)} & \textbf{0.6766 (0.2996)} & \textbf{0.7805 (0.3195)} & \textbf{0.7778 (0.4350)} & \textbf{1.6560 (0.8013)} \\
\hline
\multirow{3}{*}{\shortstack{M3 }} & FITS-G & 0.7996 (0.1013) & 0.8844 (0.1078) & 0.9800 (0.1166) & 0.9865 (0.1158) & 0.9973 (0.1187) & 1.0069 (0.1231) & 1.0121 (0.1186) & 1.0249 (0.1155) & 1.0422 (0.1198) & 1.0822 (0.1207) \\
 & HeteroFedSyn & \textbf{0.4685 (0.1792)} & \textbf{0.5287 (0.2045)} & \textbf{0.6270 (0.2308)} & \textbf{0.6319 (0.2249)} & \textbf{0.6476 (0.2347)} & \textbf{0.6526 (0.2315)} & \textbf{0.6678 (0.2340)} & \textbf{0.6724 (0.2370)} & \textbf{0.6912 (0.2367)} & \textbf{0.7284 (0.2571)} \\
 & FedTabDiff & 2.1537 (0.1338) & 2.2066 (0.1106) & 2.2579 (0.1088) & 2.2652 (0.1088) & 2.2688 (0.1093) & 2.2757 (0.1053) & 2.2846 (0.1171) & 2.2913 (0.1123) & 2.2914 (0.1118) & 2.3172 (0.1211) \\
\hline
\multirow{3}{*}{\shortstack{M4 }} & FITS-G & 0.2679 (0.0941) & 0.3063 (0.1027) & 0.4919 (0.1803) & 0.4961 (0.2276) & 0.5409 (0.1659) & 0.5342 (0.2183) & 0.5356 (0.1744) & 0.5942 (0.1824) & 0.6763 (0.2070) & 0.9566 (0.1651) \\
 & HeteroFedSyn & \textbf{0.2515 (0.1021)} & \textbf{0.2615 (0.1158)} & \textbf{0.4312 (0.1983)} & \textbf{0.4542 (0.2082)} & \textbf{0.4932 (0.2103)} & \textbf{0.4567 (0.1919)} & \textbf{0.5088 (0.2153)} & \textbf{0.5264 (0.2264)} & \textbf{0.5602 (0.2503)} & \textbf{0.6636 (0.2576)} \\
 & FedTabDiff & 0.7617 (0.2291) & 0.9986 (0.2637) & 1.4801 (0.3538) & 1.5350 (0.3922) & 1.6207 (0.3333) & 1.6829 (0.3651) & 1.7475 (0.3717) & 1.8616 (0.3082) & 1.9475 (0.2928) & 2.2367 (0.1249) \\
\specialrule{0.8pt}{3pt}{1.5pt}
\specialrule{0.8pt}{0pt}{0pt}
\end{tabular}
}
\endgroup
\end{table}

\begin{table}[H]
\centering
\caption{Adult: logistic-regression classification accuracy on the test data.}
\label{tab:adult_logistic_test_accuracy_fixed_modes}
\begingroup
\setlength{\tabcolsep}{4pt}
\renewcommand{\arraystretch}{1.35}
\resizebox{\textwidth}{!}{%
\begin{tabular}{@{}cc|*{10}{c}@{}}
\specialrule{0.8pt}{0pt}{1.5pt}
\specialrule{0.8pt}{0pt}{3pt}
\multirow{2}{*}{Mode} & \multirow{2}{*}{Baselines} & \multicolumn{10}{c}{Clients} \\
\cline{3-12}
 & & Client 0 (30\%) & Client 1 (20\%) & Client 2 (10\%) & Client 3 (9\%) & Client 4 (8\%) & Client 5 (7\%) & Client 6 (6\%) & Client 7 (5\%) & Client 8 (4\%) & Client 9 (1\%) \\
\specialrule{0.8pt}{3pt}{1.5pt}
\specialrule{0.8pt}{0pt}{3pt}
\multirow{3}{*}{M1} & FITS-G & \textbf{0.8187 (0.0030)} & \textbf{0.8187 (0.0030)} & \textbf{0.8187 (0.0030)} & \textbf{0.8187 (0.0030)} & \textbf{0.8187 (0.0030)} & \textbf{0.8187 (0.0030)} & \textbf{0.8187 (0.0030)} & \textbf{0.8187 (0.0030)} & \textbf{0.8187 (0.0030)} & \textbf{0.8187 (0.0030)} \\
 & HeteroFedSyn & 0.8179 (0.0036) & 0.8179 (0.0036) & 0.8179 (0.0036) & 0.8179 (0.0036) & 0.8179 (0.0036) & 0.8179 (0.0036) & 0.8179 (0.0036) & 0.8179 (0.0036) & 0.8179 (0.0036) & 0.8179 (0.0036) \\
 & FedTabDiff & 0.7821 (0.0083) & 0.7821 (0.0083) & 0.7821 (0.0083) & 0.7821 (0.0083) & 0.7821 (0.0083) & 0.7821 (0.0083) & 0.7821 (0.0083) & 0.7821 (0.0083) & 0.7821 (0.0083) & 0.7821 (0.0083) \\
\hline
\multirow{1}{*}{M2} & None & \textbf{0.8204 (0.0031)} & \textbf{0.8207 (0.0032)} & \textbf{0.8204 (0.0036)} & \textbf{0.8199 (0.0034)} & \textbf{0.8202 (0.0028)} & \textbf{0.8194 (0.0038)} & \textbf{0.8196 (0.0034)} & \textbf{0.8195 (0.0033)} & \textbf{0.8189 (0.0038)} & \textbf{0.8131 (0.0054)} \\
\hline
\multirow{3}{*}{\shortstack{M3 }} & FITS-G & \textbf{0.8201 (0.0029)} & 0.8198 (0.0029) & \textbf{0.8195 (0.0032)} & \textbf{0.8193 (0.0030)} & \textbf{0.8193 (0.0030)} & \textbf{0.8192 (0.0029)} & \textbf{0.8192 (0.0029)} & \textbf{0.8190 (0.0030)} & \textbf{0.8190 (0.0030)} & \textbf{0.8188 (0.0030)} \\
 & HeteroFedSyn & 0.8199 (0.0027) & \textbf{0.8198 (0.0029)} & 0.8190 (0.0032) & 0.8189 (0.0030) & 0.8189 (0.0030) & 0.8187 (0.0030) & 0.8187 (0.0030) & 0.8185 (0.0032) & 0.8183 (0.0034) & 0.8180 (0.0034) \\
 & FedTabDiff & 0.7952 (0.0066) & 0.7921 (0.0071) & 0.7879 (0.0077) & 0.7874 (0.0078) & 0.7870 (0.0078) & 0.7863 (0.0078) & 0.7858 (0.0079) & 0.7852 (0.0080) & 0.7846 (0.0081) & 0.7827 (0.0081) \\
\hline
\multirow{3}{*}{\shortstack{M4 }} & FITS-G & 0.8205 (0.0030) & 0.8208 (0.0032) & 0.8204 (0.0034) & 0.8202 (0.0032) & 0.8205 (0.0029) & 0.8199 (0.0033) & 0.8204 (0.0032) & 0.8201 (0.0033) & 0.8200 (0.0031) & \textbf{0.8197 (0.0029)} \\
 & HeteroFedSyn & \textbf{0.8205 (0.0030)} & \textbf{0.8208 (0.0031)} & \textbf{0.8208 (0.0033)} & \textbf{0.8205 (0.0032)} & \textbf{0.8208 (0.0029)} & \textbf{0.8204 (0.0033)} & \textbf{0.8206 (0.0028)} & \textbf{0.8204 (0.0029)} & \textbf{0.8201 (0.0033)} & 0.8195 (0.0030) \\
 & FedTabDiff & 0.8191 (0.0034) & 0.8180 (0.0033) & 0.8134 (0.0047) & 0.8119 (0.0054) & 0.8113 (0.0054) & 0.8097 (0.0055) & 0.8088 (0.0055) & 0.8062 (0.0057) & 0.8039 (0.0061) & 0.7916 (0.0066) \\
\specialrule{0.8pt}{3pt}{1.5pt}
\specialrule{0.8pt}{0pt}{0pt}
\end{tabular}
}
\endgroup
\end{table}

\begin{table}[t]
\centering
\caption{Adult: FT-Transformer classification accuracy on the test data.}
\label{tab:adult_ft_classification_test_accuracy_fixed_modes}
\begingroup
\setlength{\tabcolsep}{4pt}
\renewcommand{\arraystretch}{1.35}
\resizebox{\textwidth}{!}{%
\begin{tabular}{@{}cc|*{10}{c}@{}}
\specialrule{0.8pt}{0pt}{1.5pt}
\specialrule{0.8pt}{0pt}{3pt}
\multirow{2}{*}{Mode} & \multirow{2}{*}{Baselines} & \multicolumn{10}{c}{Clients} \\
\cline{3-12}
 & & Client 0 (30\%) & Client 1 (20\%) & Client 2 (10\%) & Client 3 (9\%) & Client 4 (8\%) & Client 5 (7\%) & Client 6 (6\%) & Client 7 (5\%) & Client 8 (4\%) & Client 9 (1\%) \\
\specialrule{0.8pt}{3pt}{1.5pt}
\specialrule{0.8pt}{0pt}{3pt}
\multirow{3}{*}{M1} & FITS-G & \textbf{0.8478 (0.0039)} & \textbf{0.8478 (0.0039)} & \textbf{0.8478 (0.0039)} & \textbf{0.8478 (0.0039)} & \textbf{0.8478 (0.0039)} & \textbf{0.8478 (0.0039)} & \textbf{0.8478 (0.0039)} & \textbf{0.8478 (0.0039)} & \textbf{0.8478 (0.0039)} & \textbf{0.8478 (0.0039)} \\
 & HeteroFedSyn & 0.8291 (0.0079) & 0.8291 (0.0079) & 0.8291 (0.0079) & 0.8291 (0.0079) & 0.8291 (0.0079) & 0.8291 (0.0079) & 0.8291 (0.0079) & 0.8291 (0.0079) & 0.8291 (0.0079) & 0.8291 (0.0079) \\
 & FedTabDiff & 0.8333 (0.0065) & 0.8333 (0.0065) & 0.8333 (0.0065) & 0.8333 (0.0065) & 0.8333 (0.0065) & 0.8333 (0.0065) & 0.8333 (0.0065) & 0.8333 (0.0065) & 0.8333 (0.0065) & 0.8333 (0.0065) \\
\hline
\multirow{1}{*}{M2} & None & \textbf{0.8513 (0.0045)} & \textbf{0.8506 (0.0040)} & \textbf{0.8394 (0.0076)} & \textbf{0.8353 (0.0086)} & \textbf{0.8304 (0.0081)} & \textbf{0.8290 (0.0079)} & \textbf{0.8231 (0.0091)} & \textbf{0.8176 (0.0068)} & \textbf{0.8098 (0.0075)} & \textbf{0.8015 (0.0102)} \\
\hline
\multirow{3}{*}{\shortstack{M3 }} & FITS-G & \textbf{0.8513 (0.0042)} & \textbf{0.8513 (0.0039)} & \textbf{0.8498 (0.0043)} & \textbf{0.8498 (0.0044)} & \textbf{0.8496 (0.0039)} & \textbf{0.8493 (0.0045)} & \textbf{0.8496 (0.0044)} & \textbf{0.8491 (0.0045)} & \textbf{0.8474 (0.0053)} & \textbf{0.8485 (0.0048)} \\
 & HeteroFedSyn & 0.8463 (0.0053) & 0.8437 (0.0057) & 0.8391 (0.0056) & 0.8380 (0.0055) & 0.8368 (0.0068) & 0.8370 (0.0056) & 0.8364 (0.0067) & 0.8349 (0.0062) & 0.8336 (0.0057) & 0.8298 (0.0067) \\
 & FedTabDiff & 0.8481 (0.0040) & 0.8457 (0.0045) & 0.8411 (0.0048) & 0.8408 (0.0066) & 0.8404 (0.0046) & 0.8398 (0.0054) & 0.8385 (0.0048) & 0.8375 (0.0071) & 0.8368 (0.0088) & 0.8342 (0.0065) \\
\hline
\multirow{3}{*}{\shortstack{M4 }} & FITS-G & \textbf{0.8526 (0.0036)} & \textbf{0.8521 (0.0037)} & \textbf{0.8470 (0.0044)} & \textbf{0.8472 (0.0046)} & \textbf{0.8463 (0.0050)} & \textbf{0.8457 (0.0047)} & \textbf{0.8438 (0.0056)} & \textbf{0.8441 (0.0043)} & \textbf{0.8424 (0.0070)} & \textbf{0.8445 (0.0044)} \\
 & HeteroFedSyn & 0.8524 (0.0036) & 0.8508 (0.0047) & 0.8469 (0.0052) & 0.8462 (0.0058) & 0.8435 (0.0058) & 0.8442 (0.0052) & 0.8408 (0.0072) & 0.8405 (0.0056) & 0.8356 (0.0070) & 0.8291 (0.0071) \\
 & FedTabDiff & 0.8517 (0.0038) & 0.8511 (0.0040) & 0.8469 (0.0057) & 0.8454 (0.0047) & 0.8447 (0.0041) & 0.8435 (0.0045) & 0.8418 (0.0046) & 0.8407 (0.0051) & 0.8378 (0.0050) & 0.8354 (0.0057) \\
\specialrule{0.8pt}{3pt}{1.5pt}
\specialrule{0.8pt}{0pt}{0pt}
\end{tabular}
}
\endgroup
\end{table}

\begin{table}[t]
\centering
\caption{Adult: Energy Distance between synthetic and full real data.}
\label{tab:adult_energy_distance}
\begingroup
\setlength{\tabcolsep}{4pt}
\renewcommand{\arraystretch}{1.35}
\resizebox{\textwidth}{!}{%
\begin{tabular*}{\textwidth}{@{\extracolsep{\fill}}l*{3}{c}@{}}
\specialrule{0.8pt}{0pt}{1.5pt}
\specialrule{0.8pt}{0pt}{3pt}
Variables & HeteroFedSyn & FedTabDiff & FITS-G \\
\specialrule{0.8pt}{3pt}{1.5pt}
\specialrule{0.8pt}{0pt}{3pt}
$X$ & 0.0551 (0.0017) & 0.2021 (0.0605) & \textbf{0.0543 (0.0023)} \\
\hline
$X,Y$ & 0.0552 (0.0017) & 0.2020 (0.0594) & \textbf{0.0539 (0.0023)} \\
\specialrule{0.8pt}{3pt}{1.5pt}
\specialrule{0.8pt}{0pt}{0pt}
\end{tabular*}
}
\endgroup
\end{table}

\section{Conclusion}

In this paper, we propose FITS, a federated nonparametric framework for generating synthetic mixed-type tabular data.
The framework is built on three main components.
First, we develop a two-step federated local linear procedure for conditional CDF estimation.
The first communication step aggregates local weighted Hessian-type matrices and broadcasts the corresponding intercept coefficient vector, while the second step aggregates response-dependent local CDF contributions.
Under fixed query points and weights, this construction recovers the pooled local linear estimator without centralizing raw observations.
Second, we introduce a continualization scheme for discrete variables, together with an explicit inverse mapping to the original discrete domain.
This construction lifts mixed mass-density structures to a piecewise continuous representation, allowing continuous and discrete variables to be handled within a unified smoothing framework.
Third, we develop client-designed kernels with fixed-bandwidth and KNN-based schemes, together with a penalty-adaptive kernel that adjusts penalties for mismatches in continualized discrete covariates to accommodate mixed-type data.
These designs improve the practical stability of the estimator in federated settings with mixed covariate types.

On the theoretical side, we characterize the proposed continualization at the level of distribution laws.
The results show that the continualized representation preserves probability mass while inducing a piecewise continuous joint density and corresponding conditional CDFs on the lifted domain.
For the federated conditional CDF estimator, we establish uniform convergence over both response and covariate directions under the continualized support structure and Gaussian-kernel smoothing.
These results provide the theoretical guarantees underlying the FITS framework.

As the final goal for FITS, we construct the synthetic-data generator based on the chain rule and inverse conditional-CDF sampling, and evaluate its utility through downstream learning tasks under several modes of combining local and synthetic data.
The experiments highlight three factors governing the utility of synthetic data: client sample size, model capacity, and training mode.
For clients with limited data, training on synthetic samples alone can outperform training solely on local observations.
The benefits also depend on the downstream model: expressive models such as FT-Transformer can exploit synthetic data effectively.
Across the experiments, weighted augmentation through M4 provides an effective way to combine synthetic and local observations, with the weight $\lambda$ playing a crucial role.

Several questions remain for future work.
One direction is to develop more principled weighting schemes for combining local and synthetic data in downstream learning.
Another is to characterize the utility boundary of federated synthetic data, including when synthetic samples are beneficial, redundant, or potentially harmful for a given client.
It is also of interest to study stronger privacy guarantees, communication-efficient grid construction, and ordering strategies for the chain-rule generator.
These extensions would further clarify the statistical and practical scope of federated synthetic data.

\bibliography{bib/ref.bib}
\bibliographystyle{plainnat}

\newpage

\appendix
\setcounter{table}{0}
\renewcommand{\thetable}{S\arabic{table}}
\setcounter{figure}{0}
\renewcommand{\thefigure}{S\arabic{figure}}

\section*{Supplementary Materials for "Faithful Federated Tabular Data Synthesis"}

\section{Useful tools}
\label{app:tools}

This appendix collects the empirical-process tools used in the proofs of Appendix~\ref{app:proof_simple_uniform}. We
state them without proof and refer to \citet{xie2023uniform} for
detailed arguments.

\begin{lem}[Lemma B.1 in \cite{xie2023uniform}]\label{lem:Euclidean_class}
	Let $\kappa: \mathbb{R} \to \mathbb{R}$ be a function of bounded variation. Then the following class is Euclidean:
	\begin{equation*}
		\big\{\kappa (\frac{(\cdot - x)}{w}) : x \in \mathbb{R}, w >0 \big\}.
	\end{equation*}
\end{lem}

\begin{lem}[Lemma B.2 in \cite{xie2023uniform}]\label{lem:Euclidean_class_vector}
	Any uniformly bounded and finite-dimensional vector space of functions is Euclidean.
\end{lem}

\begin{lem}[Lemma B.5 in \cite{xie2023uniform}]\label{lem:Euclidean_class_times}
	Let $\mathcal{G}_1$ be a class of functions that is uniformly bounded by $M_1$ and Euclidean with coefficients $(A_1, v_1)$ and $\mathcal{G}_2$ a class of functions that is uniformly bounded by $M_2$ and Euclidean with coefficients $(A_2, v_2)$. Then the class $\mathcal{G}_1 \mathcal{G}_2
=
\{ g_1 \cdot g_2 : g_1 \in \mathcal{G}_1,\, g_2 \in \mathcal{G}_2 \}$ is uniformly bounded by $M_1 M_2$ and Euclidean with coefficients $
\left( A_1 A_2 (M_1 + M_2)^{v_1 + v_2},\, v_1 + v_2 \right).$
\end{lem}

\begin{lem}[Lemma 2 in \cite{xie2023uniform}]\label{lem:empirical_bound}
Let $\xi_1,\ldots,\xi_n$ be an iid sample of a random vector $\xi$ in $\mathbb{R}^d$.
Let $\mathcal{G}_n$ be a sequence of classes of measurable real-valued functions defined on $\mathbb{R}^d$.
Assume that there is a uniformly bounded Euclidean class $\mathcal{G}$ with coefficients $A$ and $v$
such that $\mathcal{G}_n \subset \mathcal{G}$ for all $n$.
Let $\sigma_n^2$ be a positive sequence such that
\begin{equation*}
\sigma_n^2 \ge \sup_{g\in\mathcal{G}_n} E[g(\xi)^2].
\end{equation*}
Then
\begin{equation*}
\epsilon
=
\sup_{g\in\mathcal{G}_n}
\left|
\sum_{i=1}^n \bigl(g(\xi_i)-Eg(\xi_i)\bigr)
\right|
=
O_p\!\left(
\sqrt{n\sigma_n^2 |\log \sigma_n|}
+ |\log \sigma_n|
\right).
\end{equation*}
In particular, if
\begin{equation*}
\frac{n\sigma_n^2}{|\log \sigma_n|}\to\infty,
\end{equation*}
then
\begin{equation*}
\epsilon
=
O_p\!\left(
\sqrt{n\sigma_n^2 |\log \sigma_n|}
\right).
\end{equation*}
\end{lem}

\begin{lem}\label{lem:eig_bound}
Denote $\Omega(\x,w_x)=\int \bm\iota(\mathbf{t})\bm\iota(\mathbf{t})^\top \bm\kappa(\mathbf{t})\mathbb{I}\{\x+w_x \mathbf{t}\in\mathcal{X}\}\,d\mathbf{t}.$ Under Assumptions \ref{ass:marginal} and \ref{ass:kernel}, there exists a constant $C>0$ such that, for all $\x\in\mathcal{X}$ and all sufficiently small $w_x>0$, the eigenvalues of $\Omega(\x,w_x)$ and $\bm{\varXi}(\x,w_x)$ lie in $[1/C,C]$.
\end{lem}

\section{Proof of Theorem~\ref{thm:simple_uniform}}
\label{app:proof_simple_uniform}

Throughout this and the subsequent appendices, we work under
Assumptions~\ref{ass:marginal}--\ref{ass:kernel} with the pooled
sample of size $M=\sum_{k=1}^{K}m_{k}$, treating
$\widehat F(z\mid\x)$ as the centralized local linear estimator on
$\{(\h_{i,:d},h_{i,d+1})\}_{i=1}^{M}$. By
Proposition~\ref{prop:fed_pooled_equiv}, this is the same object
as the federated estimator $\widehat F^{\rm raw}(z\mid\x)$ in Algorithm~\ref{alg:two_step_fllr} under matched
kernel weights, so all rates established in this appendix
transfer to the federated estimator without further argument.

\subsection{Setup and notation used throughout the proofs}
\label{app:proof_setup}

We collect here the notation that recurs in all subsequent
proofs. Recall the smoothed response indicator from
\eqref{eq:smoothed_indicator}:
\[
    \mathcal K_{i}(z)
    :=
    \mathcal K\!\left(\frac{z-h_{i,d+1}}{w_{z}}\right),
    \qquad i\in[M],
\]
where $\mathcal K=\varPhi$ is the standard normal CDF. The
local linear regression \eqref{eq:llr_global} targets the
population parameter
\begin{equation}\label{eq:true_param}
    \bm\beta^{\star}(z,\x)
    =
    \bigl(\beta_{0}^{\star}(z,\x),
        \beta_{1}^{\star}(z,\x),\ldots,
        \beta_{d}^{\star}(z,\x)\bigr)^{\top}
    =
    \bigl(F(z\mid\x),\,
        \nabla_{\x}F(z\mid\x)^{\top}\bigr)^{\top},
\end{equation}
where
$\nabla_{\x}F(z\mid\x)
=\bigl(\partial F(z\mid\x)/\partial x_{1},
\ldots,\partial F(z\mid\x)/\partial x_{d}\bigr)^{\top}$
is the gradient with respect to $\x$. The intercept
$\beta_{0}^{\star}(z,\x)=F(z\mid\x)$ is the estimand of interest.

For each $i\in[M]$, define the centered covariate increment and
covariate-kernel weight at the query point $\x$:
\begin{equation}\label{eq:design_quantities}
    \u_{i}
    :=
    \h_{i,:d}-\x \ \in\mathbb R^{d},
    \qquad
    \varpi_{i}
    :=
    \bm\kappa_{w_x}\!\left(\h_{i,:d}-\x\right)
    =
    \prod_{j=1}^{d}\kappa\!\left(\frac{h_{ij}-x_{j}}{w_{x}}\right),
\end{equation}
and define the design matrix and the diagonal weight matrix as
\[
    \U
    :=
    \begin{pmatrix}
        1 & \u_{1}^{\top}\\
        \vdots & \vdots\\
        1 & \u_{M}^{\top}
    \end{pmatrix}
    \in\mathbb R^{M\times(d+1)},
    \qquad
    \W
    :=
    \mathrm{diag}(\varpi_{1},\ldots,\varpi_{M}).
\]
With $\U_{i}:=\bm\iota(\u_{i})=(1,\u_{i}^{\top})^{\top}
\in\mathbb R^{d+1}$, the weighted-least-squares normal equations
yield the closed form
\begin{equation}\label{eq:llr_normal_form}
    \widehat{\bm\beta}(z,\x,w_{z},w_{x})
    =
    \bigl(\U^{\top}\W\U\bigr)^{-1}\U^{\top}\W\,\bm{\mathcal K}(z),
    \qquad
    \bm{\mathcal K}(z)
    :=
    \bigl(\mathcal K_{1}(z),\ldots,\mathcal K_{M}(z)\bigr)^{\top},
\end{equation}
and the intercept used as
$\widehat F(z\mid\x)=\hat\beta_{0}(z,\x,w_{z},w_{x})$ admits the
linear-functional representation
\begin{equation}\label{eq:linear_form}
    \widehat F(z\mid\x)
    =
    \hat\beta_{0}(z,\x,w_{z},w_{x})
    =
    \e_{1}^{\top}\bigl(\U^{\top}\W\U\bigr)^{-1}\U^{\top}\W\,\bm{\mathcal K}(z)
    =
    \sum_{i=1}^{M}\alpha_{i}(\x)\,\mathcal K_{i}(z),
\end{equation}
where $\e_{1}=(1,0,\ldots,0)^{\top}\in\mathbb R^{d+1}$ and
$\alpha_{i}(\x):=\varpi_{i}\,\e_{1}^{\top}(\U^{\top}\W\U)^{-1}\U_{i}$.

\paragraph{Pseudo-true parameter and the error decomposition.}
The empirical objective in \eqref{eq:llr_global} is the
finite-sample analogue of the population weighted least-squares
problem whose minimizer is the deterministic pseudo-true parameter
\begin{equation}\label{eq:pseudo_est}
    \bm\beta^{\diamond}(z,\x,w_{z},w_{x})
    :=
    \arg\min_{\bm\beta\in\mathbb R^{d+1}}
    \mathbb E_{Z,\X}\!\left[
        \Bigl(
            \mathcal K\!\Bigl(\frac{z-Z}{w_{z}}\Bigr)
            -\bm\iota(\X-\x)^{\top}\bm\beta
        \Bigr)^{2}
        \bm\kappa_{w_{x}}(\X-\x)
    \right],
\end{equation}
where $\bm\iota(\x)=(1,x_{1},\ldots,x_{d})^{\top} \in \mathbb{R}^{d+1}$ and
$\bm\kappa_{w_{x}}(\X-\x)=\prod_{j=1}^{d}
\kappa((X_{j}-x_{j})/w_{x})$. The minimizer of
\eqref{eq:pseudo_est} admits the closed form
\begin{equation}\label{eq:pseudo_closed_form}
    \D_{w_{x}}\bm\beta^{\diamond}(z,\x,w_{z},w_{x})
    =
    \bm\varXi(\x,w_{x})^{-1}\,\bm\mu(z,\x,w_{z},w_{x}),
\end{equation}
where $\D_{w_{x}}=\mathrm{diag}(1,w_{x},\ldots,w_{x})\in \mathbb{R}^{(d+1)\times (d+1) }$  and
\begin{align*}
    \bm\varXi(\x,w_{x})
    &:=
    \frac{1}{w_{x}^{d}}\,
    \mathbb E_{\X}\!\left[
        \bm\iota\!\Bigl(\frac{\X-\x}{w_{x}}\Bigr)
        \bm\iota\!\Bigl(\frac{\X-\x}{w_{x}}\Bigr)^{\top}
        \bm\kappa_{w_{x}}(\X-\x)
    \right],
    \\[1mm]
    \bm\mu(z,\x,w_{z},w_{x})
    &:=
    \frac{1}{w_{x}^{d}}\,
    \mathbb E_{Z,\X}\!\left[
        \bm\iota\!\Bigl(\frac{\X-\x}{w_{x}}\Bigr)
        \mathcal K\!\Bigl(\frac{z-Z}{w_{z}}\Bigr)
        \bm\kappa_{w_{x}}(\X-\x)
    \right].
\end{align*}
The pseudo-true parameter $\bm\beta^{\diamond}$ is deterministic
and depends on the bandwidths $(w_{z},w_{x})$; it converges to
$\bm\beta^{\star}(z,\x)$ as $w_{z},w_{x}\to 0$. This deterministic
intermediate quantity gives the natural decomposition of the
overall estimation error into a stochastic component and a
smoothing-induced (oracle) bias:
\begin{equation}
    \widehat{\bm\beta}(z,\x,w_{z},w_{x})
    -\bm\beta^{\star}(z,\x)
    =
    \underbrace{
        \widehat{\bm\beta}(z,\x,w_{z},w_{x})
        -\bm\beta^{\diamond}(z,\x,w_{z},w_{x})
    }_{\text{stochastic error}}
    +
    \underbrace{
        \bm\beta^{\diamond}(z,\x,w_{z},w_{x})
        -\bm\beta^{\star}(z,\x)
    }_{\text{oracle bias}}.
\end{equation}
The stochastic error captures the empirical fluctuation of the
finite-sample estimator around its deterministic population
counterpart, and is controlled by an empirical-process argument
over the kernel class
\begin{equation}\label{eq:kernel_class}
    \mathcal G_{M}
    :=
    \Bigl\{
        (\x',z')\mapsto
        \bm\kappa\!\Bigl(\frac{\x'-\x}{w_{x}}\Bigr)
        \mathcal K\!\Bigl(\frac{z'-z}{w_{z}}\Bigr)
        :\ \x\in\mathcal X,\ z\in\mathbb R
    \Bigr\},
\end{equation}
which is uniformly bounded and Euclidean by
Lemmas~\ref{lem:Euclidean_class}--\ref{lem:empirical_bound}.
The oracle bias captures the deterministic smoothing distortion
introduced by the bandwidths $(w_{x},w_{z})$ and is controlled by
a Taylor expansion of $F(\cdot\mid\cdot)$ under
Assumption~\ref{ass:conditional}. The two components are bounded
separately in Sections~\ref{app:simple_bias} and \ref{app:stochastic_bound} below; combining them yields the conclusion
of Theorem~\ref{thm:simple_uniform}.

\subsection{Uniform bias bound}
\label{app:simple_bias}

\begin{lem}[Uniform bias bound]
\label{lem:simple_bias}
Under Assumptions~\ref{ass:marginal}--\ref{ass:kernel},
	\begin{equation*}
		\begin{aligned}
			&\D_{w_x}\bigl(\bm \beta^\diamond(z,\x,w_z, w_x) - \bm \beta^\star(z, \x)\bigr)\\
			& = \frac{w_x^2}{2}\, \Omega(\x,w_x)^{-1} \sum_{i,j=1}^{d} \frac{\partial^2 F(z\mid\x)}{\partial x_i \partial x_j} \int \bm \iota(\mathbf{t})\, t_i t_j\, \bm \kappa (\mathbf{t})\, \mathbb{I}\{\x + w_x \mathbf{t} \in \mathcal{X}\}\,d \mathbf{t}\\
			& \quad + \frac{w_z^2}{2}\, \Omega(\x,w_x)^{-1}\, \frac{\partial^2 F(z\mid \x)}{\partial z^2}\, \Big(\int v^2 \kappa(v)\,dv\Big) \int \bm \iota(\mathbf{t})\,\bm \kappa(\mathbf{t})\, \mathbb{I}\{\x + w_x \mathbf{t}\in \mathcal{X}\}\,d\mathbf{t} \\
			& \quad + o(w_x^2 + w_z^2)
		\end{aligned}
	\end{equation*}
	holds uniformly over $z \in \mathbb{R}$ and $\x \in \mathcal{X}$. In particular, for the intercept term,
\[
    \sup_{z\in\mathbb R,\ \x\in\mathcal X}
    \Bigl|\,
        \beta_0^{\diamond}(z,\x,w_{z},w_{x})
        -F(z\mid\x)
    \,\Bigr|
    =
    O\!\left(w_{x}^{2}+w_{z}^{2}\right).
\]
\end{lem}

\begin{prf}
The proof follows the same decomposition as in the proof of Theorem 1 of \cite{xie2023uniform},
with the only difference that the kernel $\bm\kappa$ is allowed to have unbounded support.
Hence we only need to verify that the remainder terms remain uniformly negligible under the
Gaussian kernel. The original proof writes the population matrices and moments as integral
expressions and then applies second-order Taylor expansions in $z$ and $\x$. 
Write
\begin{equation*}
\bm\varXi(\x,w_x)
=
\int \bm\iota(\mathbf t)\bm\iota(\mathbf t)^\top
\bm\kappa(\mathbf t)\,
f(\x+w_x\mathbf t)\,
\mathbb I\{\x+w_x\mathbf t\in\mathcal X\}\,d\mathbf t \in \mathbb{R}^{(d+1) \times (d+1)},
\end{equation*}
and
\begin{equation*}
\bm\mu (z,\x,w_z,w_x)
=
\int \bm\iota(\mathbf t)\bm\kappa(\mathbf t)\,
\widetilde F(z\mid \x+w_x\mathbf t)\,
f(\x+w_x\mathbf t)\,
\mathbb I\{\x+w_x\mathbf t\in\mathcal X\}\,d\mathbf t \in \mathbb{R}^{d+1},
\end{equation*}
where
\begin{equation*}
\widetilde F(z\mid \x)
=
\mathbb{E}\!\left[
\mathcal K\!\left(\frac{z-Z}{w_z}\right)\Bigm| \x
\right].
\end{equation*}
By definition of $\bm\beta^\diamond$,
\begin{equation*}
\D_{w_x}\bm\beta^\diamond(z,\x,w_z,w_x)
=
\bm \varXi(\x,w_x)^{-1}\bm\mu (z,\x,w_z, w_x).
\end{equation*} 

We first expand $\widetilde F(z\mid \x+w_x\mathbf t)$ in the response direction,
\begin{equation}\label{eq:response_direction}
\widetilde F(z\mid \x+w_x\mathbf t)
=
\int \kappa(v)\,F(z-w_z v\mid \x+w_x\mathbf t)\,dv. 
\end{equation}
Since $F$ is twice continuously differentiable in $z$ and $\x$, and its second-order derivatives
are uniformly continuous on $\mathbb R\times\mathcal X$ by Assumption \ref{ass:conditional},
we have the second-order expansion
\begin{equation}\label{eq:z_difference}
F(z-w_z v\mid \x+w_x\mathbf t)
=
F(z\mid \x+w_x\mathbf t)
-
\partial_zF(z\mid \x+w_x\mathbf t)\,w_z v
+
\frac{w_z^2}{2}\partial_z^2F(z\mid \x+w_x\mathbf t)\,v^2
+
R_z(z,\x,\mathbf t,v),
\end{equation}
where the remainder
\begin{equation*}
R_z(z,\x,\mathbf t,v) = \frac{w_z^2}{2} v^2 \Big(\partial_z^2F(\tilde{z}\mid \x+w_x\mathbf t) -\partial_z^2F(z\mid \x+w_x\mathbf t)\Big)
\end{equation*} and $\tilde{z} \in [z-w_z v, z]$.
Because $\kappa$ is symmetric, $\int v\,\kappa(v)\,dv=0$ implies that the integral of the first-order term $\partial_zF(z\mid \x+w_x\mathbf t)\,w_z v$ vanishes. Because $\kappa$ is Gaussian,
\begin{equation*}
\int v^2 \kappa(v)\,dv<\infty.
\end{equation*}

To control the integrated remainder, fix any $\eta>0$. Since \(\partial_z^2F\) is uniformly bounded and the Gaussian kernel has a finite second moment, we can choose \(M>0\) sufficiently large such that
\[
2\|\partial_z^2F\|_\infty \int_{|v|>M}v^2\kappa(v)\,dv <\eta.
\]
For this fixed \(M\), the uniform continuity of \(\partial_z^2F\) implies that there exists \(\delta>0\) such that
\[
|z'-z|<\delta \quad\Longrightarrow\quad \left|
\partial_z^2F(z'\mid\x') - \partial_z^2F(z\mid\x')\right|
<\eta
\]
uniformly over \(\x'\in\mathcal X\). For sufficiently small \(w_z\), \(w_zM<\delta\), and hence, whenever \(|v|\le M\),
\[
|\tilde z-z|\le w_z|v|\le w_zM<\delta,
\]
so that
\[
\left|
\partial_z^2F(\tilde z\mid\x+w_x\mathbf t) - \partial_z^2F(z\mid\x+w_x\mathbf t)
\right| <\eta.
\]
Therefore,
\[
\frac{1}{w_z^2}
\left|
\int_{\mathbb R}
R_z(z,\x,\mathbf t,v)\kappa(v)\,dv
\right|
\le
\frac{\eta}{2}
\int_{|v|\le M}v^2\kappa(v)\,dv
+
\|\partial_z^2F\|_\infty \int_{|v|>M}v^2\kappa(v)\,dv \le C\eta,
\]
uniformly over the admissible \((z,\x,\mathbf t)\). Since \(\eta>0\) is arbitrary,
\[
\int_{\mathbb R}
R_z(z,\x,\mathbf t,v)\kappa(v)\,dv
=
o(w_z^2)
\]
uniformly. Therefore, by dominated convergence,
\begin{equation*}
\widetilde F(z\mid \x+w_x\mathbf t)
=
F(z\mid \x+w_x\mathbf t)
+
\frac{w_z^2}{2}\partial_z^2F(z\mid \x+w_x\mathbf t)\int v^2\kappa(v)\,dv
+
o(w_z^2),
\end{equation*}
uniformly over $z\in\mathbb R$, $\x\in\mathcal X$, and $\x+w_x\mathbf t\in\mathcal X$.

Next, expand the map
\begin{equation*}
\mathbf u\mapsto F(z\mid \mathbf u)
\end{equation*}
at $\mathbf u=\x$ in the covariate direction. By the second-order Taylor expansion,
\begin{equation}\label{eq:Taylor_expansion}
F(z\mid \x+w_x\mathbf t)
=
F(z\mid \x)
+
w_x \nabla_\x^\top\!\{F(z\mid \x)\}\mathbf t
+
\frac{w_x^2}{2}\mathbf t^\top
\nabla_\x^\top\nabla_\x\!\{F(z\mid \x)\}\mathbf t
+
R_x(z,\x,\mathbf t),
\end{equation}
where 
\begin{equation*}
R_x(z,\x,\mathbf t) = \frac{w_x^2}{2}\mathbf t^\top
\Big(\nabla_\x^\top\nabla_\x\!\{F(z\mid \tilde{\x})\}  - \nabla_\x^\top\nabla_\x\!\{F(z\mid \x)\}\Big)\mathbf t
\end{equation*} and $\tilde{\x} \in [\x, \x+ w_x\mathbf{t}]$. 

Multiplying by $\bm\iota(\mathbf t)\bm\kappa(\mathbf t)$, using the bound
$\|\bm\iota(\mathbf t)\|\lesssim 1+\|\mathbf t\|$, and combining the $z$- and $\x$-expansions, we obtain
\begin{equation*}
\bm\mu(z,\x,w_z,w_x)
=
\bm \varXi(\x,w_x)\D_{w_x}\bm\beta^\star(z,\x)
+
 f(\x)\frac{w_x^2}{2}
\sum_{i,j=1}^d 
\frac{\partial^2}{\partial x_i\partial x_j}F(z\mid\x)
\int \bm\iota(\mathbf t)t_it_j\bm\kappa(\mathbf t)
\mathbb I\{\x+w_x\mathbf t\in\mathcal X\}\,d\mathbf t
\end{equation*}
\begin{equation}\label{eq:mu}
\qquad
+
 f(\x)\frac{w_z^2}{2}
\frac{\partial^2F(z\mid\x)}{\partial z^2}
\left(\int v^2\kappa(v)\,dv\right)
\left(
\int \bm\iota(\mathbf t)\bm\kappa(\mathbf t)
\mathbb I\{\x+w_x\mathbf t\in\mathcal X\}\,d\mathbf t
\right)
+
o(w_x^2+w_z^2),
\end{equation}
uniformly over $z\in\mathbb R$ and $\x\in\mathcal X$. We explain how Eq.~\eqref{eq:mu} is obtained as follows.

To keep all powers of $w_x$ explicit, 
we work with a remainder that does \emph{not} carry a $w_x^2$ factor inside. 
Define the Hessian-increment remainder
\begin{equation*}
\widetilde R_x(z,\x,\mathbf t)
:=
\mathbf t^\top
\Big(
\nabla_\x^\top \nabla_\x F(z\mid \tilde{\x})
-
\nabla_\x^\top \nabla_\x F(z\mid \x)
\Big)\mathbf t,
\end{equation*}
where $\tilde{\x}$ lies on the line segment joining $\x$ and $\x+w_x\mathbf t$. 
With this convention, the second-order Taylor expansion of $F(z\mid\,\cdot\,)$ in the 
covariate direction reads
\begin{equation*}
F(z\mid \x+w_x\mathbf t)
=
F(z\mid \x)
+
w_x \nabla_\x^\top F(z\mid \x)\,\mathbf t
+
\frac{w_x^2}{2}\Big(\mathbf t^\top\nabla_\x^\top\nabla_\x F(z\mid\x)\,\mathbf t
+
\widetilde R_x(z,\x,\mathbf t)\Big).
\end{equation*}
By the definition $\bm\iota(\mathbf t)=(1,t_1,\dots,t_d)^\top$ and 
$\D_{w_x}=\operatorname{diag}(1,w_x,\dots,w_x)$, the zeroth- and first-order terms 
combine into the identity
\begin{equation}\label{eq:identity_iota_beta}
F(z\mid\x)+w_x\nabla_\x^\top F(z\mid\x)\,\mathbf t
=
\bm\iota(\mathbf t)^\top \D_{w_x}\bm\beta^\star(z,\x).
\end{equation}
Combining \eqref{eq:identity_iota_beta} with the $z$-direction expansion
\begin{equation*}
\widetilde F(z\mid \x+w_x\mathbf t)
=
F(z\mid \x+w_x\mathbf t)
+
\frac{w_z^2}{2}\partial_z^2 F(z\mid \x+w_x\mathbf t)\int v^2\kappa(v)\,dv
+
o(w_z^2),
\end{equation*}
we obtain
\begin{equation}\label{eq:tildeF_expansion}
    \begin{aligned}
        \widetilde F(z\mid \x+w_x\mathbf t)&=
\bm\iota(\mathbf t)^\top \D_{w_x}\bm\beta^\star(z,\x)
\\
&+\frac{w_x^2}{2}\Big(\mathbf t^\top\nabla_\x^\top\nabla_\x F(z\mid\x)\,\mathbf t+\widetilde R_x(z,\x,\mathbf t)\Big)
\\
&+\frac{w_z^2}{2}\partial_z^2 F(z\mid \x+w_x\mathbf t)\int v^2\kappa(v)\,dv
+
o(w_z^2).\\
    \end{aligned}
\end{equation}
Substituting \eqref{eq:tildeF_expansion} into the definition of $\bm\mu(z,\x,w_z,w_x)$ yields
\begin{equation}\label{eq:mu_decomposition}
\begin{aligned}
\bm\mu(z,\x,w_z,w_x)
=&\,
\bm\varXi(\x,w_x)\,\D_{w_x}\bm\beta^\star(z,\x)\\
&+
\frac{w_x^2}{2}\int \bm\iota(\mathbf t)\bm\kappa(\mathbf t)
\Big(\mathbf t^\top\nabla_\x^\top\nabla_\x F(z\mid\x)\,\mathbf t+\widetilde R_x(z,\x,\mathbf t)\Big)
f(\x+w_x\mathbf t)\,\mathbb I\{\x+w_x\mathbf t\in\mathcal X\}\,d\mathbf t\\
&+
\frac{w_z^2}{2}\!\int v^2\kappa(v)\,dv\!
\int \bm\iota(\mathbf t)\bm\kappa(\mathbf t)\,\partial_z^2 F(z\mid \x+w_x\mathbf t)\,
f(\x+w_x\mathbf t)\,\mathbb I\{\x+w_x\mathbf t\in\mathcal X\}\,d\mathbf t\\
&+
o(w_z^2),
\end{aligned}
\end{equation}
uniformly over $z\in\mathbb R$ and $\x\in\mathcal X$. 
The first-order Taylor expansion of \(f\) at \(\x\) gives
\begin{equation*}
f(\x+w_x\mathbf t)
=
f(\x)+w_x\nabla_\x^\top f(\check{\x})\,\mathbf t,
\end{equation*}where \(\check{\x}\) lies between \(\x\) and \(\x+w_x\mathbf t\). Since \(\nabla_\x f\) is uniformly bounded by Assumption~\ref{ass:marginal},

\begin{equation*}
f(\x+w_x\mathbf t)-f(\x)=O(w_x\|\mathbf t\|).
\end{equation*}
For $\partial_z^2F$, instead of using a first-order Taylor expansion, we use Assumption~\ref{ass:conditional} to conclude that $\partial_z^2F$ is bounded and uniformly continuous. Hence, for any fixed \(M>0\),
\begin{equation*}
\sup_{\substack{z,\x,\mathbf t:\,\|\mathbf t\|\le M\\
\x+w_x\mathbf t\in\mathcal X}}
\left|
\partial_z^2F(z\mid\x+w_x\mathbf t)
-
\partial_z^2F(z\mid\x)
\right|
=o(1), \text{  as  } w_x\to 0.
\end{equation*}
For \(\|\mathbf t\|>M\), boundedness of \(\partial_z^2F\), together with the finite moments and exponentially decaying tails of the Gaussian kernel, makes the corresponding integral arbitrarily small by taking \(M\) sufficiently large. Therefore, by first taking \(M\) large and then letting \(w_x\to0\),
\begin{equation*}
\int
\bm\iota(\mathbf t)\bm\kappa(\mathbf t)
\left\{
\partial_z^2F(z\mid\x+w_x\mathbf t)
-
\partial_z^2F(z\mid\x)
\right\}
f(\x+w_x\mathbf t)\,d\mathbf t
=o(1)
\end{equation*}
uniformly in \(z\) and \(\x\). Consequently, the corresponding term multiplied by \(w_z^2\) is \(o(w_z^2)\). The cross terms arising from the expansion of \(f\) are of orders \(w_x\cdot w_x^2=o(w_x^2)\) and \(w_x\cdot w_z^2=o(w_x^2+w_z^2)\), respectively, since $
w_xw_z^2
\le
\frac12(w_x^2+w_z^2)\max(w_x,w_z)
=o(w_x^2+w_z^2).
$ Hence \eqref{eq:mu_decomposition} reduces to
\begin{equation*}
\begin{aligned}
\bm\mu(z,\x,w_z,w_x)
=&\,
\bm\varXi(\x,w_x)\,\D_{w_x}\bm\beta^\star(z,\x)\\
&+
f(\x)\,\frac{w_x^2}{2}
\int \bm\iota(\mathbf t)\bm\kappa(\mathbf t)\,
\mathbf t^\top\nabla_\x^\top\nabla_\x F(z\mid\x)\,\mathbf t\,
\mathbb I\{\x+w_x\mathbf t\in\mathcal X\}\,d\mathbf t\\
&+
f(\x)\,\frac{w_x^2}{2}
\int \bm\iota(\mathbf t)\bm\kappa(\mathbf t)\,
\widetilde R_x(z,\x,\mathbf t)\,
\mathbb I\{\x+w_x\mathbf t\in\mathcal X\}\,d\mathbf t\\
&+
f(\x)\,\frac{w_z^2}{2}\,\partial_z^2 F(z\mid\x)\int v^2\kappa(v)\,dv
\int \bm\iota(\mathbf t)\bm\kappa(\mathbf t)\,
\mathbb I\{\x+w_x\mathbf t\in\mathcal X\}\,d\mathbf t
+
o(w_x^2+w_z^2).
\end{aligned}
\end{equation*}

It remains to show that the term involving $\widetilde R_x$ is $o(w_x^2)$ uniformly.
The approach is the same: use the finite moments of the Gaussian kernel and the uniform continuity of the Hessian of $F$ with respect to $\x$.
By Assumption \ref{ass:conditional}, $\nabla_\x^\top\nabla_\x F(z\mid\,\cdot\,)$ admits 
a common modulus of continuity $\omega_{\mathbf{x}}$ on $\mathbb R\times\mathcal X$, so
\begin{equation*}
\bigl\|\bm\iota(\mathbf t)\bm\kappa(\mathbf t)\,\widetilde R_x(z,\x,\mathbf t)\bigr\|
\lesssim
(1+\|\mathbf t\|)\,\|\mathbf t\|^2\,\bm\kappa(\mathbf t)\,\omega_{\mathbf{x}}(w_x\|\mathbf t\|).
\end{equation*}
Fix $\varepsilon>0$. Since the Gaussian kernel has finite moments of all orders, 
there exists $M=M(\varepsilon)$ such that
\begin{equation*}
2\,\|\nabla_\x^\top\nabla_\x F\|_\infty\,\int_{\|\mathbf t\|>M}(1+\|\mathbf t\|^3)\,\bm\kappa(\mathbf t)\,d\mathbf t
<\varepsilon.
\end{equation*}
For this fixed $M$, by uniform continuity of $\nabla_\x^\top\nabla_\x F$ on $\mathbb R\times\mathcal X$, 
there exists $w_x^{(0)}>0$ such that $\omega_{\mathbf{x}}(w_x M)<\varepsilon$ for all $w_x<w_x^{(0)}$. 
Splitting the integral at $\|\mathbf t\|=M$, we obtain
\begin{equation*}
\left\|\int \bm\iota(\mathbf t)\bm\kappa(\mathbf t)\,\widetilde R_x(z,\x,\mathbf t)\,
\mathbb I\{\x+w_x\mathbf t\in\mathcal X\}\,d\mathbf t\right\|
\lesssim
\varepsilon\!\int_{\|\mathbf t\|\le M}(1+\|\mathbf t\|^3)\bm\kappa(\mathbf t)\,d\mathbf t
+\varepsilon
\le C\varepsilon,
\end{equation*}
uniformly over $z\in\mathbb R$ and $\x\in\mathcal X$. Since $\varepsilon$ is arbitrary, 
the $\widetilde R_x$-term is $o(1)$, and after multiplication by the prefactor $\tfrac{w_x^2}{2}$, 
it is $o(w_x^2)$ uniformly. This yields \eqref{eq:mu}.

We now translate the expansion of $\bm\mu(z,\x,w_z,w_x)$ into an expansion of 
$\D_{w_x}\bm\beta^\diamond(z,\x,w_z,w_x)=\bm\varXi(\x,w_x)^{-1}\bm\mu(z,\x,w_z,w_x)$. 
By Lemma \ref{lem:eig_bound}, the eigenvalues of $\Omega(\x,w_x)$ and $\bm\varXi(\x,w_x)$ 
are uniformly bounded away from zero and infinity for all $\x\in\mathcal X$ and all 
sufficiently small $w_x>0$, so
\begin{equation*}
\sup_{\x\in\mathcal X}\|\Omega(\x,w_x)^{-1}\|+\sup_{\x\in\mathcal X}\|\bm\varXi(\x,w_x)^{-1}\|<\infty.
\end{equation*}

We next show that
\begin{equation}\label{eq:Xi_minus_fOmega}
\bm\varXi(\x,w_x)-f(\x)\,\Omega(\x,w_x) = O(w_x),
\qquad
\bm\varXi(\x,w_x)^{-1} f(\x) = \Omega(\x,w_x)^{-1}+O(w_x),
\end{equation}
both uniformly over $\x\in\mathcal X$. Indeed, by Assumption \ref{ass:marginal}, $f$ admits a 
first-order Taylor expansion $f(\x+w_x\mathbf t) = f(\x)+w_x\nabla_\x^\top f(\check{\x})\,\mathbf t$ 
with $\check{\x}$ between $\x$ and $\x+w_x\mathbf t$, and $\nabla f$ is uniformly 
bounded by some constant $L_f$ on the closure of $\mathcal X$. Hence
\begin{equation*}
\begin{aligned}
\bm\varXi(\x,w_x)-f(\x)\,\Omega(\x,w_x)
&= \int \bm\iota(\mathbf t)\bm\iota(\mathbf t)^\top\bm\kappa(\mathbf t)
\bigl(f(\x+w_x\mathbf t)-f(\x)\bigr)\mathbb I\{\x+w_x\mathbf t\in\mathcal X\}\,d\mathbf t\\
&= w_x\int \bm\iota(\mathbf t)\bm\iota(\mathbf t)^\top\bm\kappa(\mathbf t)
\nabla_\x^\top f(\check{\x})\,\mathbf t\,
\mathbb I\{\x+w_x\mathbf t\in\mathcal X\}\,d\mathbf t,
\end{aligned}
\end{equation*}
whose operator norm is bounded by $w_x L_f \int(1+\|\mathbf t\|)^2\|\mathbf t\|\,\bm\kappa(\mathbf t)\,d\mathbf t<\infty$ 
since the Gaussian kernel has finite moments of all orders. This proves the first 
equality in \eqref{eq:Xi_minus_fOmega}. The second equality follows from the resolvent identity
\begin{equation*}
\bm\varXi^{-1} f(\x) - \Omega^{-1}
=
\bm\varXi^{-1}\bigl(f(\x)\Omega-\bm\varXi\bigr)\Omega^{-1},
\end{equation*}
combined with the uniform boundedness of $\bm\varXi^{-1}$ and $\Omega^{-1}$.

Multiplying \eqref{eq:mu} by $\bm\varXi(\x,w_x)^{-1}$ and noting that 
$\bm\varXi(\x,w_x)^{-1}\bigl(\bm\varXi(\x,w_x)\D_{w_x}\bm\beta^\star(z,\x)\bigr)
=\D_{w_x}\bm\beta^\star(z,\x)$, we obtain
\begin{equation*}
\begin{aligned}
&\D_{w_x}\bm\beta^\diamond(z,\x,w_z,w_x) - \D_{w_x}\bm\beta^\star(z,\x)\\
&\qquad=
\frac{w_x^2}{2}\,\bm\varXi(\x,w_x)^{-1} f(\x)\,
\sum_{i,j=1}^d \frac{\partial^2 F(z\mid\x)}{\partial x_i\partial x_j}
\int \bm\iota(\mathbf t)\,t_i t_j\,\bm\kappa(\mathbf t)\,\mathbb I\{\x+w_x\mathbf t\in\mathcal X\}\,d\mathbf t\\
&\qquad\quad+
\frac{w_z^2}{2}\,\bm\varXi(\x,w_x)^{-1} f(\x)\,\frac{\partial^2 F(z\mid\x)}{\partial z^2}\int v^2\kappa(v)\,dv
\int \bm\iota(\mathbf t)\bm\kappa(\mathbf t)\,\mathbb I\{\x+w_x\mathbf t\in\mathcal X\}\,d\mathbf t
+o(w_x^2+w_z^2).
\end{aligned}
\end{equation*}
Applying the second identity in \eqref{eq:Xi_minus_fOmega} to replace 
$\bm\varXi(\x,w_x)^{-1} f(\x)$ by $\Omega(\x,w_x)^{-1}+O(w_x)$, and noting that the 
$O(w_x)$ part multiplies an $O(w_x^2+w_z^2)$ quantity (so their product is  $o(w_x^2+w_z^2)$), we conclude
\begin{equation*}
\begin{aligned}
\D_{w_x}\bigl(\bm\beta^\diamond(z,\x,w_z,w_x)-\bm\beta^\star(z,\x)\bigr)
=&\,
\frac{w_x^2}{2}\,\Omega(\x,w_x)^{-1}
\sum_{i,j=1}^d \frac{\partial^2 F(z\mid\x)}{\partial x_i\partial x_j}
\int \bm\iota(\mathbf t)\,t_i t_j\,\bm\kappa(\mathbf t)\,\mathbb I\{\x+w_x\mathbf t\in\mathcal X\}\,d\mathbf t\\
&+
\frac{w_z^2}{2}\,\Omega(\x,w_x)^{-1}\frac{\partial^2 F(z\mid\x)}{\partial z^2}\int v^2\kappa(v)\,dv
\int \bm\iota(\mathbf t)\bm\kappa(\mathbf t)\,\mathbb I\{\x+w_x\mathbf t\in\mathcal X\}\,d\mathbf t\\
&+o(w_x^2+w_z^2),
\end{aligned}
\end{equation*}
uniformly over $z\in\mathbb R$ and $\x\in\mathcal X$. This proves the first claim. 
The second claim follows by taking the first coordinate of the above expansion 
and noting that the first diagonal entry of $\D_{w_x}$ is $1$, so that
\begin{equation*}
\beta_0^\diamond(z,\x,w_z,w_x)-\beta_0^\star(z,\x) = O(w_x^2+w_z^2)
\end{equation*}
uniformly over $z\in\mathbb R$ and $\x\in\mathcal X$.

\end{prf}

\subsection{Stochastic bound}\label{app:stochastic_bound}
\begin{lem}[Uniform stochastic bound]
\label{lem:simple_stochastic}
Under Assumptions~\ref{ass:marginal}--\ref{ass:kernel},
let
\begin{equation*}
	\D_{w_x}\hat{\bm\beta}(z,\x,w_z,w_x) = \hat{\bm\varXi}(\x,w_x)^{-1}\hat{\bm\mu}(z,\x,w_z,w_x).
\end{equation*}
If $\frac{M w_x^d}{|\log w_x|}\to\infty$, then
\begin{equation*}
\sup_{z\in\mathbb R,\ \x\in\mathcal X}
\left\| \D_{w_x}\bigl(\hat{\bm\beta}(z,\x,w_z,w_x)-\bm\beta^\diamond(z,\x,w_z,w_x)\bigr) \right\|
= O_p\!\left( \sqrt{\frac{|\log w_x|}{M w_x^d}} \right).
\end{equation*}
In particular, for the intercept term,
\begin{equation*}
\sup_{z\in\mathbb R,\ \x\in\mathcal X}
\left| \hat\beta_0(z,\x,w_z,w_x)-\beta_0^\diamond(z,\x,w_z,w_x) \right|
= O_p\!\left( \sqrt{\frac{|\log w_x|}{M w_x^d}} \right).
\end{equation*}
\end{lem}

\begin{prf}
Let
 \begin{equation*}
\hat{\bm\varXi}(\x,w_x)
=
\frac{1}{Mw_x^d}
\sum_{i=1}^M
\bm\iota\!\left(\frac{\X_i-\x}{w_x}\right)
\bm\iota\!\left(\frac{\X_i-\x}{w_x}\right)^\top
\bm\kappa\!\left(\frac{\X_i-\x}{w_x}\right),
 \end{equation*}
and
 \begin{equation*}
\hat{\bm\mu}(z,\x,w_z,w_x)
=
\frac{1}{Mw_x^d}
\sum_{i=1}^M
\bm\iota\!\left(\frac{\X_i-\x}{w_x}\right)
\mathcal K\!\left(\frac{z-Z_i}{w_z}\right)
\bm\kappa\!\left(\frac{\X_i-\x}{w_x}\right).
 \end{equation*}
Recall that
 \begin{equation*}
\D_{w_x}\hat{\bm\beta}(z,\x,w_z,w_x)
=
\hat{\bm\varXi}(\x,w_x)^{-1}\hat{\bm\mu}(z,\x,w_z,w_x),
\qquad
\D_{w_x}\bm\beta^\diamond(z,\x,w_z,w_x)
=
\bm\varXi(\x,w_x)^{-1}\bm\mu(z,\x,w_z,w_x).
 \end{equation*}
Hence
\begin{equation}\label{eq:decomp_thm2_gaussian}
\begin{aligned}
&\D_{w_x}\bigl(
\hat{\bm\beta}(z,\x,w_z,w_x)-\bm\beta^\diamond(z,\x,w_z,w_x)
\bigr)\\
&\qquad=
\bigl(
\hat{\bm\varXi}(\x,w_x)^{-1}-\bm\varXi(\x,w_x)^{-1}
\bigr)\hat{\bm\mu}(z,\x,w_z,w_x)
+
\bm\varXi(\x,w_x)^{-1}
\bigl(
\hat{\bm\mu}(z,\x,w_z,w_x)-\bm\mu(z,\x,w_z,w_x)
\bigr).
\end{aligned}
\end{equation}
We therefore only need to derive uniform convergence rates for 
$\hat{\bm\mu}-\bm\mu$ and $\hat{\bm\varXi}-\bm\varXi$, and then use the uniform invertibility
of $\bm\varXi(\x,w_x)$.

We first study the stochastic fluctuation of $\hat{\bm\mu}(z,\x,w_z,w_x)$. 
For each coordinate indexed by $\pi=(\pi_1,\dots,\pi_d)\in\{0,1\}^d$ with $\sum^{d}_{j=1} \pi_j \leq 1$, define
\begin{equation*}
\hat\mu_\pi(z,\x,w_z,w_x)
=
\frac{1}{M w_x^d}
\sum_{i=1}^M
\mathcal K\!\left(\frac{z-Z_i}{w_z}\right)
\prod_{j=1}^d
\left(\frac{X_{ij}-x_j}{w_x}\right)^{\pi_j}
\kappa\!\left(\frac{X_{ij}-x_j}{w_x}\right),
\end{equation*}
and let
 \begin{equation*}
\mu_\pi(z,\x,w_z,w_x)=\mathbb E[\hat\mu_\pi(z,\x,w_z,w_x)].
 \end{equation*}

We introduce the following one-dimensional classes:
\begin{equation*}
\mathcal F_z
:=
\left\{
Z \mapsto \mathcal K\!\left(\frac{z-Z}{w_z}\right)
:\ z\in\mathbb R,\ w_z>0
\right\},
\end{equation*}
and for each $j=1,\dots,d$,
\begin{equation*}
\mathcal F_{x,j}^{(\pi_j)}
:=
\left\{
\X \mapsto
\left(\frac{X_j-x_j}{w_x}\right)^{\pi_j}
\kappa\!\left(\frac{X_j-x_j}{w_x}\right)
:\ x_j\in\mathcal X_j,\ w_x>0
\right\}.
\end{equation*}
Then the class indexing $\hat\mu_\pi$ can be written as the product class
\begin{equation*}
\mathcal G_\pi
=
\mathcal F_z \cdot \prod_{j=1}^d \mathcal F_{x,j}^{(\pi_j)}.
\end{equation*} We now verify that each factor class is uniformly bounded and Euclidean. It then follows from Lemmas \ref{lem:Euclidean_class_vector} and \ref{lem:Euclidean_class_times} that $\mathcal G_\pi$ is uniformly bounded and Euclidean.

First, since $\mathcal K=\Phi$ is the Gaussian distribution function,
 \begin{equation*}
0\le \mathcal K(u)\le 1
\qquad\text{for all }u\in\mathbb R,
 \end{equation*}
so $\mathcal F_z$ is uniformly bounded by $1$. Moreover, $\Phi$ is the integral of the Gaussian kernel;
hence it is of bounded variation. Therefore, by Lemma~\ref{lem:Euclidean_class},
the location-scale class $\mathcal F_z$ is Euclidean.

Next consider $\mathcal F_{x,j}^{(\pi_j)}$. Since $\pi_j\in\{0,1\}$, there are only two cases.

If $\pi_j=0$, then the generating function is
 \begin{equation*}
u\mapsto \kappa(u)=\phi(u),
 \end{equation*}
which is uniformly bounded on $\mathbb R$, and
 \begin{equation*}
\phi'(u)=-u\phi(u)\in L^1(\mathbb R).
 \end{equation*}
Hence $\phi$ is of bounded variation, and thus $\mathcal F_{x,j}^{(0)}$ is Euclidean by Lemma~\ref{lem:Euclidean_class}.

If $\pi_j=1$, then the generating function is
 \begin{equation*}
u\mapsto u\kappa(u)=u\phi(u).
 \end{equation*}
This function is also uniformly bounded on $\mathbb R$, since
 \begin{equation*}
\sup_{u\in\mathbb R}|u\phi(u)|<\infty.
 \end{equation*}
Moreover,
 \begin{equation*}
\frac{d}{du}\{u\phi(u)\}=(1-u^2)\phi(u)\in L^1(\mathbb R),
 \end{equation*}
so $u\phi(u)$ is of bounded variation. Therefore $\mathcal F_{x,j}^{(1)}$ is also Euclidean.

Consequently, for every $j=1,\dots,d$, the class $\mathcal F_{x,j}^{(\pi_j)}$ is uniformly bounded
and Euclidean. Since $\mathcal G_\pi$ is a finite product of uniformly bounded Euclidean classes,
it follows from the closure property of Euclidean classes under finite products that
$\mathcal G_\pi$ is itself uniformly bounded and Euclidean; see Lemmas \ref{lem:Euclidean_class_vector} and \ref{lem:Euclidean_class_times}. And the variance bound of the class $\mathcal G_\pi$ is of order $w_x^d$,

 \begin{equation*}
\begin{aligned}
\sup_{g\in\mathcal G_\pi}\mathbb E[g(Z,\X)^2]
&=
\sup_{z\in\mathbb R,\x\in\mathcal X}
\mathbb E\!\left[
\mathcal K^2\!\left(\frac{z-Z}{w_z}\right)
\prod_{j=1}^d
\left(\frac{X_j-x_j}{w_x}\right)^{2\pi_j}
\kappa^2\!\left(\frac{X_j-x_j}{w_x}\right)
\right]\\
&\le
\sup_{\x\in\mathcal X}
\int_{\mathbb R^d}
\prod_{j=1}^d t_j^{2\pi_j}\bm\kappa(\mathbf t)^2\,
f(\x+w_x\mathbf t)\,w_x^d\,d\mathbf t\\
&\le
w_x^d\|f\|_\infty
\int_{\mathbb R^d}
\prod_{j=1}^d t_j^{2\pi_j}\bm\kappa(\mathbf t)^2\,d\mathbf t\\
&=
w_x^d\|f\|_\infty
\prod_{j=1}^d
\int_{\mathbb R}u^{2\pi_j}\kappa(u)^2\,du\\
&\lesssim w_x^d.
\end{aligned}
 \end{equation*} Therefore, applying Lemma~\ref{lem:empirical_bound} to
\begin{equation*}
	\sum_{i=1}^M\{g(Z_i,\X_i)-\mathbb{E}g(Z_i,\X_i)\}
\end{equation*} with variance scale $\sigma_n^2\asymp w_x^d$ and then dividing by $M w_x^d$ yields 
\begin{equation*}
\sup_{z\in\mathbb R,\x\in\mathcal X}
\left|
\hat\mu_\pi(z,\x,w_z,w_x)-\mu_\pi(z,\x,w_z,w_x)
\right|
=O_p\!\left(
\frac{\sqrt{M w_x^d |\log w_x|}}{Mw_x^d}
\right) = 
O_p\!\left(
\sqrt{\frac{|\log w_x|}{Mw_x^d}}
\right).
 \end{equation*}
Since the number of coordinates $\pi$ is finite, we conclude that
\begin{equation}\label{eq:mu_rate_gaussian}
\sup_{z\in\mathbb R,\x\in\mathcal X}
\left\|
\hat{\bm\mu}(z,\x,w_z,w_x)-\bm\mu(z,\x,w_z,w_x)
\right\|
= 
O_p\!\left(
\sqrt{\frac{|\log w_x|}{Mw_x^d}}
\right).
\end{equation}

We now turn to $\hat{\bm\varXi}(\x,w_x)$. Each entry of $\hat{\bm\varXi}(\x,w_x)$ is of the form
 \begin{equation*}
\frac{1}{Mw_x^d}
\sum_{i=1}^M
q\!\left(\frac{\X_i-\x}{w_x}\right)
\bm\kappa\!\left(\frac{\X_i-\x}{w_x}\right),
 \end{equation*}
where $q(\mathbf t)$ is a polynomial of total degree at most $2$. 
Thus the relevant one-dimensional building blocks are $u^m\phi(u)$ for $m=0,1,2$.
For each such $m$,
 \begin{equation*}
\sup_{u\in\mathbb R}|u^m\phi(u)|<\infty,
\qquad
\int_{\mathbb R}u^{2m}\phi(u)^2\,du<\infty,
 \end{equation*}
and $u\mapsto u^m\phi(u)$ is of bounded variation because its derivative is a polynomial times $\phi(u)$ and hence belongs to $L^1(\mathbb R)$. Therefore the same empirical-process argument as above yields
\begin{equation}\label{eq:Xi_rate_gaussian}
\sup_{\x\in\mathcal X}
\left\|
\hat{\bm\varXi}(\x,w_x)-\bm\varXi(\x,w_x)
\right\|_\infty
=
O_p\!\left(
\sqrt{\frac{|\log w_x|}{Mw_x^d}}
\right).
\end{equation}
Since the dimension of $\bm\varXi(\x,w_x)$ is fixed, all matrix norms are equivalent.
In particular, there exists a constant $C_d<\infty$ such that
$\|A\|_2\le C_d\|A\|_\infty$ for every matrix $A$ of the same dimension.
Hence \eqref{eq:Xi_rate_gaussian} implies
\[
\sup_{\x\in\mathcal X}
\left\|
\hat{\bm\varXi}(\x,w_x)-\bm\varXi(\x,w_x)
\right\|_2
=
O_p\!\left(
\sqrt{\frac{|\log w_x|}{Mw_x^d}}
\right),
\]
under the $2$-norm. By Lemma~\ref{lem:eig_bound}, the eigenvalues of $\bm\varXi(\x,w_x)$ are uniformly bounded away from zero and infinity; that is, there exists a constant $C>1$ such that
\[
1/C \le \lambda_{\min}\bigl(\bm\varXi(\x,w_x)\bigr)
\le \lambda_{\max}\bigl(\bm\varXi(\x,w_x)\bigr)
\le C
\]
for all $\x\in\mathcal X$ and all sufficiently small $w_x$. Hence, by Weyl's inequality,
the eigenvalues of $\hat{\bm\varXi}(\x,w_x)$ also lie in $[1/(2C),2C]$ uniformly over
$\x\in\mathcal X$ with probability approaching $1$. In particular,
\[
\sup_{\x\in\mathcal X}\|\hat{\bm\varXi}(\x,w_x)^{-1}\|_2=O_p(1).
\]

By Lemma~\ref{lem:eig_bound} and Eq.~\eqref{eq:Xi_rate_gaussian}, we obtain
\begin{equation}\label{eq:invXi_rate_gaussian}
\begin{aligned}
	\sup_{\x\in\mathcal X}
\left\|
\hat{\bm\varXi}(\x,w_x)^{-1}-\bm\varXi(\x,w_x)^{-1}
\right\|_2
&= \sup_{\x\in\mathcal X} \left\|
\hat{\bm\varXi}(\x,w_x)^{-1} \Big( \hat{\bm\varXi}(\x,w_x)-\bm\varXi(\x,w_x)\Big)\bm\varXi(\x,w_x)^{-1}
\right\|_2\\ 
&\leq\sup_{\x\in\mathcal X} \left\|
\hat{\bm\varXi}(\x,w_x)^{-1}\right\|_2 \left\|\hat{\bm\varXi}(\x,w_x)-\bm\varXi(\x,w_x)\right\|_2\left\|\bm\varXi(\x,w_x)^{-1}
\right\|_2\\
&=O_p\!\left(
\sqrt{\frac{|\log w_x|}{M w_x^d}}
\right).\\
\end{aligned}
\end{equation}

Substituting \eqref{eq:mu_rate_gaussian}, \eqref{eq:invXi_rate_gaussian}, and the above bound into
\eqref{eq:decomp_thm2_gaussian}, we conclude that
 \begin{equation*}
\sup_{z\in\mathbb R,\ \x\in\mathcal X}
\left\|
\D_{w_x}\bigl(
\hat{\bm\beta}(z,\x,w_z,w_x)-\bm\beta^\diamond(z,\x,w_z,w_x)
\bigr)
\right\|
=
O_p\!\left(
\sqrt{\frac{|\log w_x|}{M w_x^d}}
\right).
 \end{equation*}
This proves the first claim.

The intercept statement follows immediately by taking the first coordinate of the above vector bound,
since the first diagonal entry of $\D_{w_x}$ is $1$.
\end{prf}

\subsection{Completion of the proof of Theorem~\ref{thm:simple_uniform}}
Combining Lemma~\ref{lem:simple_bias} and
Lemma~\ref{lem:simple_stochastic} yields
\[
    \sup_{z\in\mathbb R,\ \x\in\mathcal X}
    \bigl|\widehat F(z\mid\x)-F(z\mid\x)\bigr|
    =
    O_{p}\!\left(
        w_{x}^{2}+w_{z}^{2}
        +\sqrt{\tfrac{|\log w_{x}|}{Mw_{x}^{d}}}
    \right),
\]
which is \eqref{eq:simple_uniform}. \qed

\subsection{Proof of Lemma \ref{lem:eig_bound}}
\begin{prf}
The argument is analogous to the proof of Lemma 11 in \cite{fan2016multivariate}, except that here
$\bm\kappa(\mathbf{t})$ is allowed to have unbounded support, as is the case for the Gaussian kernel.

We first consider $\Omega(\x,w_x)$. For the upper bound, note that
\begin{equation*}
\|\Omega(\x,w_x)\|
\le
\int \|\bm\iota(\mathbf{t})\bm\iota(\mathbf{t})^\top\|\,\bm\kappa(\mathbf{t})\,d\mathbf{t} .
\end{equation*}
Since $\bm\iota(\mathbf{t})\bm\iota(\mathbf{t})^\top$ grows at most quadratically in $\|\mathbf{t}\|$ and
$\bm\kappa(\mathbf{t})$ is Gaussian, the integral on the right-hand side is finite. Hence the largest
eigenvalue of $\Omega(\x,w_x)$ is uniformly bounded from above over $\x\in\mathcal{X}$ and all
$w_x>0$.

For the lower bound, since $\bm\kappa(\mathbf{t})$ is continuous and strictly positive on $\mathbb{R}^d$,
\begin{equation*}
c_\kappa:=\inf_{\mathbf{t}\in B(0,1)} \bm\kappa(\mathbf{t})>0.
\end{equation*}
Therefore,
\begin{equation*}
\Omega(\x,w_x)
\succeq
c_\kappa
\int \bm\iota(\mathbf{t})\bm\iota(\mathbf{t})^\top
\mathbb{I}\{\x+w_x\mathbf{t}\in\mathcal{X},\ \mathbf{t}\in B(0,1)\}\,d\mathbf{t}.
\end{equation*} Setting $\mathbf{t}= \frac{ \bm \eta- \x}{w_x}$, we have
\begin{equation*}
	\begin{aligned}
		&\int \bm\iota(\mathbf{t})\bm\iota(\mathbf{t})^\top\mathbb{I}\{\x+w_x\mathbf{t}\in\mathcal{X},\ \mathbf{t}\in  B(0,1)\}\,d\mathbf{t}\\
		& =\int \bm\iota(\frac{ \bm \eta- \x}{w_x})\bm\iota(\frac{ \bm \eta- \x}{w_x})^\top\mathbb{I}\{\bm \eta \in\mathcal{X},\ \bm \eta\in B(\x,w_x)\}\,\frac{d\bm \eta}{w_x^d}.
	\end{aligned}
\end{equation*}
By the geometric condition in Assumption \ref{ass:marginal}, there exists $\lambda_1>0$ such that,
for all $\x\in\mathcal{X}$ and all sufficiently small $w_x>0$, the set $\{\bm \eta \in B(\x,w_x)\}$ contains a ball $B(\x',\lambda_1 w_x)$ for some $\x'\in B(\x,w_x)$. Let $\y_x = \frac{\x' - \x}{w_x}$. Then $\y_x \in B(0,1)$. Hence, 
\begin{equation*}
	\begin{aligned}
		&\int \bm\iota(\frac{ \bm \eta- \x}{w_x})\bm\iota(\frac{ \bm \eta- \x}{w_x})^\top\mathbb{I}\{ \bm \eta\in B(\x,w_x) \cap \mathcal{X}\}\,\frac{d\bm \eta}{w_x^d}\\
		&\succeq  \int \bm\iota(\frac{ \bm \eta- \x}{w_x})\bm\iota(\frac{ \bm \eta- \x}{w_x})^\top\mathbb{I}\{\bm \eta\in B(\x',\lambda_1 w_x)\}\,\frac{d\bm \eta}{w_x^d}\\
		&\succeq  \int \bm\iota(\mathbf{t})\bm\iota(\mathbf{t})^\top\mathbb{I}\{\mathbf{t} \in B(\y_x,\lambda_1 )\}\,d\mathbf{t}.
	\end{aligned}
\end{equation*}
It follows that the smallest eigenvalue of $\Omega(\x,w_x)$ is bounded below by
\begin{equation*}
c_\kappa
\inf_{\y\in B(0,1)}\min_{\|\bm \xi\|=1}
\bm \xi^\top
\left(
\int \bm\iota(\mathbf{t})\bm\iota(\mathbf{t})^\top
\mathbb{I}\{\mathbf{t}\in B(\y,\lambda_1)\}\,d\mathbf{t}
\right)\bm \xi .
\end{equation*}
For each fixed $(\y,\bm \xi)$ with $\|\bm \xi\|=1$,
\begin{equation*}
\bm \xi^\top
\left(
\int \bm\iota(\mathbf{t})\bm\iota(\mathbf{t})^\top
\mathbb{I}\{\mathbf{t}\in B(\y,\lambda_1)\}\,d\mathbf{t}
\right)\bm \xi
=
\int_{B(\y,\lambda_1)} (\bm \xi^\top\bm\iota(\mathbf{t}))^2\,d\mathbf{t} .
\end{equation*}
Since $\bm \xi^\top\bm\iota(\mathbf{t})$ is a nonzero affine function of $\mathbf{t}$, it cannot vanish identically
on the open ball $B(\y,\lambda_1)$. Therefore the above integral is strictly positive. Moreover,
the mapping
\begin{equation*}
(\y,\bm \xi)\mapsto
\int_{B(\y,\lambda_1)} (\bm \xi^\top\bm\iota(\mathbf{t}))^2\,d\mathbf{t}
\end{equation*}
is continuous, and the set
\begin{equation*}
\{(\y,\bm \xi):\ \y\in \bar B(0,1 -\lambda_1),\ \|\bm \xi\|=1\}
\end{equation*}
is compact. The reduction from $B(0,1)$ to $\bar B(0,1-\lambda_1)$ is sufficient for the surrounding argument. Hence, the infimum of this mapping over the compact set is strictly positive. This yields a uniform positive lower bound
for the eigenvalues of $\Omega(\x,w_x)$.

Next consider $\bm\varXi(\x,w_x)$. By definition,
\begin{equation*}
\bm\varXi(\x,w_x)
=
\int \bm\iota(\mathbf{t})\bm\iota(\mathbf{t})^\top \bm\kappa(\mathbf{t})
f(\x+w_x\mathbf{t})\,d\mathbf{t} .
\end{equation*}
By Assumption \ref{ass:marginal}, there exist constants $0<c_f\le C_f<\infty$ such that
\begin{equation*}
c_f\le f(\x)\le C_f,\qquad \forall\,\x\in\mathcal{X} .
\end{equation*}
Therefore,
\begin{equation*}
c_f\,\Omega(\x,w_x)\preceq \bm\varXi(\x,w_x)\preceq C_f\,\Omega(\x,w_x),
\end{equation*}
which implies that the eigenvalues of $\bm\varXi(\x,w_x)$ are bounded above and below uniformly in
$\x\in\mathcal{X}$ and all sufficiently small $w_x>0$.

Combining the above bounds for $\Omega(\x,w_x)$ and $\bm\varXi(\x,w_x)$ completes the proof.
\end{prf}

\section{Other proofs in Section~\ref{sec:theory}}

\subsection{Proofs for the lifted joint and conditional laws}
\label{app:proof_lifted_laws}

This appendix contains the proofs of the three regularity
results stated in Section~\ref{subsec:piecewise_regularity}:
Theorem~\ref{thm:continuity_continualization_pdf}
(blockwise continuity and one-sided limits of the lifted joint
density), Corollary~\ref{cor:Cr_continualization} (the $C^{r}$
extension), and
Theorem~\ref{thm:conditional_distribution} (the corresponding
piecewise structure of the lifted conditional distribution
function). The three results form a logical chain in which each
proof relies on the preceding one: $C^{0}$-regularity provides
the geometric backbone, $C^{r}$-regularity follows by applying
the $C^{0}$ argument coordinatewise to each partial derivative,
and the conditional-distribution structure is then derived by
combining $C^{r}$-regularity with the probability law identity
$\mathbb P(Z\le z\mid\X^{(c)},\Y\in I(\bm\upsilon))
=\mathbb P(Z\le z\mid\X^{(c)},\X^{(d)} = \bm\upsilon )$.

Throughout this appendix we adopt the notation introduced in
Sections~\ref{subsec:continualize_univariate} and
\ref{subsec:continualize_multi}: the discrete support
$\varUpsilon=\prod_{j=1}^{\tilde d}\varUpsilon_{j}$ with blocks
$\{I(\bm\upsilon)\}_{\bm\upsilon\in\varUpsilon}$, the
continualized covariate $\widetilde\X=(\X^{(c)},\Y)$ on
$\widetilde{\mathcal X}=\mathcal X_{c}\times\mathcal Y$, the
interior region $\widetilde{\mathcal X}^{\circ}$ and the
covariate jumping set $\widetilde{\mathcal X}^{\bullet}$.

\subsection{Proof of Theorem~\ref{thm:continuity_continualization_pdf}}
\label{app:proof_continuity_continualization_pdf}

\begin{thm}[Precise version of Theorem~\ref{thm:continuity_continualization_pdf}]
Let $\tilde{\x}=(\x^{(c)},\y)\in \mathcal X_c\times\mathcal Y$ be a generic point in the support of the continualized random vector defined above, where
\begin{equation*}
	\mathcal Y=\prod_{j=1}^{\tilde{d}}\mathcal Y_j,
	\qquad
	\mathcal Y_j=\bigcup_{\ell=1}^{\tau_j} I_{j,\ell},
	\qquad
	I_{j,\ell}=(\upsilon_{j,\ell-1},\upsilon_{j,\ell}].
\end{equation*}
For each $\bm\upsilon=(\upsilon_{1,\ell_1},\dots,\upsilon_{\tilde d,\ell_{\tilde d}})\in\varUpsilon$, let
\begin{equation*}
	I(\bm\upsilon)=I_{1,\ell_1}\times\cdots\times I_{\tilde d,\ell_{\tilde d}}
	\subset \mathcal Y,
	\qquad
	I(\bm\upsilon)^\circ=\prod_{j=1}^{\tilde d}(\upsilon_{j,\ell_j-1},\upsilon_{j,\ell_j}).
\end{equation*}
Assume that for every $\bm\upsilon\in\varUpsilon$, the function
\begin{equation}\label{eq:continuity}
	\x^{(c)}\mapsto f(\x^{(c)},\bm\upsilon)
\end{equation}
is continuous on $\mathcal X_c$ (in the subspace topology of $\mathbb  R^{d-\tilde d}$).

Throughout, all neighborhoods of points in $\widetilde{\mathcal X} = \mathcal X_c\times\mathcal Y$ are understood in the subspace topology of $\mathcal X_c\times\mathcal Y\subset \mathbb R^{d}$.
Then the lifted joint density
\begin{equation*}
	g(\x^{(c)},\y)
	=
	\sum_{\bm\upsilon\in\varUpsilon}
	\frac{f(\x^{(c)},\bm\upsilon)}{|I(\bm\upsilon)|}\,
	\mathbb I\{\y\in I(\bm\upsilon)\}
\end{equation*}
satisfies the following properties.

\begin{enumerate}[(1)]
	\item \emph{Interior continuity.}
	$g$ is jointly continuous on
		\begin{equation}\label{eq:interior_app}
			\widetilde{\mathcal X}^\circ
			:=
			\mathcal X_c\times
			\bigcup_{\bm\upsilon\in\varUpsilon}
			I(\bm\upsilon)^\circ,
		\end{equation}

	\item \emph{Localization of the discontinuity set.}
	The discontinuity set of $g$ on $\mathcal X_c\times\mathcal Y$ is contained in

	\begin{equation*}
		\widetilde{\mathcal X}^\bullet
		:=
		\mathcal X_c\times\mathcal Y^\bullet .
	\end{equation*}

	\item \emph{One-sided limits at partition boundary points in $\widetilde{\mathcal X}^\bullet$.}
	Let $(\x_1^{(c)},\y_1)\in \mathcal X_c\times\mathcal Y$, and let $\bm\upsilon_1\in\varUpsilon$ denote the unique block index for which $\y_1\in I(\bm\upsilon_1)$ (well-defined by the left-open right-closed convention).

	\begin{enumerate}[(3a)]
		\item \emph{Within-block continuity.}
		\begin{equation*}
			\lim_{\substack{(\x^{(c)},\y)\to(\x_1^{(c)},\y_1)\\ (\x^{(c)},\y)\in \mathcal X_c\times I(\bm\upsilon_1)}}
			g(\x^{(c)},\y)
			=
			g(\x_1^{(c)},\y_1)
			=
			\frac{f(\x_1^{(c)},\bm\upsilon_1)}{|I(\bm\upsilon_1)|}.
		\end{equation*}

		\item \emph{Limits from adjacent blocks (jump structure).}
		For every block index $\bm\upsilon'\in\varUpsilon$ with $\bm\upsilon'\neq\bm\upsilon_1$ and $\y_1\in \overline{I(\bm\upsilon')}$ (i.e.\ $\y_1$ lies in the closure of an adjacent block),
		\begin{equation*}
			\lim_{\substack{(\x^{(c)},\y)\to(\x_1^{(c)},\y_1)\\ (\x^{(c)},\y)\in \mathcal X_c\times I(\bm\upsilon')}}
			g(\x^{(c)},\y)
			=
			\frac{f(\x_1^{(c)},\bm\upsilon')}{|I(\bm\upsilon')|}.
		\end{equation*}
		In general this limit differs from $g(\x_1^{(c)},\y_1)$, so $g$ exhibits a jump discontinuity across the corresponding partition interface. The set of such adjacent blocks $\bm\upsilon'$ is empty if and only if every boundary coordinate $y_{1,j}=\upsilon_{j,\tau_j}$ lies on the outermost boundary of $\mathcal Y_j$ on its right side, in which case (3b) is vacuous.
	\end{enumerate}

	\item \emph{Piecewise continuity.}
	The function $g$ is continuous on $\widetilde{\mathcal X}^\circ$, and at every $(\x_1^{(c)},\y_1)\in\widetilde{\mathcal X}^{\bullet}$ admits the one-sided continuous extension described in (3a). Equivalently, $g$ is piecewise continuous in the sense that, for each $\bm\upsilon\in\varUpsilon$, the restriction $g|_{\mathcal X_c\times I(\bm\upsilon)}$ extends continuously to $\mathcal X_c\times \overline{I(\bm\upsilon)}$.
\end{enumerate}
\end{thm}

\begin{prf}

We first verify that the lifted joint density of
\(\widetilde\X=(\X^{(c)},\Y)\) on
\(\widetilde{\mathcal X}=\mathcal X_c\times\mathcal Y\)
equals the displayed formula. We compute the joint distribution of \((\X^{(c)},\Y)\) by
disintegrating over this partition. For any measurable
rectangle \(A\times B\subset\mathcal X_c\times\mathcal Y\),
\[
\begin{aligned}
    \mathbb P\bigl(\X^{(c)}\in A,\,\Y\in B\bigr)
    &=
    \sum_{\bm\upsilon\in\varUpsilon}
    \mathbb P\bigl(\X^{(c)}\in A,\,\Y\in B,\,
        \X^{(d)}=\bm\upsilon\bigr)\\
    &=
    \sum_{\bm\upsilon\in\varUpsilon}
    \mathbb P\bigl(\X^{(c)}\in A,\,\X^{(d)}=\bm\upsilon,\,
        \Y\in B\cap I(\bm\upsilon)\bigr).
\end{aligned}
\]
Conditional on \(\X^{(d)}=\bm\upsilon\), the continualization
takes the form \(Y_j=\upsilon_{j,\ell_j}-U_j\) for each \(j\),
and the \(U_j\)'s are jointly independent of \(\X^{(c)}\); the
conditional distribution of \(\Y\) given
\((\X^{(c)},\X^{(d)}=\bm\upsilon)\) is therefore the uniform
distribution on \(I(\bm\upsilon)\), with constant density
\(1/|I(\bm\upsilon)|\) on this block. Hence
\[
    \mathbb P\bigl(\X^{(c)}\in A,\,\X^{(d)}=\bm\upsilon,\,
        \Y\in B\cap I(\bm\upsilon)\bigr)
    =
    \int_{A}
    f(\x^{(c)},\bm\upsilon)\,d\x^{(c)}
    \cdot
    \frac{\lambda\bigl(B\cap I(\bm\upsilon)\bigr)}
         {|I(\bm\upsilon)|},
\]
where \(\lambda(\cdot)\) denotes Lebesgue measure on
\(\mathbb R^{\tilde d}\). Writing
\(\lambda(B\cap I(\bm\upsilon))=\int_{B}
\mathbb I\{\y\in I(\bm\upsilon)\}\,d\y\), we obtain
\[
    \mathbb P\bigl(\X^{(c)}\in A,\,\Y\in B\bigr)
    =
    \int_{A\times B}
    \biggl\{
    \sum_{\bm\upsilon\in\varUpsilon}
    \frac{f(\x^{(c)},\bm\upsilon)}{|I(\bm\upsilon)|}\,
    \mathbb I\{\y\in I(\bm\upsilon)\}
    \biggr\}
    d\x^{(c)}\,d\y.
\]
Since this holds for every measurable rectangle in
\(\mathcal X_c\times\mathcal Y\), the integrand is a version of
the joint density of \((\X^{(c)},\Y)\) with respect to
Lebesgue measure on \(\mathcal X_c\times\mathcal Y\). This
verifies the displayed formula
\[
    g(\x^{(c)},\y)
    =
    \sum_{\bm\upsilon\in\varUpsilon}
    \frac{f(\x^{(c)},\bm\upsilon)}{|I(\bm\upsilon)|}\,
    \mathbb I\{\y\in I(\bm\upsilon)\}.
\]
We further verify that \(g\) integrates to one:
\[
    \int_{\mathcal X_c\times\mathcal Y}
    g(\x^{(c)},\y)\,d\x^{(c)}\,d\y
    =
    \sum_{\bm\upsilon\in\varUpsilon}
    \int_{\mathcal X_c}f(\x^{(c)},\bm\upsilon)\,d\x^{(c)}
    \cdot
    \frac{|I(\bm\upsilon)|}{|I(\bm\upsilon)|}
    =
    \sum_{\bm\upsilon\in\varUpsilon}
    \mathbb P(\X^{(d)}=\bm\upsilon)
    =1,
\]
where the second equality uses
\(\int_{I(\bm\upsilon)}\mathbb I\{\y\in I(\bm\upsilon)\}d\y
=|I(\bm\upsilon)|\) and the third equality uses the partition
property of \(\X^{(d)}\). This confirms that
\(g\) is a probability density on
\(\mathcal X_c\times\mathcal Y\).

We now turn to the regularity properties of \(g\) stated in
the theorem. We write $\overline B_{\mathcal X_c}(\x_0^{(c)};\delta):=B(\x_0^{(c)};\delta)\cap\mathcal X_c$ for the relative open ball in $\mathcal X_c$, and similarly for $\mathcal Y$.

\smallskip
\textbf{Property (1).}
Fix $\bm\upsilon\in\varUpsilon$ and $(\x_0^{(c)},\y_0)\in\widetilde{\mathcal X}^\circ$. Since $\y_0\in I(\bm\upsilon)^\circ$, there exists $\delta_0>0$ such that $B(\y_0;\delta_0)\subset I(\bm\upsilon)^\circ$.

For all $0<\delta<\delta_0$ and every $(\x^{(c)},\y)\in \overline B_{\mathcal X_c}(\x_0^{(c)};\delta)\times B(\y_0;\delta)$, we have $\y\in I(\bm\upsilon)^\circ$, hence
\begin{equation*}
	g(\x^{(c)},\y)=\frac{f(\x^{(c)},\bm\upsilon)}{|I(\bm\upsilon)|},\qquad
	g(\x_0^{(c)},\y_0)=\frac{f(\x_0^{(c)},\bm\upsilon)}{|I(\bm\upsilon)|}.
\end{equation*}
In particular, on this relative neighborhood $g$ does not depend on $\y$, so the joint continuity at $(\x_0^{(c)},\y_0)$ reduces to the continuity of $\x^{(c)}\mapsto f(\x^{(c)},\bm\upsilon)$ on $\mathcal X_c$, which is guaranteed by~\eqref{eq:continuity}. Thus
\begin{equation*}
	\bigl|g(\x^{(c)},\y)-g(\x_0^{(c)},\y_0)\bigr|
	=\frac{1}{|I(\bm\upsilon)|}\bigl|f(\x^{(c)},\bm\upsilon)-f(\x_0^{(c)},\bm\upsilon)\bigr|
	\longrightarrow 0
\end{equation*}
as $(\x^{(c)},\y)\to(\x_0^{(c)},\y_0)$.

\smallskip
\textbf{Property (2).}
If $(\x_0^{(c)},\y_0)\notin\widetilde{\mathcal X}^{\bullet}$, then for every $j\in[\tilde d]$ we have $y_{0,j}\neq\upsilon_{j,\ell}$ for all $\ell$, so $\y_0$ lies in the interior of some unique block $I(\bm\upsilon_0)$, i.e.\ $\y_0\in I(\bm\upsilon_0)^\circ$. By Property~(1), $g$ is continuous at $(\x_0^{(c)},\y_0)$.

\smallskip
\textbf{Property (3a).}
Let $(\x_1^{(c)},\y_1)\in\widetilde{\mathcal X}$ and let $\bm\upsilon_1$ be the unique block index with $\y_1\in I(\bm\upsilon_1)$. Decompose the discrete coordinate index set as
\begin{equation*}
	\tilde{\mathcal D}_1
	:=
	\bigl\{j\in[\tilde d]:y_{1,j}=\upsilon_{j,\ell_j}\bigr\},
	\qquad
	\tilde{\mathcal D}_2
	:=
	[\tilde d]\setminus \tilde{\mathcal D}_1,
\end{equation*}
where $\ell_j$ is determined by $\bm\upsilon_1=(\upsilon_{1,\ell_1},\dots,\upsilon_{\tilde d,\ell_{\tilde d}})$. Coordinates in $\tilde{\mathcal D}_1$ are the boundary coordinates (right endpoints of their respective intervals), and coordinates in $\tilde{\mathcal D}_2$ lie strictly inside. Write $\y_1=(\y_1^\bullet, \y_1^\circ)$ accordingly.

Define
\begin{equation*}
	\delta_1
	:=
	\min\biggl\{
	\min_{j\in\tilde{\mathcal D}_1}\bigl(\upsilon_{j,\ell_j}-\upsilon_{j,\ell_j-1}\bigr),\ 
	\min_{j\in\tilde{\mathcal D}_2}\operatorname{dist}\bigl(y_{1,j},\{\upsilon_{j,\ell_j-1},\upsilon_{j,\ell_j}\}\bigr)
	\biggr\}>0,
\end{equation*}with the convention that the minimum over an empty set is \(+\infty\).
Construct the one-sided relative neighborhood
\begin{equation*}
	B_-(\x_1^{(c)},\y_1;\delta)
	:=
	\overline B_{\mathcal X_c}(\x_1^{(c)};\delta)
	\times
	B(\y_1^\circ;\delta)
	\times
	\prod_{j\in\tilde{\mathcal D}_1}\bigl(\upsilon_{j,\ell_j}-\delta,\upsilon_{j,\ell_j}\bigr],
\end{equation*}
which is the trace on $\widetilde{\mathcal X}$ of the set of points approaching $\y_1$ from inside $I(\bm\upsilon_1)$ along all boundary coordinates.

For any $0<\delta<\delta_1$ and any $(\x^{(c)},\y)\in B_-(\x_1^{(c)},\y_1;\delta)$, both boundary and interior coordinates satisfy $y_j\in(\upsilon_{j,\ell_j-1},\upsilon_{j,\ell_j}]$, so $\y\in I(\bm\upsilon_1)$, and
\begin{equation*}
	g(\x^{(c)},\y)=\frac{f(\x^{(c)},\bm\upsilon_1)}{|I(\bm\upsilon_1)|}.
\end{equation*}
Hence the limit along $\mathcal X_c\times I(\bm\upsilon_1)$ reduces to the limit in $\x^{(c)}$ alone, and \eqref{eq:continuity} yields
\begin{equation*}
	\lim_{\substack{(\x^{(c)},\y)\to(\x_1^{(c)},\y_1)\\(\x^{(c)},\y)\in\mathcal X_c\times I(\bm\upsilon_1)}}
	g(\x^{(c)},\y)
	=
	g(\x_1^{(c)},\y_1)
	=
	\frac{f(\x_1^{(c)},\bm\upsilon_1)}{|I(\bm\upsilon_1)|}.
\end{equation*}

\smallskip
\textbf{Property (3b).}
Let $\bm\upsilon'\neq\bm\upsilon_1$ be such that $\y_1\in\overline{I(\bm\upsilon')}$. Such $\bm\upsilon'$ exists if and only if some coordinate $j\in\tilde{\mathcal D}_1$ satisfies $\ell_j<\tau_j$ (an internal partition hyperplane); if every $j\in\tilde{\mathcal D}_1$ has $\ell_j=\tau_j$, the boundary coordinates lie on the outermost right boundary of $\mathcal Y$ and no adjacent block exists, so (3b) is vacuous. Assume hereafter that such $\bm\upsilon'$ exists.

Write $\bm\upsilon'=(\upsilon_{1,\ell'_1},\ldots,\upsilon_{\tilde d,\ell'_{\tilde d}})$. Define the one-sided relative neighborhood approaching $\y_1$ from inside $I(\bm\upsilon')$,
\begin{equation*}
	B'(\x_1^{(c)},\y_1;\delta)
	:=
	\overline B_{\mathcal X_c}(\x_1^{(c)};\delta)
	\times
	B(\y_1^\circ;\delta)
	\times
	\prod_{j\in\tilde{\mathcal D}_1}T_j(\delta),
\end{equation*}
where $T_j(\delta)\subset(\upsilon_{j,\ell'_j-1},\upsilon_{j,\ell'_j}]$ is the appropriate one-sided sliver inside $I_{j,\ell'_j}$ with $y_{1,j}\in\overline{T_j(\delta)}$ (e.g.\ $(\upsilon_{j,\ell_j},\upsilon_{j,\ell_j}+\delta]$ if $\ell'_j=\ell_j+1$). For all sufficiently small $\delta$, $B'(\x_1^{(c)},\y_1;\delta)\cap(\mathcal X_c\times\mathcal Y)\subset \mathcal X_c\times I(\bm\upsilon')$, and on this set
\begin{equation*}
	g(\x^{(c)},\y)=\frac{f(\x^{(c)},\bm\upsilon')}{|I(\bm\upsilon')|}.
\end{equation*}
Then \eqref{eq:continuity} gives
\begin{equation*}
	\lim_{\substack{(\x^{(c)},\y)\to(\x_1^{(c)},\y_1)\\(\x^{(c)},\y)\in\mathcal X_c\times I(\bm\upsilon')}}
	g(\x^{(c)},\y)
	=
	\frac{f(\x_1^{(c)},\bm\upsilon')}{|I(\bm\upsilon')|}.
\end{equation*}
This limit equals $g(\x_1^{(c)},\y_1)=f(\x_1^{(c)},\bm\upsilon_1)/|I(\bm\upsilon_1)|$ if and only if
\begin{equation*}
	\frac{f(\x_1^{(c)},\bm\upsilon_1)}{|I(\bm\upsilon_1)|}=\frac{f(\x_1^{(c)},\bm\upsilon')}{|I(\bm\upsilon')|},
\end{equation*}
which fails generically. Hence $g$ admits a jump discontinuity across every partition interface separating two blocks with distinct values of $f(\cdot,\bm\upsilon)/|I(\bm\upsilon)|$.

\smallskip
\textbf{Property (4).}
On $\tilde{\mathcal{X}}^\circ$, every point has the form $(\x_0^{(c)},\y_0)$ with $\y_0\in I(\bm\upsilon_0)^\circ$ for a unique $\bm\upsilon_0$; by Property~(1), $g$ is continuous there. Hence the discontinuity set of $g$ is contained in $\widetilde{\mathcal X}^\bullet$, proving Property~(2) and the first claim of Property~(4). For the second claim, Property~(3a) shows that for each $\bm\upsilon\in\varUpsilon$, the restriction $g|_{\mathcal X_c\times I(\bm\upsilon)}$ admits a continuous extension to $\mathcal X_c\times\overline{I(\bm\upsilon)}$, given on the closure by $f(\x^{(c)},\bm\upsilon)/|I(\bm\upsilon)|$. This is precisely the piecewise continuity statement.
\end{prf}

\subsection{Proof of Corollary~\ref{cor:Cr_continualization}}
\label{app:proof_Cr_continualization}

\begin{cor}[Precise version of Corollary~\ref{cor:Cr_continualization}]
\label{cor:Cr_continualization_precision}
Let \(r\ge 0\) be an integer. Throughout,
\(C^r(\mathcal X_c)\) denotes the class of functions whose partial
derivatives up to order \(r\) exist on \(\mathcal X_c\). Assume that, for every
\(\bm\upsilon\in\varUpsilon\),
\[
    \x^{(c)}\mapsto f(\x^{(c)},\bm\upsilon)
\]
belongs to \(C^r(\mathcal X_c)\), and that all partial derivatives
of order up to \(r\) are continuous on
\(\mathcal X_c\), uniformly over \(\bm\upsilon\in\varUpsilon\).

Let \(g\) denote the lifted density of
\(\widetilde\X=(\X^{(c)},\Y)\), namely
\[
    g(\x^{(c)},\y)
    =
    \sum_{\bm\upsilon\in\varUpsilon}
    \frac{f(\x^{(c)},\bm\upsilon)}
    {|I(\bm\upsilon)|}
    \mathbb I\{\y\in I(\bm\upsilon)\}.
\]
Then the following hold.

\begin{enumerate}[(1)]
    \item \emph{Interior \(C^r\) regularity and explicit derivatives.}
    For each \(\bm\upsilon\in\varUpsilon\), on
    \(\mathcal X_c\times I(\bm\upsilon)^\circ\),
    \[
        g(\x^{(c)},\y)
        =
        \frac{f(\x^{(c)},\bm\upsilon)}
        {|I(\bm\upsilon)|}.
    \]
    Hence \(g\in C^r(\mathcal X_c\times I(\bm\upsilon)^\circ)\)
    in the relative product sense. Moreover, for every multi-index
    \(\alpha=(\alpha_c,\alpha_y)\) with \(|\alpha|\le r\),
    \[
        \partial^\alpha g(\x^{(c)},\y)
        =
        \begin{cases}
        \dfrac{
        \partial^{\alpha_c}_{\x^{(c)}} f(\x^{(c)},\bm\upsilon)}
        {|I(\bm\upsilon)|},
        & \alpha_y=0,\\[6pt]
        0,
        & \alpha_y\neq 0,
        \end{cases}
        \qquad
        (\x^{(c)},\y)\in
        \mathcal X_c\times I(\bm\upsilon)^\circ .
    \]

    \item \emph{Blockwise \(C^r\) extension.}
    For each \(\bm\upsilon\in\varUpsilon\), the restriction
    \(g|_{\mathcal X_c\times I(\bm\upsilon)}\) admits a \(C^r\)
    extension to \(\mathcal X_c\times\overline{I(\bm\upsilon)}\),
    given by
    \[
        g_{\bm\upsilon}^{\rm ext}(\x^{(c)},\y)
        =
        \frac{f(\x^{(c)},\bm\upsilon)}
        {|I(\bm\upsilon)|},
        \qquad
        (\x^{(c)},\y)\in
        \mathcal X_c\times\overline{I(\bm\upsilon)}.
    \]
    Its derivatives satisfy
    \[
        \partial^\alpha g_{\bm\upsilon}^{\rm ext}(\x^{(c)},\y)
        =
        \begin{cases}
        \dfrac{
        \partial^{\alpha_c}_{\x^{(c)}} f(\x^{(c)},\bm\upsilon)}
        {|I(\bm\upsilon)|},
        & \alpha_y=0,\\[6pt]
        0,
        & \alpha_y\neq 0,
        \end{cases}
        \]
    for every \(|\alpha|\le r\).

    \item \emph{One-sided derivative limits across block boundaries.}
    Let \(\y_0\in\mathcal Y^\bullet\), and let
    \(\bm\upsilon_0\in\varUpsilon\) be the unique block index such
    that \(\y_0\in I(\bm\upsilon_0)\). If
    \(\bm\upsilon'\neq\bm\upsilon_0\) satisfies
    \(\y_0\in\overline{I(\bm\upsilon')}\), then for every
    \(\x_0^{(c)}\in\mathcal X_c\) and every multi-index
    \(\alpha=(\alpha_c,0)\) with \(|\alpha|\le r\),
    \[
        \lim_{\substack{
        (\x^{(c)},\y)\to(\x_0^{(c)},\y_0)\\
        \y\in I(\bm\upsilon_0)}}
        \partial^\alpha g(\x^{(c)},\y)
        =
        \frac{
        \partial^{\alpha_c}_{\x^{(c)}}f(\x_0^{(c)},\bm\upsilon_0)}
        {|I(\bm\upsilon_0)|},
    \]
    whereas
    \[
        \lim_{\substack{
        (\x^{(c)},\y)\to(\x_0^{(c)},\y_0)\\
        \y\in I(\bm\upsilon')}}
        \partial^\alpha g(\x^{(c)},\y)
        =
        \frac{
        \partial^{\alpha_c}_{\x^{(c)}}f(\x_0^{(c)},\bm\upsilon')}
        {|I(\bm\upsilon')|}.
    \]
    Therefore the one-sided limits of the \(\x^{(c)}\)-derivative
    \(\partial^\alpha g\) coincide if and only if
    \[
        \frac{
        \partial^{\alpha_c}_{\x^{(c)}}f(\x_0^{(c)},\bm\upsilon_0)}
        {|I(\bm\upsilon_0)|}
        =
        \frac{
        \partial^{\alpha_c}_{\x^{(c)}}f(\x_0^{(c)},\bm\upsilon')}
        {|I(\bm\upsilon')|}.
    \]
    If no such adjacent block index \(\bm\upsilon'\) exists, the
    corresponding cross-block statement is vacuous.

    \item \emph{Piecewise \(C^r\) structure.}
    The collection
    \[
        \left\{
        g|_{\mathcal X_c\times I(\bm\upsilon)}
        :
        \bm\upsilon\in\varUpsilon
        \right\}
    \]
    constitutes a piecewise \(C^r\) structure on
    \(\widetilde{\mathcal X}=\mathcal X_c\times\mathcal Y\). The
    blockwise \(C^r\) extensions need not agree on shared block
    boundaries.
\end{enumerate}
\end{cor}

\begin{prf}
Fix \(\bm\upsilon\in\varUpsilon\). On the open block
\(\mathcal X_c\times I(\bm\upsilon)^\circ\), the block indicators
satisfy
\[
    \mathbb I\{\y\in I(\bm\upsilon)\}
    =
    \mathbb I\{cat(\y)=\bm\upsilon\}.
\]
Therefore
\[
    g(\x^{(c)},\y)
    =
    \frac{f(\x^{(c)},\bm\upsilon)}
    {|I(\bm\upsilon)|},
    \qquad
    (\x^{(c)},\y)\in
    \mathcal X_c\times I(\bm\upsilon)^\circ .
\]
The right-hand side is \(C^r\) in \(\x^{(c)}\) by assumption and is
constant in \(\y\). Hence \(g\in C^r\) on the open block, and for
every multi-index \(\alpha=(\alpha_c,\alpha_y)\) with \(|\alpha|\le r\),
\[
    \partial^\alpha g(\x^{(c)},\y)
    =
    \begin{cases}
    \dfrac{
    \partial^{\alpha_c}_{\x^{(c)}} f(\x^{(c)},\bm\upsilon)}
    {|I(\bm\upsilon)|},
    & \alpha_y=0,\\[6pt]
    0,
    & \alpha_y\neq 0.
    \end{cases}
\]
This proves the interior derivative formula.

Next define, for the same block,
\[
    g_{\bm\upsilon}^{\rm ext}(\x^{(c)},\y)
    =
    \frac{f(\x^{(c)},\bm\upsilon)}
    {|I(\bm\upsilon)|},
    \qquad
    (\x^{(c)},\y)\in
    \mathcal X_c\times\overline{I(\bm\upsilon)}.
\]
Since \(f(\cdot,\bm\upsilon)\in C^r(\mathcal X_c)\) and the
right-hand side is independent of \(\y\),
\(g_{\bm\upsilon}^{\rm ext}\in
C^r(\mathcal X_c\times\overline{I(\bm\upsilon)})\). It agrees with
\(g\) on \(\mathcal X_c\times I(\bm\upsilon)\). This proves the
blockwise \(C^r\) extension statement.

We now verify the one-sided derivative limits. Let
\(\y_0\in\mathcal Y^\bullet\), and let
\(\bm\upsilon_0\) be the unique block index satisfying
\(\y_0\in I(\bm\upsilon_0)\). Consider any sequence
\[
    (\x_n^{(c)},\y_n)\to(\x_0^{(c)},\y_0),
    \qquad
    \y_n\in I(\bm\upsilon_0).
\]
For any multi-index \(\alpha=(\alpha_c,0)\) with \(|\alpha|\le r\),
the interior formula gives
\[
    \partial^\alpha g(\x_n^{(c)},\y_n)
    =
    \frac{
    \partial^{\alpha_c}_{\x^{(c)}}f(\x_n^{(c)},\bm\upsilon_0)}
    {|I(\bm\upsilon_0)|}.
\]
By continuity of
\(\partial^{\alpha_c}_{\x^{(c)}}f(\cdot,\bm\upsilon_0)\) on
\(\mathcal X_c\), this converges to
\[
    \frac{
    \partial^{\alpha_c}_{\x^{(c)}}f(\x_0^{(c)},\bm\upsilon_0)}
    {|I(\bm\upsilon_0)|}.
\]

Similarly, if \(\bm\upsilon'\neq\bm\upsilon_0\) satisfies
\(\y_0\in\overline{I(\bm\upsilon')}\), then along any sequence
\[
    (\x_n^{(c)},\y_n)\to(\x_0^{(c)},\y_0),
    \qquad
    \y_n\in I(\bm\upsilon'),
\]
we have
\[
    \partial^\alpha g(\x_n^{(c)},\y_n)
    =
    \frac{
    \partial^{\alpha_c}_{\x^{(c)}}f(\x_n^{(c)},\bm\upsilon')}
    {|I(\bm\upsilon')|}
    \to
    \frac{
    \partial^{\alpha_c}_{\x^{(c)}}f(\x_0^{(c)},\bm\upsilon')}
    {|I(\bm\upsilon')|}.
\]
The two one-sided limits coincide exactly when the two displayed
blockwise derivative values are equal. This proves the jump
statement for \(\x^{(c)}\)-derivatives.

Finally, because each block restriction admits a \(C^r\) extension
to its closed block, the family of block restrictions defines a
piecewise \(C^r\) structure on
\(\widetilde{\mathcal X}=\mathcal X_c\times\mathcal Y\). The preceding
one-sided limit calculation shows that these extensions need not
match on shared block boundaries.
\end{prf}

\subsection{Proof of Theorem~\ref{thm:conditional_distribution}}
\label{app:proof_conditional_distribution}

\begin{thm}[Precise version of Theorem~\ref{thm:conditional_distribution}]
\label{thm:conditional_distribution_precision}
Let \(\widetilde\X=(\X^{(c)},\Y)\) be the continualized covariate vector,
and write a generic point as
\[
    \tilde\x=(\x^{(c)},\y)\in\mathcal X_c\times\mathcal Y,
    \qquad
    \mathcal Y=\prod_{j=1}^{\tilde d}\mathcal Y_j .
\]
For each \(\bm\upsilon\in\varUpsilon\), let \(I(\bm\upsilon)\)
be the associated covariate block. Under the left-open right-closed
convention, there is a unique block index
\[
    \mathrm{cat}(\y) = \bm \upsilon \in\varUpsilon
\]
such that \(\y\in I(\mathrm{cat}(\y))\). Then the following statements hold.

\begin{enumerate}[(i)]
    \item \textbf{Continuous response variable.}
    Suppose \(Z\) is continuous on \(\mathcal Z\subset\mathbb R\). Assume that for every
    \(\bm\upsilon\in\varUpsilon\), the original conditional
    distribution function
    \[
        (z,\x^{(c)})\mapsto F(z\mid\x^{(c)},\bm\upsilon)
    \]
    belongs to \(C^2(\mathcal Z\times\mathcal X_c)\), with
    second-order partial derivatives uniformly continuous and bounded
    on \(\mathcal Z\times\mathcal X_c\), uniformly in
    \(\bm\upsilon\). Then the lifted conditional distribution
    \[
        F(z\mid\x^{(c)},\y)
        :=
        \mathbb P(Z\le z\mid \X^{(c)}=\x^{(c)},\Y=\y)
    \]
    satisfies
    \[
        F(z\mid\x^{(c)},\y)
        =
        \sum_{\bm\upsilon\in\varUpsilon}
        F(z\mid\x^{(c)},\bm\upsilon)
        \mathbb I\{\y\in I(\bm\upsilon)\}.
    \]
    Consequently, for each \(\bm\upsilon\in\varUpsilon\),
    \[
        (z,\x^{(c)},\y)\mapsto F(z\mid\x^{(c)},\y)
        \in
        C^2\bigl(
        \mathcal Z\times\mathcal X_c\times I(\bm\upsilon)^\circ
        \bigr)
    \]
    in the relative product sense. On this block, \(F\) is independent
    of \(\y\), and hence all derivatives involving at least one
    derivative with respect to \(\y\) vanish identically.

    \item \textbf{Discrete response variable after continualization.}
    Suppose \(Z\) is ordered discrete on
    \[
        \mathcal Z=\{z_1,\ldots,z_L\},
        \qquad
        z_1<\cdots<z_L.
    \]
    Let
    \[
        z_0:=z_1-1,\qquad
        I_i=(z_{i-1},z_i],
        \qquad
        \widetilde{\mathcal Z}=\bigcup_{i=1}^L I_i,
    \]
    and let \(\tilde Z=Z-V\), where, conditional on $Z=z_i$, $V\sim\mathrm{Unif}(0,|I_i|)$ independently of $\widetilde\X$. Assume that,
    for every \(i\in[L]\) and \(\bm\upsilon\in\varUpsilon\), the
    conditional mass
    \[
        \x^{(c)}\mapsto P(z_i\mid\x^{(c)},\bm\upsilon)
    \]
    belongs to \(C^2(\mathcal X_c)\), with second-order partial
    derivatives uniformly continuous and bounded on \(\mathcal X_c\),
    uniformly in \(i\) and \(\bm\upsilon\).

    Then the lifted conditional density of \(\tilde Z\) given
    \((\x^{(c)},\y)\) is
    \[
        f(\tilde z\mid\x^{(c)},\y)
        =
        \sum_{\bm\upsilon\in\varUpsilon}
        \sum_{i=1}^L
        \frac{P(z_i\mid\x^{(c)},\bm\upsilon)}{|I_i|}
        \mathbb I\{\tilde z\in I_i\}
        \mathbb I\{\y\in I(\bm\upsilon)\}.
    \]
    This is a valid conditional density on
    \(\widetilde{\mathcal Z}\). Moreover, for
    \(\tilde z\in I_i\) and \(\y\in I(\bm\upsilon)\),
    \[
        F(\tilde z\mid\x^{(c)},\y)
        =
        \sum_{k=1}^{i-1}P(z_k\mid\x^{(c)},\bm\upsilon)
        +
        \frac{\tilde z-z_{i-1}}{|I_i|}
        P(z_i\mid\x^{(c)},\bm\upsilon).
    \]
    Consequently, for every \(i\in[L]\) and
    \(\bm\upsilon\in\varUpsilon\),
    \[
        (\tilde z,\x^{(c)},\y)\mapsto
        F(\tilde z\mid\x^{(c)},\y)
        \in
        C^2\bigl(
        I_i^\circ\times\mathcal X_c\times I(\bm\upsilon)^\circ
        \bigr).
    \]
    On each such block, \(F\) is independent of \(\y\), and is exactly
    affine in \(\tilde z\). In particular,
    \[
        \frac{\partial^2 F}{\partial \tilde z^2}\equiv 0,
        \qquad
        \frac{\partial^2 F}
        {\partial \tilde z\,\partial x_a^{(c)}}
        =
        \frac{1}{|I_i|}
        \frac{\partial}{\partial x_a^{(c)}}
        P(z_i\mid\x^{(c)},\bm\upsilon)
    \]
    on \(I_i^\circ\times\mathcal X_c\times I(\bm\upsilon)^\circ\).
    Across a response interface \(\tilde z=z_i\), \(1\le i<L\),
    the lifted CDF is continuous in \(\tilde z\), while its
    \(\tilde z\)-derivative may jump from
    \[
        \frac{P(z_i\mid\x^{(c)},\bm\upsilon)}{|I_i|}
        \quad\text{to}\quad
        \frac{P(z_{i+1}\mid\x^{(c)},\bm\upsilon)}{|I_{i+1}|}.
    \]

    \item \textbf{Covariate jumping set.}
    In both cases, at every covariate boundary point
    \[
        (\x_0^{(c)},\y_0)\in
        \widetilde{\mathcal X}^{\bullet}
        =
        \mathcal X_c\times\mathcal Y^\bullet,
    \]
    the lifted conditional CDF \(F\) admits blockwise one-sided
    \(C^2\) limits from every covariate block
    \(I(\bm\upsilon)\) whose closure contains \(\y_0\). The values
    of \(F\) and its blockwise \(\x^{(c)}\)-derivatives on adjacent
    covariate blocks need not agree.
\end{enumerate}
\end{thm}

\begin{prf}
We first record the basic conditioning identity induced by the
continualization. Since \(\X^{(d)}=\mathrm{cat}(\Y)\) is a measurable
function of \(\Y\), and since the uniform variables used to generate
\(\Y\) are conditionally independent of \(Z\) given
\((\X^{(c)},\X^{(d)})\), the conditional law of \(Z\) given
\((\X^{(c)},\Y)\) depends on \(\Y\) only through its block label.
Thus, for \(\y\in I(\bm\upsilon)\),
\[
    \mathcal L(Z\mid \X^{(c)}=\x^{(c)},\Y=\y)
    =
    \mathcal L(Z\mid \X^{(c)}=\x^{(c)},\X^{(d)}=\bm\upsilon).
\]
Equivalently,
\[
    \mathbb P(Z\le z\mid \X^{(c)}=\x^{(c)},\Y=\y)
    =
    F(z\mid\x^{(c)},\bm\upsilon),
    \qquad
    \y\in I(\bm\upsilon).
\]

\smallskip
\textbf{(i) Continuous response.}
The preceding identity immediately gives
\[
    F(z\mid\x^{(c)},\y)
    =
    \sum_{\bm\upsilon\in\varUpsilon}
    F(z\mid\x^{(c)},\bm\upsilon)
    \mathbb I\{\y\in I(\bm\upsilon)\}.
\]
Fix \(\bm\upsilon\). On
\(\mathcal Z\times\mathcal X_c\times I(\bm\upsilon)^\circ\), this
function equals \(F(z\mid\x^{(c)},\bm\upsilon)\). By assumption this
is \(C^2\) in \((z,\x^{(c)})\), with uniformly continuous and bounded
second derivatives. Since it is independent of \(\y\) on this block,
viewed as a function of \((z,\x^{(c)},\y)\), it belongs to
\[
    C^2\bigl(
    \mathcal Z\times\mathcal X_c\times I(\bm\upsilon)^\circ
    \bigr),
\]
with all \(\y\)-derivatives equal to zero.

\smallskip
\textbf{(ii) Discrete response.}
For \(\y\in I(\bm\upsilon)\), the same conditioning identity gives
\[
    \mathbb P(Z=z_i\mid \X^{(c)}=\x^{(c)},\Y=\y)
    =
    P(z_i\mid\x^{(c)},\bm\upsilon).
\]
Given \(Z=z_i\), the response continualization places \(\tilde Z\)
uniformly on \(I_i\). Therefore, for
\(\tilde z\in I_i\) and \(\y\in I(\bm\upsilon)\),
\[
    f(\tilde z\mid\x^{(c)},\y)
    =
    \frac{P(z_i\mid\x^{(c)},\bm\upsilon)}{|I_i|}.
\]
Summing over \(i\) and \(\bm\upsilon\) yields the displayed density.
It integrates to one because
\[
    \int_{\widetilde{\mathcal Z}}
    f(\tilde z\mid\x^{(c)},\y)\,d\tilde z
    =
    \sum_{i=1}^L
    \int_{I_i}
    \frac{P(z_i\mid\x^{(c)},\bm\upsilon)}{|I_i|}
    \,d\tilde z
    =
    \sum_{i=1}^L
    P(z_i\mid\x^{(c)},\bm\upsilon)
    =
    1
\]
whenever \(\y\in I(\bm\upsilon)\).

For \(\tilde z\in I_i\) and \(\y\in I(\bm\upsilon)\),
\[
\begin{aligned}
    F(\tilde z\mid\x^{(c)},\y)
    &=
    \int_{z_0}^{\tilde z}
    f(s\mid\x^{(c)},\y)\,ds  \\
    &=
    \sum_{k=1}^{i-1}
    \int_{I_k}
    \frac{P(z_k\mid\x^{(c)},\bm\upsilon)}{|I_k|}
    \,ds
    +
    \int_{z_{i-1}}^{\tilde z}
    \frac{P(z_i\mid\x^{(c)},\bm\upsilon)}{|I_i|}
    \,ds  \\
    &=
    \sum_{k=1}^{i-1}
    P(z_k\mid\x^{(c)},\bm\upsilon)
    +
    \frac{\tilde z-z_{i-1}}{|I_i|}
    P(z_i\mid\x^{(c)},\bm\upsilon).
\end{aligned}
\]
This proves the blockwise CDF formula.

On \(I_i^\circ\times\mathcal X_c\times I(\bm\upsilon)^\circ\), the
right-hand side is affine in \(\tilde z\), \(C^2\) in
\(\x^{(c)}\), and independent of \(\y\). Hence it is \(C^2\) on the
block, with all \(\y\)-derivatives equal to zero and with
\[
    \frac{\partial^2 F}{\partial \tilde z^2}=0,
    \qquad
    \frac{\partial^2 F}
    {\partial \tilde z\,\partial x_a^{(c)}}
    =
    \frac{1}{|I_i|}
    \frac{\partial}{\partial x_a^{(c)}}
    P(z_i\mid\x^{(c)},\bm\upsilon).
\]

At a response interface \(\tilde z=z_i\), with \(1\le i<L\), the
left limit from \(I_i\) is
\[
    \sum_{k=1}^{i-1}
    P(z_k\mid\x^{(c)},\bm\upsilon)
    +
    P(z_i\mid\x^{(c)},\bm\upsilon)
    =
    \sum_{k=1}^{i}
    P(z_k\mid\x^{(c)},\bm\upsilon),
\]
while the right limit from \(I_{i+1}\) is
\[
    \sum_{k=1}^{i}
    P(z_k\mid\x^{(c)},\bm\upsilon)
    +
    0
    =
    \sum_{k=1}^{i}
    P(z_k\mid\x^{(c)},\bm\upsilon).
\]
Thus \(F\) is continuous in \(\tilde z\) across the response
interface. Its \(\tilde z\)-derivative equals
\(P(z_i\mid\x^{(c)},\bm\upsilon)/|I_i|\) on \(I_i^\circ\) and
\(P(z_{i+1}\mid\x^{(c)},\bm\upsilon)/|I_{i+1}|\) on
\(I_{i+1}^\circ\), so this derivative may jump at \(z_i\).

\smallskip
\textbf{(iii) Covariate jumping set.}
In case (i), the blockwise representation shows that, on each
covariate block \(I(\bm\upsilon)^\circ\), \(F\) is the
\(C^2\) function \(F(z\mid\x^{(c)},\bm\upsilon)\), independent of
\(\y\). In case (ii), for each fixed response interval \(I_i^\circ\),
the displayed formula expresses \(F\) as a \(C^2\) function of
\(\x^{(c)}\), independent of \(\y\), on every covariate block.

Therefore the blockwise one-sided limits at points of
\(\widetilde{\mathcal X}^{\bullet}\) exist by the same argument as
in Corollary~\ref{cor:Cr_continualization}, applied respectively to
the maps
\[
    \x^{(c)}\mapsto F(z\mid\x^{(c)},\bm\upsilon)
\]
in the continuous-response case and to the maps
\[
    \x^{(c)}\mapsto P(z_i\mid\x^{(c)},\bm\upsilon)
\]
in the discrete-response case. Since these blockwise functions need
not agree for adjacent \(\bm\upsilon\)'s, the values of \(F\) and its
blockwise \(\x^{(c)}\)-derivatives may jump across the covariate
jumping set.
\end{prf}

\subsection{Proof of Theorem~\ref{thm:int_uniform}}
\label{app:proof_int_uniform}
\begin{prf}
\textbf{(i) Continuous response.} For the mixed covariate case, the only additional issue is that the Gaussian kernel has unbounded support and may assign positive weight to observations lying in blocks adjacent to the block containing the query point. We therefore require only the continualized coordinates to stay away from the internal partition interfaces. 
Denote 
\begin{equation*}
\mathcal Y_{\epsilon}^{\circ}
:=
\left\{
\y\in  \mathcal Y^\circ:
\min_{\substack{j\in[\tilde d]\\1\le \ell\le \tau_j-1}}
|y_j-\upsilon_{j,\ell}|
\ge \epsilon
\right\},
\qquad
\widetilde{\mathcal X}_{\epsilon}^{\circ}
:=
\mathcal X_c\times\mathcal Y_{\epsilon}^{\circ},
\end{equation*} where $\epsilon\ge C_\epsilon w_x \sqrt{|\log w_x|}$, for some constant \(C_\epsilon>2\). Notice that no additional distance restriction is imposed on the original continuous coordinates \(\x^{(c)}\). 

Fix \(\bm\upsilon=(\upsilon_{1,\ell_1},\ldots,\upsilon_{\tilde d,\ell_{\tilde d}})\in\varUpsilon\) and
\begin{equation*}
\tilde{\x}=(\x^{(c)},\y)
\in
\bigl(\mathcal X_c\times I(\bm\upsilon)^\circ\bigr)
\cap\widetilde{\mathcal X}_{\epsilon}^{\circ}.    
\end{equation*}
Write \(\mathbf t=(\mathbf t^{(c)},\mathbf t^{(d)})\in\mathbb R^{d - \tilde d}\times\mathbb R^{\tilde d}\), and define
\begin{equation*}
\mathcal A_{\bm\upsilon}(\tilde{\x})
:=
\left\{
\mathbf t:
\y+w_x\mathbf t^{(d)}\in I(\bm\upsilon)
\right\}.    
\end{equation*}
If $
\tilde{\x}+w_x\mathbf t\in\mathcal X_c\times\mathcal Y$ and $\mathbf t\notin\mathcal A_{\bm\upsilon}(\tilde{\x})$, then the continualized component \(\y+w_x\mathbf t^{(d)}\) lies in a block different from \(I(\bm\upsilon)\). Hence at least one internal partition interface must be crossed. By the definition of \(\widetilde{\mathcal X}_{\epsilon}^{\circ}\), this implies
\begin{equation*}
\max_{j\in[\tilde d]}|t_j^{(d)}|
\ge
\frac{\epsilon}{w_x}.
\end{equation*}
Consequently, defining $\mathcal{S}(\tilde{\x}):=\{\mathbf t:
\tilde{\x}+w_x\mathbf t\in\mathcal X_c\times\mathcal Y\} $, we obtain
\begin{equation*}
 \mathcal{S}(\tilde{\x}) \cap\mathcal A^c_{\bm\upsilon}(\tilde{\x})
\subset
\left\{
\mathbf t:
\max_{j\in[\tilde d]}|t_j^{(d)}|
\ge\frac{\epsilon}{w_x}
\right\}.
\end{equation*}
Then, for $\bm\mu = \bm\mu(z,\tilde{\x},w_z,w_x)$ and $\bm\varXi = \bm\varXi(\tilde{\x},w_x)$, we have decompositions of the integrals over $\mathcal{S}(\tilde{\x}) \cap\mathcal A_{\bm\upsilon}(\tilde{\x})$ and $\mathcal{S}(\tilde{\x}) \cap\mathcal A^c_{\bm\upsilon}(\tilde{\x})$ as follows:
\begin{equation*}
    \bm\mu = \bm\mu_{in} + \bm\mu_{out}, \qquad \bm\varXi = \bm\varXi_{in} + \bm\varXi_{out},
\end{equation*} where the subscripts "in" and "out" denote the integrals over the sets $\mathcal{S}(\tilde{\x}) \cap\mathcal A_{\bm\upsilon}(\tilde{\x})$ and $\mathcal{S}(\tilde{\x}) \cap\mathcal A^c_{\bm\upsilon}(\tilde{\x})$, respectively.

On $\mathcal{S}(\tilde{\x}) \cap\mathcal A_{\bm\upsilon}(\tilde{\x})$, both $\y$ and $\y+w_x\mathbf t^{(d)}$ belong to the same continualization block. By Corollary~\ref{cor:Cr_continualization} and Theorem~\ref{thm:conditional_distribution},

\begin{equation*}
f(\x^{(c)},\y)
=
\frac{f(\x^{(c)},\bm\upsilon)}{|I(\bm\upsilon)|},
\qquad
F(z\mid\x^{(c)},\y)
=
F(z\mid\x^{(c)},\bm\upsilon),
\end{equation*} throughout this block. Hence \(f\) and \(F\) have the required \(C^2\) regularity there, and all derivatives of \(F\) with respect to \(\y\) vanish. Therefore, for $\mathbf t\in\mathcal{S}(\tilde{\x}) \cap\mathcal A_{\bm\upsilon}(\tilde{\x})$, the covariate Taylor expansion in Eq.~\eqref{eq:Taylor_expansion} used in the proof of Lemma~\ref{lem:simple_bias} remains valid:
\begin{equation*}
F(z\mid\tilde{\x}+w_x\mathbf t)
=
F(z\mid\tilde{\x})
+
w_x\nabla_{\tilde{\x}}^\top F(z\mid\tilde{\x})\,\mathbf t
+
\frac{w_x^2}{2}
\mathbf t^\top
\nabla_{\tilde{\x}}^\top\nabla_{\tilde{\x}}
F(z\mid\tilde{\x})\,\mathbf t
+
R_x(z,\tilde{\x},\mathbf t).
\end{equation*} Thus the argument leading to Eq.~\eqref{eq:mu} applies without modification to $\bm\mu_{in}$.

It remains to control $\bm\mu_{out}$ via the contribution from the cross-block region $\left\{\mathbf t:\max_{j\in[\tilde d]}|t_j^{(d)}|\ge\frac{\epsilon}{w_x}\right\}$. Since $0\le \widetilde F\le1$, the continualized density $f$ is uniformly bounded over the finitely many blocks, and
\(\|\bm\iota(\mathbf t)\|\lesssim1+\|\mathbf t\|\), the contribution to \(\bm\mu(z,\tilde{\x},w_z,w_x)\) from observations outside the query block is bounded uniformly by
\begin{equation*}
C_f
\int_{\max_{j\in[\tilde d]}|t_j^{(d)}|
\ge \epsilon/w_x}
(1+\|\mathbf t\|)\bm\kappa(\mathbf t)\,d\mathbf t,
\end{equation*} where $C_f$ is an upper bound on the density function.
Similarly, the corresponding contribution of $\bm\varXi_{out}$ to
$\bm\varXi(\tilde{\x},w_x)$ is bounded by
\begin{equation*}
C_f
\int_{\max_{j\in[\tilde d]}|t_j^{(d)}|
\ge \epsilon/w_x}
(1+\|\mathbf t\|)^2\bm\kappa(\mathbf t)\,d\mathbf t.
\end{equation*}

Since \(\bm\kappa\) is the product Gaussian kernel and \(d\) is fixed, for \(m=1,2\),
    \begin{align*}
        \int_{\max_{j\in[\tilde d]}|t_j^{(d)}|\ge a}
        (1+\|\mathbf t\|)^m\bm\kappa(\mathbf t)\,d\mathbf t
        &\le
        2 \sum_{j=1}^{\tilde d} \mathbb E[(1 + \|\mathbf T\|^2) \mathbb I\{|T^{(d)}_j|\ge a\}] \\
        &\le 
        \frac{4\tilde{d}(d+2)}{\sqrt{2\pi}}(1+a)e^{-a^2/2},\ \forall a \ge 1,
\end{align*} where $\mathbf T=(\mathbf T^{(c)},\mathbf T^{(d)})\sim N(0,\mathbf I_d)$ and $\mathbf T^{(d)}= (T_1^{(d)},\ldots,T_{\tilde d}^{(d)})$; the first inequality follows from $2x \le 1+x^2$; the last inequality follows from $\mathbb E\!\left[V^2\mathbb I\{|V|\ge a\}\right]
=2\{a\phi(a)+1-\varPhi(a)\}$ and the Gaussian Mills bound $1-\varPhi(a)\le\phi(a)/a$.

Taking $a=\epsilon/w_x$ and using $\epsilon \ge C_\epsilon w_x \sqrt{|\log w_x|}$, we have 
\begin{align*}
\int_{\max_j|t_j^{(d)}|\ge\epsilon/w_x}
(1+\|\mathbf t\|)^m\bm\kappa(\mathbf t)\,d\mathbf t
&\le
\frac{4\tilde{d}(d+2)}{\sqrt{2\pi}}(1+C_\epsilon\sqrt{|\log w_x|})w_x^{C_\epsilon^2/2}\\
&\overset{C_\epsilon > 2}{=} o(w_x^2),
\qquad m=1,2,
\end{align*} where the first inequality follows from the fact that $g(a)=(1+a)\exp(-a^2/2)$ is monotonically decreasing for $a\ge1$. When $w_x$ is sufficiently small, $a\ge a_0:=C_\epsilon\sqrt{|\log w_x|}\ge1$, and hence $g(a)\le g(a_0)$.
Hence, the total cross-block contribution is \(o(w_x^2)\), uniformly over \(\tilde{\x}\in\widetilde{\mathcal X}_{\epsilon}^{\circ}\).

Accordingly, the expansion in Eq.~\eqref{eq:mu} applied to the within-block integral, together with the above cross-block bounds, gives
\begin{equation*}
\bm\mu
=
\bm\mu_{in}+\bm\mu_{out}
=
\bm\varXi_{in}\D_{w_x}\bm\beta^\star(z,\tilde{\x})
+O(w_x^2+w_z^2)+o(w_x^2)
=
\bm\varXi\D_{w_x}\bm\beta^\star(z,\tilde{\x})
+O(w_x^2+w_z^2),
\end{equation*}
uniformly over $z\in\mathbb R$ and $\tilde{\x}\in\mathcal X_{\epsilon}^{\circ}$, where the last equality follows from $\bm\varXi_{out}=o(w_x^2)$ and $|\mathbf{D}_{w_x}\bm\beta^\star(z,\tilde{\x})|=O(1)$ uniformly.

Moreover, since the continualized density is uniformly bounded away from zero over the finitely many blocks, the eigenvalue argument of Lemma~\ref{lem:eig_bound}, applied to the full continualized support, yields
\begin{equation*}
\sup_{\tilde{\x}\in\widetilde{\mathcal X}_{\epsilon}^{\circ}}
|\bm\varXi^{-1}(\tilde{\x},w_x)|=O(1).
\end{equation*}
Therefore,
\begin{equation*}
\sup_{\substack{z\in\mathbb R\\
\tilde{\x}\in\widetilde{\mathcal X}_{\epsilon}^{\circ}}}
\left|
\beta_0^\diamond(z,\tilde{\x},w_z,w_x)
-\beta_0^\star(z,\tilde{\x})
\right|=
O(w_x^2+w_z^2).
\end{equation*}

Finally, Lemma~\ref{lem:simple_stochastic} remains valid on
\(\mathcal X_{\epsilon}^{\circ}\), since this set is a subset of the original index set and the stochastic argument does not require differentiation across the partition interfaces. Therefore,
\begin{equation*}
\sup_{\substack{
z\in\mathbb R\\
\tilde{\x}\in\widetilde{\mathcal X}_{\epsilon}^{\circ}
}}
\bigl|
\hat F(z\mid\tilde{\x})
-
F(z\mid\tilde{\x})
\bigr|
=
O_p\!\left(
w_x^2+w_z^2+
\sqrt{\frac{|\log w_x|}{M w_x^d}}
\right).
\end{equation*}

\textbf{(ii) Discrete response after continualization.} Similarly, we need to avoid the interface points in \(\{z_0,\ldots,z_L\}\). Define
\begin{equation*}
\widetilde{\mathcal Z}_{\epsilon'}^{\circ}
:=
\left\{
\tilde z\in\widetilde{\mathcal Z}^{\circ}:
\min_{0\le i\le L}
|\tilde z-z_i|
\ge\epsilon'
\right\}.
\end{equation*}
Suppose
\begin{equation*}
\epsilon'\ge C_{\epsilon'}w_z\sqrt{|\log w_z|},
\qquad C_{\epsilon'}>2.
\end{equation*}
For any \(\tilde z\in\widetilde{\mathcal Z}_{\epsilon'}^{\circ}\), let \(I_i\) denote the continualization interval containing \(\tilde z\), and set \(a=\epsilon'/w_z\). By the definition of \(\widetilde{\mathcal Z}_{\epsilon'}^{\circ}\), whenever \(|v|\le a\), both \(\tilde z\) and \(\tilde z-w_zv\) belong to \(I_i\). By Theorem~\ref{thm:conditional_distribution}, \(F(\cdot\mid\tilde{\x})\) is affine on \(I_i\), with slope \(p_i(\tilde{\x})/|I_i|\). Hence
\[
F(\tilde z-w_zv\mid\tilde{\x})
-F(\tilde z\mid\tilde{\x})
=-w_zv\frac{p_i(\tilde{\x})}{|I_i|},
\qquad |v|\le a.
\]
By the symmetry of the Gaussian kernel,
\[
\int_{|v|\le a}
\bigl\{F(\tilde z-w_zv\mid\tilde{\x})
-F(\tilde z\mid\tilde{\x})\bigr\}
\kappa(v)\,dv =
-w_z\frac{p_i(\tilde{\x})}{|I_i|}
\int_{-a}^{a}v\kappa(v)\,dv=0.
\]
Thus, the bias from the symmetric region $|v|\le a$ vanishes exactly. Using \(0\le F\le1\) and referring to Eq.~\eqref{eq:response_direction}, we bound the remaining contribution as follows:
\[
\begin{aligned}
\left|\widetilde F(\tilde z\mid\tilde{\x})
-F(\tilde z\mid\tilde{\x})\right|
&\quad=
\left|\int_{|v|>a}
\bigl\{F(\tilde z-w_zv\mid\tilde{\x})
-F(\tilde z\mid\tilde{\x})\bigr\}
\kappa(v)\,dv\right|\\
&\quad\le
2\{1-\varPhi(a)\}
\le \frac{2}{a\sqrt{2\pi}}e^{-a^2/2}\le
\frac{2}{C_{\epsilon'}\sqrt{2\pi|\log w_z|}}
w_z^{C_{\epsilon'}^2/2}
=o(w_z^2),
\end{aligned}
\]
where the last inequality follows from
\(\epsilon'\ge C_{\epsilon'}w_z\sqrt{|\log w_z|}\) with \(C_{\epsilon'}>2\). 
The response-smoothing argument in fact holds uniformly over all continualized covariate blocks, since it only uses the piecewise-affine structure of \(F\) in the response direction and \(0\le F\le1\). Hence,
\[
\sup_{\substack{
\tilde z\in\widetilde{\mathcal Z}_{\epsilon'}^\circ\\
\tilde{\x}\in\widetilde{\mathcal X}
}}
\left|
\widetilde F(\tilde z\mid\tilde{\x})
-
F(\tilde z\mid\tilde{\x})
\right|
=o(w_z^2).
\]

Substituting the preceding response-smoothing bound into the integral expression for \(\bm\mu\), we write
\[
\begin{aligned}
\bm\mu-\bm\varXi\D_{w_x}\bm\beta^\star(\tilde z,\tilde{\x})
&=
\underbrace{
\int_{\mathcal S(\tilde{\x})}
\bm\iota(\mathbf t)\bm\kappa(\mathbf t)
f(\tilde{\x}+w_x\mathbf t)
\Bigl\{
\widetilde F(\tilde z\mid\tilde{\x}+w_x\mathbf t)
-
F(\tilde z\mid\tilde{\x}+w_x\mathbf t)
\Bigr\}\,d\mathbf t
}_{=:R_z}\\
&\quad+
\underbrace{
\int_{\mathcal S(\tilde{\x})}
\bm\iota(\mathbf t)\bm\kappa(\mathbf t)
f(\tilde{\x}+w_x\mathbf t)
\Bigl\{
F(\tilde z\mid\tilde{\x}+w_x\mathbf t)
-
\bm\iota(\mathbf t)^\top
\D_{w_x}\bm\beta^\star(\tilde z,\tilde{\x})
\Bigr\}\,d\mathbf t
}_{=:R_x}.
\end{aligned}
\]
Since \(f\) is uniformly bounded and
\(\|\bm\iota(\mathbf t)\|\lesssim1+\|\mathbf t\|\), it follows that
\[
\|R_z\|
\le
C
\int_{\mathbb R^d}
(1+\|\mathbf t\|)\bm\kappa(\mathbf t)\,d\mathbf t
\cdot o(w_z^2)
=
o(w_z^2).
\]

For \(R_x\), decompose the integral over
\(\mathcal S(\tilde{\x})\cap\mathcal A_{\bm\upsilon}(\tilde{\x})\)
and
\(\mathcal S(\tilde{\x})\cap\mathcal A_{\bm\upsilon}^c(\tilde{\x})\).
On the within-block region, the covariate Taylor expansion gives
\[
F(\tilde z\mid\tilde{\x}+w_x\mathbf t)
-
\bm\iota(\mathbf t)^\top
\D_{w_x}\bm\beta^\star(\tilde z,\tilde{\x})
=
O(w_x^2\|\mathbf t\|^2),
\]
uniformly, and therefore its contribution to \(R_x\) is \(O(w_x^2)\).
On the cross-block region, the integrand is uniformly bounded by a polynomial in \(\|\mathbf t\|\), while
\[
\mathcal S(\tilde{\x})\cap
\mathcal A_{\bm\upsilon}^c(\tilde{\x})
\subset
\left\{
\mathbf t:
\max_{j\in[\tilde d]}|t_j^{(d)}|
\ge\frac{\epsilon}{w_x}
\right\}.
\]
The Gaussian-tail bound established in part~(i) therefore gives an
\(o(w_x^2)\) contribution. Consequently,
\[
R_x=O(w_x^2)+o(w_x^2)=O(w_x^2),
\]
and hence
\[
\bm\mu-\bm\varXi\D_{w_x}\bm\beta^\star(\tilde z,\tilde{\x})
=
O(w_x^2)+o(w_z^2)
\]
uniformly over
\(\tilde z\in\widetilde{\mathcal Z}_{\epsilon'}^\circ\) and
\(\tilde{\x}\in\widetilde{\mathcal X}_{\epsilon}^\circ\).

By the population normal equation,
\[
\D_{w_x}
\bigl\{
\bm\beta^\diamond(\tilde z,\tilde{\x},w_z,w_x)
-
\bm\beta^\star(\tilde z,\tilde{\x})
\bigr\}
=
\bm\varXi^{-1}
\bigl\{
\bm\mu-
\bm\varXi\D_{w_x}\bm\beta^\star(\tilde z,\tilde{\x})
\bigr\}.
\]
Since \(\bm\varXi^{-1}\) is uniformly bounded and the first diagonal entry of \(\D_{w_x}\) equals one,
\[
\sup_{\substack{
\tilde z\in\widetilde{\mathcal Z}_{\epsilon'}^{\circ}\\
\tilde{\x}\in\widetilde{\mathcal X}_{\epsilon}^{\circ}
}}
\left|
\beta_0^\diamond(\tilde z,\tilde{\x},w_z,w_x)
-
F(\tilde z\mid\tilde{\x})
\right|
=
O(w_x^2)+o(w_z^2).
\]
Combining this oracle bias bound with the same stochastic bound as in part (i) yields
\begin{equation*}
\sup_{\substack{
\tilde z\in\widetilde{\mathcal Z}_{\epsilon'}^{\circ}\\
\tilde{\x}\in\widetilde{\mathcal X}_{\epsilon}^{\circ}
}}
\left|
\hat F(\tilde z\mid\tilde{\x})
-
F(\tilde z\mid\tilde{\x})
\right|
=
O_p\!\left(
w_x^2+w_z^2+
\sqrt{\frac{|\log w_x|}{M w_x^d}}
\right).
\end{equation*}
\end{prf}

\subsection{Proof of Corollary~\ref{cor:mass_uniform}}
\label{app:proof_mass_uniform}

\begin{prf}
Fix $i$ and $\tilde\x\in\mathcal X_\epsilon^\circ$.
By Theorem~\ref{thm:conditional_distribution}(ii),
$F(\cdot\mid\tilde\x)$ is affine on
$I_i\cap\widetilde{\mathcal Z}_{\epsilon'}^\circ$
with slope $p_i(\tilde\x)/|I_i|$. Hence
\begin{equation*}
    F(\tilde z_{i,+}^\circ\mid\tilde\x)
    -
    F(\tilde z_{i,-}^\circ\mid\tilde\x)
    =
    \frac{\tilde z_{i,+}^\circ-\tilde z_{i,-}^\circ}
    {|I_i|}p_i(\tilde\x),
\end{equation*}
giving the exact identity
\begin{equation*}
    p_i(\tilde\x)
    =
    \frac{|I_i|}
    {\tilde z_{i,+}^\circ-\tilde z_{i,-}^\circ}
    \Bigl(
        F(\tilde z_{i,+}^\circ\mid\tilde\x)
        -
        F(\tilde z_{i,-}^\circ\mid\tilde\x)
    \Bigr).
\end{equation*}
Subtracting,
\begin{equation}\label{eq:phat_decomp}
    \widehat p_i(\tilde\x)-p_i(\tilde\x)
    =
    \frac{|I_i|}
    {\tilde z_{i,+}^\circ-\tilde z_{i,-}^\circ}
    \Bigl[ (\widehat F-F)(\tilde z_{i,+}^\circ\mid\tilde\x)-(\widehat F-F)(\tilde z_{i,-}^\circ\mid\tilde\x)\Bigr].
\end{equation}
By the separation condition~\eqref{eq:separation},
$|I_i|/(\tilde z_{i,+}^\circ-\tilde z_{i,-}^\circ)
\le c^{-1}$ is uniformly bounded in $i$.

By the refined bias bound established in the proof of
Theorem~\ref{thm:int_uniform}(ii), uniformly over
\(\tilde z\in\widetilde{\mathcal Z}_{\epsilon'}^\circ\) and
\(\tilde\x\in\widetilde{\mathcal X}_\epsilon^\circ\),
\[
\bigl|\widehat F(\tilde z\mid\tilde\x)
-F(\tilde z\mid\tilde\x)\bigr|
=
O_p\!\left(
w_x^2+
\sqrt{\frac{|\log w_x|}{Mw_x^d}}
\right)
+o(w_z^2).
\]
Since
\(\tilde z_{i,\pm}^\circ\in\widetilde{\mathcal Z}_{\epsilon'}^\circ\)
and \(L\) is fixed,
\[
\max_{i\in[L],\,\tilde z_{i,\pm}^\circ}
\sup_{\tilde\x\in\mathcal X_\epsilon^\circ}
\left|
(\widehat F-F)
(\tilde z_{i,\pm}^\circ \mid\tilde\x)
\right|
=
O_p\!\left(
w_x^2+
\sqrt{\frac{|\log w_x|}{Mw_x^d}}
\right)
+o(w_z^2).
\]
Combining this bound with \eqref{eq:phat_decomp} and
\eqref{eq:separation} proves \eqref{eq:mass_uniform}.

For
\(\tilde\x\in\widetilde{\mathcal X}_\epsilon^\bullet\),
the same result follows by evaluating \(\widehat p_i\) at the
corresponding interior representative
\(\tilde\x_{\epsilon}^\circ\), since the conditional
probability masses are invariant in \(\y\) within each
continualization block.
\end{prf}

\section{Experimental settings}

\subsection{Data usage modes for downstream learning}
\label{subsec:integration_modes}

A federated synthetic dataset is most useful when it can be
combined with the original local data to support downstream
learning tasks at each client. This subsection formalizes four
data usage modes for downstream training that differ in how
synthetic and local data is weighted in the empirical loss.
These four modes are not algorithmic variants of the
synthetic data generator itself but four well-defined ways of
\emph{using} the generated data in downstream learning; they
constitute the comparison framework that
Section~\ref{sec:experiments} evaluates empirically.

Throughout this subsection, let $\mathcal B=\{\hat{\bm x}_b\}_{b=1}^B$ denote the synthetic dataset
generated by Algorithm~\ref{alg:chain_generation} and shared
across all clients, and let $\mathcal H^{(k)}$ denote
client $k$'s local dataset. We consider a generic downstream
loss $\ell$ indexed by a parameter $\theta$ (for example, a
squared-error loss for linear regression or a negative log-likelihood
loss for logistic regression), and define the four data usage 
modes as follows.

\textbf{M1: synthetic-only learning.}
Client $k$ fits its downstream model entirely on the shared
synthetic dataset:
\begin{equation}\label{eq:mode_syn_only}
    \widehat\theta^{(k)}_{\rm syn}
    =
    \arg\min_{\theta}
    \sum_{\hat{\x}_b\in\mathcal B}
    \ell(\theta;\hat{\x}_b).
\end{equation}
This mode is the natural baseline for clients with no or
extremely limited local data. It captures all the information
that the federated framework has aggregated about the joint
law, but it does not exploit any client-specific signal.

\textbf{M2: local-only learning.}
Client $k$ fits the downstream model on its local data alone:
\begin{equation}\label{eq:mode_pri_only}
    \widehat\theta^{(k)}_{\rm loc}
    =
    \arg\min_{\theta}
    \sum_{\h_i\in\mathcal H^{(k)}}
    \ell(\theta;\h_i).
\end{equation}
This is the standard non-federated baseline: no benefit is
drawn from the synthetic data, and each client's estimator is
governed by its local sample size $m_k$.

\textbf{M3: mixed-data learning.}
Client $k$ pools a subset of the synthetic observations ${\mathcal B}_{part}$ and $|{\mathcal B}_{part}| = r |\mathcal{H}^{(k)}|,r\in[0,\infty]$, whose size is controlled by $r$ relative to the local sample size, with its local dataset to form a single training set with uniform weight per sample.
\begin{equation}\label{eq:mode_mixed}
    \widehat\theta^{(k)}_{\rm mix}
    =
    \arg\min_{\theta}
    \biggl\{
    \sum_{\h_i\in\mathcal H^{(k)}}
    \ell(\theta;\h_i)
    +
    \sum_{\hat{\x}_b\in {\mathcal B}_{part}}
    \ell(\theta;\hat{\x}_b)
    \biggr\}.
\end{equation}
Each synthetic and local sample contributes equally to the empirical loss, but the number of synthetic observations is controlled. 
Note that the parameter $r = \infty$ means that all synthetic data is used.

\textbf{M4: loss-reweighted synthetic-local learning.}
Client $k$ combines synthetic and local data with a tunable
weight $\lambda\in[0,+\infty)$ on the synthetic loss:
\begin{equation}\label{eq:mode_reweight}
    \widehat\theta^{(k)}_{\rm rew}(\lambda)
    =
    \arg\min_{\theta}
    \biggl\{
    \sum_{\h_i\in\mathcal H^{(k)}}
    \!\!
    \ell(\theta;\h_i)
    +
    \lambda
    \!\!
    \sum_{\hat{\x}_b \in{\mathcal B}}
    \!\!
    \ell(\theta;\hat{\x}_b)
    \biggr\}.
\end{equation}
The weight $\lambda$ controls the relative influence of the two
data sources. The special case $\lambda=0$ recovers Mode~2, whereas $\lambda=1$ recovers Mode~3 with the full synthetic dataset.

The four modes differ in how synthetic data replaces or complements local data. Mode~1 completely replaces local data with synthetic data; Mode~2 ignores the synthetic data entirely; and Mode~3 controls the relative contribution of synthetic data through the number of synthetic observations selected by $r$ relative to $m_k$. Mode~4 controls the trade-off explicitly at the loss-function level through $\lambda$.

\subsection{Simulation data}\label{subsec:cov_ship}

The simulated covariates are generated as follows. First,
\[
X_{B,i}\sim \mathrm{Bernoulli}(\pi),\qquad \pi=0.45.
\]
Conditional on \(X_{B,i}\), the pair \((X_{G_1,i},X_{G_2,i})\) is sampled from a bivariate Gaussian distribution:
\[
(X_{G_1,i},X_{G_2,i})^\top\mid X_{B,i}=0
\sim
N\left(
\begin{pmatrix}-1.0\\0.5\end{pmatrix},
\begin{pmatrix}1.0&0.6\\0.6&1.5\end{pmatrix}
\right),
\]
\[
(X_{G_1,i},X_{G_2,i})^\top\mid X_{B,i}=1
\sim
N\left(
\begin{pmatrix}1.2\\-0.7\end{pmatrix},
\begin{pmatrix}1.6&-0.4\\-0.4&0.8\end{pmatrix}
\right).
\]
The independent continuous covariate is generated as
\[
X_{G_3,i}\sim N(0,1).
\]
The ordinal covariate $X_{O,i}\in\{1,2,3,4\}$ is obtained by discretizing the latent variable
\[
Z_i=0.65X_{B,i}+0.45X_{G_1,i}-0.35X_{G_2,i}+\varepsilon_i,
\qquad
\varepsilon_i\sim N(0,1),
\]
using cut points (-0.8,0.25,1.1):
\[
X_{O,i}=1+\sum_{c\in\{-0.8,0.25,1.1\}}\mathbf 1\{Z_i>c\}.
\]
Finally, with $X_{O,i}^c=X_{O,i}-2.5$, the Laplace covariate is generated as
\[
X_{L,i}\sim \mathrm{Laplace}(\mu_i,s_i),
\]
where $\mu_i=0.40X_{B,i}+0.25X_{G_1,i}-0.15X_{G_2,i}+0.20X_{O,i}^c$, and
$
s_i=\exp(-0.10+0.25X_{B,i}+0.08X_{O,i}^c).
$
Thus $X_B$ drives the distribution of $(X_{G_1},X_{G_2})$, the ordinal variable $X_O$ depends on $(X_B,X_{G_1},X_{G_2})$, and $X_L$ depends on $(X_B,X_{G_1},X_{G_2},X_O)$, while $X_{G_3}$ is generated independently.

\begin{table}[ht]
\centering
\caption{Chain order, task-specific parameters, and evaluation metrics.}
\label{tab:task_chain_parameters}
\begin{tabular}{p{0.26\textwidth}p{0.34\textwidth}p{0.30\textwidth}}
\specialrule{0.8pt}{0pt}{1.5pt}
\specialrule{0.8pt}{0pt}{3pt}
Component & Specification & Description \\
\specialrule{0.8pt}{3pt}{1.5pt}
\specialrule{0.8pt}{0pt}{3pt}
Covariates
& $(X_B,X_{G_1},X_{G_2},X_O,X_L,X_{G_3})$ 
& {\small Chain order used to generate the synthetic covariate table.} \\

Linear task 
& $(X_B,X_{G_1},X_{G_2},X_O,X_L,X_{G_3},Y_{\rm lin})$ 
& {\small The response $Y_{\rm lin}$ is generated after the covariates.} \\

Logistic task 
& $(X_B,X_{G_1},X_{G_2},X_O,X_L,X_{G_3},Y_{\rm log})$ 
& {\small The response $Y_{\rm log}$ is generated after the same covariates.} \\

Parameter order 
& $(\mathrm{intercept},X_B,X_{G_1},X_{G_2},X_O,X_L,$ $X_{G_3})$ 
& {\small Coefficients are aligned with this ordering.} \\

True linear-model parameter $\beta^\star$
& $(3.75,1.5,1.2,-0.8,-1.1,0.5,0.7)^\top$ 
& {\small Ground-truth coefficient vector for the linear response model.} \\

True logistic-model parameter $\beta^\star$
& $(0,1.5,1.2,-0.8,-1.1,0.5,0.7)^\top$ 
& {\small Ground-truth coefficient vector for the logistic response model.} \\

Test MSE
& Test Mean Squared Error
& {\small Reported only for the downstream regression tasks.} \\

Parameter metric 
& $\|\widehat\beta-\beta^\star\|_2$
& {\small Euclidean error of the estimated parameter vector.} \\

Accuracy metric 
& Test accuracy 
& {\small Reported only for the downstream classification tasks.} \\
\specialrule{0.8pt}{3pt}{1.5pt}
\specialrule{0.8pt}{0pt}{0pt}
\end{tabular}
\end{table}

\subsection{Real data}\label{app:realdata_setting}
\begin{table}[H]
\centering
\caption{Preprocessing and variable-type assignments for the UCI datasets.}
\label{tab:uci_preprocessing}
\resizebox{\textwidth}{!}{
\begin{tabular}{lll}
\specialrule{0.8pt}{0pt}{1.5pt}
\specialrule{0.8pt}{0pt}{3pt}
 & Bike Sharing & Adult Income \\
\specialrule{0.8pt}{3pt}{1.5pt}
\specialrule{0.8pt}{0pt}{3pt}
Response
& $Y=$\texttt{cnt}, standardized
& $Y=\mathbf{1}\{\texttt{income}>\texttt{50K}\}$ \\
Full sample size
& $17{,}379$
& $45{,}222$ \\
Train/test split
& $13{,}903/3{,}476$
& $36{,}177/9{,}045$ \\
\midrule
Variable mapping
& 
\begin{tabular}[t]{@{}l@{\;}l@{}}
$X_1$: & \texttt{temp} \\
$X_2$: & \texttt{atemp} \\
$X_3$: & \texttt{hum} \\
$X_4$: & \texttt{windspeed} \\
$X_5$: & \texttt{season} \\
$X_6$: & \texttt{yr} \\
$X_7$: & \texttt{mnth} \\
$X_8$: & \texttt{hr} \\
$X_9$: & \texttt{holiday} \\
$X_{10}$: & \texttt{weekday} \\
$X_{11}$: & \texttt{workingday} \\
$X_{12}$: & \texttt{weathersit}
\end{tabular}
&
\begin{tabular}[t]{@{}l@{\;}l@{}}
$X_1$: & \texttt{age} \\
$X_2$: & \texttt{fnlwgt} \\
$X_3$: & \texttt{capital-gain} \\
$X_4$: & \texttt{capital-loss} \\
$X_5$: & \texttt{hours-per-week} \\
$X_6$: & \texttt{education-num} \\
$X_7$: & \texttt{workclass} \\
$X_8$: & \texttt{education} \\
$X_9$: & \texttt{marital-status} \\
$X_{10}$: & \texttt{occupation} \\
$X_{11}$: & \texttt{relationship} \\
$X_{12}$: & \texttt{race} \\
$X_{13}$: & \texttt{sex} \\
$X_{14}$: & \texttt{native-country}
\end{tabular}
\\
\midrule
Continuous predictors
& $X_1,X_2,X_3,X_4$
& $X_1,X_2,X_3,X_4,X_5$ \\
Discrete predictors
& $X_5,X_6,X_7,X_8,X_9,X_{10},X_{11},X_{12}$
& $X_6,X_7,X_8,X_9,X_{10},X_{11},X_{12},X_{13},X_{14}$  \\
First chain variable
& $X_3$ \;(\texttt{hum})
& $X_2$ \;(\texttt{fnlwgt}) \\
Chain order
& $X_3,X_1,X_2,X_4,X_5,\ldots,X_{12},Y$
& $X_2,X_1,X_3,X_4,X_5,\ldots,X_{14},Y$ \\
\specialrule{0.8pt}{3pt}{1.5pt}
\specialrule{0.8pt}{0pt}{0pt}
\end{tabular}
}
\end{table}

\section{Baseline methods}
\label{app:baseline_methods}

We compare FITS with HeteroFedSyn \citep{li2026heterofedsyn} and FedTabDiff \citep{sattarov2024fedtabdiff}. Both baselines generate a shared synthetic tabular dataset $\mathcal B=\{\hat{\bm x}_b\}_{b=1}^{B}$ from the same client datasets $\{\mathcal H^{(k)}\}_{k=1}^{K}$, where $m_k=|\mathcal H^{(k)}|$ and $M=\sum_{k=1}^{K}m_k$. FT-Transformer provides an additional downstream evaluation model. The descriptions below summarize the baseline implementations used in our experiments.

\subsection{HeteroFedSyn}
\label{app:heterofedsyn}

We use the adaptive marginal-selection variant, AdaFedPrivSyn. Continuous columns are discretized into at most $100$ equal-width bins, while discrete columns retain their observed categories. For a table with $q$ columns, client $k$ computes one-way count vectors $\bm c_a^{(k)}$ and two-way count vectors $\bm c_{ab}^{(k)}$. Each attribute pair uses a shared Gaussian projection matrix $\Pi_{ab}\in\mathbb R^{c_a c_b\times 10}$, where $c_a$ is the encoded domain size of column $a$. The server aggregates locally perturbed statistics as
\[
\hat{\bm p}_a
=\frac{1}{M}\sum_{k=1}^{K}\bigl(\bm c_a^{(k)}+\bm\xi_a^{(k)}\bigr),
\qquad
\hat{\bm s}_{ab}
=\frac{1}{M}\sum_{k=1}^{K}\bigl(\Pi_{ab}^{\top}\bm c_{ab}^{(k)}+\bm\zeta_{ab}^{(k)}\bigr).
\]
Here $\bm\xi_a^{(k)}$ and $\bm\zeta_{ab}^{(k)}$ are Gaussian noise vectors. 
Dependency scores compare $\hat{\bm s}_{ab}$ with the projected product $\Pi_{ab}^{\top}(\hat{\bm p}_a\otimes\hat{\bm p}_b)$, with a correction for the aggregated noise variance. 
A pair is selected when its dependency score exceeds its expected release-noise error. 
We select at most $\lfloor\binom{q}{2}/3\rfloor$ pairs. 
After each selection, marginal consistency adjustment and $30$ iterations of the gradual update method (GUM, \cite{zhang2021privsyn}) produce a synthetic tabular dataset; projected residuals update the scores of remaining pairs involving covered attributes. 
Previously released one-way marginals are reused. 
Continuous values are recovered as bin midpoints, and discrete values are mapped back to their original categories.

We set $\epsilon=10^8$ and $\delta=10^{-5}$. Writing $L=\log(1/\delta)$, the implementation uses the zCDP budget
\[
\rho=\bigl(\sqrt{L+\epsilon}-\sqrt L\bigr)^2,
\qquad
(\rho_1,\rho_2,\rho_3)=(0.1\rho,0.1\rho,0.8\rho),
\]
for one-way statistics, projected two-way statistics, and selected two-way marginals, respectively. This deliberately weak privacy setting emphasizes synthesis utility rather than a privacy--utility trade-off.

\subsection{FedTabDiff}
\label{app:fedtabdiff}

FedTabDiff learns a denoising diffusion model on an encoded mixed-type tabular dataset. Continuous columns are transformed to approximately standard normal marginals using a training-data quantile transform with at most $1{,}000$ quantiles. Each discrete column uses a two-dimensional frozen embedding; generated embeddings are decoded to the nearest category embedding within that column. The encoded record $\bm z_0$ is perturbed according to
\[
\bm z_t=\sqrt{\bar\alpha_t}\bm z_0+\sqrt{1-\bar\alpha_t}\bm\varepsilon,
\qquad
\bm\varepsilon\sim N(\bm 0,I),
\qquad
\bar\alpha_t=\prod_{u=0}^{t}(1-\beta_u).
\]
Clients minimize the noise-prediction loss $\mathbb E\|\bm\varepsilon-\bm\varepsilon_\psi(\bm z_t,t)\|_2^2$. All clients participate in every communication round, and the server uses sample-size-weighted FedAvg:
\[
\psi^{(s+1)}=\sum_{k=1}^{K}\frac{m_k}{M}\psi_k^{(s+1)},
\] where $\psi_k^{(s+1)}$ is the $k$-th client's model at the $(s+1)$-th iteration and $\psi^{(s+1)}$ is the corresponding central model.

The noise predictor uses four hidden layers of width $1{,}024$. We use $500$ diffusion steps with a linear noise schedule from $2\times10^{-4}$ to $0.04$, and $1{,}000$ FedAvg rounds with Adam (learning rate $10^{-4}$), batch size $512$, and at most $20$ local updates per client per round. Covariates and responses are synthesized jointly.

\subsection{FT-Transformer for downstream evaluation}
\label{app:ft_transformer}

We use FT-Transformer \citep{gorishniy2021revisiting} to evaluate nonlinear regression and classification on the generated data. For each M1--M4 fit, only covariates enter the model. Continuous features are standardized using the weighted mean and standard deviation of that fit's training data, while nominal and ordinal predictors are mapped to categorical indices.

Each standardized continuous feature $x^{(c)}_j$ receives an affine token, and each discrete feature $x^{(d)}_j$ receives a learned category embedding:
\[
\bm t_j=
\begin{cases}
 x^{(c)}_j\bm a_j+\bm b_j, & x^{(c)}_j\text{ continuous},\\
 E_j[x^{(d)}_j]+\bm b_j, & x^{(d)}_j\text{ discrete},
\end{cases}
\qquad \bm t_j\in\mathbb R^{32},
\]where $\bm a_j,\ \bm b_j \in \mathbb R^{32}$ are learnable parameter vectors and $E_j$ is a learnable embedding matrix for discrete feature $j$.

The feature tokens and a learned $\texttt{[CLS]}$ token have dimension $32$ and pass through two Transformer blocks. Each block contains residual self-attention with four heads and a feed-forward network of width $64$ with ReGLU activation. Attention, feed-forward, and residual dropout are all set to zero. The final $\texttt{[CLS]}$ representation passes through LayerNorm, ReLU, and a linear head to produce a scalar output $f_\theta(\bm x)$, used for regression or as a binary-classification logit.

For a training table $\mathcal T=\{(\bm x_i,y_i,w_i)\}_{i=1}^{n}$ specified by the data usage mode, the objective is
\[
\mathcal L(\theta)=\frac{\sum_{i=1}^{n}w_i\ell(\theta;\bm x_i,y_i)}{\sum_{i=1}^{n}w_i},
\qquad
\ell(\theta;\bm x,y)=
\begin{cases}
 (f_\theta(\bm x)-y)^2, & \text{regression},\\
 \log(1+e^{f_\theta(\bm x)})-yf_\theta(\bm x), & \text{classification}.
\end{cases}
\]
Regression responses are additionally standardized using training weights, and predictions are mapped back to the evaluation scale. Binary responses retain their $0/1$ encoding; the sigmoid of the model output gives the class probability, with threshold $0.5$.

Each model uses $500$ AdamW updates with learning rate $10^{-3}$, batch size $256$, and weight decay $10^{-5}$. Minibatches are sampled with replacement with probabilities $w_i/\sum_j w_j$. Training settings and seed rules are identical across methods. We report regression test MSE on the original response scale and classification test accuracy.

\section{Supplementary experimental results}
\label{app:supp_results}

\begin{table}[ht]
\centering
\caption{Simulation: linear-regression parameter estimation error (FITS-G vs FITS-K).}
\label{tab:simulation_linear_parameter_l2_error_fixed_modes_fits_g_vs_fits_k}
\begingroup
\setlength{\tabcolsep}{4pt}
\renewcommand{\arraystretch}{1.35}
\resizebox{\textwidth}{!}{%
\begin{tabular}{@{}cc|*{10}{c}@{}}
\specialrule{0.8pt}{0pt}{1.5pt}
\specialrule{0.8pt}{0pt}{3pt}
\multirow{2}{*}{Mode} & \multirow{2}{*}{Baselines} & \multicolumn{10}{c}{Clients} \\
\cline{3-12}
 & & Client 0 (30\%) & Client 1 (20\%) & Client 2 (10\%) & Client 3 (9\%) & Client 4 (8\%) & Client 5 (7\%) & Client 6 (6\%) & Client 7 (5\%) & Client 8 (4\%) & Client 9 (1\%) \\
\specialrule{0.8pt}{3pt}{1.5pt}
\specialrule{0.8pt}{0pt}{3pt}
\multirow{2}{*}{M1} & FITS-G & \textbf{0.1584 (0.0698)} & \textbf{0.1584 (0.0698)} & \textbf{0.1584 (0.0698)} & \textbf{0.1584 (0.0698)} & \textbf{0.1584 (0.0698)} & \textbf{0.1584 (0.0698)} & \textbf{0.1584 (0.0698)} & \textbf{0.1584 (0.0698)} & \textbf{0.1584 (0.0698)} & \textbf{0.1584 (0.0698)} \\
 & FITS-K & 0.1721 (0.0580) & 0.1721 (0.0580) & 0.1721 (0.0580) & 0.1721 (0.0580) & 0.1721 (0.0580) & 0.1721 (0.0580) & 0.1721 (0.0580) & 0.1721 (0.0580) & 0.1721 (0.0580) & 0.1721 (0.0580) \\
\hline
\multirow{1}{*}{M2} & None & \textbf{0.0751 (0.0329)} & \textbf{0.0930 (0.0385)} & \textbf{0.1302 (0.0568)} & \textbf{0.1446 (0.0698)} & \textbf{0.1310 (0.0616)} & \textbf{0.1541 (0.0704)} & \textbf{0.1751 (0.0725)} & \textbf{0.2063 (0.0856)} & \textbf{0.2198 (0.1036)} & \textbf{0.4730 (0.1856)} \\
\hline
\multirow{2}{*}{\shortstack{M3 }} & FITS-G & \textbf{0.1060 (0.0435)} & \textbf{0.1225 (0.0549)} & \textbf{0.1350 (0.0622)} & \textbf{0.1406 (0.0620)} & \textbf{0.1405 (0.0638)} & \textbf{0.1413 (0.0629)} & \textbf{0.1438 (0.0654)} & \textbf{0.1467 (0.0663)} & \textbf{0.1501 (0.0669)} & \textbf{0.1545 (0.0682)} \\
 & FITS-K & 0.1140 (0.0398) & 0.1315 (0.0493) & 0.1462 (0.0526) & 0.1500 (0.0538) & 0.1511 (0.0542) & 0.1522 (0.0547) & 0.1565 (0.0560) & 0.1584 (0.0554) & 0.1635 (0.0550) & 0.1680 (0.0554) \\
\end{tabular}
}
\endgroup
\end{table}

\begin{table}[ht]
\centering
\caption{Simulation: logistic-regression parameter estimation error (FITS-G vs FITS-K).}
\label{tab:simulation_logistic_parameter_l2_error_fixed_modes_fits_g_vs_fits_k}
\begingroup
\setlength{\tabcolsep}{4pt}
\renewcommand{\arraystretch}{1.35}
\resizebox{\textwidth}{!}{%
\begin{tabular}{@{}cc|*{10}{c}@{}}
\specialrule{0.8pt}{0pt}{1.5pt}
\specialrule{0.8pt}{0pt}{3pt}
\multirow{2}{*}{Mode} & \multirow{2}{*}{Baselines} & \multicolumn{10}{c}{Clients} \\
\cline{3-12}
 & & Client 0 (30\%) & Client 1 (20\%) & Client 2 (10\%) & Client 3 (9\%) & Client 4 (8\%) & Client 5 (7\%) & Client 6 (6\%) & Client 7 (5\%) & Client 8 (4\%) & Client 9 (1\%) \\
\specialrule{0.8pt}{3pt}{1.5pt}
\specialrule{0.8pt}{0pt}{3pt}
\multirow{2}{*}{M1} & FITS-G & \textbf{0.7111 (0.1401)} & \textbf{0.7111 (0.1401)} & \textbf{0.7111 (0.1401)} & \textbf{0.7111 (0.1401)} & \textbf{0.7111 (0.1401)} & \textbf{0.7111 (0.1401)} & \textbf{0.7111 (0.1401)} & \textbf{0.7111 (0.1401)} & \textbf{0.7111 (0.1401)} & \textbf{0.7111 (0.1401)} \\
 & FITS-K & 0.7289 (0.1278) & 0.7289 (0.1278) & 0.7289 (0.1278) & 0.7289 (0.1278) & 0.7289 (0.1278) & 0.7289 (0.1278) & 0.7289 (0.1278) & 0.7289 (0.1278) & 0.7289 (0.1278) & 0.7289 (0.1278) \\
\hline
\multirow{1}{*}{M2} & None & \textbf{0.3923 (0.1852)} & \textbf{0.5028 (0.2397)} & \textbf{0.6821 (0.3112)} & \textbf{0.8635 (0.3364)} & \textbf{0.7619 (0.3377)} & \textbf{0.8787 (0.4005)} & \textbf{0.8757 (0.4659)} & \textbf{1.0606 (0.5028)} & \textbf{1.1961 (0.6224)} & \textbf{10.6726 (16.6535)} \\
\hline
\multirow{2}{*}{\shortstack{M3 }} & FITS-G & \textbf{0.5117 (0.1249)} & \textbf{0.5694 (0.1246)} & \textbf{0.6281 (0.1209)} & \textbf{0.6370 (0.1343)} & \textbf{0.6446 (0.1351)} & \textbf{0.6602 (0.1370)} & \textbf{0.6615 (0.1381)} & \textbf{0.6714 (0.1420)} & \textbf{0.6681 (0.1364)} & \textbf{0.6983 (0.1363)} \\
 & FITS-K & 0.5252 (0.1168) & 0.5823 (0.1136) & 0.6448 (0.1178) & 0.6529 (0.1194) & 0.6603 (0.1237) & 0.6752 (0.1200) & 0.6779 (0.1286) & 0.6896 (0.1315) & 0.6856 (0.1328) & 0.7172 (0.1244) \\
\end{tabular}
}
\endgroup
\end{table}

\begin{table}[ht]
\centering
\caption{Simulation: FT-Transformer classification accuracy on the test data (FITS-G vs FITS-K).}
\label{tab:simulation_ft_classification_test_accuracy_fixed_modes_fits_g_vs_fits_k}
\begingroup
\setlength{\tabcolsep}{4pt}
\renewcommand{\arraystretch}{1.35}
\resizebox{\textwidth}{!}{%
\begin{tabular}{@{}cc|*{10}{c}@{}}
\specialrule{0.8pt}{0pt}{1.5pt}
\specialrule{0.8pt}{0pt}{3pt}
\multirow{2}{*}{Mode} & \multirow{2}{*}{Baselines} & \multicolumn{10}{c}{Clients} \\
\cline{3-12}
 & & Client 0 (30\%) & Client 1 (20\%) & Client 2 (10\%) & Client 3 (9\%) & Client 4 (8\%) & Client 5 (7\%) & Client 6 (6\%) & Client 7 (5\%) & Client 8 (4\%) & Client 9 (1\%) \\
\specialrule{0.8pt}{3pt}{1.5pt}
\specialrule{0.8pt}{0pt}{3pt}
\multirow{2}{*}{M1} & FITS-G & \textbf{0.8608 (0.0079)} & \textbf{0.8608 (0.0079)} & \textbf{0.8608 (0.0079)} & \textbf{0.8608 (0.0079)} & \textbf{0.8608 (0.0079)} & \textbf{0.8608 (0.0079)} & \textbf{0.8608 (0.0079)} & \textbf{0.8608 (0.0079)} & \textbf{0.8608 (0.0079)} & \textbf{0.8608 (0.0079)} \\
 & FITS-K & 0.8603 (0.0082) & 0.8603 (0.0082) & 0.8603 (0.0082) & 0.8603 (0.0082) & 0.8603 (0.0082) & 0.8603 (0.0082) & 0.8603 (0.0082) & 0.8603 (0.0082) & 0.8603 (0.0082) & 0.8603 (0.0082) \\
\hline
\multirow{1}{*}{M2} & None & \textbf{0.8547 (0.0101)} & \textbf{0.8406 (0.0093)} & \textbf{0.8280 (0.0105)} & \textbf{0.8278 (0.0115)} & \textbf{0.8237 (0.0127)} & \textbf{0.8242 (0.0123)} & \textbf{0.8259 (0.0122)} & \textbf{0.8246 (0.0135)} & \textbf{0.8234 (0.0134)} & \textbf{0.8193 (0.0235)} \\
\hline
\multirow{2}{*}{\shortstack{M3 }} & FITS-G & 0.8692 (0.0080) & \textbf{0.8672 (0.0070)} & \textbf{0.8644 (0.0074)} & 0.8645 (0.0075) & 0.8629 (0.0082) & \textbf{0.8638 (0.0072)} & 0.8635 (0.0084) & \textbf{0.8634 (0.0081)} & \textbf{0.8617 (0.0079)} & 0.8603 (0.0081) \\
 & FITS-K & \textbf{0.8695 (0.0084)} & 0.8671 (0.0076) & 0.8638 (0.0080) & \textbf{0.8656 (0.0078)} & \textbf{0.8637 (0.0076)} & 0.8627 (0.0085) & \textbf{0.8640 (0.0077)} & 0.8633 (0.0078) & 0.8608 (0.0083) & \textbf{0.8615 (0.0080)} \\
\hline
\multirow{2}{*}{\shortstack{M4 }} & FITS-G & \textbf{0.8689 (0.0084)} & 0.8667 (0.0079) & 0.8647 (0.0085) & 0.8633 (0.0078) & 0.8634 (0.0071) & 0.8620 (0.0080) & 0.8627 (0.0083) & 0.8612 (0.0096) & 0.8606 (0.0080) & 0.8605 (0.0075) \\
\end{tabular}
}
\endgroup
\end{table}

\begin{table}[ht]
\centering
\caption{Simulation: linear-regression parameter estimation error (FITS-G: all modes).}
\label{tab:simulation_linear_parameter_l2_error_fixed_modes_fits_g_all_modes_parameters}
\begingroup
\setlength{\tabcolsep}{4pt}
\renewcommand{\arraystretch}{1.35}
\resizebox{\textwidth}{!}{%
\begin{tabular}{@{}cc|*{10}{c}@{}}
\specialrule{0.8pt}{0pt}{1.5pt}
\specialrule{0.8pt}{0pt}{3pt}
\multirow{2}{*}{Mode} & \multirow{2}{*}{Parameters} & \multicolumn{10}{c}{Clients} \\
\cline{3-12}
 & & Client 0 (30\%) & Client 1 (20\%) & Client 2 (10\%) & Client 3 (9\%) & Client 4 (8\%) & Client 5 (7\%) & Client 6 (6\%) & Client 7 (5\%) & Client 8 (4\%) & Client 9 (1\%) \\
\specialrule{0.8pt}{3pt}{1.5pt}
\specialrule{0.8pt}{0pt}{3pt}
\multirow{1}{*}{M1} & -- & 0.1584 (0.0698) & 0.1584 (0.0698) & 0.1584 (0.0698) & 0.1584 (0.0698) & 0.1584 (0.0698) & 0.1584 (0.0698) & 0.1584 (0.0698) & 0.1584 (0.0698) & 0.1584 (0.0698) & 0.1584 (0.0698) \\
\hline
\multirow{1}{*}{M2} & -- & 0.0751 (0.0329) & 0.0930 (0.0385) & 0.1302 (0.0568) & 0.1446 (0.0698) & 0.1310 (0.0616) & 0.1541 (0.0704) & 0.1751 (0.0725) & 0.2063 (0.0856) & 0.2198 (0.1036) & 0.4730 (0.1856) \\
\hline
\multirow{5}{*}{M3} & $r=0$ & 0.0751 (0.0329) & 0.0930 (0.0385) & 0.1302 (0.0568) & 0.1446 (0.0698) & 0.1310 (0.0616) & 0.1541 (0.0704) & 0.1751 (0.0725) & 0.2063 (0.0856) & 0.2198 (0.1036) & 0.4730 (0.1856) \\
 & $r=0.5$ & 0.0938 (0.0376) & 0.1082 (0.0464) & 0.1364 (0.0651) & 0.1582 (0.0761) & 0.1490 (0.0804) & 0.1616 (0.0748) & 0.1836 (0.0678) & 0.2218 (0.1041) & 0.2308 (0.0934) & 0.4392 (0.2013) \\
 & $r=1$ & 0.0999 (0.0401) & 0.1184 (0.0481) & 0.1491 (0.0679) & 0.1544 (0.0687) & 0.1642 (0.0853) & 0.1580 (0.0759) & 0.1823 (0.0736) & 0.2009 (0.1022) & 0.2129 (0.0876) & 0.3802 (0.1548) \\
 & $r=2$ & 0.1060 (0.0435) & 0.1182 (0.0559) & 0.1349 (0.0535) & 0.1424 (0.0669) & 0.1588 (0.0800) & 0.1470 (0.0642) & 0.1622 (0.0779) & 0.1769 (0.0887) & 0.1863 (0.0841) & 0.3667 (0.1730) \\
\end{tabular}
}
\endgroup
\end{table}

\begin{table}[ht]
\centering
\caption{Simulation: logistic-regression parameter estimation error (FITS-G: all modes).}
\label{tab:simulation_logistic_parameter_l2_error_fixed_modes_fits_g_all_modes_parameters}
\begingroup
\setlength{\tabcolsep}{4pt}
\renewcommand{\arraystretch}{1.35}
\resizebox{\textwidth}{!}{%
\begin{tabular}{@{}cc|*{10}{c}@{}}
\specialrule{0.8pt}{0pt}{1.5pt}
\specialrule{0.8pt}{0pt}{3pt}
\multirow{2}{*}{Mode} & \multirow{2}{*}{Parameters} & \multicolumn{10}{c}{Clients} \\
\cline{3-12}
 & & Client 0 (30\%) & Client 1 (20\%) & Client 2 (10\%) & Client 3 (9\%) & Client 4 (8\%) & Client 5 (7\%) & Client 6 (6\%) & Client 7 (5\%) & Client 8 (4\%) & Client 9 (1\%) \\
\specialrule{0.8pt}{3pt}{1.5pt}
\specialrule{0.8pt}{0pt}{3pt}
\multirow{1}{*}{M1} & -- & 0.7111 (0.1401) & 0.7111 (0.1401) & 0.7111 (0.1401) & 0.7111 (0.1401) & 0.7111 (0.1401) & 0.7111 (0.1401) & 0.7111 (0.1401) & 0.7111 (0.1401) & 0.7111 (0.1401) & 0.7111 (0.1401) \\
\hline
\multirow{1}{*}{M2} & -- & 0.3923 (0.1852) & 0.5028 (0.2397) & 0.6821 (0.3112) & 0.8635 (0.3364) & 0.7619 (0.3377) & 0.8787 (0.4005) & 0.8757 (0.4659) & 1.0606 (0.5028) & 1.1961 (0.6224) & 10.6726 (16.6535) \\
\hline
\multirow{5}{*}{M3} & $r=0$ & 0.3923 (0.1852) & 0.5028 (0.2397) & 0.6821 (0.3112) & 0.8635 (0.3364) & 0.7619 (0.3377) & 0.8787 (0.4005) & 0.8757 (0.4659) & 1.0606 (0.5028) & 1.1961 (0.6224) & 10.6726 (16.6535) \\
 & $r=0.5$ & 0.4041 (0.1665) & 0.4679 (0.1653) & 0.5697 (0.2247) & 0.6907 (0.2742) & 0.6489 (0.2302) & 0.6532 (0.2362) & 0.7458 (0.3929) & 0.8008 (0.3280) & 0.8787 (0.3414) & 3.0192 (3.4075) \\
 & $r=1$ & 0.4546 (0.1308) & 0.4850 (0.1498) & 0.5876 (0.2121) & 0.6201 (0.2102) & 0.6232 (0.2326) & 0.6382 (0.2552) & 0.6440 (0.3432) & 0.7573 (0.2600) & 0.7669 (0.2690) & 1.7311 (0.8331) \\
 & $r=2$ & 0.5117 (0.1249) & 0.5510 (0.1356) & 0.5632 (0.1836) & 0.6277 (0.2068) & 0.6002 (0.2122) & 0.6313 (0.2036) & 0.6204 (0.2093) & 0.6557 (0.2232) & 0.7169 (0.2454) & 1.2980 (0.5683) \\
\end{tabular}
}
\endgroup
\end{table}

\begin{table}[ht]
\centering
\caption{Simulation: FT-Transformer classification accuracy on the test data (FITS-G: all modes).}
\label{tab:simulation_ft_classification_test_accuracy_fixed_modes_fits_g_all_modes_parameters}
\begingroup
\setlength{\tabcolsep}{4pt}
\renewcommand{\arraystretch}{1.35}
\resizebox{\textwidth}{!}{%
\begin{tabular}{@{}cc|*{10}{c}@{}}
\specialrule{0.8pt}{0pt}{1.5pt}
\specialrule{0.8pt}{0pt}{3pt}
\multirow{2}{*}{Mode} & \multirow{2}{*}{Parameters} & \multicolumn{10}{c}{Clients} \\
\cline{3-12}
 & & Client 0 (30\%) & Client 1 (20\%) & Client 2 (10\%) & Client 3 (9\%) & Client 4 (8\%) & Client 5 (7\%) & Client 6 (6\%) & Client 7 (5\%) & Client 8 (4\%) & Client 9 (1\%) \\
\specialrule{0.8pt}{3pt}{1.5pt}
\specialrule{0.8pt}{0pt}{3pt}
\multirow{1}{*}{M1} & -- & 0.8608 (0.0079) & 0.8608 (0.0079) & 0.8608 (0.0079) & 0.8608 (0.0079) & 0.8608 (0.0079) & 0.8608 (0.0079) & 0.8608 (0.0079) & 0.8608 (0.0079) & 0.8608 (0.0079) & \textbf{0.8608 (0.0079)} \\
\hline
\multirow{1}{*}{M2} & -- & 0.8547 (0.0101) & 0.8406 (0.0093) & 0.8280 (0.0105) & 0.8278 (0.0115) & 0.8237 (0.0127) & 0.8242 (0.0123) & 0.8259 (0.0122) & 0.8246 (0.0135) & 0.8234 (0.0134) & 0.8193 (0.0235) \\
\hline
\multirow{5}{*}{M3} & $r=0$ & 0.8547 (0.0101) & 0.8406 (0.0093) & 0.8280 (0.0105) & 0.8278 (0.0115) & 0.8237 (0.0127) & 0.8242 (0.0123) & 0.8259 (0.0122) & 0.8246 (0.0135) & 0.8234 (0.0134) & 0.8193 (0.0235) \\
 & $r=0.5$ & 0.8639 (0.0067) & 0.8486 (0.0097) & 0.8306 (0.0111) & 0.8251 (0.0121) & 0.8214 (0.0137) & 0.8187 (0.0139) & 0.8210 (0.0114) & 0.8208 (0.0128) & 0.8179 (0.0137) & 0.8115 (0.0229) \\
 & $r=1$ & 0.8661 (0.0070) & 0.8566 (0.0088) & 0.8323 (0.0111) & 0.8295 (0.0097) & 0.8232 (0.0111) & 0.8191 (0.0111) & 0.8206 (0.0116) & 0.8167 (0.0138) & 0.8120 (0.0160) & 0.8078 (0.0217) \\
 & $r=2$ & \textbf{0.8695 (0.0070)} & 0.8651 (0.0074) & 0.8474 (0.0097) & 0.8420 (0.0093) & 0.8391 (0.0102) & 0.8320 (0.0118) & 0.8279 (0.0123) & 0.8214 (0.0127) & 0.8110 (0.0154) & 0.8057 (0.0184) \\
\end{tabular}
}
\endgroup
\end{table}

\begin{table}[ht]
\centering
\caption{Bike: linear-regression parameter estimation error (FITS-G vs FITS-K).}
\label{tab:bike_linear_parameter_l2_error_fixed_modes_fits_g_vs_fits_k}
\begingroup
\setlength{\tabcolsep}{4pt}
\renewcommand{\arraystretch}{1.35}
\resizebox{\textwidth}{!}{%
\begin{tabular}{@{}cc|*{10}{c}@{}}
\specialrule{0.8pt}{0pt}{1.5pt}
\specialrule{0.8pt}{0pt}{3pt}
\multirow{2}{*}{Mode} & \multirow{2}{*}{Baselines} & \multicolumn{10}{c}{Clients} \\
\cline{3-12}
 & & Client 0 (30\%) & Client 1 (20\%) & Client 2 (10\%) & Client 3 (9\%) & Client 4 (8\%) & Client 5 (7\%) & Client 6 (6\%) & Client 7 (5\%) & Client 8 (4\%) & Client 9 (1\%) \\
\specialrule{0.8pt}{3pt}{1.5pt}
\specialrule{0.8pt}{0pt}{3pt}
\multirow{2}{*}{M1} & FITS-G & \textbf{0.1240 (0.0443)} & \textbf{0.1240 (0.0443)} & \textbf{0.1240 (0.0443)} & \textbf{0.1240 (0.0443)} & \textbf{0.1240 (0.0443)} & \textbf{0.1240 (0.0443)} & \textbf{0.1240 (0.0443)} & \textbf{0.1240 (0.0443)} & \textbf{0.1240 (0.0443)} & \textbf{0.1240 (0.0443)} \\
 & FITS-K & 0.1577 (0.0603) & 0.1577 (0.0603) & 0.1577 (0.0603) & 0.1577 (0.0603) & 0.1577 (0.0603) & 0.1577 (0.0603) & 0.1577 (0.0603) & 0.1577 (0.0603) & 0.1577 (0.0603) & 0.1577 (0.0603) \\
\hline
\multirow{1}{*}{M2} & None & \textbf{0.1059 (0.0711)} & \textbf{0.1397 (0.0924)} & \textbf{0.2044 (0.1343)} & \textbf{0.2043 (0.1359)} & \textbf{0.2331 (0.1749)} & \textbf{0.2088 (0.1111)} & \textbf{0.2900 (0.1840)} & \textbf{0.3137 (0.1962)} & \textbf{0.3291 (0.2270)} & \textbf{0.6422 (0.4136)} \\
\hline
\multirow{2}{*}{\shortstack{M3 }} & FITS-G & \textbf{0.0933 (0.0370)} & \textbf{0.0976 (0.0402)} & \textbf{0.1138 (0.0442)} & \textbf{0.1117 (0.0460)} & \textbf{0.1126 (0.0413)} & \textbf{0.1114 (0.0425)} & \textbf{0.1186 (0.0406)} & \textbf{0.1169 (0.0424)} & \textbf{0.1173 (0.0399)} & \textbf{0.1224 (0.0445)} \\
 & FITS-K & 0.1095 (0.0516) & 0.1169 (0.0508) & 0.1347 (0.0557) & 0.1350 (0.0556) & 0.1377 (0.0552) & 0.1373 (0.0568) & 0.1428 (0.0572) & 0.1440 (0.0595) & 0.1460 (0.0563) & 0.1551 (0.0602) \\
\end{tabular}
}
\endgroup
\end{table}

\begin{table}[ht]
\centering
\caption{Bike: linear-regression test MSE (FITS-G vs FITS-K).}
\label{tab:bike_linear_test_mse_fixed_modes_fits_g_vs_fits_k}
\begingroup
\setlength{\tabcolsep}{4pt}
\renewcommand{\arraystretch}{1.35}
\resizebox{\textwidth}{!}{%
\begin{tabular}{@{}cc|*{10}{c}@{}}
\specialrule{0.8pt}{0pt}{1.5pt}
\specialrule{0.8pt}{0pt}{3pt}
\multirow{2}{*}{Mode} & \multirow{2}{*}{Baselines} & \multicolumn{10}{c}{Clients} \\
\cline{3-12}
 & & Client 0 (30\%) & Client 1 (20\%) & Client 2 (10\%) & Client 3 (9\%) & Client 4 (8\%) & Client 5 (7\%) & Client 6 (6\%) & Client 7 (5\%) & Client 8 (4\%) & Client 9 (1\%) \\
\specialrule{0.8pt}{3pt}{1.5pt}
\specialrule{0.8pt}{0pt}{3pt}
\multirow{2}{*}{M1} & FITS-G & \textbf{0.6117 (0.0195)} & \textbf{0.6117 (0.0195)} & \textbf{0.6117 (0.0195)} & \textbf{0.6117 (0.0195)} & \textbf{0.6117 (0.0195)} & \textbf{0.6117 (0.0195)} & \textbf{0.6117 (0.0195)} & \textbf{0.6117 (0.0195)} & \textbf{0.6117 (0.0195)} & \textbf{0.6117 (0.0195)} \\
 & FITS-K & 0.6187 (0.0184) & 0.6187 (0.0184) & 0.6187 (0.0184) & 0.6187 (0.0184) & 0.6187 (0.0184) & 0.6187 (0.0184) & 0.6187 (0.0184) & 0.6187 (0.0184) & 0.6187 (0.0184) & 0.6187 (0.0184) \\
\hline
\multirow{1}{*}{M2} & None & \textbf{0.6099 (0.0189)} & \textbf{0.6106 (0.0189)} & \textbf{0.6135 (0.0193)} & \textbf{0.6140 (0.0190)} & \textbf{0.6157 (0.0189)} & \textbf{0.6152 (0.0196)} & \textbf{0.6172 (0.0198)} & \textbf{0.6192 (0.0200)} & \textbf{0.6216 (0.0209)} & \textbf{0.6582 (0.0274)} \\
\hline
\multirow{2}{*}{\shortstack{M3 }} & FITS-G & \textbf{0.6098 (0.0192)} & \textbf{0.6101 (0.0192)} & \textbf{0.6109 (0.0194)} & \textbf{0.6107 (0.0195)} & \textbf{0.6110 (0.0194)} & \textbf{0.6110 (0.0194)} & \textbf{0.6111 (0.0194)} & \textbf{0.6111 (0.0195)} & \textbf{0.6113 (0.0194)} & \textbf{0.6116 (0.0194)} \\
 & FITS-K & 0.6119 (0.0185) & 0.6130 (0.0186) & 0.6151 (0.0184) & 0.6153 (0.0186) & 0.6158 (0.0184) & 0.6158 (0.0185) & 0.6162 (0.0185) & 0.6166 (0.0185) & 0.6169 (0.0184) & 0.6182 (0.0185) \\
\end{tabular}
}
\endgroup
\end{table}

\begin{table}[ht]
\centering
\caption{Bike: FT-Transformer regression MSE on the test data (FITS-G vs FITS-K).}
\label{tab:bike_ft_regression_test_mse_fixed_modes_fits_g_vs_fits_k}
\begingroup
\setlength{\tabcolsep}{4pt}
\renewcommand{\arraystretch}{1.35}
\resizebox{\textwidth}{!}{%
\begin{tabular}{@{}cc|*{10}{c}@{}}
\specialrule{0.8pt}{0pt}{1.5pt}
\specialrule{0.8pt}{0pt}{3pt}
\multirow{2}{*}{Mode} & \multirow{2}{*}{Baselines} & \multicolumn{10}{c}{Clients} \\
\cline{3-12}
 & & Client 0 (30\%) & Client 1 (20\%) & Client 2 (10\%) & Client 3 (9\%) & Client 4 (8\%) & Client 5 (7\%) & Client 6 (6\%) & Client 7 (5\%) & Client 8 (4\%) & Client 9 (1\%) \\
\specialrule{0.8pt}{3pt}{1.5pt}
\specialrule{0.8pt}{0pt}{3pt}
\multirow{2}{*}{M1} & FITS-G & \textbf{0.2348 (0.0209)} & \textbf{0.2348 (0.0209)} & \textbf{0.2348 (0.0209)} & \textbf{0.2348 (0.0209)} & \textbf{0.2348 (0.0209)} & \textbf{0.2348 (0.0209)} & \textbf{0.2348 (0.0209)} & \textbf{0.2348 (0.0209)} & \textbf{0.2348 (0.0209)} & \textbf{0.2348 (0.0209)} \\
 & FITS-K & 0.2826 (0.0270) & 0.2826 (0.0270) & 0.2826 (0.0270) & 0.2826 (0.0270) & 0.2826 (0.0270) & 0.2826 (0.0270) & 0.2826 (0.0270) & 0.2826 (0.0270) & 0.2826 (0.0270) & 0.2826 (0.0270) \\
\hline
\multirow{1}{*}{M2} & None & \textbf{0.0687 (0.0060)} & \textbf{0.0727 (0.0041)} & \textbf{0.0896 (0.0065)} & \textbf{0.0944 (0.0056)} & \textbf{0.0990 (0.0070)} & \textbf{0.1079 (0.0079)} & \textbf{0.1154 (0.0107)} & \textbf{0.1262 (0.0142)} & \textbf{0.1522 (0.0192)} & \textbf{0.4313 (0.0741)} \\
\hline
\multirow{2}{*}{\shortstack{M3 }} & FITS-G & \textbf{0.1150 (0.0122)} & \textbf{0.1324 (0.0143)} & \textbf{0.1636 (0.0177)} & \textbf{0.1697 (0.0191)} & \textbf{0.1732 (0.0152)} & \textbf{0.1790 (0.0210)} & \textbf{0.1840 (0.0212)} & \textbf{0.1908 (0.0158)} & \textbf{0.1984 (0.0210)} & \textbf{0.2213 (0.0199)} \\
 & FITS-K & 0.1253 (0.0128) & 0.1524 (0.0140) & 0.1911 (0.0196) & 0.1988 (0.0215) & 0.2061 (0.0163) & 0.2148 (0.0225) & 0.2220 (0.0261) & 0.2336 (0.0237) & 0.2414 (0.0241) & 0.2753 (0.0256) \\
\end{tabular}
}
\endgroup
\end{table}

\begin{table}[ht]
\centering
\caption{Bike: linear-regression parameter estimation error (FITS-G: all modes).}
\label{tab:bike_linear_parameter_l2_error_fixed_modes_fits_g_all_modes_parameters}
\begingroup
\setlength{\tabcolsep}{4pt}
\renewcommand{\arraystretch}{1.35}
\resizebox{\textwidth}{!}{%
\begin{tabular}{@{}cc|*{10}{c}@{}}
\specialrule{0.8pt}{0pt}{1.5pt}
\specialrule{0.8pt}{0pt}{3pt}
\multirow{2}{*}{Mode} & \multirow{2}{*}{Parameters} & \multicolumn{10}{c}{Clients} \\
\cline{3-12}
 & & Client 0 (30\%) & Client 1 (20\%) & Client 2 (10\%) & Client 3 (9\%) & Client 4 (8\%) & Client 5 (7\%) & Client 6 (6\%) & Client 7 (5\%) & Client 8 (4\%) & Client 9 (1\%) \\
\specialrule{0.8pt}{3pt}{1.5pt}
\specialrule{0.8pt}{0pt}{3pt}
\multirow{1}{*}{M1} & -- & 0.1240 (0.0443) & 0.1240 (0.0443) & 0.1240 (0.0443) & 0.1240 (0.0443) & 0.1240 (0.0443) & 0.1240 (0.0443) & 0.1240 (0.0443) & 0.1240 (0.0443) & 0.1240 (0.0443) & 0.1240 (0.0443) \\
\hline
\multirow{1}{*}{M2} & -- & 0.1059 (0.0711) & 0.1397 (0.0924) & 0.2044 (0.1343) & 0.2043 (0.1359) & 0.2331 (0.1749) & 0.2088 (0.1111) & 0.2900 (0.1840) & 0.3137 (0.1962) & 0.3291 (0.2270) & 0.6422 (0.4136) \\
\hline
\multirow{5}{*}{M3} & $r=0$ & 0.1059 (0.0711) & 0.1397 (0.0924) & 0.2044 (0.1343) & 0.2043 (0.1359) & 0.2331 (0.1749) & 0.2088 (0.1111) & 0.2900 (0.1840) & 0.3137 (0.1962) & 0.3291 (0.2270) & 0.6422 (0.4136) \\
 & $r=0.5$ & 0.0872 (0.0571) & 0.1019 (0.0543) & 0.1626 (0.0885) & 0.1498 (0.0676) & 0.1713 (0.1141) & 0.1391 (0.0711) & 0.1881 (0.1157) & 0.2089 (0.1124) & 0.2287 (0.1309) & 0.4503 (0.2013) \\
 & $r=1$ & 0.0871 (0.0389) & 0.0983 (0.0587) & 0.1493 (0.0802) & 0.1302 (0.0707) & 0.1179 (0.0663) & 0.1328 (0.0521) & 0.1719 (0.0813) & 0.1727 (0.0921) & 0.1821 (0.0818) & 0.3670 (0.1706) \\
 & $r=2$ & 0.0933 (0.0370) & 0.0976 (0.0402) & 0.1145 (0.0557) & 0.1311 (0.0733) & 0.1253 (0.0492) & 0.1084 (0.0497) & 0.1317 (0.0578) & 0.1241 (0.0601) & 0.1443 (0.0692) & 0.2833 (0.1474) \\
\end{tabular}
}
\endgroup
\end{table}

\begin{table}[ht]
\centering
\caption{Bike: linear-regression test MSE (FITS-G: all modes).}
\label{tab:bike_linear_test_mse_fixed_modes_fits_g_all_modes_parameters}
\begingroup
\setlength{\tabcolsep}{4pt}
\renewcommand{\arraystretch}{1.35}
\resizebox{\textwidth}{!}{%
\begin{tabular}{@{}cc|*{10}{c}@{}}
\specialrule{0.8pt}{0pt}{1.5pt}
\specialrule{0.8pt}{0pt}{3pt}
\multirow{2}{*}{Mode} & \multirow{2}{*}{Parameters} & \multicolumn{10}{c}{Clients} \\
\cline{3-12}
 & & Client 0 (30\%) & Client 1 (20\%) & Client 2 (10\%) & Client 3 (9\%) & Client 4 (8\%) & Client 5 (7\%) & Client 6 (6\%) & Client 7 (5\%) & Client 8 (4\%) & Client 9 (1\%) \\
\specialrule{0.8pt}{3pt}{1.5pt}
\specialrule{0.8pt}{0pt}{3pt}
\multirow{1}{*}{M1} & -- & 0.6117 (0.0195) & 0.6117 (0.0195) & 0.6117 (0.0195) & 0.6117 (0.0195) & 0.6117 (0.0195) & 0.6117 (0.0195) & 0.6117 (0.0195) & 0.6117 (0.0195) & 0.6117 (0.0195) & 0.6117 (0.0195) \\
\hline
\multirow{1}{*}{M2} & -- & 0.6099 (0.0189) & 0.6106 (0.0189) & 0.6135 (0.0193) & 0.6140 (0.0190) & 0.6157 (0.0189) & 0.6152 (0.0196) & 0.6172 (0.0198) & 0.6192 (0.0200) & 0.6216 (0.0209) & 0.6582 (0.0274) \\
\hline
\multirow{5}{*}{M3} & $r=0$ & 0.6099 (0.0189) & 0.6106 (0.0189) & 0.6135 (0.0193) & 0.6140 (0.0190) & 0.6157 (0.0189) & 0.6152 (0.0196) & 0.6172 (0.0198) & 0.6192 (0.0200) & 0.6216 (0.0209) & 0.6582 (0.0274) \\
 & $r=0.5$ & \textbf{0.6095 (0.0191)} & 0.6102 (0.0193) & 0.6125 (0.0195) & 0.6123 (0.0196) & 0.6133 (0.0190) & 0.6134 (0.0192) & 0.6144 (0.0196) & 0.6156 (0.0191) & 0.6176 (0.0194) & 0.6388 (0.0234) \\
 & $r=1$ & 0.6097 (0.0192) & 0.6102 (0.0189) & 0.6116 (0.0195) & 0.6113 (0.0196) & 0.6123 (0.0193) & 0.6127 (0.0197) & 0.6130 (0.0188) & 0.6135 (0.0194) & 0.6154 (0.0201) & 0.6299 (0.0203) \\
 & $r=2$ & 0.6098 (0.0192) & 0.6101 (0.0192) & 0.6112 (0.0192) & 0.6112 (0.0195) & 0.6116 (0.0192) & 0.6117 (0.0192) & 0.6122 (0.0191) & 0.6120 (0.0195) & 0.6130 (0.0195) & 0.6245 (0.0198) \\
\end{tabular}
}
\endgroup
\end{table}

\begin{table}[ht]
\centering
\caption{Bike: FT-Transformer regression test MSE (FITS-G: all modes).}
\label{tab:bike_ft_regression_test_mse_fixed_modes_fits_g_all_modes_parameters}
\begingroup
\setlength{\tabcolsep}{4pt}
\renewcommand{\arraystretch}{1.35}
\resizebox{\textwidth}{!}{%
\begin{tabular}{@{}cc|*{10}{c}@{}}
\specialrule{0.8pt}{0pt}{1.5pt}
\specialrule{0.8pt}{0pt}{3pt}
\multirow{2}{*}{Mode} & \multirow{2}{*}{Parameters} & \multicolumn{10}{c}{Clients} \\
\cline{3-12}
 & & Client 0 (30\%) & Client 1 (20\%) & Client 2 (10\%) & Client 3 (9\%) & Client 4 (8\%) & Client 5 (7\%) & Client 6 (6\%) & Client 7 (5\%) & Client 8 (4\%) & Client 9 (1\%) \\
\specialrule{0.8pt}{3pt}{1.5pt}
\specialrule{0.8pt}{0pt}{3pt}
\multirow{1}{*}{M1} & -- & 0.2348 (0.0209) & 0.2348 (0.0209) & 0.2348 (0.0209) & 0.2348 (0.0209) & 0.2348 (0.0209) & 0.2348 (0.0209) & 0.2348 (0.0209) & 0.2348 (0.0209) & 0.2348 (0.0209) & 0.2348 (0.0209) \\
\hline
\multirow{1}{*}{M2} & -- & 0.0687 (0.0060) & \textbf{0.0727 (0.0041)} & 0.0896 (0.0065) & 0.0944 (0.0056) & 0.0990 (0.0070) & 0.1079 (0.0079) & 0.1154 (0.0107) & 0.1262 (0.0142) & 0.1522 (0.0192) & 0.4313 (0.0741) \\
\hline
\multirow{5}{*}{M3} & $r=0$ & 0.0687 (0.0060) & \textbf{0.0727 (0.0041)} & 0.0896 (0.0065) & 0.0944 (0.0056) & 0.0990 (0.0070) & 0.1079 (0.0079) & 0.1154 (0.0107) & 0.1262 (0.0142) & 0.1522 (0.0192) & 0.4313 (0.0741) \\
 & $r=0.5$ & 0.0935 (0.0091) & 0.1047 (0.0092) & 0.1467 (0.0123) & 0.1597 (0.0171) & 0.1661 (0.0204) & 0.1760 (0.0186) & 0.1970 (0.0202) & 0.2112 (0.0231) & 0.2334 (0.0240) & 0.4408 (0.0804) \\
 & $r=1$ & 0.1077 (0.0108) & 0.1215 (0.0133) & 0.1578 (0.0167) & 0.1751 (0.0173) & 0.1848 (0.0244) & 0.1994 (0.0196) & 0.2175 (0.0233) & 0.2352 (0.0261) & 0.2596 (0.0254) & 0.4393 (0.0654) \\
 & $r=2$ & 0.1119 (0.0112) & 0.1299 (0.0109) & 0.1649 (0.0174) & 0.1711 (0.0170) & 0.1769 (0.0202) & 0.1946 (0.0201) & 0.2165 (0.0253) & 0.2418 (0.0224) & 0.2722 (0.0296) & 0.4420 (0.0555) \\
\end{tabular}
}
\endgroup
\end{table}

\begin{table}[ht]
\centering
\caption{Adult: logistic-regression parameter estimation error (FITS-G vs FITS-K).}
\label{tab:adult_logistic_parameter_l2_error_fixed_modes_fits_g_vs_fits_k}
\begingroup
\setlength{\tabcolsep}{4pt}
\renewcommand{\arraystretch}{1.35}
\resizebox{\textwidth}{!}{%
\begin{tabular}{@{}cc|*{10}{c}@{}}
\specialrule{0.8pt}{0pt}{1.5pt}
\specialrule{0.8pt}{0pt}{3pt}
\multirow{2}{*}{Mode} & \multirow{2}{*}{Baselines} & \multicolumn{10}{c}{Clients} \\
\cline{3-12}
 & & Client 0 (30\%) & Client 1 (20\%) & Client 2 (10\%) & Client 3 (9\%) & Client 4 (8\%) & Client 5 (7\%) & Client 6 (6\%) & Client 7 (5\%) & Client 8 (4\%) & Client 9 (1\%) \\
\specialrule{0.8pt}{3pt}{1.5pt}
\specialrule{0.8pt}{0pt}{3pt}
\multirow{2}{*}{M1} & FITS-G & \textbf{1.0930 (0.1225)} & \textbf{1.0930 (0.1225)} & \textbf{1.0930 (0.1225)} & \textbf{1.0930 (0.1225)} & \textbf{1.0930 (0.1225)} & \textbf{1.0930 (0.1225)} & \textbf{1.0930 (0.1225)} & \textbf{1.0930 (0.1225)} & \textbf{1.0930 (0.1225)} & \textbf{1.0930 (0.1225)} \\
 & FITS-K & 1.2821 (0.0896) & 1.2821 (0.0896) & 1.2821 (0.0896) & 1.2821 (0.0896) & 1.2821 (0.0896) & 1.2821 (0.0896) & 1.2821 (0.0896) & 1.2821 (0.0896) & 1.2821 (0.0896) & 1.2821 (0.0896) \\
\hline
\multirow{1}{*}{M2} & None & \textbf{0.2712 (0.1082)} & \textbf{0.2810 (0.1239)} & \textbf{0.4893 (0.2531)} & \textbf{0.5538 (0.2660)} & \textbf{0.5820 (0.2656)} & \textbf{0.5930 (0.2823)} & \textbf{0.6766 (0.2996)} & \textbf{0.7805 (0.3195)} & \textbf{0.7778 (0.4350)} & \textbf{1.6560 (0.8013)} \\
\hline
\multirow{2}{*}{\shortstack{M3 }} & FITS-G & \textbf{0.7996 (0.1013)} & \textbf{0.8844 (0.1078)} & \textbf{0.9800 (0.1166)} & \textbf{0.9865 (0.1158)} & \textbf{0.9973 (0.1187)} & \textbf{1.0069 (0.1231)} & \textbf{1.0121 (0.1186)} & \textbf{1.0249 (0.1155)} & \textbf{1.0422 (0.1198)} & \textbf{1.0822 (0.1207)} \\
 & FITS-K & 0.9506 (0.0771) & 1.0472 (0.0796) & 1.1552 (0.0915) & 1.1634 (0.0879) & 1.1752 (0.0891) & 1.1857 (0.0927) & 1.1962 (0.0903) & 1.2082 (0.0856) & 1.2248 (0.0852) & 1.2697 (0.0898) \\
\end{tabular}
}
\endgroup
\end{table}

\begin{table}[ht]
\centering
\caption{Adult: logistic-regression test accuracy (FITS-G vs FITS-K).}
\label{tab:adult_logistic_test_accuracy_fixed_modes_fits_g_vs_fits_k}
\begingroup
\setlength{\tabcolsep}{4pt}
\renewcommand{\arraystretch}{1.35}
\resizebox{\textwidth}{!}{%
\begin{tabular}{@{}cc|*{10}{c}@{}}
\specialrule{0.8pt}{0pt}{1.5pt}
\specialrule{0.8pt}{0pt}{3pt}
\multirow{2}{*}{Mode} & \multirow{2}{*}{Baselines} & \multicolumn{10}{c}{Clients} \\
\cline{3-12}
 & & Client 0 (30\%) & Client 1 (20\%) & Client 2 (10\%) & Client 3 (9\%) & Client 4 (8\%) & Client 5 (7\%) & Client 6 (6\%) & Client 7 (5\%) & Client 8 (4\%) & Client 9 (1\%) \\
\specialrule{0.8pt}{3pt}{1.5pt}
\specialrule{0.8pt}{0pt}{3pt}
\multirow{2}{*}{M1} & FITS-G & \textbf{0.8187 (0.0030)} & \textbf{0.8187 (0.0030)} & \textbf{0.8187 (0.0030)} & \textbf{0.8187 (0.0030)} & \textbf{0.8187 (0.0030)} & \textbf{0.8187 (0.0030)} & \textbf{0.8187 (0.0030)} & \textbf{0.8187 (0.0030)} & \textbf{0.8187 (0.0030)} & \textbf{0.8187 (0.0030)} \\
 & FITS-K & 0.8174 (0.0030) & 0.8174 (0.0030) & 0.8174 (0.0030) & 0.8174 (0.0030) & 0.8174 (0.0030) & 0.8174 (0.0030) & 0.8174 (0.0030) & 0.8174 (0.0030) & 0.8174 (0.0030) & 0.8174 (0.0030) \\
\hline
\multirow{1}{*}{M2} & None & \textbf{0.8204 (0.0031)} & \textbf{0.8207 (0.0032)} & \textbf{0.8204 (0.0036)} & \textbf{0.8199 (0.0034)} & \textbf{0.8202 (0.0028)} & \textbf{0.8194 (0.0038)} & \textbf{0.8196 (0.0034)} & \textbf{0.8195 (0.0033)} & \textbf{0.8189 (0.0038)} & \textbf{0.8131 (0.0054)} \\
\hline
\multirow{2}{*}{\shortstack{M3 }} & FITS-G & \textbf{0.8201 (0.0029)} & \textbf{0.8198 (0.0029)} & \textbf{0.8195 (0.0032)} & \textbf{0.8193 (0.0030)} & \textbf{0.8193 (0.0030)} & \textbf{0.8192 (0.0029)} & \textbf{0.8192 (0.0029)} & \textbf{0.8190 (0.0030)} & \textbf{0.8190 (0.0030)} & \textbf{0.8188 (0.0030)} \\
 & FITS-K & 0.8195 (0.0031) & 0.8191 (0.0032) & 0.8184 (0.0032) & 0.8185 (0.0032) & 0.8181 (0.0030) & 0.8182 (0.0030) & 0.8182 (0.0030) & 0.8180 (0.0030) & 0.8178 (0.0030) & 0.8176 (0.0030) \\
\end{tabular}
}
\endgroup
\end{table}

\begin{table}[ht]
\centering
\caption{Adult: FT-Transformer classification test accuracy (FITS-G vs FITS-K).}
\label{tab:adult_ft_classification_test_accuracy_fixed_modes_fits_g_vs_fits_k}
\begingroup
\setlength{\tabcolsep}{4pt}
\renewcommand{\arraystretch}{1.35}
\resizebox{\textwidth}{!}{%
\begin{tabular}{@{}cc|*{10}{c}@{}}
\specialrule{0.8pt}{0pt}{1.5pt}
\specialrule{0.8pt}{0pt}{3pt}
\multirow{2}{*}{Mode} & \multirow{2}{*}{Baselines} & \multicolumn{10}{c}{Clients} \\
\cline{3-12}
 & & Client 0 (30\%) & Client 1 (20\%) & Client 2 (10\%) & Client 3 (9\%) & Client 4 (8\%) & Client 5 (7\%) & Client 6 (6\%) & Client 7 (5\%) & Client 8 (4\%) & Client 9 (1\%) \\
\specialrule{0.8pt}{3pt}{1.5pt}
\specialrule{0.8pt}{0pt}{3pt}
\multirow{2}{*}{M1} & FITS-G & \textbf{0.8478 (0.0039)} & \textbf{0.8478 (0.0039)} & \textbf{0.8478 (0.0039)} & \textbf{0.8478 (0.0039)} & \textbf{0.8478 (0.0039)} & \textbf{0.8478 (0.0039)} & \textbf{0.8478 (0.0039)} & \textbf{0.8478 (0.0039)} & \textbf{0.8478 (0.0039)} & \textbf{0.8478 (0.0039)} \\
 & FITS-K & 0.8465 (0.0038) & 0.8465 (0.0038) & 0.8465 (0.0038) & 0.8465 (0.0038) & 0.8465 (0.0038) & 0.8465 (0.0038) & 0.8465 (0.0038) & 0.8465 (0.0038) & 0.8465 (0.0038) & 0.8465 (0.0038) \\
\hline
\multirow{1}{*}{M2} & None & \textbf{0.8513 (0.0045)} & \textbf{0.8506 (0.0040)} & \textbf{0.8394 (0.0076)} & \textbf{0.8353 (0.0086)} & \textbf{0.8304 (0.0081)} & \textbf{0.8290 (0.0079)} & \textbf{0.8231 (0.0091)} & \textbf{0.8176 (0.0068)} & \textbf{0.8098 (0.0075)} & \textbf{0.8015 (0.0102)} \\
\hline
\multirow{2}{*}{\shortstack{M3 }} & FITS-G & \textbf{0.8513 (0.0042)} & \textbf{0.8513 (0.0039)} & \textbf{0.8498 (0.0043)} & \textbf{0.8498 (0.0044)} & \textbf{0.8496 (0.0039)} & \textbf{0.8493 (0.0045)} & \textbf{0.8496 (0.0044)} & \textbf{0.8491 (0.0045)} & \textbf{0.8474 (0.0053)} & \textbf{0.8485 (0.0048)} \\
 & FITS-K & 0.8507 (0.0038) & 0.8499 (0.0048) & 0.8487 (0.0039) & 0.8484 (0.0043) & 0.8475 (0.0044) & 0.8477 (0.0049) & 0.8482 (0.0040) & 0.8476 (0.0052) & 0.8461 (0.0042) & 0.8467 (0.0042) \\
\end{tabular}
}
\endgroup
\end{table}

\begin{table}[ht]
\centering
\caption{Adult: logistic-regression parameter estimation error (FITS-G: all modes).}
\label{tab:adult_logistic_parameter_l2_error_fixed_modes_fits_g_all_modes_parameters}
\begingroup
\setlength{\tabcolsep}{4pt}
\renewcommand{\arraystretch}{1.35}
\resizebox{\textwidth}{!}{%
\begin{tabular}{@{}cc|*{10}{c}@{}}
\specialrule{0.8pt}{0pt}{1.5pt}
\specialrule{0.8pt}{0pt}{3pt}
\multirow{2}{*}{Mode} & \multirow{2}{*}{Parameters} & \multicolumn{10}{c}{Clients} \\
\cline{3-12}
 & & Client 0 (30\%) & Client 1 (20\%) & Client 2 (10\%) & Client 3 (9\%) & Client 4 (8\%) & Client 5 (7\%) & Client 6 (6\%) & Client 7 (5\%) & Client 8 (4\%) & Client 9 (1\%) \\
\specialrule{0.8pt}{3pt}{1.5pt}
\specialrule{0.8pt}{0pt}{3pt}
\multirow{1}{*}{M1} & -- & 1.0930 (0.1225) & 1.0930 (0.1225) & 1.0930 (0.1225) & 1.0930 (0.1225) & 1.0930 (0.1225) & 1.0930 (0.1225) & 1.0930 (0.1225) & 1.0930 (0.1225) & 1.0930 (0.1225) & 1.0930 (0.1225) \\
\hline
\multirow{1}{*}{M2} & -- & 0.2712 (0.1082) & 0.2810 (0.1239) & 0.4893 (0.2531) & 0.5538 (0.2660) & 0.5820 (0.2656) & 0.5930 (0.2823) & 0.6766 (0.2996) & 0.7805 (0.3195) & 0.7778 (0.4350) & 1.6560 (0.8013) \\
\hline
\multirow{5}{*}{M3} & $r=0$ & 0.2712 (0.1082) & 0.2810 (0.1239) & 0.4893 (0.2531) & 0.5538 (0.2660) & 0.5820 (0.2656) & 0.5930 (0.2823) & 0.6766 (0.2996) & 0.7805 (0.3195) & 0.7778 (0.4350) & 1.6560 (0.8013) \\
 & $r=0.5$ & 0.4937 (0.1383) & 0.5016 (0.1498) & 0.6129 (0.2177) & 0.5926 (0.2377) & 0.6516 (0.2328) & 0.6575 (0.2170) & 0.6018 (0.2645) & 0.6413 (0.2764) & 0.7271 (0.3544) & 1.3037 (0.5863) \\
 & $r=1$ & 0.6709 (0.1233) & 0.6605 (0.1319) & 0.7212 (0.2112) & 0.7154 (0.2037) & 0.7412 (0.2095) & 0.7044 (0.2303) & 0.6906 (0.2636) & 0.7160 (0.3141) & 0.8068 (0.2605) & 1.2527 (0.6042) \\
 & $r=2$ & 0.7989 (0.1008) & 0.8299 (0.1220) & 0.8310 (0.1454) & 0.8544 (0.1743) & 0.8657 (0.1873) & 0.8399 (0.1944) & 0.8511 (0.1913) & 0.8614 (0.2080) & 0.8943 (0.2216) & 1.1270 (0.4059) \\
\end{tabular}
}
\endgroup
\end{table}

\begin{table}[ht]
\centering
\caption{Adult: logistic-regression test accuracy (FITS-G: all modes).}
\label{tab:adult_logistic_test_accuracy_fixed_modes_fits_g_all_modes_parameters}
\begingroup
\setlength{\tabcolsep}{4pt}
\renewcommand{\arraystretch}{1.35}
\resizebox{\textwidth}{!}{%
\begin{tabular}{@{}cc|*{10}{c}@{}}
\specialrule{0.8pt}{0pt}{1.5pt}
\specialrule{0.8pt}{0pt}{3pt}
\multirow{2}{*}{Mode} & \multirow{2}{*}{Parameters} & \multicolumn{10}{c}{Clients} \\
\cline{3-12}
 & & Client 0 (30\%) & Client 1 (20\%) & Client 2 (10\%) & Client 3 (9\%) & Client 4 (8\%) & Client 5 (7\%) & Client 6 (6\%) & Client 7 (5\%) & Client 8 (4\%) & Client 9 (1\%) \\
\specialrule{0.8pt}{3pt}{1.5pt}
\specialrule{0.8pt}{0pt}{3pt}
\multirow{1}{*}{M1} & -- & 0.8187 (0.0030) & 0.8187 (0.0030) & 0.8187 (0.0030) & 0.8187 (0.0030) & 0.8187 (0.0030) & 0.8187 (0.0030) & 0.8187 (0.0030) & 0.8187 (0.0030) & 0.8187 (0.0030) & 0.8187 (0.0030) \\
\hline
\multirow{1}{*}{M2} & -- & 0.8204 (0.0031) & 0.8207 (0.0032) & 0.8204 (0.0036) & 0.8199 (0.0034) & 0.8202 (0.0028) & 0.8194 (0.0038) & 0.8196 (0.0034) & 0.8195 (0.0033) & 0.8189 (0.0038) & 0.8131 (0.0054) \\
\hline
\multirow{5}{*}{M3} & $r=0$ & 0.8204 (0.0031) & 0.8207 (0.0032) & 0.8204 (0.0036) & 0.8199 (0.0034) & 0.8202 (0.0028) & 0.8194 (0.0038) & 0.8196 (0.0034) & 0.8195 (0.0033) & 0.8189 (0.0038) & 0.8131 (0.0054) \\
 & $r=0.5$ & 0.8204 (0.0032) & 0.8205 (0.0031) & 0.8204 (0.0034) & 0.8202 (0.0033) & 0.8201 (0.0028) & 0.8199 (0.0036) & 0.8196 (0.0033) & 0.8196 (0.0033) & 0.8191 (0.0039) & 0.8152 (0.0046) \\
 & $r=1$ & 0.8202 (0.0030) & 0.8204 (0.0031) & 0.8202 (0.0034) & 0.8199 (0.0033) & 0.8199 (0.0030) & 0.8195 (0.0035) & 0.8198 (0.0032) & 0.8193 (0.0036) & 0.8188 (0.0039) & 0.8163 (0.0047) \\
 & $r=2$ & 0.8201 (0.0029) & 0.8199 (0.0029) & 0.8198 (0.0033) & 0.8198 (0.0029) & 0.8197 (0.0029) & 0.8196 (0.0031) & 0.8194 (0.0030) & 0.8196 (0.0034) & 0.8191 (0.0033) & 0.8177 (0.0036) \\
\end{tabular}
}
\endgroup
\end{table}

\begin{table}[ht]
\centering
\caption{Adult: FT-Transformer classification test accuracy (FITS-G: all modes).}
\label{tab:adult_ft_classification_test_accuracy_fixed_modes_fits_g_all_modes_parameters}
\begingroup
\setlength{\tabcolsep}{4pt}
\renewcommand{\arraystretch}{1.35}
\resizebox{\textwidth}{!}{%
\begin{tabular}{@{}cc|*{10}{c}@{}}
\specialrule{0.8pt}{0pt}{1.5pt}
\specialrule{0.8pt}{0pt}{3pt}
\multirow{2}{*}{Mode} & \multirow{2}{*}{Parameters} & \multicolumn{10}{c}{Clients} \\
\cline{3-12}
 & & Client 0 (30\%) & Client 1 (20\%) & Client 2 (10\%) & Client 3 (9\%) & Client 4 (8\%) & Client 5 (7\%) & Client 6 (6\%) & Client 7 (5\%) & Client 8 (4\%) & Client 9 (1\%) \\
\specialrule{0.8pt}{3pt}{1.5pt}
\specialrule{0.8pt}{0pt}{3pt}
\multirow{1}{*}{M1} & -- & 0.8478 (0.0039) & 0.8478 (0.0039) & 0.8478 (0.0039) & 0.8478 (0.0039) & 0.8478 (0.0039) & 0.8478 (0.0039) & 0.8478 (0.0039) & 0.8478 (0.0039) & 0.8478 (0.0039) & 0.8478 (0.0039) \\
\hline
\multirow{1}{*}{M2} & -- & 0.8513 (0.0045) & 0.8506 (0.0040) & 0.8394 (0.0076) & 0.8353 (0.0086) & 0.8304 (0.0081) & 0.8290 (0.0079) & 0.8231 (0.0091) & 0.8176 (0.0068) & 0.8098 (0.0075) & 0.8015 (0.0102) \\
\hline
\multirow{5}{*}{M3} & $r=0$ & 0.8513 (0.0045) & 0.8506 (0.0040) & 0.8394 (0.0076) & 0.8353 (0.0086) & 0.8304 (0.0081) & 0.8290 (0.0079) & 0.8231 (0.0091) & 0.8176 (0.0068) & 0.8098 (0.0075) & 0.8015 (0.0102) \\
 & $r=0.5$ & 0.8526 (0.0040) & 0.8498 (0.0050) & 0.8444 (0.0062) & 0.8445 (0.0054) & 0.8397 (0.0065) & 0.8369 (0.0074) & 0.8318 (0.0068) & 0.8268 (0.0076) & 0.8181 (0.0076) & 0.7954 (0.0112) \\
 & $r=1$ & 0.8514 (0.0049) & 0.8501 (0.0044) & 0.8475 (0.0041) & 0.8451 (0.0047) & 0.8438 (0.0051) & 0.8425 (0.0051) & 0.8382 (0.0056) & 0.8329 (0.0083) & 0.8242 (0.0087) & 0.7942 (0.0106) \\
 & $r=2$ & 0.8516 (0.0041) & 0.8503 (0.0049) & 0.8485 (0.0039) & 0.8477 (0.0053) & 0.8463 (0.0053) & 0.8465 (0.0047) & 0.8442 (0.0065) & 0.8418 (0.0060) & 0.8371 (0.0061) & 0.7967 (0.0092) \\
\end{tabular}
}
\endgroup
\end{table}

\begin{table}[ht]
\centering
\caption{Simulation: linear-regression test MSE.}
\label{tab:simulation_linear_test_mse_fixed_modes}
\begingroup
\setlength{\tabcolsep}{4pt}
\renewcommand{\arraystretch}{1.35}
\resizebox{\textwidth}{!}{%
\begin{tabular}{@{}cc|*{10}{c}@{}}
\specialrule{0.8pt}{0pt}{1.5pt}
\specialrule{0.8pt}{0pt}{3pt}
\multirow{2}{*}{Mode} & \multirow{2}{*}{Baselines} & \multicolumn{10}{c}{Clients} \\
\cline{3-12}
 & & Client 0 (30\%) & Client 1 (20\%) & Client 2 (10\%) & Client 3 (9\%) & Client 4 (8\%) & Client 5 (7\%) & Client 6 (6\%) & Client 7 (5\%) & Client 8 (4\%) & Client 9 (1\%) \\
\specialrule{0.8pt}{3pt}{1.5pt}
\specialrule{0.8pt}{0pt}{3pt}
\multirow{3}{*}{M1} & FITS-G & \textbf{0.4971 (0.0127)} & \textbf{0.4971 (0.0127)} & \textbf{0.4971 (0.0127)} & \textbf{0.4971 (0.0127)} & \textbf{0.4971 (0.0127)} & \textbf{0.4971 (0.0127)} & \textbf{0.4971 (0.0127)} & \textbf{0.4971 (0.0127)} & \textbf{0.4971 (0.0127)} & \textbf{0.4971 (0.0127)} \\
 & HeteroFedSyn & 3.5888 (0.3499) & 3.5888 (0.3499) & 3.5888 (0.3499) & 3.5888 (0.3499) & 3.5888 (0.3499) & 3.5888 (0.3499) & 3.5888 (0.3499) & 3.5888 (0.3499) & 3.5888 (0.3499) & 3.5888 (0.3499) \\
 & FedTabDiff & 0.7209 (0.1486) & 0.7209 (0.1486) & 0.7209 (0.1486) & 0.7209 (0.1486) & 0.7209 (0.1486) & 0.7209 (0.1486) & 0.7209 (0.1486) & 0.7209 (0.1486) & 0.7209 (0.1486) & 0.7209 (0.1486) \\
\hline
\multirow{1}{*}{M2} & None & \textbf{0.4926 (0.0113)} & \textbf{0.4942 (0.0111)} & \textbf{0.4968 (0.0115)} & \textbf{0.4970 (0.0120)} & \textbf{0.4975 (0.0116)} & \textbf{0.4997 (0.0123)} & \textbf{0.5015 (0.0135)} & \textbf{0.5032 (0.0133)} & \textbf{0.5057 (0.0126)} & \textbf{0.5602 (0.0375)} \\
\hline
\multirow{3}{*}{\shortstack{M3 }} & FITS-G & \textbf{0.4938 (0.0120)} & \textbf{0.4947 (0.0120)} & \textbf{0.4956 (0.0123)} & \textbf{0.4958 (0.0123)} & \textbf{0.4960 (0.0125)} & \textbf{0.4960 (0.0124)} & \textbf{0.4962 (0.0125)} & \textbf{0.4963 (0.0124)} & \textbf{0.4966 (0.0126)} & \textbf{0.4969 (0.0127)} \\
 & HeteroFedSyn & 1.8255 (0.1602) & 2.1875 (0.1923) & 2.7177 (0.2436) & 2.8014 (0.2706) & 2.8567 (0.2786) & 2.9276 (0.2814) & 3.0330 (0.2946) & 3.1035 (0.2907) & 3.1842 (0.2984) & 3.4802 (0.3405) \\
\end{tabular}
}
\endgroup
\end{table}

\begin{table}[ht]
\centering
\caption{Simulation: logistic-regression test accuracy.}
\label{tab:simulation_logistic_test_accuracy_fixed_modes}
\begingroup
\setlength{\tabcolsep}{4pt}
\renewcommand{\arraystretch}{1.35}
\resizebox{\textwidth}{!}{%
\begin{tabular}{@{}cc|*{10}{c}@{}}
\specialrule{0.8pt}{0pt}{1.5pt}
\specialrule{0.8pt}{0pt}{3pt}
\multirow{2}{*}{Mode} & \multirow{2}{*}{Baselines} & \multicolumn{10}{c}{Clients} \\
\cline{3-12}
 & & Client 0 (30\%) & Client 1 (20\%) & Client 2 (10\%) & Client 3 (9\%) & Client 4 (8\%) & Client 5 (7\%) & Client 6 (6\%) & Client 7 (5\%) & Client 8 (4\%) & Client 9 (1\%) \\
\specialrule{0.8pt}{3pt}{1.5pt}
\specialrule{0.8pt}{0pt}{3pt}
\multirow{3}{*}{M1} & FITS-G & \textbf{0.8759 (0.0059)} & \textbf{0.8759 (0.0059)} & \textbf{0.8759 (0.0059)} & \textbf{0.8759 (0.0059)} & \textbf{0.8759 (0.0059)} & \textbf{0.8759 (0.0059)} & \textbf{0.8759 (0.0059)} & \textbf{0.8759 (0.0059)} & \textbf{0.8759 (0.0059)} & \textbf{0.8759 (0.0059)} \\
 & HeteroFedSyn & 0.8147 (0.0184) & 0.8147 (0.0184) & 0.8147 (0.0184) & 0.8147 (0.0184) & 0.8147 (0.0184) & 0.8147 (0.0184) & 0.8147 (0.0184) & 0.8147 (0.0184) & 0.8147 (0.0184) & 0.8147 (0.0184) \\
 & FedTabDiff & 0.8687 (0.0058) & 0.8687 (0.0058) & 0.8687 (0.0058) & 0.8687 (0.0058) & 0.8687 (0.0058) & 0.8687 (0.0058) & 0.8687 (0.0058) & 0.8687 (0.0058) & 0.8687 (0.0058) & 0.8687 (0.0058) \\
\hline
\multirow{1}{*}{M2} & None & \textbf{0.8763 (0.0054)} & \textbf{0.8760 (0.0055)} & \textbf{0.8748 (0.0055)} & \textbf{0.8736 (0.0058)} & \textbf{0.8745 (0.0058)} & \textbf{0.8738 (0.0059)} & \textbf{0.8727 (0.0058)} & \textbf{0.8722 (0.0065)} & \textbf{0.8705 (0.0077)} & \textbf{0.8488 (0.0177)} \\
\hline
\multirow{3}{*}{\shortstack{M3 }} & FITS-G & \textbf{0.8763 (0.0056)} & \textbf{0.8761 (0.0056)} & \textbf{0.8762 (0.0058)} & \textbf{0.8758 (0.0059)} & \textbf{0.8760 (0.0059)} & \textbf{0.8760 (0.0057)} & \textbf{0.8762 (0.0058)} & \textbf{0.8759 (0.0059)} & \textbf{0.8760 (0.0060)} & \textbf{0.8760 (0.0059)} \\
 & HeteroFedSyn & 0.8583 (0.0081) & 0.8509 (0.0092) & 0.8397 (0.0102) & 0.8373 (0.0110) & 0.8356 (0.0114) & 0.8331 (0.0116) & 0.8308 (0.0125) & 0.8288 (0.0127) & 0.8268 (0.0137) & 0.8180 (0.0170) \\
\end{tabular}
}
\endgroup
\end{table}

\begin{table}[ht]
\centering
\caption{Simulation: FT-Transformer regression test MSE.}
\label{tab:simulation_ft_regression_test_mse_fixed_modes}
\begingroup
\setlength{\tabcolsep}{4pt}
\renewcommand{\arraystretch}{1.35}
\resizebox{\textwidth}{!}{%
\begin{tabular}{@{}cc|*{10}{c}@{}}
\specialrule{0.8pt}{0pt}{1.5pt}
\specialrule{0.8pt}{0pt}{3pt}
\multirow{2}{*}{Mode} & \multirow{2}{*}{Baselines} & \multicolumn{10}{c}{Clients} \\
\cline{3-12}
 & & Client 0 (30\%) & Client 1 (20\%) & Client 2 (10\%) & Client 3 (9\%) & Client 4 (8\%) & Client 5 (7\%) & Client 6 (6\%) & Client 7 (5\%) & Client 8 (4\%) & Client 9 (1\%) \\
\specialrule{0.8pt}{3pt}{1.5pt}
\specialrule{0.8pt}{0pt}{3pt}
\multirow{3}{*}{M1} & FITS-G & \textbf{0.6028 (0.0433)} & \textbf{0.6028 (0.0433)} & \textbf{0.6028 (0.0433)} & \textbf{0.6028 (0.0433)} & \textbf{0.6028 (0.0433)} & \textbf{0.6028 (0.0433)} & \textbf{0.6028 (0.0433)} & \textbf{0.6028 (0.0433)} & \textbf{0.6028 (0.0433)} & \textbf{0.6028 (0.0433)} \\
 & HeteroFedSyn & 5.5571 (0.9458) & 5.5571 (0.9458) & 5.5571 (0.9458) & 5.5571 (0.9458) & 5.5571 (0.9458) & 5.5571 (0.9458) & 5.5571 (0.9458) & 5.5571 (0.9458) & 5.5571 (0.9458) & 5.5571 (0.9458) \\
 & FedTabDiff & 0.7994 (0.1271) & 0.7994 (0.1271) & 0.7994 (0.1271) & 0.7994 (0.1271) & 0.7994 (0.1271) & 0.7994 (0.1271) & 0.7994 (0.1271) & 0.7994 (0.1271) & 0.7994 (0.1271) & 0.7994 (0.1271) \\
\hline
\multirow{1}{*}{M2} & None & \textbf{0.5567 (0.0273)} & \textbf{0.5889 (0.0319)} & \textbf{0.7232 (0.0644)} & \textbf{0.7731 (0.0580)} & \textbf{0.8112 (0.0694)} & \textbf{0.8747 (0.0853)} & \textbf{0.9386 (0.1001)} & \textbf{1.0237 (0.0881)} & \textbf{1.1248 (0.1095)} & \textbf{1.4740 (0.2824)} \\
\hline
\multirow{3}{*}{\shortstack{M3 }} & FITS-G & \textbf{0.5625 (0.0341)} & \textbf{0.5688 (0.0345)} & \textbf{0.5777 (0.0468)} & \textbf{0.5756 (0.0371)} & \textbf{0.5806 (0.0384)} & \textbf{0.5763 (0.0269)} & \textbf{0.5970 (0.0547)} & \textbf{0.5820 (0.0376)} & \textbf{0.5748 (0.0277)} & \textbf{0.5910 (0.0437)} \\
 & HeteroFedSyn & 2.1663 (0.2894) & 2.6479 (0.3289) & 3.5003 (0.4756) & 3.6347 (0.4695) & 3.7814 (0.4829) & 3.9253 (0.5653) & 4.1741 (0.6395) & 4.2038 (0.6069) & 4.4703 (0.6628) & 5.0962 (0.7845) \\
\end{tabular}
}
\endgroup
\end{table}

\begin{table}[ht]
\centering
\caption{Simulation: linear-regression test MSE (FITS-G vs FITS-K).}
\label{tab:simulation_linear_test_mse_fixed_modes_fits_g_vs_fits_k}
\begingroup
\setlength{\tabcolsep}{4pt}
\renewcommand{\arraystretch}{1.35}
\resizebox{\textwidth}{!}{%
\begin{tabular}{@{}cc|*{10}{c}@{}}
\specialrule{0.8pt}{0pt}{1.5pt}
\specialrule{0.8pt}{0pt}{3pt}
\multirow{2}{*}{Mode} & \multirow{2}{*}{Baselines} & \multicolumn{10}{c}{Clients} \\
\cline{3-12}
 & & Client 0 (30\%) & Client 1 (20\%) & Client 2 (10\%) & Client 3 (9\%) & Client 4 (8\%) & Client 5 (7\%) & Client 6 (6\%) & Client 7 (5\%) & Client 8 (4\%) & Client 9 (1\%) \\
\specialrule{0.8pt}{3pt}{1.5pt}
\specialrule{0.8pt}{0pt}{3pt}
\multirow{2}{*}{M1} & FITS-G & \textbf{0.4971 (0.0127)} & \textbf{0.4971 (0.0127)} & \textbf{0.4971 (0.0127)} & \textbf{0.4971 (0.0127)} & \textbf{0.4971 (0.0127)} & \textbf{0.4971 (0.0127)} & \textbf{0.4971 (0.0127)} & \textbf{0.4971 (0.0127)} & \textbf{0.4971 (0.0127)} & \textbf{0.4971 (0.0127)} \\
 & FITS-K & 0.4986 (0.0124) & 0.4986 (0.0124) & 0.4986 (0.0124) & 0.4986 (0.0124) & 0.4986 (0.0124) & 0.4986 (0.0124) & 0.4986 (0.0124) & 0.4986 (0.0124) & 0.4986 (0.0124) & 0.4986 (0.0124) \\
\hline
\multirow{1}{*}{M2} & None & \textbf{0.4926 (0.0113)} & \textbf{0.4942 (0.0111)} & \textbf{0.4968 (0.0115)} & \textbf{0.4970 (0.0120)} & \textbf{0.4975 (0.0116)} & \textbf{0.4997 (0.0123)} & \textbf{0.5015 (0.0135)} & \textbf{0.5032 (0.0133)} & \textbf{0.5057 (0.0126)} & \textbf{0.5602 (0.0375)} \\
\hline
\multirow{2}{*}{\shortstack{M3 }} & FITS-G & \textbf{0.4938 (0.0120)} & \textbf{0.4947 (0.0120)} & \textbf{0.4956 (0.0123)} & \textbf{0.4958 (0.0123)} & \textbf{0.4960 (0.0125)} & \textbf{0.4960 (0.0124)} & \textbf{0.4962 (0.0125)} & \textbf{0.4963 (0.0124)} & \textbf{0.4966 (0.0126)} & \textbf{0.4969 (0.0127)} \\
 & FITS-K & 0.4944 (0.0118) & 0.4955 (0.0118) & 0.4966 (0.0120) & 0.4967 (0.0120) & 0.4969 (0.0122) & 0.4971 (0.0121) & 0.4974 (0.0121) & 0.4975 (0.0121) & 0.4978 (0.0123) & 0.4982 (0.0124) \\
\end{tabular}
}
\endgroup
\end{table}

\begin{table}[ht]
\centering
\caption{Simulation: logistic-regression test accuracy (FITS-G vs FITS-K).}
\label{tab:simulation_logistic_test_accuracy_fixed_modes_fits_g_vs_fits_k}
\begingroup
\setlength{\tabcolsep}{4pt}
\renewcommand{\arraystretch}{1.35}
\resizebox{\textwidth}{!}{%
\begin{tabular}{@{}cc|*{10}{c}@{}}
\specialrule{0.8pt}{0pt}{1.5pt}
\specialrule{0.8pt}{0pt}{3pt}
\multirow{2}{*}{Mode} & \multirow{2}{*}{Baselines} & \multicolumn{10}{c}{Clients} \\
\cline{3-12}
 & & Client 0 (30\%) & Client 1 (20\%) & Client 2 (10\%) & Client 3 (9\%) & Client 4 (8\%) & Client 5 (7\%) & Client 6 (6\%) & Client 7 (5\%) & Client 8 (4\%) & Client 9 (1\%) \\
\specialrule{0.8pt}{3pt}{1.5pt}
\specialrule{0.8pt}{0pt}{3pt}
\multirow{2}{*}{M1} & FITS-G & 0.8759 (0.0059) & 0.8759 (0.0059) & 0.8759 (0.0059) & 0.8759 (0.0059) & 0.8759 (0.0059) & 0.8759 (0.0059) & 0.8759 (0.0059) & 0.8759 (0.0059) & 0.8759 (0.0059) & 0.8759 (0.0059) \\
 & FITS-K & \textbf{0.8759 (0.0055)} & \textbf{0.8759 (0.0055)} & \textbf{0.8759 (0.0055)} & \textbf{0.8759 (0.0055)} & \textbf{0.8759 (0.0055)} & \textbf{0.8759 (0.0055)} & \textbf{0.8759 (0.0055)} & \textbf{0.8759 (0.0055)} & \textbf{0.8759 (0.0055)} & \textbf{0.8759 (0.0055)} \\
\hline
\multirow{1}{*}{M2} & None & \textbf{0.8763 (0.0054)} & \textbf{0.8760 (0.0055)} & \textbf{0.8748 (0.0055)} & \textbf{0.8736 (0.0058)} & \textbf{0.8745 (0.0058)} & \textbf{0.8738 (0.0059)} & \textbf{0.8727 (0.0058)} & \textbf{0.8722 (0.0065)} & \textbf{0.8705 (0.0077)} & \textbf{0.8488 (0.0177)} \\
\hline
\multirow{2}{*}{\shortstack{M3 }} & FITS-G & 0.8763 (0.0056) & 0.8761 (0.0056) & 0.8762 (0.0058) & 0.8758 (0.0059) & 0.8760 (0.0059) & 0.8760 (0.0057) & \textbf{0.8762 (0.0058)} & 0.8759 (0.0059) & 0.8760 (0.0060) & \textbf{0.8760 (0.0059)} \\
 & FITS-K & \textbf{0.8765 (0.0054)} & \textbf{0.8763 (0.0053)} & \textbf{0.8763 (0.0054)} & \textbf{0.8761 (0.0057)} & \textbf{0.8762 (0.0053)} & \textbf{0.8761 (0.0054)} & 0.8761 (0.0052) & \textbf{0.8761 (0.0056)} & \textbf{0.8761 (0.0055)} & 0.8759 (0.0054) \\
\hline
\multirow{2}{*}{\shortstack{M4 }} & FITS-G & 0.8764 (0.0056) & 0.8763 (0.0055) & 0.8764 (0.0057) & 0.8758 (0.0058) & \textbf{0.8761 (0.0057)} & 0.8760 (0.0056) & \textbf{0.8760 (0.0056)} & 0.8759 (0.0057) & 0.8760 (0.0060) & \textbf{0.8761 (0.0059)} \\
\end{tabular}
}
\endgroup
\end{table}

\begin{table}[ht]
\centering
\caption{Simulation: FT-Transformer regression test MSE (FITS-G vs FITS-K).}
\label{tab:simulation_ft_regression_test_mse_fixed_modes_fits_g_vs_fits_k}
\begingroup
\setlength{\tabcolsep}{4pt}
\renewcommand{\arraystretch}{1.35}
\resizebox{\textwidth}{!}{%
\begin{tabular}{@{}cc|*{10}{c}@{}}
\specialrule{0.8pt}{0pt}{1.5pt}
\specialrule{0.8pt}{0pt}{3pt}
\multirow{2}{*}{Mode} & \multirow{2}{*}{Baselines} & \multicolumn{10}{c}{Clients} \\
\cline{3-12}
 & & Client 0 (30\%) & Client 1 (20\%) & Client 2 (10\%) & Client 3 (9\%) & Client 4 (8\%) & Client 5 (7\%) & Client 6 (6\%) & Client 7 (5\%) & Client 8 (4\%) & Client 9 (1\%) \\
\specialrule{0.8pt}{3pt}{1.5pt}
\specialrule{0.8pt}{0pt}{3pt}
\multirow{2}{*}{M1} & FITS-G & \textbf{0.6028 (0.0433)} & \textbf{0.6028 (0.0433)} & \textbf{0.6028 (0.0433)} & \textbf{0.6028 (0.0433)} & \textbf{0.6028 (0.0433)} & \textbf{0.6028 (0.0433)} & \textbf{0.6028 (0.0433)} & \textbf{0.6028 (0.0433)} & \textbf{0.6028 (0.0433)} & \textbf{0.6028 (0.0433)} \\
 & FITS-K & 0.6358 (0.0519) & 0.6358 (0.0519) & 0.6358 (0.0519) & 0.6358 (0.0519) & 0.6358 (0.0519) & 0.6358 (0.0519) & 0.6358 (0.0519) & 0.6358 (0.0519) & 0.6358 (0.0519) & 0.6358 (0.0519) \\
\hline
\multirow{1}{*}{M2} & None & \textbf{0.5567 (0.0273)} & \textbf{0.5889 (0.0319)} & \textbf{0.7232 (0.0644)} & \textbf{0.7731 (0.0580)} & \textbf{0.8112 (0.0694)} & \textbf{0.8747 (0.0853)} & \textbf{0.9386 (0.1001)} & \textbf{1.0237 (0.0881)} & \textbf{1.1248 (0.1095)} & \textbf{1.4740 (0.2824)} \\
\hline
\multirow{2}{*}{\shortstack{M3 }} & FITS-G & \textbf{0.5625 (0.0341)} & \textbf{0.5688 (0.0345)} & \textbf{0.5777 (0.0468)} & \textbf{0.5756 (0.0371)} & \textbf{0.5806 (0.0384)} & \textbf{0.5763 (0.0269)} & \textbf{0.5970 (0.0547)} & \textbf{0.5820 (0.0376)} & \textbf{0.5748 (0.0277)} & \textbf{0.5910 (0.0437)} \\
 & FITS-K & 0.5714 (0.0423) & 0.5757 (0.0258) & 0.5883 (0.0368) & 0.5934 (0.0364) & 0.6013 (0.0440) & 0.5951 (0.0382) & 0.6050 (0.0395) & 0.6055 (0.0391) & 0.5987 (0.0392) & 0.6189 (0.0497) \\
\end{tabular}
}
\endgroup
\end{table}

\begin{table}[ht]
\centering
\caption{Simulation: linear-regression test MSE (FITS-G: all modes).}
\label{tab:simulation_linear_test_mse_fixed_modes_fits_g_all_modes_parameters}
\begingroup
\setlength{\tabcolsep}{4pt}
\renewcommand{\arraystretch}{1.35}
\resizebox{\textwidth}{!}{%
\begin{tabular}{@{}cc|*{10}{c}@{}}
\specialrule{0.8pt}{0pt}{1.5pt}
\specialrule{0.8pt}{0pt}{3pt}
\multirow{2}{*}{Mode} & \multirow{2}{*}{Parameters} & \multicolumn{10}{c}{Clients} \\
\cline{3-12}
 & & Client 0 (30\%) & Client 1 (20\%) & Client 2 (10\%) & Client 3 (9\%) & Client 4 (8\%) & Client 5 (7\%) & Client 6 (6\%) & Client 7 (5\%) & Client 8 (4\%) & Client 9 (1\%) \\
\specialrule{0.8pt}{3pt}{1.5pt}
\specialrule{0.8pt}{0pt}{3pt}
\multirow{1}{*}{M1} & -- & 0.4971 (0.0127) & 0.4971 (0.0127) & 0.4971 (0.0127) & 0.4971 (0.0127) & 0.4971 (0.0127) & 0.4971 (0.0127) & 0.4971 (0.0127) & 0.4971 (0.0127) & 0.4971 (0.0127) & 0.4971 (0.0127) \\
\hline
\multirow{1}{*}{M2} & -- & 0.4926 (0.0113) & 0.4942 (0.0111) & 0.4968 (0.0115) & 0.4970 (0.0120) & 0.4975 (0.0116) & 0.4997 (0.0123) & 0.5015 (0.0135) & 0.5032 (0.0133) & 0.5057 (0.0126) & 0.5602 (0.0375) \\
\hline
\multirow{5}{*}{M3} & $r=0$ & 0.4926 (0.0113) & 0.4942 (0.0111) & 0.4968 (0.0115) & 0.4970 (0.0120) & 0.4975 (0.0116) & 0.4997 (0.0123) & 0.5015 (0.0135) & 0.5032 (0.0133) & 0.5057 (0.0126) & 0.5602 (0.0375) \\
 & $r=0.5$ & 0.4934 (0.0117) & 0.4947 (0.0112) & 0.4978 (0.0119) & 0.4979 (0.0118) & 0.4979 (0.0118) & 0.5005 (0.0120) & 0.5023 (0.0146) & 0.5037 (0.0133) & 0.5074 (0.0138) & 0.5500 (0.0358) \\
 & $r=1$ & 0.4937 (0.0121) & 0.4953 (0.0118) & 0.4964 (0.0114) & 0.4979 (0.0128) & 0.4982 (0.0132) & 0.4994 (0.0126) & 0.5017 (0.0144) & 0.5024 (0.0122) & 0.5057 (0.0142) & 0.5400 (0.0294) \\
 & $r=2$ & 0.4938 (0.0120) & 0.4947 (0.0119) & 0.4963 (0.0119) & 0.4968 (0.0117) & 0.4984 (0.0132) & 0.4973 (0.0120) & 0.4986 (0.0124) & 0.5002 (0.0130) & 0.5040 (0.0129) & 0.5332 (0.0258) \\
\end{tabular}
}
\endgroup
\end{table}

\begin{table}[ht]
\centering
\caption{Simulation: logistic-regression test accuracy (FITS-G: all modes).}
\label{tab:simulation_logistic_test_accuracy_fixed_modes_fits_g_all_modes_parameters}
\begingroup
\setlength{\tabcolsep}{4pt}
\renewcommand{\arraystretch}{1.35}
\resizebox{\textwidth}{!}{%
\begin{tabular}{@{}cc|*{10}{c}@{}}
\specialrule{0.8pt}{0pt}{1.5pt}
\specialrule{0.8pt}{0pt}{3pt}
\multirow{2}{*}{Mode} & \multirow{2}{*}{Parameters} & \multicolumn{10}{c}{Clients} \\
\cline{3-12}
 & & Client 0 (30\%) & Client 1 (20\%) & Client 2 (10\%) & Client 3 (9\%) & Client 4 (8\%) & Client 5 (7\%) & Client 6 (6\%) & Client 7 (5\%) & Client 8 (4\%) & Client 9 (1\%) \\
\specialrule{0.8pt}{3pt}{1.5pt}
\specialrule{0.8pt}{0pt}{3pt}
\multirow{1}{*}{M1} & -- & 0.8759 (0.0059) & 0.8759 (0.0059) & 0.8759 (0.0059) & 0.8759 (0.0059) & 0.8759 (0.0059) & 0.8759 (0.0059) & 0.8759 (0.0059) & 0.8759 (0.0059) & 0.8759 (0.0059) & 0.8759 (0.0059) \\
\hline
\multirow{1}{*}{M2} & -- & 0.8763 (0.0054) & 0.8760 (0.0055) & 0.8748 (0.0055) & 0.8736 (0.0058) & 0.8745 (0.0058) & 0.8738 (0.0059) & 0.8727 (0.0058) & 0.8722 (0.0065) & 0.8705 (0.0077) & 0.8488 (0.0177) \\
\hline
\multirow{5}{*}{M3} & $r=0$ & 0.8763 (0.0054) & 0.8760 (0.0055) & 0.8748 (0.0055) & 0.8736 (0.0058) & 0.8745 (0.0058) & 0.8738 (0.0059) & 0.8727 (0.0058) & 0.8722 (0.0065) & 0.8705 (0.0077) & 0.8488 (0.0177) \\
 & $r=0.5$ & 0.8763 (0.0054) & 0.8761 (0.0055) & 0.8753 (0.0057) & 0.8743 (0.0055) & 0.8749 (0.0059) & 0.8746 (0.0059) & 0.8731 (0.0060) & 0.8732 (0.0061) & 0.8719 (0.0067) & 0.8570 (0.0114) \\
 & $r=1$ & \textbf{0.8764 (0.0056)} & 0.8762 (0.0054) & 0.8758 (0.0058) & 0.8750 (0.0059) & 0.8753 (0.0055) & 0.8751 (0.0060) & 0.8745 (0.0055) & 0.8742 (0.0054) & 0.8733 (0.0067) & 0.8632 (0.0115) \\
 & $r=2$ & 0.8763 (0.0056) & 0.8764 (0.0056) & 0.8760 (0.0058) & 0.8752 (0.0054) & 0.8756 (0.0054) & 0.8752 (0.0060) & 0.8748 (0.0058) & 0.8748 (0.0061) & 0.8745 (0.0062) & 0.8659 (0.0074) \\
\end{tabular}
}
\endgroup
\end{table}

\begin{table}[ht]
\centering
\caption{Simulation: FT-Transformer regression test MSE (FITS-G: all modes).}
\label{tab:simulation_ft_regression_test_mse_fixed_modes_fits_g_all_modes_parameters}
\begingroup
\setlength{\tabcolsep}{4pt}
\renewcommand{\arraystretch}{1.35}
\resizebox{\textwidth}{!}{%
\begin{tabular}{@{}cc|*{10}{c}@{}}
\specialrule{0.8pt}{0pt}{1.5pt}
\specialrule{0.8pt}{0pt}{3pt}
\multirow{2}{*}{Mode} & \multirow{2}{*}{Parameters} & \multicolumn{10}{c}{Clients} \\
\cline{3-12}
 & & Client 0 (30\%) & Client 1 (20\%) & Client 2 (10\%) & Client 3 (9\%) & Client 4 (8\%) & Client 5 (7\%) & Client 6 (6\%) & Client 7 (5\%) & Client 8 (4\%) & Client 9 (1\%) \\
\specialrule{0.8pt}{3pt}{1.5pt}
\specialrule{0.8pt}{0pt}{3pt}
\multirow{1}{*}{M1} & -- & 0.6028 (0.0433) & 0.6028 (0.0433) & 0.6028 (0.0433) & 0.6028 (0.0433) & 0.6028 (0.0433) & 0.6028 (0.0433) & 0.6028 (0.0433) & 0.6028 (0.0433) & 0.6028 (0.0433) & 0.6028 (0.0433) \\
\hline
\multirow{1}{*}{M2} & -- & 0.5567 (0.0273) & 0.5889 (0.0319) & 0.7232 (0.0644) & 0.7731 (0.0580) & 0.8112 (0.0694) & 0.8747 (0.0853) & 0.9386 (0.1001) & 1.0237 (0.0881) & 1.1248 (0.1095) & 1.4740 (0.2824) \\
\hline
\multirow{5}{*}{M3} & $r=0$ & 0.5567 (0.0273) & 0.5889 (0.0319) & 0.7232 (0.0644) & 0.7731 (0.0580) & 0.8112 (0.0694) & 0.8747 (0.0853) & 0.9386 (0.1001) & 1.0237 (0.0881) & 1.1248 (0.1095) & 1.4740 (0.2824) \\
 & $r=0.5$ & 0.5646 (0.0365) & 0.5899 (0.0343) & 0.6999 (0.0600) & 0.7414 (0.0689) & 0.7943 (0.0727) & 0.8534 (0.1054) & 0.9549 (0.1172) & 1.0441 (0.1100) & 1.2009 (0.1289) & 1.5537 (0.2274) \\
 & $r=1$ & 0.5654 (0.0329) & 0.5804 (0.0322) & 0.6707 (0.0498) & 0.6957 (0.0524) & 0.7449 (0.0674) & 0.7833 (0.0765) & 0.8586 (0.1190) & 0.9823 (0.1009) & 1.1235 (0.1313) & 1.6559 (0.2200) \\
 & $r=2$ & 0.5561 (0.0289) & 0.5725 (0.0404) & 0.6212 (0.0492) & 0.6372 (0.0351) & 0.6718 (0.0485) & 0.6919 (0.0619) & 0.7523 (0.0708) & 0.8267 (0.0824) & 0.9167 (0.1145) & 1.6683 (0.1677) \\
\end{tabular}
}
\endgroup
\end{table}

\end{document}